\documentclass[oneside,english,british,12pt,a4]{book}
\usepackage{helvet}

\usepackage[T1]{fontenc}
\usepackage{array}
\usepackage{verbatim}
\usepackage{amssymb}
\usepackage{graphicx}
\usepackage[authoryear]{natbib}
\usepackage[section]{placeins}

\usepackage{amsmath}
\newcommand{\argmax}{\operatorname*{arg\,max}}
\usepackage{cases}

\usepackage{amsfonts,amssymb,amsthm,mathtools}
\allowdisplaybreaks
\DeclareMathOperator*{\argmin}{arg\,min}
\DeclareMathOperator*{\arginf}{arg\,inf}

\DeclarePairedDelimiter\floor{\lfloor}{\rfloor}

\makeatletter

\theoremstyle{definition}
\newtheorem{defn}{Definition}[section]

\newtheorem{assumption}{Assumption}[section]
\newtheorem{thm}{Theorem}[section]
\newtheorem{prop}{Proposition}[section]
\newtheorem{lem}{Lemma}[section]

\theoremstyle{remark}
\newtheorem{rem}{Remark}
\newtheorem{eg}{Example}

\newtheorem*{rem*}{Remark}
\usepackage{longtable}
\usepackage{subcaption}
\usepackage{adjustbox}
\usepackage{tabularx}
\providecommand{\tabularnewline}{\\}

\usepackage{textcomp}
\usepackage{setspace}

\usepackage{graphicx}
\usepackage{latexsym}
\usepackage{colortbl}
\usepackage{algolyx}
\usepackage{cancel}

\usepackage{helvet}
\usepackage{flushend}

\usepackage{multirow}

\usepackage{fancyhdr}
\usepackage{colortbl}
\definecolor{cobalt}{rgb}{0.0, 0.28, 0.67}
\definecolor{carmine}{rgb}{0.59, 0.0, 0.09}

\usepackage[colorlinks=true,linkcolor=blue,citecolor=blue]{hyperref}

\usepackage{lmodern}
\setcitestyle{round}

\usepackage{tikz}

\usepackage{epigraph,varwidth}
\usepackage{type1cm}
\usepackage{lettrine}
\definecolor{mygray}{gray}{0.6}

\makeatother

\usepackage{babel}
\begin{document}
\selectlanguage{english}%

\global\long\def\tr{\mathrm{tr}}%

\global\long\def\realset{\mathbb{R}}%

\global\long\def\E{\mathbb{E}}%

\global\long\def\H{\mathcal{H}}%

\global\long\def\realn{\real^{n}}%

\global\long\def\natset{\integerset}%

\global\long\def\interger{\integerset}%

\global\long\def\integerset{\mathbb{Z}}%

\global\long\def\natn{\natset^{n}}%

\global\long\def\partf{\mathcal{Z}}%

\global\long\def\rational{\mathbb{Q}}%

\global\long\def\realPlusn{\mathbb{R_{+}^{n}}}%

\global\long\def\comp{\complexset}%
 
\global\long\def\complexset{\mathbb{C}}%

\global\long\def\dataset{\mathcal{D}}%

\global\long\def\class{\mathcal{C}}%

\global\long\def\likelihood{\mathcal{L}}%

\global\long\def\normal{\mathcal{N}}%

\global\long\def\argmax#1{\underset{#1}{\text{argmax}}}%

\global\long\def\bphi{\boldsymbol{\phi}}%

\global\long\def\btheta{\boldsymbol{\theta}}%

\global\long\def\bx{\boldsymbol{x}}%

\global\long\def\by{\boldsymbol{y}}%

\global\long\def\bt{\boldsymbol{t}}%

\global\long\def\bf{\boldsymbol{f}}%

\global\long\def\bX{\boldsymbol{X}}%

\global\long\def\bY{\boldsymbol{Y}}%

\global\long\def\bp{\boldsymbol{p}}%

\global\long\def\bP{\boldsymbol{P}}%

\global\long\def\bh{\boldsymbol{h}}%

\global\long\def\bH{\boldsymbol{H}}%

\global\long\def\bv{\boldsymbol{v}}%

\global\long\def\bV{\boldsymbol{V}}%

\global\long\def\btv{\tilde{\boldsymbol{v}}}%

\global\long\def\tv{\tilde{v}}%

\global\long\def\tq{\tilde{q}}%

\global\long\def\baq{\bar{q}}%

\global\long\def\ba{\boldsymbol{a}}%

\global\long\def\bta{\tilde{\boldsymbol{a}}}%

\global\long\def\ta{\tilde{a}}%

\global\long\def\bA{\boldsymbol{A}}%

\global\long\def\bb{\boldsymbol{b}}%

\global\long\def\bB{\boldsymbol{B}}%

\global\long\def\bD{\boldsymbol{D}}%

\global\long\def\ble{\boldsymbol{e}}%

\global\long\def\bE{\boldsymbol{E}}%

\global\long\def\bI{\boldsymbol{I}}%

\global\long\def\bS{\boldsymbol{S}}%

\global\long\def\bw{\boldsymbol{w}}%

\global\long\def\bW{\boldsymbol{W}}%

\global\long\def\bL{\boldsymbol{L}}%

\global\long\def\btW{\tilde{\boldsymbol{W}}}%

\global\long\def\tW{\tilde{W}}%

\global\long\def\bz{\boldsymbol{z}}%

\global\long\def\bu{\boldsymbol{u}}%

\global\long\def\bI{\boldsymbol{I}}%

\global\long\def\bmu{\boldsymbol{\mu}}%

\global\long\def\tmu{\tilde{\mu}}%

\global\long\def\btmu{\tilde{\bmu}}%

\global\long\def\bamu{\bar{\mu}}%

\global\long\def\bbmu{\bar{\bmu}}%

\global\long\def\tlamda{\tilde{\lambda}}%

\global\long\def\balamda{\bar{\lambda}}%

\global\long\def\balpha{\boldsymbol{\alpha}}%

\global\long\def\bbeta{\boldsymbol{\beta}}%

\global\long\def\bSigma{\boldsymbol{\Sigma}}%

\global\long\def\grad{\bigtriangleup}%

\global\long\def\cov{\textnormal{cov}}%

\selectlanguage{british}%
\global\long\def\xb{\boldsymbol{x}}%
\global\long\def\wb{\boldsymbol{w}}%
\global\long\def\ub{\boldsymbol{u}}%
\global\long\def\hb{\boldsymbol{h}}%
\global\long\def\bb{\boldsymbol{b}}%
\global\long\def\zb{\boldsymbol{z}}%

\selectlanguage{english}%
\pagenumbering{roman} 
\setcounter{page}{1}
\pagestyle{empty}

\begin{center}
\textbf{\Large{} On Practical
Reinforcement Learning: Provable Robustness, Scalability, and Statistical Efficiency}{\Large\par}
\par\end{center}

\vspace{1cm}

\vspace{1cm}

\begin{center}
by
\par\end{center}

\vspace{-3cc}

\begin{center}
\textbf{\large{} Thanh Nguyen-Tang }{\large\par} M.S. 
\par\end{center}

\vspace{-3cc}

\begin{center}
\par\end{center}

\vfill{}

\begin{center}
\textbf{Submitted in fulfilment of the requirements for the degree
of}\\
 \textbf{Doctor of Philosophy}
\par\end{center}

\vspace{2cm}

\begin{center}
\textbf{\large{}Deakin University}{\large\par}
\par\end{center}

\vspace{-3cc}

\begin{center}
\textbf{\large{}June 2021}{\large\par}
\par\end{center}\selectlanguage{british}%

\newpage{}





\newpage{}

\tableofcontents{}
\listoffigures
\listoftables
\listofalgorithms

\newpage{}

\pagestyle{empty}

\chapter*{Abstract}

\addcontentsline{toc}{chapter}{\numberline{}{Abstract}}

\noindent


A central cognitive capability for many modern artificial intelligence (AI) systems is the ability to learn to make decisions only through interactions with the surrounding environment. Reinforcement learning (RL) is one such a powerful sequential decision making framework for learning to make decisions from interactions, where the primary goal is to maximize the total return via a sequence of strategic decisions. In many real-world scenarios, it is essential for the learner to learn robust and generalizable decisions in practical settings such as the presence of adversaries, the intrinsic randomness of the total reward, and the difficulty of acquiring new data in a high-dimensional space. In this thesis, we present a novel and unifying approach to close the gap of several current RL methods in practical considerations with a chief focus on provable robustness, scalability and statistical efficiency. 

First, we study the problem of robust decision making under the presence of an uncontrollable environmental variable. We study this setting in the framework of Bayesian quadrature optimization, a special instance of RL that aims at maximizing a black-box, expensive quadrature objective function. The environmental variable is uncontrollable by the learner and can mislead the feedback of any decision the learner makes, creating spurious learning signal. The standard methods in Bayesian (quadrature) optimization become ineffective for this setting. We propose a novel solution based on Thompson sampling that provably trade-offs between the accuracy and robustness against the disturbance caused by the uncontrollable environmental variable. In particular, our method quantifies robustness via a $\chi^2$-divergence between the empirical distribution of the environmental variable and its model distribution. We prove that our proposed method efficiently achieves a robust solution in a sublinear time. Our method also outperforms the standard methods in finding a robust solution in this setting in both synthetic and real-world applications. 

Second, we study the problem of distributional learning in RL. This problem is known as distributional RL where the goal is to learn the entire distribution, rather than only the expected value of the total reward. The intrinsic randomness encoded in the total reward is crucial to design risk-sensitive agents in safety-critical applications. Under this practical consideration, while the standard RL methods entirely ignore such the intrinsic randomness, all the predominant distributional RL methods rely on learning a set of \textit{predefined statistics} of the return distribution. This conventional wisdom hinders these methods from effectively exploiting the powerful idea of distributional RL for better learning representation. We propose a novel framework based on statistical hypothesis testing that allows the pseudo-samples of the return distribution to deterministically evolve into any functional form as long as they can simulate the distribution well in terms of a distributional discrepancy known as maximum mean discrepancy. We prove the validity, convergence and stability property of our method. Our method maintains scalability, outperforms the predominant methods and even achieves a new state-of-the-art performance when extending to the large-scale Atari game environments. 

Third, we study the offline RL setting where it is expensive or even prohibited to further interact with the underlying environment but a batch of offline data from previous interactions is available. Especially, we study this under the high-dimensional setting where it is almost impossible to visit every high-dimensional states thus it is necessary to generalize from observed states to unseen ones via function approximation such as deep neural networks. Under the offline setting, the standard online RL methods fail to evaluate or learn a new policy as there is a mismatch between the distribution of the offline data and that of the data generated by the target policy. This so-called \textit{distributional shift} is further exaggerated under the high-dimensional setting. We build a new statistical theory of offline RL with deep ReLU network function approximation. Deep ReLU networks are known for their expressiveness and adaptivity in supervised learning. In offline RL, we show that such benefits retain in offline RL. Our theory provides a tight dependence of the offline RL sup-optimality on the distributional shift, input dimension, and the smoothness of the underlying environment when employing deep neural networks.




\newpage{}

\pagestyle{empty}

\chapter*{Acknowledgements}

\addcontentsline{toc}{chapter}{\numberline{}{Acknowledgements}}

I would like to thank my co-authors at A$^2$I$^2$ for the fruitful discussions and collaborations: Sunil Gupta (A/Prof. at Deakin University), Svetha Venkatesh (Prof. at Deakin University), Huong Ha (Lecturer at RMIT), Hung Tran-The (Research Fellow at Deakin University), and Santu Rana (A/Prof. at Deakin University). Especially thanks Sunil for the very detailed comments on my drafts. I also thank all of my friends and colleagues at A$^2$I$^2$ for creating a friendly and fun environment, thank the Australian Research Council and PRaDA Postgraduate Research Scholarship for funding my PhD, and thank Australia for being such a nice place for me to work on my PhD. I do not thank the global pandemic COVID-19 for ruining all my conference travels. 

I would also like to thank other researchers for the interesting discussions and the constructive feedback when I was working in individual chapters: my anonymous reviewers at AISTATS'20, AISTATS'21, NeurIPS'20, NeurIPS'21, AAAI'21, and ICML'21; Will Dabney (Deepmind); Richard Nickl (Prof. at the University of Cambridge); and my friends at the Machine Learning Summer School, Max Planck Institute for Intelligent Systems, Tübingen, Germany, 2020.

Last but not least, I would like to give my special thanks to my parents, Tan Nguyen Tang and Tung Truong Thi, and my partner Le Minh Khue Nguyen. Without their love, trust and support, this thesis would have not been possible. 


\newpage{}

\pagestyle{empty}

\selectlanguage{english}%

\chapter*{Relevant Publications}

\addcontentsline{toc}{chapter}{\numberline{}{Relevant Publications}}

Parts of this thesis have been written based on the following manuscripts: 

\textbf{Chapter \ref{chap:three}}
\begin{itemize}
    \item \textbf{T. Nguyen}, S. Gupta, H. Ha, S. Rana, and S. Venkatesh. \textit{Distributionally Robust Bayesian Quadrature Optimization}. Proceedings of the 23rd International Conference on Artificial Intelligence and Statistics (AISTATS), Palermo, Italy, 2020.
\end{itemize}

\textbf{Chapter \ref{chap:four}}
\begin{itemize}
    \item \textbf{T. Nguyen-Tang}, S. Gupta, and S. Venkatesh. \textit{Distributional Reinforcement Learning via Moment Matching}. Proceedings of the 35th AAAI Conference on Artificial Intelligence (AAAI), Vancouver, Canada, Feb. 2-9, 2021.
\end{itemize}

\textbf{Chapter \ref{chap:five}} 
\begin{itemize}
    \item \textbf{T. Nguyen-Tang}, S. Gupta, H. Tran-The, and S. Venkatesh. \textit{Sample Complexity of Offline Reinforcement Learning with Deep ReLU Networks}, Workshop on Reinforcement Learning Theory, ICML 2021.
    
    \item \textbf{T. Nguyen-Tang}, S. Gupta, A.Tuan Nguyen, and S. Venkatesh. \textit{Offline Neural Contextual Bandits: Pessimism, Optimization and Generalization}, ICLR 2022.
\end{itemize}

\newpage{}

\pagestyle{empty}
\section*{Notations} 
\addcontentsline{toc}{chapter}{\numberline{}{Notations}}\medskip{}

\begin{longtable}{l | l } 
    \textbf{Notation}   & \textbf{Meaning}  \\
    \hline 
    \hline 
    $\Delta_N$ & the $N$-dimensional simplex \\  
    $\mathcal{P}(\mathcal{X})$ & the set of (Borel) probability measures supported on domain $\mathcal{X}$ \\ 
    $dom(X)$ & the support of the random variable $X$ \\
    $\delta_x$&  the Dirac distribution concentrated at $x$\\  
    $F_{P}$&  the cumulative distribution function of distribution $P$\\   
    $F^{-1}_{P}$& the inverse cumulative distribution function of distribution $P$\\  
    $KL[P \| Q]$& the KL divergence between distribution $P$ and $Q$\\  
    $W_p(P, Q)$& the $p$-Wasserstein metric between distribution $P$ and $Q$\\  
    $1\{A\}$&  an indicator function, i.e., $1\{A\} = 1$ if $A$ occurs and $1\{A\}= 0$ o.w.\\  
    $f_{\#} \mu$& the pushforward measure of measure $\mu$ by function $f$\\  
    $\| \cdot \|_{\infty}$& the $\infty$-norm\\ 
    $x \sim P$&  $x$ is sampled according to distribution $P$\\  
    $\{x_i\}_{i=1}^n \overset{i.i.d.}{\sim} P$&  $\{x_i\}_{i=1}^n$ are i.i.d. samples of $P$\\ 
    $\mathbb{E}[\cdot]$& expectation\\  
    $\mathbb{V}[\cdot]$& variance\\  
    $\tilde{O}(\cdot)$& the big-O notation that hides any log factors\\
    $\lambda_{\max}(A)$& the largest eigenvalue of matrix $A$ \\ 
    $\sigma_{\max}(A)$ &the largest singular value of matrix $A$ \\
    $tr(A)$ & the trace of matrix $A$ \\
    $det(A)$ & the determinant of matrix $A$ \\ 
    $A \succeq B$ & $A - B$ is semi-definite \\
    $A \preceq B$ & $B - A$ is semi-definite \\
    $\mathcal{F} - \mathcal{G}$ &  $\{f - g: f \in \mathcal{F}, g \in \mathcal{G}\}$ \\ 
    $N(\epsilon, \mathcal{F}, \|\cdot\|)$ & the $\epsilon$-covering number of $\mathcal{F}$ w.r.t. $\|\cdot\|$ \\
    $H(\epsilon, \mathcal{F}, \|\cdot\|)$ & entropic number, i.e., $\log N(\epsilon, \mathcal{F}, \|\cdot\|)$ \\
    $N_{[]}(\epsilon, \mathcal{F}, \|\cdot\|)$ & the $\|\cdot\|$-bracketing metric entropy of $\mathcal{F}$\\
    $\mathcal{F} | \{x_i\}_{i=1}^n$ & $\{(f(x_1), ..., f(x_n)) \in \mathbb{R}^n | f \in \mathcal{F}\}$ \\
    $T \mathcal{F}$ & $\{ Tf | f \in \mathcal{F} \}$ \\ 
    $\|f \|_n$ & $\sqrt{\frac{1}{n} \sum_{i=1}^n f(x_i)^2}$ \\ 
    $\|f\|_{p, \mu}$ & $(\int_{\mathcal{X}} |f|^p d\mu)^{1/p}$ \\ 
    $\|f \|_{\mu}$ & $\|f\|_{2, \mu}$ \\ 
    $L^p(\mathcal{X}, \mu)$ & $\{f: \mathcal{X} \rightarrow \mathbb{R} \text{ }|\text{ } \|f\|_{p, \mu} < \infty \}$ \\ 
    $C^0(\mathcal{X})$ & $\{f: \mathcal{X} \rightarrow \mathbb{R} \text{ }|\text{ } f \text{ is continuous and } \|f\|_{\infty}  < \infty \}$ \\ 
    $C^{\alpha}(\mathcal{X})$ & the H\"older space with smoothness parameter $\alpha \in (0, \infty) \backslash \mathbb{N}$ \\ 
    $W^{m}_p(\mathcal{X})$ & the Sobolev space with regularity $m \in \mathbb{N}$ and parameter $p \in [1,\infty]$ \\ 
    $X \hookrightarrow Y$ & \textit{continuous embedding} from a metric space $X$ to a metric space $Y$ \\
    $f(\epsilon, n) \lesssim g(\epsilon,n)$ & $\exists \text{ absolute constant } c$ such that $f(\epsilon, n) \leq c \cdot g(\epsilon, n), \forall \epsilon > 0, n \in \mathbb{N}$ \\ 
    $f(\epsilon, n) \asymp g(\epsilon, n)$ & $f(\epsilon, n) \lesssim g(\epsilon, n)$ and $g(\epsilon, n) \lesssim f(\epsilon, n)$ \\ 
    $f(\epsilon, n) \simeq g(\epsilon, n)$ & $\exists$ absolute constant $c$ such that $f(\epsilon, n) = c \cdot g(\epsilon, n), \forall \epsilon, n$ \\ 
    $a \lor b$ & $\max \{a , b\}$ \\ 
    $a \land b$ & $\min \{a,b\}$ 
\end{longtable}

\newpage{}

\pagestyle{empty}
\section*{Abbreviations}

\addcontentsline{toc}{chapter}{\numberline{}{Abbreviations}}\medskip{}

\begin{tabular}{>{\raggedright}p{0.2\linewidth}>{\raggedright}p{0.8\linewidth}}
\hline 
 \textbf{Abbreviation} & \textbf{Description}\tabularnewline
\hline 
\noalign{\vskip\doublerulesep}
AI & Artificial Intelligence \tabularnewline
ML & Machine learning \tabularnewline
RL & Reinforcement learning \tabularnewline 
MDP & Markov decision process \tabularnewline
DRL & Distributional reinforcement learning \tabularnewline
CDRL & Categorical distributional reinforcement learning \tabularnewline 
QRDRL & Quantile regression distributional reinforcement learning \tabularnewline 
MMDRL & Moment matching distributional reinforcement learning \tabularnewline
MMDQN & Moment matching deep Q-networks \tabularnewline
OPE & Off-policy evaluation \tabularnewline 
OPL & Off-policy learning (offline learning) \tabularnewline
MAB & Multi-armed bandits \tabularnewline
BO & Bayesian optimization \tabularnewline
GP & Gaussian process \tabularnewline
BQO & Bayesian quadrature optimization \tabularnewline
DRBQO & Distributionally robust Bayesian quadrature optimization \tabularnewline 
OFU & Optimism in the Face of Uncertainty\tabularnewline
UCB & Upper confidence bound \tabularnewline 
PS & Posterior (Thompson) sampling \tabularnewline
ReLU & Rectified Linear Unit \tabularnewline 
MMD & Maximum mean discrepancy \tabularnewline 
(S)GD & (Stochastic) gradient descent \tabularnewline 
FQI & Fitted Q-iteration \tabularnewline
LSVI & Least-squares value iteration \tabularnewline
PAC & Probably Approximately Correct \tabularnewline
iff & if and only if \tabularnewline
LHS & left-hand side \tabularnewline
RHS & right-hand side \tabularnewline

\noalign{\vskip\doublerulesep}
\hline 
\end{tabular}

\newpage\pagenumbering{arabic} 
\setcounter{page}{1}
\pagestyle{fancy}

\chapter{Introduction\label{chap:Introduction}}

Learning by interacting with an environment is one of the most integral cognitive abilities of humans. As an infant, we can learn to walk and play from the consequences of our own actions rather than from any explicit teacher. Throughout our lives, much of our knowledge is acquired via learning from our own experiences with the environment around us, from primitive activities such as learning to walk to more intricate ones such as learning to drive a car. Learning from interactions is thus a key learning paradigm, and is one of the main goals that a modern artificial intelligence (AI) system aims to possess. 

Sequential decision making is a framework for learning from interactions that aims at learning the underlying dynamics and making decisions accordingly for maximizing the total reward (a.k.a. long-term return) accumulated along the way. An AI system that can perform such capability finds vast applications in robotics, autonomous driving, personalized medicine, healthcare, material discovery, games, recommender system, marketing, and dynamic pricing \citep{sutton2018reinforcement}. Reinforcement learning (RL) is a sequential decision making framework where the decisions it makes can affect the data it gets and the rewards it obtains. This differs from
supervised learning and unsupervised learning which focus on learning only from fixed datasets. Another learning characteristic of RL that sets it aside from the other two learning frameworks is bandit feedback. That is, an RL agent only receives a feedback
from the action it has taken, not from the other actions; thus it faces a fundamental exploration-exploitation trade-offs where the agent has to decide whether it should gather more data by exploring its ``poorly-understood'' actions or exploiting
its current knowledge to maximize its return. 

Though RL has been extensively studied both theoretically and empirically in the literature, the main challenges in bridging the gap between theory and practice in RL still remain. In particular, empirical works often lack theoretical guarantees and theoretical works often focus on simple models that are relatively disconnected from practical settings. 
It is thus crucial to design an agent that works provably with theoretical guarantees in many practical settings from a unifying perspective. 

In this thesis, we attempt to fill the gap above by addressing several practical considerations of RL problems with {provable theoretical properties}. The practical considerations in this thesis are both common and general enough spanning a large space of modern and fundamental RL problems in practical settings. In particular, we study three main practical considerations in this thesis: (i) \textit{data source}, (ii) \textit{distributional learning}, and (iii) \textit{high dimensionality}. The data source consideration refers to the nature of the data available to an RL agent for its learning and decision-making. The data source includes online data, offline data and a hybrid online-offline data. While online data is a standard data source that a RL agent can actively acquire from the environment, offline data about (part of) the environment is collected a priori from historical interactions and is commonly encountered in real-life applications such as recommender system, autonomous driving, and robotics. We consider all three possible data source scenarios: online data as the sole data source (Chapter \ref{chap:four}), offline data as the sole data source (Chapter \ref{chap:five}), and offline data as a supplementary to the online data (Chapter \ref{chap:three}). The distributional learning consideration (Chapter \ref{chap:four}) refers to learning the entire distribution, as opposed to only the expected value, of the total reward an agent accumulates over time. The return distribution results from the intrinsic randomness of  the underlying environment and the agent policy. Standard value-based RL focuses on learning the expected value of the total reward and ignores this intrinsic randomness. Such intrinsic randomness is however both crucial to claim a significant empirical improvement and design a risk-sensitive policy in RL. Finally, the high dimensionality consideration (Chapter \ref{chap:four}-\ref{chap:five}) refers to the high-dimensional input space of the environment. In practical domains such as robotics, video games and continuous control, most environments are complex with high-dimensional states (and/or actions) that require a RL agent to be able to generalize across different input regions using function approximation such as deep neural networks. 

Beside the main goal of maximizing the total reward in a long run, we place a chief emphasis on three desirable properties for a RL agent to be able to address a decision-making and learning problem in practical considerations: \textit{robustness} (Chapter \ref{chap:three}), \textit{scalability} (Chapter \ref{chap:four}) and \textit{statistical efficiency} (Chapter \ref{chap:five}). Robustness is the ability of an RL agent to continue to perform well even in an adversarially corrupted system. For example, an empirical model constructed from an available dataset can mislead the decision-making and learning of an agent due to the finite data; a robust agent should therefore account for the worst-case scenario depicted through the empirical model. Scalability is the ability to scale efficiently with the problem and data size. In an era of ever increasing datasets and deep learning, scalability might be as important as accuracy for an AI system. Statistical efficiency is a central issue in RL which aims at using a minimum number of samples to learn to behave optimally.


To make it  concrete under the practical considerations, we identify and consider the following challenges in this thesis: 

\subsection*{Challenge 1: Robustness for reinforcement learning under adversarial contexts} 
As a concrete example, consider an alloy design application whose goal is to combine several elements into a new composition with enhanced performance. Alloy design is an extremely extensive and costly process as it involves multiple alloying elements and processing steps. Bayesian optimization is an instance of RL that is commonly used in alloy design to explore and optimize for new alloy composition \citep{DBLP:journals/access/GreenhillRGVV20}. In practice, the alloying element powders are often slightly impure as it is hard to find $100\%$ pure element powder. While not under control of the designer, these impurities however also affect the performance of the alloy composition. While Bayesian optimization can effectively apply to standard alloy design, it can fail to search for a robust alloy composition in the presence of impurities as it ignores this intervention. In other words, the uncontrollable environmental variables (a.k.a. adversarial contexts) can mislead the bandit feedback of a chosen action. This misleading is more detrimental than that by noisy feedback corrupted by Gaussian noise as commonly assumed in Bayesian optimization because the former can shift the intended input to a completely new one. Thus, the problem of designing a robust learning algorithm under such adversarial contexts remains open. 

\subsection*{Challenge 2: Scalability for distributional reinforcement learning}
As a concrete example, consider placing a bet which leads to the winning of $\$10$ with probability $0.95$ and a loss of $\$1000$ with probability $0.05$. The expected reward for playing this bet is $100 \times 0.95 - 1000 \times 0.05 = \$45  $; thus it is a positive gain on average if the agent picks this bet. Although the probability of losing the bet is lower, once it occurs (i.e., losing $\$1000$), the consequence could be detrimental from the agent's perspective. Thus, a risk-averse agent would play safe by avoiding this bet even if it leads to a positive gain on average. To obtain such a risk-sensitive behaviour requires the knowledge of the entire distribution of the total reward, i.e., the return distribution. 
Learning return distribution is crucial in not only many safety-critical applications such as medicine, healthcare and autonomous driving where a bad decision might lead to a fatal outcome, but also stabilizing and improving many current algorithms via auxiliary learning. As return distribution is often useful for high-dimensional complex models in practice, a key challenge lies in scalable learning that can readily and flexibly scale well with the complex environment while maintaining reliable estimation of the return distribution. Though prior methods have attempted to address this challenge to some extent, they all suffer from the so-called \textit{curse of predefined statistics} that prevent them from fully exploiting the benefit of learning the return distribution. Thus, the problem of designing a scalable learning algorithm of the return distribution that goes beyond the curse of predefined statistics remains open.

\subsection*{Challenge 3: Statistical efficiency of offline reinforcement learning with function approximation}
In many practical settings such as robotics, healthcare, educational games, and autonomous driving, an exploratory interaction with the environment is often expensive or even prohibited, making the RL agent difficult to obtain any online data. Instead, an offline dataset from historical interactions (e.g., demonstration data from human experts or data from any previous policies) is largely available. Offline RL \citep{levine2020offline} is the problem of how to leverage any offline data to evaluate or learn new policies. 
The offline RL problems are further exacerbated in the high dimensionality consideration where the dimension of the state space can be infinitely large that it is necessary for a RL agent to generalize from observed states to unobserved ones. In such cases, function approximation such as deep neural networks is a typical approach to deal with high-dimensional offline RL. Despite the popularity of the high-dimensional offline RL setting in practice, the statistical efficiency of offline RL under deep neural network function approximation remains to be analyzed and understood.



\section{Aims and Approaches}

The goal of this thesis is to address the aforementioned challenges via the following respective aims:
\begin{enumerate}
    \item To design a new algorithm with a theoretical guarantee that is robust to the uncontrollable environmental variables in sequential decision making. We study this setting under a new variant of Bayesian optimization. 
    
    \item To design a novel theoretically grounded method for distributional RL that can scalably learn the return distribution beyond the curse of predefined statistics.
    
    \item To build a statistical theory for high-dimensional offline RL with deep neural network function approximation that can shed light on the statistical efficiency of offline learning under deep neural networks. 
\end{enumerate}

We address these challenges from a unified perspective: a \textit{distributional} perspective. In the distributional perspective, we formulate a learning problem at hand on the basis of learning over a space of probability distributions. On this basis, we lift the original problem into a structural space that allows more information to be learned and leveraged. Especially, we quantify the learning progress in the distributional perspective via distributional discrepancies such as $\chi^2$-divergence (Challenge 1 - Aim 1 - Chapter \ref{chap:three}), maximum mean discrepancy (Challenge 2 - Aim 2 - Chapter \ref{chap:four}) and Radon-Nikodym derivative (Challenge 3 - Aim 3 - Chapter \ref{chap:five}). More specifically, 

\begin{enumerate}
    \item To realize Aim 1, we study a variant of Bayesian optimization which includes an environmental variable in its setup, namely Bayesian quadrature optimization. We consider a practical setting for Bayesian quadrature optimization in which  the environmental variable is not under the control of the agent but affects the reward received by the agent. Especially, the environmental variable is only observed via a limited set of empirical samples obtained via historical interactions. We propose a new method based on Thompson sampling that can trade-off between the regret of the agent and the robustness against the disturbance corrupted by the unknown environmental variable. The robustness is quantified via the $\chi^2$-divergence between the empirical distribution of the environmental variable and its model from a class of distributions. We prove that our method can achieve a provable robustness in sublinear time. We also provide empirical results to verify the effectiveness of our proposed method in both synthetic and real-world experiments.
    
    \item To realize Aim 2, we propose a novel method for distributional RL based on the idea of statistical hypothesis testing. In particular, the return distribution in distributional RL is represented by a set of pseudo-samples which will evolve during the learning process via stochastic gradient descent for minimizing the so-called maximum mean discrepancy between the return distribution and its Bellman target. In contrast to the predominant distributional RL methods, these pseudo-samples in our new framework do not subscribe to any specific functional form, and thus giving them the freedom to simulate the return distributions. Subsequently, we provide a new understanding of distributional RL under our proposed framework including its convergence and stability, and discuss various practical relaxations of our method when it is applied to practical problems. Finally, we propose a novel deep RL agent that outperforms the predominant distributional RL methods and establishes a new state-of-the-art result for the Atari games.

    \item To realize Aim 3, we build a statistical theory of offline RL under deep ReLU network function approximation.
    Specifically, we introduce a new general dynamic condition namely \textit{Besov dynamic closure} that generalizes the dynamic conditions in the prior analyses. We obtain the sample complexity of the offline RL under deep ReLU network in the general dynamic condition and a data-dependent structure that is previously ignored in the prior algorithms and analyses. This is the first general and comprehensive analysis with an improved sample complexity that gives an important insight into the statistical efficiency of offline RL under deep ReLU network function approximation. Technically, this result is established using a uniform-convergence argument and local Rademacher complexities via localization argument. This technique could be of independent interest for studying offline RL with non-linear function approximation. 
\end{enumerate}
\section{Thesis Contributions}
This thesis has made significant contributions to the respective challenges stated via our aims, both algorithmically, empirically and theoretically. More specifically,  

\begin{itemize}
    \item For the first challenge, we derive a novel algorithm that is provably guaranteed to find a robust solution in a sublinear time. This is the first work of its kind to establish such a guarantee for robustness in the face of uncontrollable environmental variables. Our proof technique may also have its own independent usage when deriving robust regret in decision making frameworks. Our method is experimentally effective in finding a robust solution despite the reward being corrupted by the uncontrollable environmental variables. 
    
    \item For the second challenge, we propose a novel perspective of distributional RL inspired by statistical hypothesis testing. This perspective leads to a novel framework with new understanding of distributional RL which goes beyond the limitations of the predominant methods and results in a scalable algorithm with the state-of-the-art result in Atari games. This is the first work that has changed the predominant approach shared by all the previous distributional RL methods.
    
    \item For the third challenge, we provide the first comprehensive analysis of offline RL under deep ReLU function approximation. Our analysis is established under a new general dynamic condition and a practical algorithmic condition that covers the prior analyses and provides an improved sample complexity for offline RL.


\end{itemize}

\section{Thesis Structure}

\begin{itemize}
    \item Chapter \ref{chap:Background} provides background for the thesis. We first present the mathematical preliminaries helpful to establish our results in this thesis. These include basic linear algebra, reproducing kernel Hilbert spaces, concentration phenomenon, and brief generalization theory. We then present the relevant sequential decision making frameworks including bandits, Bayesian optimization and Markov decision processes followed by their foundations, technical background, and brief review. We then cover the background of distributional RL and offline RL. These areas are also discussed in further details in their respective chapters in this thesis. 
    
    \item Chapter \ref{chap:three} introduces our first work in this thesis, namely distributionally robust Bayesian quadrature optimization (DRBQO). We first motivate for the proposed framework for robust decision making under uncontrollable environmental variables, detail how a practical algorithm is derived and how to implement the proposed method in practice. We also provide our detailed proof to guarantee the robustness of the proposed algorithm under the uncontrollability of the environmental variables. Finally, we provide experimental experiments to verify the effectiveness of our our proposed framework in both synthetic and real-world problems. 
    
    
    \item In Chapter \ref{chap:four}, we introduce our novel framework of distributional RL via moment matching. We first review the predominant methods in distributional RL and their limitations. Based on these, we motivate our idea of moment matching and present detailed derivations for our proposed method. We also study the theoretical aspects of distribution RL through the lens of moment matching, providing new insights into distributional RL. Next, we show that our novel idea does not only overcome the fundamental limitations of the predominant methods, but also has a strong scalability to large-scale environments in Atari games. In fact, our proposed method establishes a new state-of-the-art result in the Atari games for non-distributed agents in terms of the best mean normalized scores at the time our work is published. 
    
    
    \item In Chapter \ref{chap:five}, we present our statistical theory of offline RL under deep neural network function approximation. We first motivate the offline RL under function approximation and present limitations of the predominant analysis in the setting. We then introduce a new dynamic condition and explain the data-dependent structure that is previously ignored in the prior analyses. We propose our new analysis and theoretical result on the sample efficiency of offline RL under deep ReLU networks. Finally, we interpret our result and its significance in light of the related literature. 
    
    
    \item Chapter \ref{chap:Conclusion} concludes this thesis by a summary of the thesis contributions and the future research directions including various sources of improvements and open questions. 
    
\end{itemize}




\newpage{}

\chapter{Background} 
\label{chap:Background}
This chapter provides a concise background and literature for the thesis. It starts with mathematical preliminaries. Next, it provides the technical background and a review of the literature related to the aims of this thesis. The chapter ends with a brief discussion of the challenges that will be addressed in this thesis.

\section{Mathematical Preliminaries}
In this section, we review some non-elementary mathematical background including linear algebra, concentration phenomenon, reproducing kernel Hilbert spaces and classical generalization theory. We also briefly review the modern generalization theory for overparameterized models such as deep neural networks. 

\subsection{Linear Algebra}
Many machine learning tasks such as algorithmic design and analysis require the knowledge about basic linear algebra. Here, we review key concepts of linear algebra that are particularly relevant to machine learning, based on our own preference. For a comprehensive review of linear algebra for machine learning, we refer the readers to \citep{Deisenroth2020}. Note that we only consider real-valued matrices in this thesis. For most cases, we state the basic results without a proof but in some cases that are more convenient or are not so obvious, we provide the proofs. 

\subsubsection*{Matrix Decomposition}
An eigenvector $v$ of a matrix $A$ is a non-zero vector that has invariant direction under the linear transformation by $A$, i.e., $A v = \lambda v $
where $\lambda$ is a scalar called eigenvalue corresponding to eigenvector $v$. If $v$ is an eigenvector with respect to the eigenvalue $\lambda$, $c \cdot v$ is also an eigenvector for $\lambda$ for any scalar $c \neq 0$. The set of all eigenvectors with respect to an eigenvalue $\lambda$ is called  $\lambda$-eigenspace.

A $n \times n$ matrix has at most $n$ eigenvalues, denoted by $(\lambda_i)_{i=1}^n$. If the eigenvalues are distinct, the eigenvectors are linearly independent. If $0$ is an eigenvalue of $A$, $A$ it is not invertible. The eigenvalues of $A$ are the roots of the characteristic equation $\text{det}(A - \lambda I) = 0$. 

An important result in matrix decomposition is eigendecomposition.
\begin{thm}[{Eigendecomposition}]
Let $A \in \mathbb{R}^{n \times n}$ be a matrix with $n$ linearly independent eigenvectors $(q_i)_{i=1}^n$ associated with eigenvalues $(\lambda_i)_{i=1}^n$, and let $Q = [q_1, ..., q_n] \in \mathbb{R}^{n \times n}$ and $\Lambda = \text{diag}(\lambda_1, ..., \lambda_n)$. Then, $A$ can be decomposed as $A = Q \Lambda Q^{-1}$.
\end{thm}
\begin{rem}
Given independent eigenvectors $\{q_i\}_{i=1}^n$ of matrix $A$, we can normalize them  to form an orthonormal basis via Gram-Schmidt: Let
$v_i = q_i - \sum_{j=1}^{i-1} \frac{\langle q_i, v_j \rangle}{\|v_j\|^2} v_j, \forall i$, $w_i = v_i / \|v_i\|, \forall i$, then $\{w_i\}$ is an orthonormal basis of $\text{Span}(q_1,...,q_n)$. Let $Q = [w_1, ..., w_n]$, then $Q$ is an orthonormal matrix and $A = Q \Lambda Q^{-1} = Q \Lambda Q^T$.
\end{rem}
\begin{rem}
If a square matrix $A$ has distinct eigenvalues, it then has an eigendecomposition as distinct eigenvalues imply linearly independent eigenvectors. 
\end{rem}

\begin{rem}
 Any real symmetric matrix has an eigendecomposition. Moreover, any two eigenvectors from different eigenspaces are orthogornal. 
\end{rem}
We present several useful applications of the eigendecomposition. 
\begin{eg}
If $A$ has an eigendecomposition $A = Q \Lambda Q^{-1}$ where $Q$ is an orthonormal matrix, then $A^n = Q \Lambda^n Q^{-1}$.
\end{eg}

\begin{eg}
Let $\Sigma \in \mathbb{R}^{d \times d}$ be a symmetric matrix and $\lambda > 0$, then $\Sigma + \lambda I$ is symmetric and invertible, and
\begin{align*}
    (\Sigma + \lambda I)^{-1/2} \Sigma (\Sigma + \lambda I)^{-1/2} = (\Sigma + \lambda I)^{-1} \Sigma. 
\end{align*}
\end{eg}

\begin{proof}
Since $\Sigma$ is symmetric, it admits an eigendecomposition $\Sigma = Q \Lambda Q^{-1}$ where $Q$ is an orthonormal matrix, $\Lambda = \text{diag}([\lambda_1, ..., \lambda_d])$, and $(\lambda_i)$ are the eigenvalues. We have 
\begin{align*}
    (\Sigma + \lambda I)^{-1/2} \Sigma (\Sigma + \lambda I)^{-1/2} &=  
    Q (\Sigma + \lambda I)^{-1/2} Q^{-1} Q \Sigma Q^{-1}  Q (\Sigma + \lambda I)^{-1/2} Q^{-1}
    \\
    &= Q (\Lambda + \lambda I)^{-1/2} \Lambda (\Lambda + \lambda I)^{-1/2} Q^{-1} \\
    &= Q \text{diag}([\lambda_1 / (\lambda_1 + \lambda), ...,\lambda_d / (\lambda_d + \lambda) ] ) Q^{-1} \\
    &= (\Sigma + \lambda I)^{-1} \Sigma. 
\end{align*}
\end{proof}

While not any matrix admits an eigendecomposition, all matrices have singular value decomposition (SVD).

\begin{thm}[{SVD theorem}]
Let $A \in \mathbb{R}^{m \times n}$ be a matrix of any rank $r \in [0,  m \wedge n]$, then $A$ can be decomposed as $A = U \Sigma V^T$
where $U = [u_1, ..., u_m] \in \mathbb{R}^{m \times m}$ and $V = [v_1, ..., v_n] \in \mathbb{R}^{n \times n}$ are orthogonal matrices, and $\Sigma \in \mathbb{R}^{m \times n}$ with $\Sigma_{ii} = \sigma_i \geq 0$ and $\Sigma_{i,j} = 0, \forall i \neq j$.
\end{thm}

\begin{rem}
$\{\sigma_i \}_{i=1}^n$ are called singular values of $A$, $\{u_i\}_{i=1}^n$ are called left-singular vectors and $v_j$ are called right-singular vectors. The singular matrix $\Sigma$ is unique for each matrix.
\end{rem}

\begin{rem}
The rank of $A$ equals the number of non-zero singular values of $A$, i.e., rank($A$) is the number of non-zero diagonal elements of $\Sigma$.
\end{rem}

\subsubsection*{Matrix Norm}
Now we turn to matrix norm and its properties. This background is crucial to exploit linearity structures in machine learning, e.g., in linear and generalized linear models for provably efficient reinforcement learning. Assuming that the readers are familiar with vector norm, here we define matrix norm via vector norm.
\begin{defn}[Matrix norm induced by vector norm]
\begin{align*}
    \|A\|_p := \sup_{x \neq 0} \frac{ \|Ax\|_p  }{\|x\|_p}, \forall p \in [1, \infty].
\end{align*}
\end{defn}
The interesting case is $p=2$ where the matrix norm is called \textit{spectral norm}. In such case, we interchangeably write $\|\cdot\|_{op}$ for $\|\cdot\|_2$. A spectral norm can be alternatively characterized as 
\begin{align*}
    \|A\|_{op} := \sup \{ x^TAy : \|x\|_2 = \|y\|_2 = 1 \}. 
\end{align*}
Alternatively, the spectral norm has a closed form: 
\begin{align*}
    \|A\|_{op} = \sqrt{ \lambda_{\max}(A^T A) } = \sigma_{\max}(A),
\end{align*}
where $\lambda_{max}(X)$ denotes the largest eigenvalue of $X$, and $\sigma_{\max}(X)$ is the largest singular value of $X$.

\begin{rem}
Since $A$ and $A^T$ have the same singular values, $\|A\|_{op} = \|A^T\|_{op}$. Moreover, if $A$ is symmetric, $\|A\|_{op} = \max_{i}|\lambda_i(A)|$.
\end{rem}

We go through several simple yet useful inequalities pertaining to spectral norms: 
\begin{itemize}
    \item Let $A$ be a matrix with rank $k$, we have 
\begin{align*}
    \|A\|_{op} \leq \|A\|_F := \sqrt{\text{ tr}(A^T A) } = \sqrt{\sum_{i,j} A_{i,j}^2} \leq \sqrt{k} \|A\|_{op}.
\end{align*}

\item Let $x \in \mathbb{R}^d$ such that $x^T x \leq a$ for a given $a > 0$. Then, we have 
\begin{align*}
    xx^T \preceq aI. 
\end{align*}
\begin{proof}
Note that $xx^T \preceq aI$ iff $\lambda_{\max}(xx^T) \leq a$, but we have
\begin{align*}
    a \geq x^Tx = \|x^T \|_{op}^2 = \lambda_{\max}(xx^T). 
\end{align*}
\end{proof}

\item For any square matrix $A$, we have 
\begin{align*}
    \sqrt{x^T A x} \leq \| x\|_2 \cdot \sqrt{ \| A\|_{op}}, \text{ } \forall x. 
\end{align*}

\item Let $\Sigma \in \mathbb{R}^{d\times d}$ be a positive definite matrix (thus it is symmetric and invertible), we have 
\begin{align*}
    x^T y \leq \| x \|_{\Sigma} \|y \|_{\Sigma^{-1}}, \forall x,y \in \mathbb{R}^d. 
\end{align*}

\begin{proof}
Let $\lambda_1, ..., \lambda_d$ be the eigenvalues of $\Sigma$. Since $\Sigma$ is a real symmetric matrix, it is diagonalizable by orthogonal matrices, i.e., there exists an orthogonal matrix $P$ such that $\Sigma = P^T D P$ where $D = \text{diag}(\lambda_1, ..., \lambda_d)$. Note that $P$ is orthogonal, thus $P^T = P^{-1}$. 

Since $\Sigma$ is positive definite, $\lambda_i > 0, \forall i$. Let $\Sigma^{1/2} := P^T \text{diag}(\sqrt{\lambda_1}, ..., \sqrt{\lambda_d}) P$. It is easy to verify that $\Sigma = \Sigma^{1/2} \Sigma^{1/2}$. In addition, $\Sigma^{1/2}$ is symmetric. Since $\lambda_i > 0, \forall i$, we have $\Sigma^{-1/2}:= P^T \text{diag}(\lambda_1^{-1/2}, ..., \lambda_d^{-1/2}) P$ exists and is the inverse matrix of $\Sigma^{1/2}$.
It follows from the Cauchy-Schwartz inequality and the condition of $\Sigma$ that
\begin{align*}
    x^T y = x^T \Sigma^{1/2} \Sigma^{-1/2} y \leq \sqrt{ x^T \Sigma^{1/2} \Sigma^{1/2}x} \sqrt{y^T \Sigma^{-1/2} \Sigma^{-1/2} y}.
\end{align*}
\end{proof}

\end{itemize}






\subsubsection*{Matrix Determinant}
Another useful operation in linear algebra is matrix determinant. The determinant of a square matrix $A$, denoted by $det(A)$, is the volume of the parallelotope formed by the column vectors of $A$. We briefly iterate through several important identities of matrix determinant: 
\begin{itemize}
    \item $\text{det}(I) = 1$;
    \item $\text{det}(AB) = \text{det}(A) \text{det}(B)$;  
    \item $\text{det}(A) = \prod_{i=1}^n a_{i,i}$ if $A = (a_{i,j})$ is a triangular matrix; 
    \item For $A \in \mathbb{R}^{m \times n}, B \in \mathbb{R}^{n \times m}$ and an invertible $X \in \mathbb{R}^{m \times m}$, we have 
    \begin{align*}
        \text{det}(X + AB) = \text{det}(X)  \text{det}(I_n + B X^{-1} A);
    \end{align*}
    \item For $A \in \mathbb{R}^{n \times n}$ with eigenvalues $(\lambda_i)_{i=1}^n$, we have 
    \begin{align*}
        \text{det}(A) = \prod_{i=1}^n \lambda_i. 
    \end{align*}
\end{itemize}

\subsubsection*{Traces}
The trace of a square matrix $A$, denoted $tr(A)$, is the sum of all the elements of its main diagonal. Traces can be helpful in simplifying matrix multiplication due to its invariance under cyclic permutations. In particular, we briefly go through some basic properties of traces as below where $A,B,C,D$ are any matrices of appropriate sizes such that any matrix multiplication in the properties below is valid:
\begin{itemize}
    \item $\text{tr}(A + B) =\text{ tr}(A) + \text{tr}(B)$;
    \item $\text{tr}(A^T) = \text{tr}(A)$;
    \item $\text{tr}(ABCD) = \text{tr}(BCDA) =\text{ tr}(CDAB) = \text{tr}(DABC)$ (invariance under cyclic permutations);
    \item $\text{tr}(I_n) = n$;
    \item $\text{tr}(A) = \sum_{i=1}^n \lambda_i$ where $\{\lambda_i\}$ are the eigenvalues of $A$;
    \item $\text{tr}(A^T (A^T A)^{-1}A) = \text{rank}(A)$ where $A$ has full column rank. Note that $A^T (A^T A)^{-1}A$ is called a \textit{projection} matrix; 
    \item If $A \succeq B$, then $\text{tr}(A) \geq \text{tr}(B)$.
\end{itemize}

\subsubsection*{Other}
We present the \textit{matrix inversion lemma} that allows to obtain the inverse of a large matrix from the inverse of a smaller matrix. This property is commonly used for both algorithmic design and theoretical analysis. A simplified form of the matrix inversion lemma implies that for any matrix $A$, we have
\begin{align*}
    (A A ^T + I)^{-1} A = A (A^T A + I)^{-1}. 
\end{align*}

Now, using the above properties in linear algebra, we prove a simple yet useful inequality. 

\begin{eg}
Let $A$ and $B$ be square matrices of the same size where $A$ is symmetric and invertible. We have 
\begin{align*}
    \log \text{det} (A + B) \leq \log \text{det}(A) + \langle A^{-1}, B \rangle_{F},
\end{align*}
where $\langle U, V \rangle_F = \text{tr}(U^T V)$.
\end{eg}
\begin{proof}
First, for any square matrix $U$, let $(\lambda_i)$ be its eigenvalues. We have 
\begin{align*}
    \log \text{det}(I + U) = \sum_{i} \log (1 + \lambda_i) \leq \sum_i \lambda_i = \text{tr}(U),
\end{align*}
where we use inequality $\log(1+x ) \leq x, \forall x$.
Using this inequality, we are ready to prove the main inequality. We have 
\begin{align*}
    \log \text{det} (A + B) = \log\text{det}(A) + \log \text{det}(I + A^{-1} B) \leq \log \text{det}(A) + \text{tr}(A^{-1} B).
\end{align*}
\end{proof}

\subsection{Reproducing Kernel Hilbert Space}
A reproducing kernel Hilbert space (RKHS) is an infinite-dimensional generalization of Euclidean space. As an RKHS admits closed-form solutions to many ML problem where an RLHS is used as a hypothesis space, it allows us to derive tractable algorithms and establish theoretical properties more easily. Here, we briefly review some basic concepts and properties of an RKHS. For more detailed presentation and discussion on RKHS, we refer the readers to \citep{Sejdinovic2012WhatIA,ltfp}.

We start with an intuition about the relation of RKHS to other vector spaces. 
\begin{align*}
    \text{vector spaces} &\supset \text{ \textit{normed} vector spaces} \\
    &\supset \text{Banach spaces (\textit{complete} \textit{normed} vector spaces)} \\
    &\supset \text{Hilbert spaces (Banach spaces equipped with \textit{inner product})}  \\ 
    &\supset \text{RKHS (Hilbert spaces with \textit{continuous} evaluation functional) }.
\end{align*}

That is, a RKHS is a complete, normed vector (linear) space equipped with inner product and continuous evaluation functional. The continuous evaluation functional intuitively means that if two functions (infinite-dimensional vectors) of a RKHS are close in the RKHS norm, they are close in the element-wise manner. 

\subsubsection*{Reproducing Kernels}
Let $\mathcal{H}$ be a Hilbert space of real-valued functions defined on a non-empty domain $\mathcal{X}$. A function $k: \mathcal{X} \times \mathcal{X} \rightarrow \mathbb{R}$ is said to be a reproducing kernel of $\mathcal{H}$ if 
\begin{itemize}
    \item $\forall x \in \mathcal{X}, k(\cdot, x) \in \mathcal{H}$,
    \item $\forall x \in \mathcal{X}, \forall f \in \mathcal{H}, \langle f, k(\cdot, x) \rangle_{\mathcal{H}} = f(x)$ where $\langle \cdot, \cdot \rangle_{\mathcal{H}}: \mathcal{H} \times \mathcal{H} \rightarrow \mathbb{R}$ denotes the inner product of $\mathcal{H}$.
\end{itemize}
In particular, we have $\langle k(\cdot, x), k(\cdot, y) \rangle_{\mathcal{H}} = k(x,y)$.
\begin{rem}
If it exists, a reproducing kernel is unique. 
\end{rem}
\begin{rem}
A Hilbert space $\mathcal{H}$ is an RKHS iff it has a reproducing kernel. 
\end{rem}
\begin{rem}
Any kernel $k$ of a RKHS $\mathcal{H}$, there is a feature map $\phi: \mathcal{X} \rightarrow \mathcal{H}$ such that $k(x,y) = \langle \phi(x), \phi(y) \rangle_{\mathcal{H}}, \forall x,y \in \mathcal{X}$. The space $\mathcal{H}$ is referred to as the feature space.
\end{rem}

\subsubsection*{Kernel Ridge Regression} 
We consider a common application of RKHS in regression, namely kernel ridge regression. Let $\mathcal{H}$ be an RKHS defined on $\mathcal{X}$ with kernel function $k: \mathcal{X} \times \mathcal{X} \rightarrow \mathbb{R}$. Let $\langle \cdot, \cdot \rangle_{\mathcal{H}}$, $\cdot \otimes \cdot$ and $\|\cdot\|_{\mathcal{H}}$ be the inner product, the tensor product and the norm on $\mathcal{H}$, respectively. In addition, we also often denote $f^T \cdot$ for $\langle f, \cdot \rangle_{\mathcal{H}}$ and $f g^T$ for $f \otimes g$. Let $\phi: \mathcal{X} \rightarrow \mathcal{H}$ be the feature mapping of $\mathcal{H}$, i.e., $f(x) = \langle f , \phi(x) \rangle_{\mathcal{H}}, \forall x \in \mathcal{X}, f \in \mathcal{H}$, and $k(x,y) = \langle \phi(x), \phi(y) \rangle_{\mathcal{H}}, \forall x,y \in \mathcal{X}$. In kernel ridge regression, we aim at solving the optimization problem: 
\begin{align*}
    \min_{f \in \mathcal{H}} \frac{1}{n} \sum_{i=1}^n (y_i - f(x_i))^2 + \lambda \| f \|^2_{\mathcal{H}},
\end{align*}
where $\mathcal{H}$ is an RKHS, $\lambda > 0$ is a regularization parameter, and $\{(x_i, y_i)_{i=1}^n \}$ are i.i.d samples from an unknown data distribution $P_{\mathcal{D}}(x,y)$.

Let $I_{\mathcal{H}}$ be the identity map on $\mathcal{H}$, $y = [y_1, \ldots, y_n]^T \in \mathbb{R}^n$ be the response vector,  $\Phi = [\phi(x_1)^T, \ldots, \phi(x_n)^T]^T: \mathcal{H} \rightarrow \mathbb{R}^n$ be the data operator, and $\hat{\Lambda} := \frac{1}{n} \sum_{i=1}^n \phi(x_i) \phi(x_i)^T = \frac{1}{n} \Phi^T \Phi: \mathcal{H} \rightarrow \mathcal{H}$ be the empirical covariance self-adjoin operator. Since $\hat{\Lambda} + \lambda I_{\mathcal{H}}$ is positive-definite for $\lambda > 0$, the inverse operator $(\hat{\Lambda} + \lambda I_{\mathcal{H}})^{-1}$ is well-defined. The kernel ridge regression above admits a closed-form solution: 
\begin{align*}
    \hat{f}_{\lambda}(x) &=  \phi(x)^T (\hat{\Lambda} + \lambda I_{\mathcal{H}})^{-1} \frac{1}{n} \Phi^T y =\phi(x)^T (\hat{\Lambda} + \lambda I_{\mathcal{H}})^{-1} \frac{1}{n} \sum_{i=1}^n y_i \phi(x_i).
\end{align*}
The following lemma shows an interesting connection between a posterior variance in Bayesian linear regression and the uncertainty function $\| \phi(x)\|_{(\hat{\Lambda} + \lambda I_{\mathcal{H}})^{-1}}^2$ in $\mathcal{H}$.
\begin{lem}
Let $k^n(x) = [k(x_1, x), ..., k(x_n, x)]^T \in \mathbb{R}^{n}$ and $K^n = [k(x_i, x_j)]_{1 \leq i,j \leq n} \in \mathbb{R}^{n \times n}$. For any $x$, we have 
\begin{align*}
     \lambda \| \phi(x)\|_{(\hat{\Lambda} + \lambda I_{\mathcal{H}})^{-1}}^2 =  k(x,x) - k^n(x)^T (K^n + \lambda n I_n)^{-1} k^n(x). 
\end{align*}
\end{lem}
\begin{rem}
The posterior variance in the RHS of the expression above appears again the Subsection \ref{subsect:gp} of this chapter about Gaussian Processes while the uncertainty function (which appears in the LHS of the above expression) is often used to implement the so-called optimism in the face of uncertainty principle in reinforcement learning.
\end{rem}
\begin{proof}
We have 
\begin{align*}
    \phi(x) &= (\frac{1}{n} \Phi ^T \Phi + \lambda I_{\mathcal{H}})^{-1} ( \frac{1}{n}\Phi ^T \Phi + \lambda I_{\mathcal{H}}) \phi(x) \\ 
    &= (\frac{1}{n}\Phi ^T \Phi + \lambda I_{\mathcal{H}})^{-1} \frac{1}{n}\Phi ^T \Phi \phi(x) + \lambda (\frac{1}{n}\Phi ^T \Phi  + \lambda I_{\mathcal{H}})^{-1} \phi(x).
\end{align*}
Thus, we have 
\begin{align*}
    k(x,x) &= \|\phi(x) \|^2_{\mathcal{H}} = \phi(x)^T \phi(x)  \\
    &= \phi(x)^T ( \frac{1}{n}\Phi ^T \Phi + \lambda I_{\mathcal{H}})^{-1} \frac{1}{n} \Phi^T \Phi \phi(x)  + \lambda \phi(x)^T (\frac{1}{n}\Phi ^T \Phi + \lambda I_{\mathcal{H}})^{-1} \phi(x) \\
    &= \frac{1}{n}\phi(x^T) \Phi^T (\frac{1}{n}\Phi \Phi^T + \lambda I_n)^{-1} \Phi \phi(x) + \lambda\| \phi(x) \|_{(\hat{\Lambda} + \lambda I_{\mathcal{H}})^{-1}}^2 \\ 
    &=  \frac{1}{n} k^n(x)^T ( \frac{1}{n} K^n + \lambda I_n)^{-1} k^n(x) + + \lambda\| \phi(x) \|_{(\hat{\Lambda} + \lambda I_{\mathcal{H}})^{-1}}^2,
\end{align*}
where the third inequality follows from the inverse matrix lemma which allows the inverse of a infinite-dimensional matrix to be obtained from the inverse of a $n$-dimensional matrix. 
\end{proof}

We refer the readers to \citep[Chapter~7]{ltfp} for an analysis of generalization guarantees in kernel ridge regression, and to \citep{mollenhauer2020singular} for the preliminaries, eigendecomposition, and SVD of operators on RKHS. 

\subsection{Concentration of Measures}
The concentration of measures describes the concentration phenomenon of a random variable or a random process around its expected value. These play a central role in establishing theoretical guarantees in machine learning theory. Here, we briefly go through several important concentration inequalities that are frequently used in machine learning. For more detailed account of concentration inequalities, we refer the readers to a more comprehensive manuscript \citep{boucheron2013concentration}.

We first introduce a simple yet useful inequality, namely Markov's inequality. 
\begin{lem}[Markov's inequality]
For any non-negative random variable $X$ and $t > 0$, we have $P(X \geq t) \leq \frac{\mathbb{E}[X]}{t}$. 
\label{markov_inequality_lemma}
\end{lem}

Markov's inequality, combined with the so-called Cram\'er-Chernoff method, is often very helpful in bounding a quantity of the form $P(X < x)$. To demonstrate this, we will bound the tails of a $\sigma$-subgaussian random variable. 
\begin{defn}[$\sigma$-subgaussian random variable]
A real-valued random variable $X$ is $\sigma$-subgaussian with variance proxy $\sigma^2 > 0$ if 
\begin{align*}
    \mathbb{E} \left[ e^{\lambda (X - \mathbb{E}[X]) } \right] \leq e^{\sigma^2 \lambda^2/2}, \forall \lambda \in \mathbb{R}.
\end{align*}
\end{defn}
\begin{rem}
A Gaussian random variable with variance $\sigma$ is $\sigma$-subgaussian (but the reverse is not true). 
\end{rem}
\begin{rem}
A random variable bounded in $[a,b]$ is a $\sigma$-subgaussian with variance proxy $\sigma^2 = (b-a)^2/4$.
\end{rem}

We will prove using Markov's inequality and the Cram\'er-Chernoff method that the tails of a $\sigma$-subgaussian variable decay exponentially and at least as fast as Gaussian variables. In particular, let $X$ be $\sigma$-subgaussian, we prove that 
\begin{align*}
    P(X - \mathbb{E}[X]\geq \epsilon) \leq \exp \left({ -\frac{\epsilon^2}{2 \sigma^2}} \right), \forall \epsilon > 0. 
\end{align*}
\begin{proof}
For any $\lambda> 0$, we have 
\begin{align*}
    P(X - \mathbb{E}[X] \geq \epsilon) &= P( e^{\lambda (X- \mathbb{E}[X])} \geq e^{\lambda \epsilon}) \text{ (the Cram\'er-Chernoff trick)} \\ 
    &\leq \mathbb{E}[e^{\lambda (X- \mathbb{E}[X])}] e^{-\lambda \epsilon} \text{ (Markov's inequality)}\\ 
    &\leq e^{\sigma^2 \lambda^2/2 -\lambda \epsilon} \text{ (Definition of subgaussian)}. 
\end{align*}
Choosing $\lambda$ to minimize the LHS of the last inequality above, we obtain $\lambda=\epsilon/\sigma^2$ and the desired inequality. 
\end{proof}
\subsubsection{Independence}
The previous inequalities involve only a single sample at a time. In practice, we often work in the multi-sample regime where we want to understand the concentration phenomenon of a quantity computed by multiple samples. The basic multi-sample regime is the i.i.d.  (identically and independently distributed) structure where samples are i.i.d. A common concentration inequality for the i.i.d. structure is Hoeffding's inequality which indicates that the tails of the empirical sum of a sequence of i.i.d. samples are exponentially decayed. 
\begin{lem}[Hoeffding's inequality] For any fixed (deterministic) $n$, let $X_1, ..., X_n$ be independent real-valued $\sigma$-subgaussian random variables. Then, for any $\epsilon > 0$, we have the following upper-tail bound 
\begin{align*}
    P\left( \frac{1}{n} \sum_{i=1}^n X_i - \mathbb{E}[X_1] \geq \epsilon \right) &\leq e^{-n \epsilon^2 /(2 \sigma^2)}, \\ 
    P\left( \frac{1}{n} \sum_{i=1}^n X_i - \mathbb{E}[X_1] \leq -\epsilon \right) &\leq e^{-n \epsilon^2 /(2 \sigma^2)}. 
\end{align*}
Alternatively, for any $\delta \in [0,1]$, we have the following upper confidence bound 
\begin{align*}
    P\left( \frac{1}{n} \sum_{i=1}^n X_i - \mathbb{E}[X_1] < \sqrt{ \frac{2 \sigma^2 \log(1/\delta)}{n}  } \right) \geq 1 - \delta, \\
    P\left( \mathbb{E}[X_1] - \frac{1}{n} \sum_{i=1}^n X_i < \sqrt{ \frac{2 \sigma^2 \log(1/\delta)}{n}  } \right) \geq 1 - \delta.
\end{align*}
\label{lemma:chap3_support_hoeffding}
\end{lem}
An important generalization of Hoeffding's inequality to the case where the quantity of interest is a function of the data is McDiarmid’s inequality. 

\begin{lem}[McDiarmid’s inequality]
Let $X_1, ..., X_n$ be i.i.d. random variables and $f: \mathcal{X}^n \rightarrow \mathbb{R}$ such that for some $c_i >0, \forall i$, we have
\begin{align*}
    |f(x_1, ..., x_{i-1}, x_i, x_{i+1}, ..., x_n) - f(x_1, ..., x_{i-1}, x'_i, x_{i+1}, ..., x_n)| \leq c_i, \forall i \in [n], 
\end{align*}
for any $x_1, x_2, ..., x_n, x'_1, ..., x'_n \in \mathcal{X}$.  Then, we have 
\begin{align*}
    P\left( f(x_1,...,x_n) - \mathbb{E}[f(x_1,...,x_n)] \geq \epsilon \right) &\leq \exp \left( - \frac{2 \epsilon^2}{ \sum_{i=1}^n c_i^2} \right), \\
    P\left( f(x_1,...,x_n) - \mathbb{E}[f(x_1,...,x_n)] \leq -\epsilon \right) &\leq \exp \left( - \frac{2 \epsilon^2}{ \sum_{i=1}^n c_i^2} \right).
\end{align*}
\end{lem}

In the case that a good bound on the variance is known, we can obtain a tighter concentration than Hoeffding's inequality via Bernstein's inequality. Before stating the Bernstein's inequality, we first define the Bernstein's condition. 
\begin{defn}[One-sided Bernstein's condition]
A real-valued random variable $X$ is said to satisfy the one-sided Bernstein's condition with parameter $b > 0$ if 
\begin{align*}
    \mathbb{E} \left[ e^{\lambda( X - \mathbb{E}[X])} \right] \leq \exp \left(  \frac{(\mathbb{V}[X]) \lambda^2/2 }{1 - b \lambda}  \right), \forall \lambda \in [0, 1/b).
\end{align*}
\end{defn}

\begin{rem}
If $X - \mathbb{E}[X] \leq c$ for a given $c > 0$, then $X$ satisfies the one-sided Bernstein's condition with parameter $b = c/3$. 
\end{rem}

Now we are ready to state the Bernstein's inequality. 
\begin{lem}[Bernstein's inequality]
Let $X_1, ..., X_n \sim X$ be i.i.d. real-valued random variables that satisfy the one-sided Bernstein's condition with parameter $b >0$. Then, for any $\epsilon > 0$ and $\delta \in [0,1]$, we have 
\begin{align*}
     P\left( \frac{1}{n} \sum_{i=1}^n X_i - \mathbb{E}[X_1] \geq \epsilon \right) 
     \leq \exp \left( - \frac{n \epsilon^2 / 2}{ \mathbb{V}[X_1] + b\epsilon}  \right), \\
    P\left( \mathbb{E}[X_1] - \frac{1}{n} \sum_{i=1}^n X_i  \geq \epsilon \right) 
     \leq \exp \left( - \frac{n \epsilon^2 / 2}{ \mathbb{V}[X_1] + b\epsilon}  \right). 
\end{align*}
Alternatively, we have 
\begin{align*}
      P\left( \frac{1}{n} \sum_{i=1}^n X_i - \mathbb{E}[X_1]
      < \frac{b}{n} \log(1/\delta) + \sqrt{ \frac{2\mathbb{V}[X_1] \log(1/\delta)}{n}  }
      \right) \geq 1 - \delta, \\
       P\left( \mathbb{E}[X_1] - \frac{1}{n} \sum_{i=1}^n X_i 
      < \frac{b}{n} \log(1/\delta) + \sqrt{ \frac{2\mathbb{V}[X_1] \log(1/\delta)}{n}  }
      \right) \geq 1 - \delta.
\end{align*}
\end{lem}
The concentration inequalities we have discussed so far are all concerned with real-valued random variables. For random matrices, a similar concentration phenomenon also holds. In particular, the matrix Bernstein's inequality below indicates that the tails of the spectral norm of the sum of i.i.d. random matrices are exponentially decayed and are scaled with the dimension $d$. 
\begin{lem}[Matrix Bernstein's inequality \citep{tropp2015introduction}]
Let $X_1, X_2, ..., X_n$ be zero-mean, independent, symmetric, $d \times d$ random matrices such that $\|X_i \|_{op} \leq c, \forall i$ for a given $c>0$. Then, for any $\epsilon > 0$, we have 
\begin{align*}
    P\left(  \| \sum_{i=1}^n X_i \|_{op} \geq \epsilon \right) \leq 2d \exp \left( - \frac{\epsilon^2/2}{ \sigma^2 + c\epsilon/3} \right),
\end{align*}
where $\sigma^2 = \| \sum_{i=1}^n\mathbb{E}[X_i^2] \|_{op}$.
\end{lem}
\begin{rem}
 The matrix Bernstein's inequality reduces into Bernstein's inequality when $d=1$. 
\end{rem}

\begin{rem}
It is possible to generalize the matrix Bernstein's inequality from an Euclidean space to an RKHS using the dimension-free Bernstein's inequality \citep{minsker2017some}. 
\end{rem}

\subsubsection{Martingales}
The next important structure in the multi-sample regime is martingales. The martingale structure describes a specific structure of dependency that does not behave ``wildly''. In particular, the conditional expectation of each random variable is controllable given the prior variables in the martingale sequence. Formally, we briefly describe martingales in the following. 
\begin{defn}[Martingales]
Let $(\Omega, \mathcal{F})$ be a measurable space, $\mathcal{F}_0 := \{ \emptyset, \Omega\} \subseteq \mathcal{F}_1 \subseteq ... \subseteq \mathcal{F}$ be a sequence of sub-$\sigma$-fields. Let $\{X_k\}_{k \geq 0}$ be a sequence of random variables such that $X_k$ is $\mathcal{F}_k$-measurable. The sequence $(X_k)$ is said to be a martingale adapted to the filtration $(\mathcal{F}_k)$ if $\mathbb{E} |X_k| \leq \infty$ and $X_k = \mathbb{E}[ X_{k+1} | \mathcal{F}_k ]$ for all $k \geq 0$. 
\end{defn}
\begin{rem}[Doob construction of martingales] We can obtain martingales from arbitrary structure via the so-called Doob construction. In particular, let $Y_1, ..., Y_n$ be an arbitrary sequence of random variables, and let $X = f(Y_1, ..., Y_n)$ for some function $f$ such that $X$ is integrable. We construct a filtration as follows: let $\mathcal{F}_0 = \{\emptyset, \Omega\}$, and define the generated $\sigma$-field $\mathcal{F}_k := \sigma(Y_1, ..., Y_k), \forall k \in [1,n]$.  Then let $X_k = \mathbb{E}[ X | \mathcal{F}_k], k \in [0,n]$. It is not hard to verify that $(X_k)$ is a martingale adapted to the filtration $(\mathcal{F}_k)$. Note that $X_0 = \mathbb{E}[X | \mathcal{F}_0] = \mathbb{E}[X]$ is deterministic and $X_n = \mathbb{E}[X | \mathcal{F}_n ] = X$.
\end{rem}
We are now ready to state three basic concentration inequalities for martingales: Azuma's inequality, Freedman's inequality, and matrix Freedman's inequality. These inequalities are the martingale counterparts to Hoeffding's inequality, Bernstein's inequality and matrix Bernstein's inequality, respectively. 
\begin{lem}[Azuma's inequality] 
Let $\{X_k\}_{k \geq 0}$ be a martingale such that $|X_k - X_{k-1}| \leq c_k$ a.s. for some $0 < c_k < \infty$ for all $1 \leq k \leq n$. For any $\epsilon \geq 0$, we have 
\begin{align*}
    P(|X_n - X_0| \geq \epsilon ) \leq 2 \exp \left(  -\frac{\epsilon^2}{2 \sum_{i=1}^n c_i^2} \right). 
\end{align*}
\end{lem}

\begin{lem}[Freedman's inequality \citep{tropp2011freedman}] Let $\{X_k \}_{k \geq 0}$ be a martingale  satisfying that $X_k - X_{k-1} \overset{a.s.}{\leq} M, \forall k $ where $M$ can be random. Denote the variance process $W := \sum_{i=1}^n \mathbb{V}[X_k | F_{k-1}]$. Then, for all $\epsilon > 0, \sigma^2 > 0$, we have
\begin{align*}
    P\left(  X_n - X_0 \geq \epsilon, W \leq \sigma^2 \right)
    \leq \exp \left(  \frac{ -\epsilon^2/2 }{  \sigma^2 + M\epsilon/3 }  \right).
\end{align*} 
In addition, if $|X_k - X_{k-1}| \overset{a.s.}{\leq} c, \forall k$, we have 
\begin{align*}
    P\left(  |X_n - X_0| \geq \epsilon, W \leq \sigma^2 \right)
    \leq 2 \exp \left(  \frac{ -\epsilon^2/2 }{  \sigma^2 + M\epsilon/3 }  \right).
\end{align*} 
\end{lem}
\begin{lem}[Matrix Freedman's inequality \citep{tropp2011freedman}]
Let $\{Y_k\}_{k \geq 0}$ be a $d \times d$ symmetric matrix martingale adapted to $\{\mathcal{F}_k\}_{k \geq 0}$ with the difference sequence $\{X_k := Y_k - Y_{k-1}\}_{k \geq 1}$. Assume that the difference sequence is uniformly bounded, i.e., $\lambda_{\max}(X_k) \overset{a.s.}{\leq} c, \forall k \geq 1$ for a given $c>0$. Define the quadratic variation process 
\begin{align*}
    W_k := \sum_{j=1}^k \mathbb{E}_{j-1} [X_j^2], \forall k \geq 1, 
\end{align*}
where $\mathbb{E}_j[\cdot] : = \mathbb{E}[\cdot| \mathcal{F}_j]$. Then, for all $\epsilon > 0, \sigma^2 > 0$, we have
\begin{align*}
    P\left(  \exists k \geq 0: \lambda_{\max}(Y_k) \geq \epsilon \text{ and } \|W_k\|_2 \leq \sigma^2 \right)
    \leq d \exp \left(  \frac{ -\epsilon^2/2 }{  \sigma^2 + c\epsilon/3 }  \right).
\end{align*}
\end{lem}

\begin{rem}
For any fixed $n$, we have 
\begin{align*}
    \left\{ \lambda_{\max}(Y_n) \geq \epsilon \text{ and } \|W_n\|_2 \leq \sigma^2 \right \} \subseteq  \left\{  \exists k \geq 0: \lambda_{\max}(Y_k) \geq \epsilon \text{ and } \|W_k\|_2 \leq \sigma^2 \right \}. 
\end{align*}
Thus, the matrix Freedman's inequality also applies to the matrix martingale with a fixed number of matrices. 
\end{rem}

The connections of the considered concentration inequalities are summarized in Table \ref{tab:chap2_concentration_ineqs}. 

\begin{table}[h]
    \centering
    \begin{tabular}{l|l}
     \textbf{Independence}  & \textbf{Martingales}  \\
      \hline 
      \hline 
     Hoeffding's inequality & Azuma's inequality\\ 
     Bernstein's inequality& Freedman's inequality\\ 
     Matrix Bernstein's inequality & Matrix Freedman's inequality 
    \end{tabular}
    \caption{The ``duality'' of concentration inequalities for i.i.d. and martingale structures.}
    \label{tab:chap2_concentration_ineqs}
\end{table}


\subsection{Foundations of Generalization Theory}
The general goal of machine learning is to learn a function for a certain task from training data that can perform the task well on the unseen data. The ability to generalize from training data to unseen data is called generalization ability. Generalization theory signifies theoretical guarantees on the generalization ability of certain algorithms and provides insights into many important questions such as how an algorithm or model works, what can and cannot be learned from the data, and whether it is possible to design a sample-efficient algorithm. Moreover, generalization theory also provides insights toward suggesting a better algorithm and a better model in the future. In this subsection, we provide a concise overview of generalization theory underlying machine learning tasks. In particular, we will iterate over several important notions and foundational results in both classical and modern generalization theory. For more detailed account and literature, we refer the readers to \citep{berner2021modern,ltfp,dlt_book}.

\noindent \textbf{Formulation}. 
Let $X, Y$ be two random variables following the joint data distribution $P_{\mathcal{D}}(X,Y)$, $\mathcal{X}$ and $\mathcal{Y}$ be the domains of $X$ and $Y$, respectively, $\mathcal{H} \subset \{\mathcal{X} \rightarrow \mathcal{Y}\}$ be a hypothesis class, and $S$ be a sample of $n$ data points $\{(x_i, y_i)\}_{i=1}^n$ sampled from the data distribution $P_{\mathcal{D}}(X,Y)$. The goal of the prediction task is to learn a function $h \in \mathcal{H}$ from the training data $S$ such that $h$ is a ``good'' predictor on the unseen data from ${P}_{\mathcal{D}}$. The goodness of the predictor is measured through a loss function $l: \mathcal{Y} \times \mathcal{Y} \rightarrow \mathbb{R}^{+}$. For example, in a regression task where $\mathcal{Y} \subseteq \mathbb{R}$, one often uses the squared error as the loss function $l(\hat{y},y) = \frac{1}{2}(\hat{y} - y)^2$. Without loss of generality, we assume that $l: \mathcal{Y} \times \mathcal{Y} \rightarrow [0,1]$. \footnote{For example, in the case of squared loss function, we assume $\mathcal{Y} = [0,1]$, thus the squared loss is bounded within $[0,1]$.} We consider empirical risk minimization: 
\begin{align*}
    \hat{h} = \arginf_{h \in \mathcal{H}} \hat{L}_S(h) \text{ where } \hat{L}_{S}(h) = \frac{1}{n} \sum_{i=1}^n l(h(x_i), y_i).  
\end{align*}
The \textit{generalization error} is defined as 
\begin{align*}
    \Delta_S(h) := L_{\mathcal{D}}(h) - \hat{L}_S(h) \text{ where } L_{\mathcal{D}}(h) = \mathbb{E}_{\mathcal{D}}\left[ l(h(X), Y) \right]. 
\end{align*}
Intuitively, the generalization error signifies how much the prediction error of a hypothesis on the training data can inform about the prediction error of the hypothesis on the entire data distribution. Note that a small generalization error does not imply the hypothesis $h$ is any good. For example, a random predictor $h$ can have small generalization error as its training and expected error are both arbitrarily bad and close to each other. 

\subsubsection{A Complexity-Measure Perspective}
We review several basic classical generalization theories which are based on a \textit{complexity measure} for a hypothesis class. 

First, if the hypothesis class $\mathcal{H}$ has finite elements, a simple union bound and McDiarmid’s inequality yield a generalization bound. 
\begin{thm}[Finite hypothesis classes]
For any $\delta \in (0,1)$, with probability at least $1-\delta$, we have 
\begin{align*}
    \forall h \in \mathcal{H}, \Delta_S(h) \leq \sqrt{ \frac{\log( |\mathcal{H}|/\delta)}{2n} }.
\end{align*}
\label{chap2_gen_union}
\end{thm}
The naive generalization bound above requires the finite cardinality of the hypothesis class. For infinite hypothesis classes, it is possible to obtain a similar bound using a complexity measure known as Vapnik and Chervonenkis (VC) dimension \citep{vapnik2013nature}. 
\begin{thm}[Uniform generalization bound]
Let $d=VC(\mathcal{H})$ be the VC dimension of $\mathcal{H}$ and assume that $d$ is finite. For any $\delta \in (0,1)$, with probability at least $1 - \delta$, we have 
\begin{align*}
    \forall h \in \mathcal{H}, \Delta_S(h) \leq \sqrt{\frac{8d}{n} \log \frac{2en}{d} + \frac{8}{n}\log(4/\delta)}.
\end{align*}
\label{chap2_gen_vc}
\end{thm}
The generalization bounds in Theorem \ref{chap2_gen_union} and Theorem \ref{chap2_gen_vc} hold uniformly over the hypothesis class $\mathcal{H}$ and the data distribution. Thus, these bounds are often too conservative to be helpful in many practical settings. A different complexity measure that slightly mitigates the conservativeness of the generalization bounds above is Rademacher complexity \citep{bartlett2002rademacher} which depends on the data distribution. 
\begin{defn}[Rademacher complexity]
Let $S = \{z_1, ..., z_n\}$ be a set of i.i.d. samples drawn from a distribution $P_{Z}$ supported on domain $\mathcal{Z}$. Let $\mathcal{F}$ be a class of functions $\mathcal{Z} \rightarrow \mathbb{R}$. The \textit{empirical Rademacher complexity} of $\mathcal{F}$ is defined as 
\begin{align*}
    \hat{R}_n(\mathcal{F}; S) = \mathbb{E}_{\sigma} \left[ \sup_{f \in \mathcal{F}} \frac{1}{n} \sum_{i=1}^n \sigma_i f(z_i) \right],
\end{align*}
where $\{\sigma_i \}_{i=1}^n$ are i.i.d. samples from the uniform distribution over \{-1,1\}. 
The Rademacher complexity of $\mathcal{F}$ is defined as 
\begin{align*}
R_n(\mathcal{F}) = \mathbb{E}_{\mathcal{S}} \left[ \hat{R}_{n}(\mathcal{F}; S) \right].
\end{align*}
\end{defn}
Intuitively, the Rademacher complexity measures the ability of functions from $\mathcal{F}$ to fit random noises. Using a general Rademacher-based uniform convergence, we can bound the generalization error by the Rademacher complexity of the hypothesis class. 

\begin{thm}[Rademacher complexity-based bound]
For any $\delta \in (0,1)$, with probability at least $1 - \delta$, we have
\begin{align*}
    \forall h \in \mathcal{H}, \Delta_S(h) \leq 2R_n(\mathcal{H}) + \sqrt{ \log(1/\delta) / n}. 
\end{align*}
The generalization error can also be bounded by empirical Rademacher complexity. In particular, with probability at least $1 - \delta$, we have
\begin{align*}
    \forall h \in \mathcal{H}, \Delta_S(h) \leq 2 \hat{R}_n(\mathcal{H}; S) + 3\sqrt{ \log(2/\delta) / n}. 
\end{align*}
\end{thm}

\subsubsection{A Parameter-Dynamic Perspective}
The major limitation of the complexity-measure generalization bounds above is that they rely on uniform convergence and are independent of the training algorithm. Intuitively, they follow the Occam Razor's principle that a model with simpler complexity is more preferred to explain the same observational phenomenon. In other words, it indicates that a model that overfits the training data tends to have large test error (thus large generalization gap). This phenomenon however does not hold in modern practice with deep neural networks where an overparameterized neural network can obtain both a zero training error and small test error \citep{zhang2016understanding}. Such an excellent generalization of deep neural networks has spurred the development of new generalization theories that go beyond the limitations of classical statistical learning theory. Early efforts in understanding the generalization of deep neural networks make an interesting connection between the generalization of deep neural networks with information compression in information theory \citep{tishby2015deep,tang2019markov}. However, this approach is rather conceptual rather than providing a generalization bound. Here, we briefly mention several key concepts that have made some substantial progress in understanding generalization in deep neural networks. These approaches study the evolution of (the distribution of) the neural network parameters and connect it with the generalization ability. In particular, they explain the generalization phenomenon of overparameterized models from 
the perspective of parameter evolution where the dynamics of the neural network parameters over the training is analyzed. For technical details, we prefer the readers to the recent comprehensive manuscripts in deep learning theory \citep{berner2021modern,bartlett2021deep,roberts2021principles,dlt_lec,dlt_princeton}. 


The first concept comes from the classic PAC-Bayes \citep{mcallester1999some,mcallester1999pac} which can mitigate the uniformity over the hypothesis class of Rademacher complexity by allowing distributions over the hypothesis class. 
\begin{thm}[PAC-Bayes bound]
For any prior distribution $P$ over the hypothesis class $\mathcal{H}$, for any $\delta \in (0,1]$, with probability at least $1-\delta$, we have 
\begin{align*}
    \forall Q \in \mathcal{Q}(\mathcal{H}), \mathbb{E}_{h \sim P} \left[ \Delta_S(h)\right] \leq \sqrt{\frac{D_{KL}(Q \| P) + \log(2\sqrt{n}/\delta)}{2n}}. 
\end{align*}
\end{thm}
For example, $P$ could be the uniform distribution over all deep networks with a fixed architecture and $Q$ is the weight distribution obtained after training the deep network on $n$ samples \citep{DBLP:conf/uai/DziugaiteR17}.

Another important concept is neural tangent kernels.  \citet{jacot2018neural} show that the dynamics of  overparamterized neural networks can be described by the so-called neural tangent kernels (NTK) which kernelize the parameter dynamic equation under stochastic gradient descent (SGD). The NTK analysis is able to explain the generalization of deep neural networks under certain training conditions without using capacity-based complexity measure. While the NTK analysis provides a finite-time convergent rate of SGD for ultra-wide neural networks, the guarantee requires that the trained parameter at any time step is not too far from the initialized parameter. This closeness condition does not reflect deep neural network regime in practice. That leads to the third concept related to mean-field analysis. The mean-field analysis does not require such closeness condition to analyze the generalization of deep neural networks. Instead, mean-field analysis studies the parameter evolution from a distributional perspective. In particular, \citet{mei2018mean,chizat2018global} show that the empirical distribution of the two-layer neural network parameters can be described as a Wasserstein gradient flow which converges to the global optimum under certain structural assumptions.  


\section{Introduction to Reinforcement Learning} 
Reinforcement learning (RL) is a learning paradigm that involves interacting with an underlying environment and (strategically) taking actions to maximize a numerical goal. The learner does not receive any supervision about which action to take but instead must learn to take good actions by trying them out and learning from their consequences. In many situations, the consequence of an action affects not only the immediate reward but also the future rewards. In other words, the consequence of an action can be delayed and propagated into the future situations.

Beyond delayed rewards, interactive nature, and no direct supervision, RL is also fundamentally different from supervised learning in the trade-off between \textit{exploration} and \textit{exploitation}. On the one hand, the agent needs to explore different actions it has not tried before to acquire more reliable knowledge about the underlying environment. On the other hand, the agent should exploit its acquired knowledge to obtain higher rewards. Being too exploitative in the early stage of learning when the acquired knowledge is insufficient can get the agent stuck into taking sup-optimal actions. In the same way, being too explorative in the later stage of learning can be highly inefficient when the acquired knowledge is already more reliable. Thus, to strategically balance between exploration and exploitation is highly non-trivial and is in fact one of the fundamental challenges of designing an sample-efficient RL agent. 

The best way to gain good intuition about RL is perhaps through examples and applications: 

\begin{itemize}
    \item \textbf{Autonomous driving}: There are various aspects in autonomous driving that RL can apply. For example, we can automatically learn to park a car, change a lane or overtake when avoiding collision and maintaining a steady speed thereafter. 
    
    \item \textbf{Game playing}: A master Go player makes a move. The move is driven by the goal of winning the game and is informed by immediate judgement of the current move and by speculating possible counter moves by the opponent. 
    
    \item \textbf{Trading and finance}: A learner can compute a numerical reward function based on the loss and gain of every past transactions to inform which action to take (whether to buy, hold or sell at a particular stock price).
    
    \item \textbf{Healthcare}: Given clinical observations and assessment of a patient, a RL agent can determine the best treatment option for the patient at different treatment stages.
    
    \item \textbf{Adaptive experimental design}: A RL agent can adaptively optimize novel processes for manufacturing new materials, from short polymer fibers, alloys to food production. 
\end{itemize}

Besides the agent and the environment, their interaction is made possible by a \textit{policy} pursued by the agent to navigate the environment and a reward function received by the agent from the environment. A policy specifies for the learning agent a way of behaving in the environment. It is a mapping from any state of the environment to an action that should be taken when the agent is in that state. The mapping can also depend on the observations of the historical interactions between the agent and the environment. 

A reward function defines the goal of a RL agent. Intuitively, a reward function maps each state (or a state-action pair) to a numerical value indicating the relevance of that state (or that state-action pair) toward achieving the goal. The reward signal can distinguish between good and bad actions for the agent in the immediate term. However, what is good in the immediate term is not necessarily good in the long term. This holds for the decision problems with both no-delayed rewards and delayed rewards. In a no-delayed reward problem, the reward perceived might be noisy and thus might be a spurious indicator of the best actions when there is lack of sufficient data. In a delayed-reward problem, a state that yields a low immediate reward might in the long run lead to higher rewards as it is followed by other states with higher rewards. To make an analogy, in psychology, a similar phenomenon is the ability to resist the temptation of the immediate pleasure to receive a larger or long-lasting reward later is referred to as \textit{delayed gratification}. To measure the long term reward in RL, a \textit{value} function is used. A value function maps each state (or state-action pair) to the total amount of reward the agent is expected to accumulate over time when starting from that state (or state-action pair) onward and following a particular policy. 

There are three main problem instances of RL. The first instance is multi-armed bandits which is a simplified version of RL where the rewards are not delayed (but could be noisy) and the action space is finite. The second instance is Bayesian optimization which is identical to multi-armed bandit problem except that the action space is continuous. The third instance is a full version of RL with delayed rewards which is modelled by Markov decision processes. We formally introduce these decision-making instances and their technical background in the following subsections. 
\section{Stochastic Bandits}
\label{chap2_section_bandit}
\subsection{Introduction}
There are several models for multi-armed bandits. For the sake of introducing some fundamental aspects of decision making in bandits in particular and in RL in general, it is sufficient to focus on stochastic bandits \citep{lattimore_szepesvari_2020}.  

A stochastic bandit consists of a set of reward distributions $\nu = \{P_a: a \in \mathcal{A} \}$ where $\mathcal{A}$ is the set of all plausible actions with $|\mathcal{A}| = k$. The learner and the environment interact with each other sequentially over $n$ rounds as follows. At each round $t \in [n] := \{1,2,...,n\} $, the learner takes an action $a_t \in \mathcal{A}$ and subsequently observes the reward $r_t \sim P_{a_t}$. Note that the learner knows neither the reward distributions $\{P_a: a \in \mathcal{A}\}$ nor the reward samples of the other actions rather than $a_t$. The interaction induces a sequence of random variables $a_1, r_1, ..., a_n, r_n$. Define the expected reward for each action $a \in \mathcal{A}$ on bandit instance $\nu$ as 
\begin{align*}
    r^{\nu}(a) = \int x P_a(dx). 
\end{align*}
When the context is clear, we often drop the superscript $\nu$ in the expected reward function $r^{\nu}$. We define the sub-optimality gap as 
\begin{align*}
    \Delta_a = \mu(a^*) - \mu(a) \geq 0, 
\end{align*}
where $a^* = \operatorname*{arg\,max}_{a \in \mathcal{A}} r(a)$.

The learner's goal is to find a learning strategy, i.e., a policy, that maximizes the expected value of the total reward $S_n = \sum_{i=1}^n r_t$. Note that the total reward $S_n$ is a random variable that depends on the action sequence of the learner and the reward distributions of the environment. A learning strategy of the learner is a distribution $\pi$ over the action space $\mathcal{A}$ where $\pi(a)$ is the probability of choosing an action $a \in \mathcal{A}$. 
\subsection{Performance Metrics}
Supervised learning uses the \textit{excess risk} to evaluate the goodness of a classifier. In bandits (and RL in general), it is common to use \textit{regret} to measure the performance of a policy. 
A common performance metric is the (frequentist) {regret} of a policy $\pi$ on bandit instance $\nu$ which is defined as 
\begin{align*}
    R_n(\pi; \nu) = n r^{\nu}(a^*) - \mathbb{E}_{\pi} \left[ \sum_{i=1}^n r_i \right], 
\end{align*}
where the expectation $\mathbb{E}_{\pi}$ is taken over the randomness of the sequence $\{r_i\}_{i=1}^n$ induced by policy $\pi$ and the reward distributions $\{P_a\}_{a \in \mathcal{A}}$. Note that $R_n(\pi; \nu) \geq 0$ with equality iff $\pi$ selects the optimal action(s) for all rounds $t \in [n]$. In practice, the learner often cannot achieve a zero regret as it does not know the optimal action in advance. Instead, it must learn to select the optimal action from the interaction with the bandit instance. The ignorance of the learner about the optimal action incurs a positive regret. The goal of the learner is to achieve as a small regret as possible given $n$ rounds of interaction. The regret of a good learner at least disappears when it collects an infinite amount of data, i.e., $n \rightarrow \infty$. Formally, a \textit{good} learner on some class of bandit problems $\mathcal{E}$ has a \textit{sublinear} regret on the class, i.e.,
\begin{align*}
    \lim_{n \rightarrow \infty} \frac{R_n(\pi; \nu)}{n} = 0, \forall \nu \in \mathcal{E}. 
\end{align*}
If $\frac{R_n(\pi; \nu)}{n} = \tilde{O}(1/n^{\alpha})$ for some $\alpha > 0$, we say that the \textit{regret rate} of $\pi$ on $\nu$ is $1/n^\alpha$. The larger the $\alpha$, the faster the regret vanishes and the faster $\pi$ can find the optimal action. Often in practice, we refer to $n^{-1/2}$ as a \textit{slow} rate and $n^{-1}$ as a \textit{fast} rate. 

When we have a prior distribution $Q$ of the bandit instances over $\mathcal{E}$, it is often more helpful to use the so-called \textit{Bayesian regret} than the frequentist regret. The Bayesian regret is the average regret with respect to the prior $Q$: 
\begin{align*}
    BR_n(\pi) = \int_{\mathcal{E}} R_n(\pi; \nu) Q(d\nu).
\end{align*}
In this chapter, we focus on analyzing a bandit algorithm based on the frequentist regret. In Chapter \ref{chap:three}, we use the Bayesian regret to analyze our proposed algorithm for robust decision making under an uncontrollable environmental variable. The Bayesian regret is also discussed in Section \ref{chap2_bo} of this chapter. 
\subsection{Optimism Principle}
\label{chap2_bandit_OFU}
One of the most popular ideas that can handle exploration-exploitation trade-offs to achieve a sublinear regret in bandits is the \textit{optimism in the face of uncertainty} (OFU) principle. OFU dates back to the seminal work \citep{lai1985asymptotically} where the upper confidence bound criterion is used to balance between exploration and exploitation in multi-armed bandits. The basic idea of the principle is that at each round $t$, we construct a confident set for the value estimates, and optimistically choose an estimate from the confidence set and an action such that the predicted reward is maximized.

To demonstrate the benefit of OFU in balancing exploration-exploitation trade-off, we present a concrete algorithm of the OFU principle, namely the upper confidence bound (UCB) algorithm. For convenience, for any $a \in \mathcal{A}$ and $t \in [n]$, let $N_t(a)$ be the number of times up to round $t$ action $a$ is chosen and $\hat{r}_t(a)$ be the empirical mean of action $a$ using the collected data up to time $t$, i.e., 
\begin{align*}
    N_t(a) &:= \sum_{i=1}^t 1\{a_i = a\}, \\ 
    \hat{r}_t(a) &:= 
    \begin{cases}
        \frac{1}{N_t(a)} \sum_{j=1}^t r_i \cdot 1\{a_i = a\} \text{ if } N_t(a) > 0 \\ 
        0 \text{ if } N_t(a) = 0. 
    \end{cases}
\end{align*}
We also define a bonus function $b_t(a) =  \sqrt{\frac{\log(t^{\alpha})}{2 N_t(a)}}$ for some $\alpha > 0$. The upper confidence bound function is defined as 
\begin{align*}
    U_t(a) = \hat{r}_{t-1}(a) + b_{t-1}(a). 
\end{align*}


The details are presented in Algorithm \ref{alg:ucb}. The UCB algorithm overestimates the unknown means with high probability. We show that such a simple idea can achieve a sublinear regret (which can in fact be shown to be minimax-optimal). The idea of this proof follows from \citep{Patrick_aFoL}.

\begin{algorithm}
  \caption{UCB$(\alpha)$}
\label{alg:ucb}
\begin{algorithmic}[1]
  \STATE {\bfseries Input:} bandit instance $\nu$, parameter $\alpha > 0$ 
  \FOR{$t=1$ {\bfseries to} $n$}
  \STATE Choose an (optimistic) action $a_t \in \operatorname*{arg\, max}_{a \in \mathcal{A}} U_t(a)$ 
  \STATE Observe reward $r_t \sim P_{a_t}$ 
  \STATE Update $U_t$ using all the data collected so far 
  \ENDFOR
\end{algorithmic}
\end{algorithm}

To obtain a nontrivial rate, we need to make some structural assumption on the reward. A common and general assumption as such is that the conditional mean of $r_t$ depends only on the current action and that the tails of  $r_t$ are conditional 1-subgaussian, as presented in Assumption \ref{assumption:bandit}.
\begin{assumption} 
$\mathbb{E} \left[ \exp(\lambda(r_t - r(a_t))) | r_1, a_1, ..., r_{t-1}, a_{t-1}, a_t \right] \leq \exp(\lambda^2/2), \forall \lambda$ and $\mathbb{E} \left[r_t | r_1, a_1, ..., r_{t-1}, a_{t-1}, a_t \right] = r(a_t)$.
\label{assumption:bandit}
\end{assumption}
Let $\mathcal{B}_g$ be the set of bandits that satisfy Assumption \ref{assumption:bandit}. We can prove that the UCB algorithm achieves a sublinear regret. 
\begin{thm}
For any bandit instance $\nu \in \mathcal{B}_g$, and $\alpha > 1$,  UCB$(\alpha)$ presented in Algorithm \ref{alg:ucb} achieves a regret of $R_n(\nu) = O(\sqrt{nk \log n})$.
\label{theorem:chap2_bandit_regret}
\end{thm}
We start with two useful lemmas. The first lemma is that if a sup-optimal action has been played for a number of times, the probability that it is played again is small. Intuitively, $N_{t-1}(a)$ should be small as $\Delta_a$ is large because in such case, it is easier to distinguish the sub-optimal action $a$ from the optimal one. In the following lemma, we show that $1/\Delta_a^2$ is a ``right'' scaling for $N_{t-1}(a)$. 

\begin{lem}
For any $\delta > 0$, let $a_t = \operatorname*{arg\,max}_{a \in \mathcal{A}} U_{t-1}(a)$ where $U_{t-1}(a) = \hat{r}_{t-1}(a) + \sqrt{\frac{\log(1/\delta)}{2 N_{t-1}(a)}} $. For any sub-optimal action $a \in \mathcal{A}$, i.e., $\Delta_a > 0$, we have 
\begin{align*}
    \mathbb{P}\left( a_t = a | N_{t-1}(a) \geq 2 \frac{\log(1/\delta)}{\Delta_a^2} \right) \leq 2 \delta. 
\end{align*}
\label{lemma:chap2_bandit_conditional_action_count}
\end{lem}
\begin{proof}
For any $a \in \mathcal{A}$, define the following events 
\begin{align*}
    \mathcal{M} &:= \left\{a_t = a \bigg | N_{t-1}(a) \geq 2 \frac{\log(1/\delta)}{\Delta_a^2} \right\} \\ 
    \mathcal{E}_1 &:= \left\{r(a^*) \leq U_{t-1}(a^*) \bigg | N_{t-1}(a) \geq 2 \frac{\log(1/\delta)}{\Delta_a^2} \right\} \\ 
    \mathcal{E}_2 &:= \left\{U_{t-1}(a) < r(a^*) \bigg | N_{t-1}(a) \geq 2 \frac{\log(1/\delta)}{\Delta_a^2} \right\} \\ 
    \mathcal{E} &:=  \mathcal{E}_1 \cap \mathcal{E}_2 .
\end{align*}
We have
\begin{align*}
    &\left\{ \hat{r}_{t-1}(a) \geq r(a) + \sqrt{\frac{\log(1/\delta) }{2 N_{t-1}(a)}} \bigg | N_{t-1}(a) \geq 2 \frac{\log(1/\delta)}{\Delta_a^2}  \right\} \\
    &= \left\{ r(a^*) - \hat{r}_{t-1}(a) \leq \Delta_a - \sqrt{\frac{\log(1/\delta)}{2 N_{t-1}(a)}}  \bigg | N_{t-1}(a) \geq 2 \frac{\log(1/\delta)}{\Delta_a^2} \right\}\\ 
    &\supseteq  \left\{ r(a^*) - \hat{r}_{t-1}(a) \leq \sqrt{2\frac{\log(1/\delta)}{ N_{t-1}(a)}} - \sqrt{\frac{\log(1/\delta)}{2 N_{t-1}(a)}}  \bigg | N_{t-1}(a) \geq 2 \frac{\log(1/\delta)}{\Delta_a^2} \right\} = \mathcal{E}_2^c. 
\end{align*}
It follows from Hoeffding's inequality in Lemma \ref{lemma:chap3_support_hoeffding} that 
\begin{align*}
    \mathbb{P}(\mathcal{E}_1^c | N_{t-1}(a^*)) &\leq \delta \\ 
    \mathbb{P}(\mathcal{E}_2^c | N_{t-1}(a)) &\leq \mathbb{P}\left( \hat{r}_{t-1}(a) \geq r(a) + \sqrt{\frac{\log(1/\delta) }{2 N_{t-1}(a)}} \bigg | N_{t-1}(a) \geq 2 \frac{\log(1/\delta)}{\Delta_a^2}, N_{t-1}(a)  \right) \\
    &\leq \delta.
\end{align*}
Note that we condition on $N_{t-1}(a^*)$ and $N_{t-1}(a)$ in the inequalities above as $N_{t-1}(a^*)$ and $N_{t-1}(a)$ are random while Hoelfding's inequality is applicable only when the number of samples is deterministic. Thus we have $\mathbb{P}(\mathcal{E}_1^c) \leq \delta$ and $\mathbb{P}(\mathcal{E}_2^c) \leq \delta$.

It is clear that $\mathcal{M} \cap \mathcal{E} = \emptyset$ as on $\mathcal{E}$ we have $U_{t-1}(a^*) > U_{t-1}(a)$. Thus, we have $\mathcal{M} \subseteq \mathcal{E}^c$, or $\mathbb{P}(\mathcal{M}) \leq \mathbb{P}(\mathcal{E}^c) \leq \mathbb{P}(\mathcal{E}_1^c) + \mathbb{P}(\mathcal{E}_2^c) \leq 2 \delta$.
\end{proof}

In the second lemma, we show that the total number of times an action is played up to some time point can be bounded given that the action is played for a certain number of times in the past. 
\begin{lem}
    For any sequence $0 \leq s_1 \leq s_2 \leq ... \leq s_n$, and any $a \in \mathcal{A}$, we have 
    \begin{align*}
        \mathbb{E}[N_n(a)] \leq s_n + \sum_{t=1}^{n-1} \mathbb{P}(a_{t+1}=a| N_t(a) \geq s_t).
    \end{align*}
\label{lemma:chap2_bandit_expected_action_count}
\end{lem}
\begin{proof}
Let us define 
\begin{align*}
    I_n := 1\{a_1 = a\} + \sum_{t=1}^{n-1} 1\{a_{t+1} = a, N_t(a) < s_t \} .
\end{align*}
Assume by contradiction that $ I_n > s_n$. Then $I_n \geq \floor{s_n} + 1$ or $\sum_{t=1}^{n-1} 1\{a_{t+1} = a, N_t(a) < s_t\} \geq \floor{s_n}$. It implies that there exists $\floor{s_n}$ times $1\{a_{t+1} = a, N_t(a) < s_t\} = 1$. Let $\bar{t} \in [1,n-1]$ be the $\floor{s_n}$-th time that $1\{a_{\bar{t}+1} = a, N_{\bar{t}}(a) < s_{\bar{t}}\} = 1$. Then, up to time $\bar{t}$, the action $a$ has been played at least $\floor{s_n} + 1$ times but $N_{\bar{t}}(a) < s_{\bar{t}} \leq s_n \leq \floor{s_n} + 1$. This contradiction implies that $I_n \leq s_n$. 

We have 
\begin{align*}
    N_n(a) &= \sum_{t=1}^n 1\{a_t = a\} = 1\{a_1 = a\} + \sum_{t=1}^{n-1} 1\{a_{t+1} = a\} \\
    &= 1\{a_1 = a\} + \sum_{t=1}^{n-1} 1\{a_{t+1} = a, N_t(a) \geq s_t\} + \sum_{t=1}^{n-1} 1\{a_{t+1} = a, N_t(a) < s_t\} \\ 
    &= I_n + \sum_{t=1}^{n-1} 1\{a_{t+1} = a, N_t(a) \geq s_t\}. 
\end{align*}
Thus, 
\begin{align*}
    \mathbb{E}_n[N_n(a)] &= \mathbb{E}[I_n] + \mathbb{E} \left[ \sum_{t=1}^{n-1} 1\{A_{t+1} = a, N_t(a) \geq s_t\} \right] \\
    &\leq s_n + \sum_{t=1}^{n-1} \mathbb{P}(a_{t+1} = a | N_t(a) \geq s_t).
\end{align*}
\end{proof}

Now putting everything together, we are ready to prove Theorem \ref{theorem:chap2_bandit_regret}. 
\begin{proof}[Proof of Theorem \ref{theorem:chap2_bandit_regret}]

For any $c > 0$, it follows from Lemma \ref{lemma:chap2_bandit_conditional_action_count} and Lemma \ref{lemma:chap2_bandit_expected_action_count} with $s_t = 2 \frac{\alpha \log(t)}{\Delta_a^2}$ and $\alpha > 1$ that 
\begin{align*}
    R_n &= \sum_{a \in \mathcal{A}} \Delta_a \mathbb{E}\left[ N_n(a) \right] \\ 
    &=  \sum_{a \in \mathcal{A}, \Delta_a < c} \Delta_a \mathbb{E}\left[ N_n(a) \right] + \sum_{a \in \mathcal{A}, \Delta_a \geq c} \Delta_a \mathbb{E}\left[ N_n(a) \right] \\ 
    &\leq c n +  \sum_{a \in \mathcal{A}, \Delta_a \geq c} \Delta_a \mathbb{E}\left[ N_n(a) \right] \\ 
    &\leq c n +  \sum_{a \in \mathcal{A}, \Delta_a \geq c} \Delta_a \cdot 2 \frac{ \alpha \log(n)}{\Delta_a^2} + \sum_{a \in \mathcal{A}} \sum_{t=1}^{n-1} \Delta_a \frac{2}{t^{\epsilon}} \\ 
    &\leq c n + \sum_{a \in \mathcal{A},\Delta_a \geq c} 2 \frac{ \alpha \log(n)}{\Delta_a} + \sum_{a \in \mathcal{A}, \Delta_a \geq c} \Delta_a \sum_{t=1}^{n-1}  \frac{2}{t^{\alpha}} \\ 
    &\leq cn + \frac{2 \alpha k \log n}{c} + \frac{2}{\alpha -1 } \sum_{a \in \mathcal{A}} \Delta_a. 
\end{align*}
The RHS of the inequality above is minimized at $c = \sqrt{2 \alpha k \log n / n}$. Plugging this value into the inequality above noting that $\frac{2}{\alpha -1 } \sum_{a \in \mathcal{A}} \Delta_a$ is an constant, we have 
\begin{align*}
    R_n = O(\sqrt{kn \log n}).
\end{align*}
\end{proof}


\section{Bayesian Optimization}
\label{chap2_bo}
\subsection{Introduction}
Bayesian optimization \citep{DBLP:journals/corr/abs-1012-2599} is a simple instance of RL that is concerned with partial feedback but not with long term consequences. Bayesian optimization is a setting for learning to optimize in a continuous action space. The goal in Bayesian optimization is to optimize a noise-corrupted, gradient-absent, expensive-to-evaluate objective function within a constrained budget of function evaluations. Formally, Bayesian optimization is concerned with 
\begin{align*}
    \max_{x \in \mathcal{X}} f(x),
\end{align*}
where the action space $\mathcal{X}$ is usually a subspace of $\mathbb{R}^d$, and $f$ is a black-box, unknown function which can be queried at a high cost. In addition, we often assume the function evaluation is noisy, i.e., upon evaluating $f$ at $x$, we observe an noise-corrupted output $y = f(x) + \epsilon$ where $\epsilon$ is a zero-mean Gaussian noise. A typical example of Bayesian optimization is to sequentially activate sensors using as few sensors as possible to locate which part of a building has the highest temperature. Another typical example is experimental design in which one wants to run experiments of multiple parameters to reasonably approximate a design goal in as few numbers of experiment runs as possible. Another example is tuning the hyper-parameters of a big model (e.g. deep neural networks) for model selection. The objective function in this case is the generalization error of the models and their hyper-parameter settings on a held-out set after training the model. Note that even when the objective function has gradients, traditional optimization methods such as gradient descent are not feasible in this setting. The reason is that gradient descents demand for a high number of observation samples (function evaluations in this setting) while we only have a limited budget of evaluations in the context of Bayesian optimization. By casting function evaluations on a $n$-dimensional input space as running experiments on a specific setting of $n$ experiment parameters, \cite{DBLP:conf/nips/AzimiFF10} provide a context as to why we need Bayesian optimization instead of gradient descent. 

The main idea of Bayesian optimization is straightforward. Bayesian optimization casts the global optimization as sequential decision making with a probabilistic model: (i) we construct a probabilistic model as a surrogate model of the objective function; (ii) we update the probabilistic model as we collect more samples (in the context of Bayesian optimization, samples refer to the inputs at which we evaluate the function and the function values at these inputs). A common probabilistic modeling choice in Bayesian optimization is Gaussian Process (GP) \citep{Rasmussen:2005:GPM:1162254} due to the convenient analytical form of the posterior. To decide which point to evaluate next, Bayesian optimization maximizes a so-called \textit{acquisition function} built upon the previous observations and the probabilistic model, e.g., GP. The idea is that it sequentially selects a new point such that it both reduces the uncertainty about the unknown function and aims at moving toward the global maximum of the function. The generic Bayesian optimization procedure is presented in Algorithm \ref{alg:generic_bo}. At each time step $t$, Bayesian optimization selects next point $x_t$ by maximizing an acquisition function $u(x| \mathcal{D}_{t-1})$ conditioned on the past data $\mathcal{D}_{t-1}$. A noisy function output $y_t$ evaluated at $x_t$ is obtained and the new observation $(x_t,y_t)$ is augmented to the past data $\mathcal{D}_t := \mathcal{D}_{t-1} \cup \{(x_t, y_t)\}$. The procedure repeats in the next step $t+1$.

\begin{algorithm}[t]
  \caption{Generic Bayesian optimization}
\label{alg:generic_bo}
\begin{algorithmic}[1]
  \FOR{$t=1$ {\bfseries to} $T$}
  \STATE Select $x_t$ by maximizing an acquisition function: $$x_t \in \operatorname*{arg\,max}_{x \in \mathcal{X}} u(x|\mathcal{D}_{t-1})$$  
  \STATE Evaluate the function $y_t = f(x_t) + \epsilon_t$
  \STATE Augment the data $\mathcal{D}_t := \mathcal{D}_{t-1} \cup \{(x_t, y_t)\}$ and update the GP. 
  \ENDFOR
\end{algorithmic}
\end{algorithm}



\subsection{Performance Metrics}
Similar to multi-armed bandits discussed in Section \ref{chap2_section_bandit}, we use various notions of \textit{regret} for performance metric in Bayesian optimization. Consider any policy $\pi$ which induces a sequence of actions $\{x_t\}_{t=1}^T$ over $T$ steps. The frequentist regret of policy $\pi$ over $T$ steps for target function $f$ is defined as 
\begin{align*}
    \text{Regret}(T, \pi, f) = T \cdot \max_{x \in \mathcal{X}} f(x) - \sum_{t=1}^T f(x_t)
\end{align*}
A desirable property of a policy is to be \textit{no-regret}, i.e., $\pi$ incurs \textit{sublinear} regret: 
\begin{align*}
    \lim_{T \rightarrow \infty} \frac{\text{Regret}(T,\pi,f)}{T} = 0. 
\end{align*}

Intuitively, a no-regret policy $\pi$ is guaranteed to find the optimum after a sufficiently large number of iterations $T$. In addition, we can use Bayesian regret in Bayesian optimization: 
\begin{align*}
    \text{BayesRegret}(T,\pi) = \mathbb{E} \left[ \text{Regret}(T, \pi, f) \right],
\end{align*}
where the expectation is taken with respect to the prior distribution over the target function $f$. Compared to the frequentist regret, Bayesian regret can allow more elegant analysis and is particularly helpful for analyzing posterior sampling \citep{DBLP:journals/mor/RussoR14}. Specifically, in Chapter \ref{chap:three} we use Bayesian regret to analyze our algorithm that is based on posterior sampling. 

\subsection{Gaussian Processes} 
\label{subsect:gp}
A Gaussian Process (GP) is an infinite-dimension stochastic process where any finite combination of dimensions is a Gaussian distribution. A GP is a distribution over functions and is fully specified by a mean function $\mu$ and a covariance function $k$: 
\begin{align*}
    f(x) \sim \text{GP}(\mu, k). 
\end{align*}
Equivalently, a function  is a sample from $\text{GP}(\mu, k)$ if  
\begin{align*}
    \mathbb{E}[f(x)] &= \mu(x), \forall x \in \mathcal{X} \\ 
    C[f(x), f(x')] &:= \mathbb{E}\left[ (f(x) - \mu(x)) (f(x') - \mu(x'))\right] = k(x,x'), \forall x,x' \in \mathcal{X}.
\end{align*}
A popular choice for the covariance function $k$ is the squared exponential kernel 
\begin{align*}
    k(x,x') = \exp \left(-\frac{1}{2} \|x - x' \|^2 \right).
\end{align*}
Via the squared exponential kernel, we can see that $k(x,x')$ becomes larger and approaches the maximum value of $1$ when $x$ and $x'$ are closer in the $2$-norm. When $x$ and $x'$ become more distant, their kernel value is exponentially small and is approaching $0$. This implements the smoothness assumption that the more distant the two points, the less their function values influence each other. Different kernels impose different smoothness into functions sampled from a corresponding GP, e.g., see Figure \ref{fig:kernel_map}. We refer the readers to \citep[Chapter~4]{Rasmussen:2005:GPM:1162254} for an extensive discussion on various GP kernels and their properties. 

\begin{figure}
    \centering
    \includegraphics[scale=0.75]{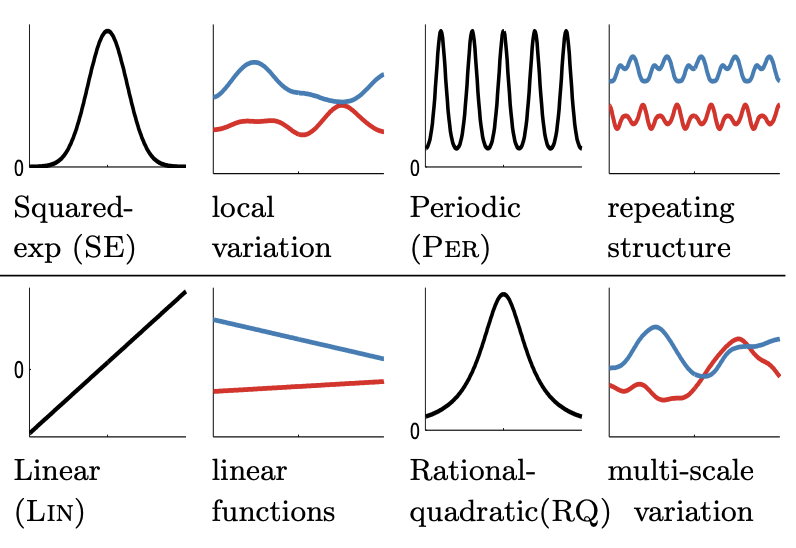}
    \caption{The first and third columns: Base kernels $k(\cdot, 0)$. Second and fourth columns: Function samples from a GP with each respective kernels \citep{duvenaud2013structure}.}
    \label{fig:kernel_map}
\end{figure}

An important issue regarding GPs is how to compute the posterior distribution given the past observations. In particular, assume that $f \sim GP(\mu,k)$ and given the past observations $\mathcal{D}_t = \{ (x_{\tau}, y_{\tau}) \}_{\tau = 1}^t$ where 
\begin{align*}
    y_{\tau} = f(x_{\tau}) + \epsilon_{\tau} \text{ and } \epsilon_{\tau} \sim \mathcal{N}(0, \sigma^2).
\end{align*}
We are interested in computing the posterior distribution of $f$ given the data $\mathcal{D}_t$. 
For any $x \in \mathcal{X}$, we denote 
\begin{align*}
    k(x,x_{1:t}) &:= [k(x,x_1) \text{ } \ldots \text{ }  k(x, x_t)]^T \in \mathbb{R}^t \\ 
    y_{1:t} &:= [y_1 \text{ } \ldots \text{ } y_t]^T \in \mathbb{R}^t \\ 
     k(x_{1:t}, x_{1:t}) &:= 
    \begin{bmatrix}
    k(x_1, x_1) & \ldots & k(x_1, x_t) \\ 
    \vdots & \ddots & \vdots \\ 
    k(x_t, x_1) & \ldots & k(x_t,x_t)
    \end{bmatrix} \in \mathbb{R}^{t \times t}.
\end{align*}
It is a celebrated result \citep{Rasmussen:2005:GPM:1162254} that the posterior distribution of $f|\mathcal{D}_t$ is another GP, namely $ f|\mathcal{D}_t \sim \text{GP}(\mu_t, k_t)$ where 
\begin{align}
    \mu_t(x) &:= \mathbb{E}\left[ f(x) | \mathcal{D}_t \right] = k(x, x_{1:t})^T (k(x_{1:t}, x_{1:t}) + \sigma^2 I_t)^{-1} (y_{1:t} - \mu(x) 1_t)\\ 
    \label{eq:chap2_bo_posterior_mean}
   k_t(x,x') &:= C \left[ f(x) \cdot f(x') | \mathcal{D}_t \right] \nonumber \\
   &= k(x,x') - k(x, x_{1:t})^T (k(x_{1:t}, x_{1:t}) + \sigma^2 I_t)^{-1} k(x', x_{1:t}).
\end{align}
Here, $I_t$ is a $t\times t$ identity matrix and $1_t = [1 \text{ } \ldots \text{ } 1]^T \in \mathbb{R}^t$. 
In addition, we denote by $\sigma^{2}_t(x)$ the posterior variance of $f(x)$ given $\mathcal{D}_t$, i.e., 
\begin{align}
    \sigma_t^2(x) &:= \mathbb{V}\left[ f(x) | \mathcal{D}_t \right] = k_t(x,x) \nonumber \\ 
    &=k(x,x) - k(x, x_{1:t})^T (k(x_{1:t}, x_{1:t}) + \sigma^2 I_t)^{-1} k(x, x_{1:t}). 
    \label{eq:chap2_bo_posterior_var}
\end{align}
The tractable posterior distribution above is perhaps the most attractive characteristic that makes GPs extensively common as a probabilistic model in Bayesian optimization. 

\subsection{Acquisition Functions}
Inherited from the general sequential decision making, one of the key challenges in Bayesian optimization is to balance exploration versus exploitation: either gathering more new data for better estimating the mean payoff function (the underlying objective function) or choosing a greedy optimal action based on the currently gathered data. Bayesian optimization maintains such balance via maximizing an acquisition function. In particular, given the past observations $\mathcal{D}_{t-1} = \{(x_{\tau}, y_{\tau})\}_{\tau=1}^{t-1}$, the next point $x_t$ is selected via 
\begin{align*}
    x_t \in \operatorname*{arg\,max}_{x \in \mathcal{X}} u(x|\mathcal{D}_{t-1}),
\end{align*}
where $u(x|\mathcal{D}_{t-1})$ is an acquisition function based on $\mathcal{D}_{t-1}$. 

There are many choices for acquisition functions in the literature. A key idea to design an acquisition function is to ensure exploitation-exploration trade-off encoded into the acquisition function. Here we discuss some of the most common acquisition functions used in Bayesian optimization. We refer the readers to \citep{DBLP:journals/corr/abs-1012-2599} for the other acquisition functions.  

\subsubsection{Expected Improvement}
Expected improvement (EI) \citep{mockus1978application} selects the next point as the one which maximizes the expected improvement with some improvement margin $\xi \geq 0$: 
\begin{align*}
    u_{EI}(x|\mathcal{D}_{t-1}) = \mathbb{E} \left[\max \{0, \mu_t(x) - \max_{1 \leq \tau \leq t-1}{y_{\tau}} - \xi \} | \mathcal{D}_{t-1} \right], 
\end{align*}
where $\xi \geq 0$ is a constant for the improvement margin, $\mu_t(\cdot)$ is the posterior mean given in Equation (\ref{eq:chap2_bo_posterior_mean}) conditioned on the data $\mathcal{D}_t = \mathcal{D}_{t-1} \cup \{(x, y)\}$ for any $y \in \mathbb{R}$ and the expectation $\mathbb{E}[\cdot]$ above is taken over the randomness of $y \sim \mathcal{N}(\mu_{t-1}(x), \sigma_{t-1}^2(x) + \sigma^2)$. The EI acquisition function above admits an analytical formula: 
\begin{align*}
    &u_{EI}(x|\mathcal{D}_{t-1}) \\
    &= \begin{cases}
    (\mu_{t-1}(x) - \displaystyle \max_{1 \leq \tau \leq t-1}{y_{\tau}} - \xi) \Phi(Z_{t-1}(x)) + \sigma_{t-1}(x) \phi(Z_{t-1}(x)) & \text{if } \sigma_{t-1}(x) > 0 \\
    0 & \text{if } \sigma_{t-1}(x) = 0.
    \end{cases}
\end{align*}
where $\mu_{t-1}$ and $\sigma_{t-1}$ are the posterior mean and variance function given in Equations (\ref{eq:chap2_bo_posterior_mean}) and (\ref{eq:chap2_bo_posterior_var}), $Z_{t-1}(x) := (\mu_{t-1}(x) - \displaystyle \max_{1 \leq \tau \leq t-1}{y_{\tau}} - \xi) \cdot \sigma_{t-1}^{-1}(x)$, and $\Phi(\cdot)$ and $\phi(\cdot)$ denote the CDF and pdf of the standard normal distribution, respectively. 


\subsubsection{Gaussian Process Upper Confidence Bounds}
One of the most popular acquisition functions in Bayesian optimization is based on the principle of optimism in the face as uncertainty as discussed in Subsection \ref{chap2_bandit_OFU} in this chapter. Following this principle, \citep{DBLP:conf/icml/SrinivasKKS10} propose the Gaussian Process Upper Confidence Bound (GP-UCB) acquisition function:  
\begin{align*}
    u_{GP-UCB}(x|\mathcal{D}_{t-1}) = \mu_{t-1}(x) + \sqrt{\beta_t} \sigma_{t-1}(x), 
\end{align*}
where $\beta_t$ is a time-dependent hyperparameter controlling the level of exploration. GP-UCB bears an appealing intuition. Intuitively, The first term $\mu_{t-1}(x)$ encourages exploitation by favoring the points with high posterior means while the second term $\sqrt{\beta_t} \sigma_{t-1}(x)$ encourages exploration by favoring points with high uncertainty (i.e., high posterior variance). The time-dependent hyperparameter $\beta_t$ is chosen such that the true function $f$ is in $[\mu_{t-1}(\cdot) - \sqrt{\beta_{t-1}} \sigma_{t-1}(\cdot), \mu_{t-1}(\cdot) + \sqrt{\beta_t} \sigma_{t-1}(\cdot)]$ for all $t$ with high probability. In particular,  \cite{DBLP:conf/icml/SrinivasKKS10} prove that for $\mathcal{X} \subseteq [0,1]^d$ and for any kernel $k$ such that 
\begin{align*}
    \mathbb{P} \left( \sup_{x \in \mathcal{X}} \bigg | \frac{\partial f}{\partial x_i} \bigg | > L \right) \leq a \exp(- (L/b)^2), \forall i = 1,..., d \text{ where } f \sim \text{GP}(0, k), 
\end{align*} 
for some constants $a,b,L > 0$, GP-UCB incurs sublinear regret with probability at least $1-\delta$ if we choose \begin{align*}
    \beta_t = 2 \log(t^2\pi^2 / (3 \delta)) + 2d \log(t^2 db \log(4 da /\delta)).
\end{align*}

\begin{rem}
As GP-UCB has a well-established theoretical result in Bayesian optimization, it is often effectively adopted to address other problems of Bayesian optimization in different settings. For example, \citet{NIPS2019_9350} modify GP-UCB to address the problem of Bayesian optimization in unknown search spaces. 
\end{rem}

\subsubsection{Posterior Sampling}
The posterior sampling (a.k.a. Thompson sampling) simply samples an action according to the posterior distribution it is optimal. In particular, the acquisition function for posterior sampling is defined as 
\begin{align*}
    u_{PS}(x|\mathcal{D}_{t-1}) = \hat{f}(x) \text{ where } \hat{f} \sim f | \mathcal{D}_{t-1}.
\end{align*}
In practice, we can approximately obtain a posterior function sample $\hat{f}$ using Bayesian linear models with random features \citep{DBLP:conf/nips/Hernandez-LobatoHG14} as shown in Algorithm \ref{alg:ps_gp}.
\begin{algorithm}[t]
  \caption{Posterior sampling for Gaussian Processes}
\label{alg:ps_gp}
\begin{algorithmic}[1]
  \STATE {\bfseries Input:} Shift-invariant kernel $k(x,y)$ (e.g., the squared exponential kernel), data $\mathcal{D}_{t-1} = \{(x_{\tau}, y_{\tau})\}_{\tau=1}^{t-1}$, the number of Monte Carlo samples $m \in \mathbb{N}$. 
\STATE Compute the \textit{spectral density} $p(w) = s(w) /\alpha$ of a shift-invariant kernel $k(x,y)$ where 
\begin{align*}
    s(w) &= \frac{1}{(2\pi)^d} \int e^{i w^T \tau} k(\tau, 0) d \tau \\ 
    \alpha &= \int s(w) dw. 
\end{align*}

\STATE Compute a random $m$-dimensional feature function 
\begin{align*}
    \Phi(x) = \sqrt{\frac{2 \alpha}{m}} cos(Wx + b) 
\end{align*}
where the row vectors in  $W \in \mathbb{R}^{m \times d}$ are i.i.d. samples of $p(w)$ and the elements of $b \in \mathbb{R}^m$ are i.i.d. samples of $U(0, 2 \pi)$.

\STATE Sample $\theta \sim \mathcal{N}(m,V)$ where 
\begin{align*}
    m &= (\Phi^T \Phi + \sigma^2 I)^{-1} \Phi^T y \\
    V &= (\Phi^T \Phi + \sigma^2 I)^{-1} \sigma^2 \\ 
    \Phi &= [\phi(x_1)^T, ... ,\phi(x_n)^T]^T \in \mathbb{R}^{n \times d } \\ 
    y &= [y_1, ..., y_n]^T \in \mathbb{R}^n. 
\end{align*}

\STATE {\bfseries Output:} Obtain a posterior function sample $\hat{f}(x) = \Phi(x)^T \theta \sim f | \mathcal{D}_{t-1}$.
\end{algorithmic}
\end{algorithm}

Posterior sampling offers several advantages over upper confidence bound methods such as GP-UCB. Computing an upper confidence bound is often more complicated than posterior sampling as the former requires either an analytic form of the posterior distribution or an extensive Monte Carlo simulation for each action. Posterior sampling does not require computing the posterior distribution but only sampling from it. In addition, designing a proper form of upper confidence bound often requires a sophisticated tool such as self-normalized martingale processes \citep{DBLP:conf/nips/Abbasi-YadkoriPS11}. For more detailed discussion of posterior sampling, we refer the readers to \citep{DBLP:journals/mor/RussoR14}. 


\subsection{Bayesian Quadrature Optimization}
\label{chap2_bqo}

The standard Bayesian optimization uses a canonical form of the unknown black-box objective function. However, in many machine learning problems, the objective function often takes a quadrature form which is an expected value with respect to an environmental variable $w \sim P$ of an expensive, black-box integrand:
\begin{align*}
    \max_{x \in \mathcal{X}} g(x) \text{ where } g(x) := \mathbb{E}_{w \sim P} \left[ f(x,w)  \right]. 
\end{align*}
This problem is known as Bayesian quadrature optimization (BQO) \citep{Toscano_IntegralBO_18}. In BQO, the environmental distribution $P(w)$ is known, either in an analytical form of $P(w)$ or in a sampling procedure $w \sim P(w)$. Even with that prior knowledge about $P(w)$, as $f$ is expensive, given a fixed input $x$, we do not expect to evaluate $f(x,w)$ for all $w$ in the domain of $w$. Thus, it is more desirable to solve the BQO by evaluating $f$ at a small set of selective $w \sim P$ (e.g., one sample of $w$) for any given $x$. 

The prior knowledge about $P(w)$ and the linearity of the expectation operator $\mathbb{E}_{w \sim P}$ allows a seamless adaptation of Gaussian process (GP) modeling from standard BO to BQO \citep{o1991bayes}. In particular, assume that $f(x,w) \sim GP(\mu, k)$. As $g$ is a linear combination of $f$, $g$ also follows a GP: $g(x) \sim GP(\mu^g, k^g)$ where 
\begin{align*}
    \mu^g(x) &= \int \mu(x,w) P(w) dw, \\
     k^g(x,x') &= \int \int k(x,w; x', w') P(w) P(w') dw dw'. 
\end{align*}
Now let $\mathcal{D}_t = \{(x_i, w_i, y_i)\}_{i=1}^t$ be the aggregated data collected after time $t$ where $y_i = f(x_i, w_i) + \epsilon_i$ with $\epsilon_i \sim \mathcal{N}(0, \sigma^2)$, the posterior distribution of $g$, i.e., $g | \mathcal{D}_t$, is also a $GP(\mu^g_t, k^g_t)$, which can be represented using the posterior distribution of $f | \mathcal{D}_t \sim GP(\mu_t, k_t)$
\begin{align*}
    \mu^g_t(x) &= \int \mu_t(x,w) P(w) dw, \\
     k^g_t(x,x') &= \int \int k_t(x,w; x', w') P(w) P(w') dw dw'.
\end{align*}

Given the above formulation, \citet{Toscano_IntegralBO_18} derive a variant of knowledge gradient \citep{frazier2009knowledge} to address the BQO problem. 

\section{Markov Decision Processes}
\subsection{Introduction}

Markov decision processes (MDPs) formally describe an environment for (full-fledged) RL. Almost all RL problems (e.g., optimal control and partially observable problems) can be formulated as MDPs. The bandit problem discussed in Section \ref{chap2_section_bandit} can be considered as MDPs with one state. 

Formally, a MDP is a tuple $(\mathcal{S}, \mathcal{A}, P, R, \gamma)$ where $\mathcal{S}$ is the state space, $\mathcal{A}$ is an action space, $\rho \in \mathcal{P}(\mathcal{S})$ is an initial state distribution, $P: \mathcal{S} \times \mathcal{A} \rightarrow \mathcal{P}(\mathcal{S})$ is a transition operator, $R: \mathcal{S} \times \mathcal{A} \rightarrow \mathcal{P}([0,1])$ is a reward function, and $\gamma \in [0,1]$ is a discount factor. Here $\mathcal{P}(\mathcal{S})$ denotes the space of probability measures supported on $\mathcal{S}$. Without loss of generality, we also assume that the reward distribution $R$ is supported on $[0,1]$.

A policy $\pi: \mathcal{S} \rightarrow \mathcal{P}(\mathcal{A})$ induces a distribution over the action space conditioned on states. A policy $\pi$ that sequentially interacts with the MDP over $T$ steps induces a random sequence of $(s_0, a_0, r_0, ..., s_T, a_T, r_T, s_{T+1})$, as illustrated in Figure \ref{fig:chap2_mdp}. An \textit{absorbing} state is a state in which the MDP terminates. Think of this as the final states of a video game at the point where it is clear that the game is over (i.e., whether you win, lose or draw the game). A terminal state remarks the end of one \textit{episode} at which any further interaction with the MDP restarts the MDP by proceeding with the next episode. Some MDPs do not have absorbing states and we can impose the time horizon $H$, i.e. the maximum number of steps within an episode. In either cases, for unifying notations, we assume that $R(s,a) = \delta_0$ (the Dirac distribution centered at $0$) if $s$ is an absorbing state or $s = s_{H+1}$ in the $H$-horizon setting.  

\begin{figure}
    \centering
    \includegraphics{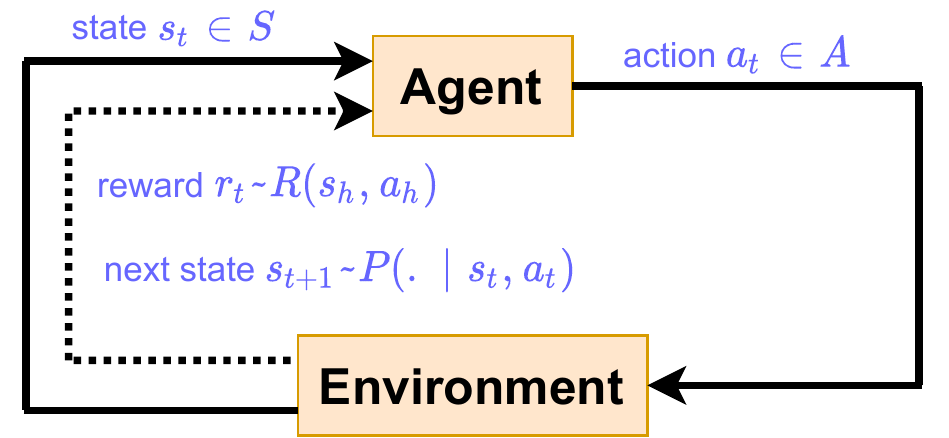}
    \caption{An illustration of an interaction with a Markov decision process.}
    \label{fig:chap2_mdp}
\end{figure}

The $Q$-value function for policy $\pi$ at state-action pair $(s,a)$, denoted by $Q^{\pi}(s,a) \in [0,1]$, is the expected discounted total reward the policy collects if it initially starts in the state-action pair, 
\begin{align*}
    Q^{\pi}(s,a) &:= \mathbb{E}_{\pi} \left[ \sum_{t=0}^{\infty} \gamma^t r_t | s_0 = s, a_0 = a \right], 
\end{align*}
where $r_t \sim R(s_t, a_t), a_t \sim \pi(\cdot|s_t)$, and $s_t \sim P(\cdot|s_{t-1}, a_{t-1})$. The value for a policy $\pi$ is simply $V^{\pi} = \mathbb{E}_{s \sim \rho, a \sim \pi(\cdot|s)} \left[Q^{\pi}(s,a) \right]$, and the optimal value is $V^* = \max_{\pi} V^{\pi}$ where the maximization is taken over all stationary policies. Equivalently, the optimal value $V^*$ can be obtained via the optimal $Q$-function $Q^* = \max_{\pi} Q^{\pi}$ via $V^* = \mathbb{E}_{s \sim \rho, a \sim \pi(\cdot|s)} \left[Q^*(s,a) \right]$. Moreover, the optimal policy is obtained via $\pi^* = \operatorname*{arg\, max}_{\pi} Q^{\pi}$. Denote by $T^{\pi}$ and $T^*$ the Bellman operator and the optimality Bellman operator, i.e., for any $f: \mathcal{S}\times \mathcal{A} \rightarrow \mathbb{R}$
\begin{align}
    [T^{\pi} f](s,a) &= \mathbb{E}_{r \sim R(s,a)}[r] + \gamma \mathbb{E}_{s' \sim P(\cdot|s,a), a' \sim \pi(\cdot|s')} \left[ f(s',a') \right] \\
    \label{chap2_mdp_bellman_eq}
    [T^* f](s,a) &= \mathbb{E}_{r \sim R(s,a)}[r] + \gamma \mathbb{E}_{s' \sim P(\cdot|s,a)} \left[ \max_{a'} f(s',a') \right], 
\end{align}
we have 
\begin{align}
    T^{\pi} Q^{\pi} = Q^{\pi} \text{ and } T^* Q^* = Q^*. 
    \label{chap2_eq_fixed_point}
\end{align}

Another intriguing property of these operators is their contraction in the infinity norm, i.e., 
\begin{align}
    \|T^{\pi} f - T^{\pi} g \| \leq \gamma \| f - g \|_{\infty} \\ 
    \|T^* f - T^* g \| \leq \gamma \| f - g \|_{\infty}, 
\end{align}
for all functions $f,g: \mathcal{S} \times \mathcal{A} \rightarrow \mathbb{R}$. The contraction of $T^{\pi}$ and $T^*$ indicates that these operators have the unique fixed points given in Equation (\ref{chap2_eq_fixed_point}), following from Banach's fixed point theorem \citep{BanachSURLO}. In addition, starting at arbitrary function $f$, $(T^{\pi})^n f$ and $(T^*)^n f$ converges at the fast linear rate to their respective fixed points. 

\subsection{Performance Metrics}
The goal of RL is how to efficiently estimate the optimal policy $\pi^*(s)$. A typical RL algorithm often iteratively produces a sequence of estimates $(V_k)_{k=1}^K$ for the optimal value $V^*$ where $K$ is the number of rounds of interactions. We measure the performance of the algorithm by the total regret at an initial state, 
\begin{align*}
    \text{Regret}(K, s_0) = \sum_{k=1}^K (V^*(s_0) - V_k(s_0)). 
\end{align*}
Similar to bandits and Bayesian optimization, a RL algorithm is said to successfully learn the optimal policy if it incurs a sublinear regret, i.e., 
\begin{align*}
    \lim_{K \rightarrow \infty} \frac{\text{Regret}(K, s_0)}{K} = 0, \forall s_0. 
\end{align*}

\subsection{Reinforcement Learning Approaches}
If the transition operator $T$ and the reward function $R$ are known to the agent, this problem can be efficiently solved by dynamic programming \citep{10.5555/3312046}. The challenge that makes this goal a learning task is that the agent knows neither the transition operator nor the reward function. Instead, the agent only observes transition samples from $P$ and reward samples from $R$ based on its active decision-making process. Though dynamic programming is no longer feasible for the RL setting, its core idea of Bellman backup and the fixed point equations in Equations (\ref{chap2_eq_fixed_point}) lay a foundational step for many RL algorithms. Typically, there are three main approaches to RL: value-based methods, policy gradient methods and model-based methods. While value-based methods learn the optimal $Q$-function and extract the optimal policy thereafter, policy gradient methods directly optimize the value function over the space of policies. Model-based methods, on the other hand, aim at estimating the transition dynamic $P$ and reward function $R$, and apply dynamic programming on the estimated models. Since value-based methods are effective in practice with good theoretical guarantees that also provide foundations to the other approaches, we focus on value-based methods in this thesis. We refer the readers to \citep{10.5555/3312046} for a review on policy gradients and model-based methods. 

\subsubsection{(Deep) Q-learning} 
One of the most common value-based methods and the earliest breakthroughs in RL is Q-learning \citep{Watkins92q-learning}, also known as off-policy temporal-difference learning, with the pseudo-code presented in Algorithm \ref{alg:q_learning}, where we consider a tabular MDP where $|\mathcal{S}| < \infty$ and $|\mathcal{A}| < \infty$. 


\begin{algorithm}[t]
  \caption{Q-learning}
\label{alg:q_learning}
\begin{algorithmic}[1]
  \STATE {\bfseries Input:} Learning rate $\{\alpha_t(s,a) > 0: \forall (s,a) \in \mathcal{S} \times \mathcal{A}, \forall t\}$. 
  \STATE Initialize $Q(s,a)$ arbitrarily for all $(s,a) \in \mathcal{S} \times \mathcal{A}$ 
  \STATE Initialize $s_0$
  \FOR{$t=0,1,2 ...$}
  \STATE Choose $a_t$ from $s_t$ using policy derived from $Q$ (e.g., $\epsilon$-greedy) and observe $r_t \sim R(s_t,a_t)$ and $s_{t+1} \sim P(\cdot|s_t,a_t)$ 
  \STATE Update $Q(s_t,a_t) \leftarrow Q(s_t,a_t) + \alpha_t(s,a) (r_t + \gamma \max_{a'} Q(s_{t+1},a') - Q(s_t,a_t))$
    \label{chap2_q-learning_SA}
  \ENDFOR 
  \STATE {\bfseries Output:} $Q(s,a)$ 
\end{algorithmic}
\end{algorithm}

The most important step is the update of $Q$ function at line \ref{chap2_q-learning_SA} which is a stochastic approximation to the optimality Bellman operator with learning rates $\{\alpha_t(s,a) > 0: \forall (s,a) \in \mathcal{S} \times \mathcal{A}, \forall t\}$. This update style is called temporal difference (TD) which allows to learn directly from experience without a model of the environment's dynamics as in dynamic programming or waiting until the end of an episode for an update as in Monte Carlo methods. This is the main advantage of $Q$-learning for solving RL problems. 

If the learning rates satisfy the Robbin-Monro conditions, i.e., 
\begin{align*}
    \sum_{t=0}^{\infty} \alpha_t(s,a) = \infty \text{ and } \sum_{t=0}^{\infty} \alpha_t^2(s,a) < \infty,
\end{align*}
for any $(s,a) \in \mathcal{S} \times \mathcal{A}$, then $Q$ converges to $Q^*$ almost surely. 

In the scenario where the state space is sufficiently large as in Atari games, Q-learning obtains scalability when using with deep neural network as function approximation. As such, the most predominant method is perhaps deep Q-network (DQN) \citep{mnih2015human}. DQN is a breakthrough in RL that achieves a human-level control on a wide ranges of Atari games using only raw pixels and scores as inputs. This empirical success comes from several algorithmic advances including the use of memory buffer and target networks to stabilize the learning. In Chapter \ref{chap:four}, we build an architecture upon DQN to achieve a so-called Moment Matching DQN (MMDQN) which learns the intrinsic randomness of the environment as an auxiliary task and achieves a new state-of-the-art empirical performance in the Atari game suite.

\subsubsection{Optimistic Least-Squares Value Iteration}
Even though the basic Q-learning presented in Algorithm \ref{alg:q_learning} obtains an asymptotic convergence under mild conditions, the asymptotic regime requires the number of samples $n$ to approach $\infty$. Thus, it is unclear whether the algorithm is (sample)-efficient, i.e., how many samples are required to obtain an estimate error within a user-specified precision? A finite-sample analysis is often more helpful to answer such question. In the following, we summarize a celebrated result of such finite-sample analysis for Q-learning using the optimism in the face of uncertainty (OFU) principle. 

The OFU principle has been shown to be effective to obtain sample efficiency in bandits and Bayesain optimization presented in Section \ref{chap2_section_bandit} and \ref{chap2_bo}, respectively. In tabular MDPs, the OFU principle also works effectively. We present a representative algorithm for such cases, namely tabular optimistic least-squares value iteration (LSVI) \citep{jin2018qlearning}. For this result, we consider a tabular episodic time-inhomogeneous MDP$(\mathcal{S}, \mathcal{A}, P, r, H)$ where $H$ is the number of steps in each episode, $P = \{P_h\}_{h=1}^{H}$ is a sequence of transition kernels $P_h: \mathcal{S} \times \mathcal{A} \rightarrow \mathcal{P}(\mathcal{S})$, and $r = \{r_h\}_{h=1}^H$ is a sequence of deterministic reward functions $r_h: \mathcal{S} \times \mathcal{A} \rightarrow [0,1]$. The pseudo-code for tabular optimistic LSVI is presented in Algorithm \ref{alg:tabular_lsvi}. Here $c$ is a constant, $\delta \in (0,1)$ and $\alpha_t = (H+1)/(t+1)$. 
\begin{algorithm}[t]
  \caption{Tabular optimistic least-squares value iteration \citep{jin2018qlearning}}
\label{alg:tabular_lsvi}
\begin{algorithmic}[1]
  \STATE Initialize $Q_h(s,a) \leftarrow H$, and $N_h(s,a) \leftarrow 0$, $\forall (s,a,h) \in \mathcal{S} \times \mathcal{A} \times [H]$
  \FOR{episode $k=1,2, ..., K$}
  \STATE Receive initial state $s_1$ 
  \FOR{step $h=1,2,...,H$}
  \STATE Take action $a_h \leftarrow \operatorname*{arg\,max}_{a'} Q_h(s,a')$ and observe $s_{h+1}$ and $r_h(s_h, a_h)$ 
  \STATE $t = N_h(s,a) \leftarrow N_h(s,a) + 1$ and  $b_t \leftarrow c \sqrt{H^3 \log(SAKH/\delta) / t}$
  \STATE $Q_h(s_h, a_h) \leftarrow (1-\alpha_t) Q_h(s_h, a_h) + \alpha_t \left[r_h(s_t,a_t) + V_{h+1}(s_{h+1}) + b_t\right]$ 
  \label{tabular_lsvi:optimistic}
  \STATE $V_h(s_h) \leftarrow \min\{H, \max_{a}Q_h(s_h,a)\}$
  \ENDFOR
  \ENDFOR 
\end{algorithmic}
\end{algorithm}
The key difference between the tabular optimistic LSVI and the basic Q-learning in Algorithm \ref{alg:q_learning} is the update step at line \ref{tabular_lsvi:optimistic} of Algorithm \ref{alg:tabular_lsvi} where a ``bonus'' function $b_t$ is augmented. With a carefully designed bonus function such as a specific one in Algorithm \ref{alg:tabular_lsvi}, for any $h \in [H] = \{1,2,...,H\}$, $Q_h$ is an optimistic estimate of the the optimal $Q^*_h$ in a sense that $Q_h \geq Q^*_h$ with probability at least $1 - \delta$. \citet{jin2018qlearning} prove that the tabular optimistic LSVI obtains a regret of $\tilde{O}( \sqrt{H^4 SAT})$ where $\tilde{O}$ hides log factors and $T = KH$ is the total number of samples. With a more intricate bonus function using Bernstein's inequality, \citet{jin2018qlearning} improve the regret of tabular optimistic LSVI to  $\tilde{O}(\sqrt{H^3 SAT})$, which nearly matches the lower bound of $\Omega(\sqrt{H^2 SAT})$. 

The OFU principle is also an effective and dominant strategy for obtaining sample efficiency (i.e. sublinear regret) in more complex environments beyond tabular MDPs. For such environments, many states may have not been visited even once during the learning, thus a function approximation is required to generalize from observed states to unseen ones. In particular, sample efficiency is attainable with the OFU principle for linear MDPs \citep{DBLP:conf/colt/JinYWJ20}, generalized linear MDPs \citep{DBLP:journals/corr/abs-1912-04136}, MDPs with finite eluder dimension \citep{DBLP:journals/corr/abs-2005-10804}, and MDPs with finite Bellman-eluder dimension \citep{DBLP:journals/corr/abs-2102-00815}. 

\section{Distributional Reinforcement Learning} 
All the value-based methods for RL we have discussed so far share the same characteristic that they learn to optimize the {expected} value of the total (discounted) reward,  $Q^{\pi}(s,a) = \mathbb{E}_{\pi}[Z^{\pi}(s,a)]$ where 
\begin{align}
    Z^{\pi}(s,a) = 
    \begin{cases}
         \sum_{t=0}^{\infty} \gamma^t r_t(s_t, a_t)  | (s_0, a_0) = (s,a)  & \text{ for infinite-horizon MDPs}, \\ 
         \label{chap2:distributional_return}
         \sum_{h=1}^{H} r_h(s_h,a_h)  | (s_1, a_1) = (s,a)  & \text{ for finite-horizon MDPs}. 
    \end{cases}
\end{align}
Here, $\mathbb{E}_{\pi}$ is the expectation taken over the randomness of the trajectory induced by policy $\pi$ and $Z^{\pi}$ is called the random return. Focusing only on the expected value of the random return discards the other valuable information of $Z^{\pi}$ which could be useful for both learning and decision-making. For example, the other values rather than the expectation such as quantile values are important to design risk-sensitive RL algorithms \citep{DBLP:journals/corr/ShenTSO13}. 

Distributional RL \citep{DBLP:conf/icml/BellemareDM17} is a recent RL paradigm which aims at learning the entire distribution, instead of just the expected value of the random return. For empirical performance, distributional RL has achieved impressive empirical successes in the Atari game benchmark  \citep{DBLP:conf/icml/BellemareDM17,DBLP:conf/aaai/DabneyRBM18,DBLP:conf/icml/DabneyOSM18,yang2019fully,Nguyen-Tang_Gupta_Venkatesh_2021}.

For notational simplicity, we focus on infinite-horizon MDPs. Let $\mu^{\pi} = law(Z^{\pi})$ be the law, i.e., the distribution of the random return $Z^{\pi}$ defined in Equation (\ref{chap2:distributional_return}). Similar to the standard expected RL with the Bellman operator defined in Equation (\ref{chap2_mdp_bellman_eq}), distribution RL has a so-called distributional Bellman operator $\mathcal{T}^{\pi}: \mathcal{P}([0,1])^{\mathcal{S} \times \mathcal{A}} \rightarrow \mathcal{P}([0,1])^{\mathcal{S} \times \mathcal{A}}$, defined as 
\begin{align}
    [\mathcal{T}^{\pi} \mu](s,a) := \int_{\mathcal{S}} \int_{\mathcal{A}} \int_{[0,1]} (f_{\gamma,r})_{\#} \mu(s',a') R(dr|s,a) \pi(d a'|s') P(ds'|s,a),
    \label{chap2:distributional_bellman_operator}
\end{align}
for any $\mu \in \mathcal{P}([0,1])$, where $f_{\gamma,r}(z) := r + \gamma z, \forall z$ and $(f_{\gamma,r})_{\#} \mu(s',a')$ is the pushforward measure of $\mu(s',a')$ by $f_{\gamma,r}$. Note that by definition, $\mu^{\pi}$ is the fixed point of $\mathcal{T}^{\pi}$, i.e., $\mathcal{T}^{\pi} \mu^{\pi} = \mu^{\pi}$. The key difference between the Bellman operator $T^{\pi}$ defined in Equation (\ref{chap2_mdp_bellman_eq}) and the distributional Bellman operator $\mathcal{T}^{\pi}$ defined in Equation (\ref{chap2:distributional_bellman_operator}) is that the latter operates on the space of probability measures instead of the function space, thus is more complicated. 

The contraction of $\mathcal{T}^{\pi}$ is relative to specific distribution discrepancies employed. In particular, $\mathcal{T}^{\pi}$ is a contraction in the $p$-Wasserstein metric \citep{DBLP:conf/icml/BellemareDM17} and Cr\'amer distance \citep{DBLP:conf/aistats/RowlandBDMT18}; but it is not a contraction in total variation distance \citep{Chung1987DiscountedMD}, Kullback-Leibler divergence and Komogorov-Smirnov distance \citep{DBLP:conf/icml/BellemareDM17}.

On the algorithmic side, it is necessary to approximate the return distributions to learn a distributional Bellman operator for distributional RL, as probability measures are infinite-dimensional objects. In what follows, we briefly review predominant distributional RL methods in the literature. 

\subsection{Categorical Distributional Reinforcement Learning} 
The main idea of categorical distributional reinforcement learning (CDRL) \citep{DBLP:conf/icml/BellemareDM17} is to approximate a distribution $\eta$ by a categorical distribution $\hat{\eta} = \sum_{i=1}^N \theta_i \delta_{z_i}$ where $z_1 \leq z_2 \leq ... \leq z_N $ is a set of fixed supports and $\{\theta_i\}_{i=1}^N$ are learnable probabilities. The optimal categorical distribution $\hat{\eta}$ in CDRL are defined via the projection of $\eta$ onto $\Theta := \{ \sum_{i=1}^N p_i \delta_{z_i}: \{p_i\}_{i=1}^N \in \Delta_N \}$ with respect to the Cr\'amer distance \citep{DBLP:conf/aistats/RowlandBDMT18}. In particular, 
\begin{align*}
    \hat{\eta} = \Pi_{C}(\eta) =: \argmin_{\iota \in \Theta} d_C(\eta, \iota), 
\end{align*}
where $d_C$ is the the Cr\'amer distance defined by 
\begin{align*}
    d_C(\nu_1, \nu_2) := \sqrt{\int (F_{\nu_1}(x) - F_{\nu_2}(x))^2 dx},  
\end{align*}
for any two distributions $\nu_1$ and $\nu_2$ with cumulative distribution functions $F_{\nu_1}$ and $F_{\nu_2}$, respectively. \citet{DBLP:conf/aistats/RowlandBDMT18} prove that such Cra\'mer projection above is equivalent to the heuristic projection in the original distributional RL \citep{DBLP:conf/icml/BellemareDM17}.  

The intriguing property of categorical distribution with Cra\'mer projection is that all the distributions are projected such that their images have the same supports. On the same supports, the KL divergence can be employed to minimize the discrepancy between two distributions. This is beneficial in practice because KL divergence is a relatively well-understood and effective loss function used in the context of deep learning. In particular, the main update in CDRL is presented in Algorithm \ref{alg:cdrl}. The main idea for this algorithm is that a Bellman target distribution $[\hat{\mathcal{T}} \hat{\eta}](s_t, a_t)$ is computed and then projected back to ensure that the resulting distribution $\nu$ is supported on $\{z_i\}_{i=1}^N$. Then, the KL divergence between $\hat{\eta}(s_t, a_t)$ and $\nu$ is minimized as an attempt to find the stationary distribution of the distributional Bellman operator. 

\begin{algorithm}
  \caption{Categorical distributional RL}
\label{alg:cdrl}
\begin{algorithmic}[1]
  \STATE {\bfseries Input:} A transition sample $(s_t, a_t, r_t, s_{t+1})$, current return distribution estimate $\hat{\eta}(s, a) = \sum_{i=1}^N \delta_{z_i} \theta_i(s, a)$ where $\{z_i\}_{i=1}^N$ is the fixed supports and $\theta(s,a) = [\theta_1(s,a), ..., \theta_N(s,a)] \in \Delta_N$. 
  \STATE $a^* \leftarrow \operatorname*{arg\,max}_{a'} \sum_{i=1}^n z_i \theta_i(s_{t+1}, a')$
  \STATE Compute an empirical Bellman target distributions: $$[\hat{\mathcal{T}} \hat{\eta}](s_t, a_t) := (f_{\gamma, r_t})_{\#} \hat{\eta}(s_{t+1}, a^*) = \sum_{i=1}^n \delta_{\gamma z_i  + r_t} \theta_i(s_{t+1}, a^*)$$ 
  \STATE Project $\hat{\mathcal{T}} \hat{\eta}$ back onto a distribution supported on $\{z_i\}_{i=1}^n$: 
  $$\nu \leftarrow \Pi_C([\hat{\mathcal{T}} \hat{\eta}](s_t, a_t))$$
  \STATE {\bfseries Output:} Compute the KL divergence loss 
  $$\mathcal{L}(\theta) = KL[\hat{\eta}(s_t, a_t) \| \nu ]$$
\end{algorithmic}
\end{algorithm}

In practice, the probabilities $\{\theta_i(s,a): 1 \leq i \leq N, (s,a) \in \mathcal{S} \times \mathcal{A}\}$ can be effectively represented by neural networks. In particular, with such neural parameterization, \citet{DBLP:conf/icml/BellemareDM17} show a successful variant of CDRL, namely C51 where $N=51$, which improves significantly over DQN in terms of empirical performance in Atari games. 

\subsection{Quantile Distributional Reinforcement Learning} 
A drawback of CDRL is that the use of fixed supports and KL divergence require a Cr\'amer projection step which is highly non-trivial and involved. This can be avoided by using a mixture of Dirac distributions, instead of a categorical distribution, to approximate a return distribution.

Quantile distributional RL or Quantile regression distributional RL (QRDRL) \citep{DBLP:conf/aaai/DabneyRBM18} approximates a distribution $\eta$ by a mixture of Diracs $\hat{\eta} = \frac{1}{N} \sum_{i=1}^N \delta_{\theta_i}$ which is the projection of $\eta$ on $\{  \frac{1}{N} \sum_{i=1}^N \delta_{z_i}: \{z_i\}_{i=1}^N \in \mathbb{R}^N \}$ with respect to the 1-Wasserstein distance. \citet{DBLP:conf/aaai/DabneyRBM18} show that the projection results into $\theta_i = F_{\eta}^{-1}( \frac{2i-1}{2N} )$ where $F_{\eta}^{-1}$ is the inverse cumulative distribution function of $\eta$. Since the quantile values $\{F_{\eta}^{-1}( \frac{2i-1}{2N})\}$ at the fixed quantiles $\{\frac{2i-1}{2N}\}$ are a minimizer of an asymmetric quantile loss from quantile regression literature and the quantile loss is compatible with stochastic gradient descent (SGD), the quantile loss is used for QRDRL in practice. The pseudo-code of QRDRL update is presented in Algorithm \ref{alg:qdrl}. 

\begin{algorithm}[t]
  \caption{Quantile regression distributional RL}
\label{alg:qdrl}
\begin{algorithmic}[1]
  \STATE {\bfseries Input:} A transition sample $(s_t, a_t, r_t, s_{t+1})$, current return distribution estimate $\hat{\eta}(s, a) = \sum_{i=1}^N \frac{1}{N}\delta_{\theta
  _i(s,a)} $. 
  \STATE $a^* \leftarrow \operatorname*{arg\,max}_{a'} \sum_{i=1}^n \frac{1}{N} \theta_i(s_{t+1}, a')$
  \STATE Compute an empirical Bellman target distribution: $$[\hat{\mathcal{T}} \hat{\eta}](s_t, a_t) := (f_{\gamma, r_t})_{\#} \hat{\eta}(s_{t+1}, a^*) = \sum_{i=1}^n \frac{1}{N} \delta_{\gamma \theta_i(s_{t+1}, a^*)  + r_t} $$ 
  
  \STATE {\bfseries Output:} Quantile regression loss 
  $$
  \mathcal{L}(\theta) = \sum_{i=1}^N \sum_{j=1}^N \frac{1}{N} \rho_{\frac{2i-1}{2N}} \left( \gamma \theta_j^{-}(s_{t+1}, a^*)  + r_t - \theta_i(s_t, a_t)\right),
  $$
  where $\theta^{-}$ is a copy of $\theta$, and $\rho_{\tau}(x) = x(\tau - 1\{x < 0\})$.
\end{algorithmic}
\end{algorithm}

In practice, a variant of QDRL with Huber loss, namely QR-DQN-1 \citep{DBLP:conf/aaai/DabneyRBM18}, can be used with deep neural networks and achieves a significant improvement in the Atari games over CDRL. 

\subsection{Implicit Distributional Reinforcement Learning} 
A drawback of QRDRL is that the projection space $\{  \frac{1}{N} \sum_{i=1}^N \delta_{z_i}: \{z_i\}_{i=1}^N \in \mathbb{R}^N \}$ have fixed probabilities $1/N$ for all support points $\{z_i\}_{i=1}^N$. Instead, the probabilities could be made implicit and learnable by deep neural networks. This modeling improvement is the main idea of implicit distributional RL \citep{DBLP:conf/icml/DabneyOSM18,yang2019fully}. In practice, this additional flexibility transfers into an empirical improvement over QRDRL in the Atari games in \citep{DBLP:conf/icml/DabneyOSM18,yang2019fully}.

\section{Offline Reinforcement Learning}
\subsection{Motivation}
The RL methods discussed so far work on the basis that they can interact with the underlying environment to acquire more data. Such interaction is however limited in many practical settings. In safety-related domains such as healthcare, medicine, and autonomous driving, it is very expensive and even unethical to try out a policy in the environment for collecting more data. Instead, the historical data collected a priori is often available, and it is desirable to evaluate a new policy or learn the optimal policy purely from the offline data. Offline RL \citep{levine2020offline} is a practical paradigm that considers such offline setting. 

There are two main tasks in offline RL: off-policy evaluation (OPE) and off-policy learning (a.k.a. offline learning) (OPL) . OPE aims at evaluating the value of a fixed target policy given the offline data generated by different policies. OPL is, on the other hand, learning an optimal policy using the offline data without any further interaction with the environment. The key conceptual differences between OPE, OPL and online learning are summarized and illustrated in Figure \ref{fig:chap2_diag_ope_offline}.

\subsection{Performance Metrics}

For offline RL, we measure the performance by the sub-optimality gaps. There are various ways to define sub-optimality gaps and there is not significant difference in a statistical sense among these ways. Here we present several of such definitions. In any concrete context, we elaborate which sub-optimality metric we use.

\noindent \textbf{OPE}. 
Given a fixed target policy $\pi$, for any value estimate $\hat{V}$ computed from the offline data, the sub-optimality of OPE is defined as
\begin{align*}
    \text{SubOpt}(\hat{V}; s) = |V^{\pi}(s) - \hat{V}(s)|, \forall s \in \mathcal{S}, 
\end{align*}
where $V^{\pi}$ is the value function for $\pi$. 

Another way to define the sub-optimality for OPE is 
\begin{align*}
    \text{SubOpt}(\hat{V}; \pi) = |\mathbb{E}_{s \sim \rho}[V^{\pi}(s)] - \mathbb{E}_{s \sim \rho}[\hat{V}(s)]|,
\end{align*}
where $\rho$ is the initial state distribution. 

\noindent \textbf{OPL}. 
Let $V^*$ be the optimal value function and $Q^*$ be the optimal $Q$-value function, for any policy $\pi$, its sub-optimality is defined as 
\begin{align*}
    \text{SubOpt}(\pi; s) = | V^*(s) - \mathbb{E}_{a \sim \pi}[Q^*(s,a)]|, \forall s \in \mathcal{S}.
\end{align*}

Alternatively, we can define the sup-optimality gap in OPL as 
\begin{align*}
    \text{SubOpt}(\pi) = \mathbb{E}_{s \sim \rho}[V^*(s)] - \mathbb{E}_{a \sim \pi(s), s \sim \rho}[Q^*(s,a)].
\end{align*}
\subsection{Approaches}
A key question in offline RL tasks (including both OPE and OPL) is about sample efficiency: how well we can efficiently leverage the previous data for the OPE and offline learning tasks? We briefly summarize recent approaches in offline RL. 

\begin{figure}
    \centering
    \includegraphics[scale=0.35]{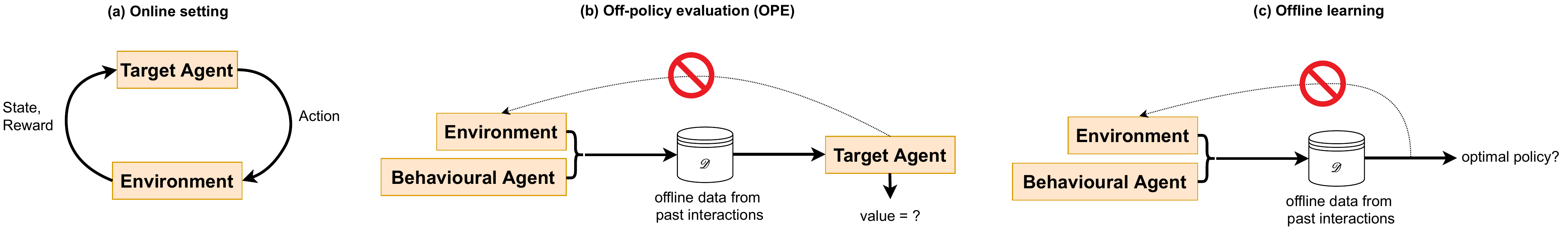}
    \caption{A diagram for online RL setting, off-policy evaluation setting and offline learning setting.}
    \label{fig:chap2_diag_ope_offline}
\end{figure}

Algorithmically, there are three main approaches to offline RL problems. Direct methods aim at learning a model of the system (e.g., the value functions, transition kernels or reward functions) and use this model to estimate the performance of the evaluation policy or learn an optimal policy. This has been studied in the tabular case in \citep{DBLP:conf/icml/MannorSST04}. 
However, in practice, the state space of MDPs is  often infinite or continuous, thus function approximation is often deployed in approximate dynamic programming such as fitted Q-iteration, least squared policy iteration \citep{bertsekas1995neuro,DBLP:conf/atal/JongS07,DBLP:journals/jmlr/LagoudakisP03,DBLP:conf/icml/GrunewalderLBPG12,DBLP:conf/icml/Munos03,DBLP:journals/jmlr/MunosS08,10.1007/s10994-007-5038-2,DBLP:conf/icml/TosattoPDR17}, and fitted Q-evaluation (FQE) \citep{DBLP:conf/icml/0002VY19}. Another popular approach uses importance sampling (IS) to obtain an unbiased value estimate of new policies by re-weighing sample rewards \citep{first_is}. Hybrid methods such as doubly robust estimations combine IS with model-based estimators to reduce the high variance while keeping the estimate unbiased \citep{dudik2011doubly,jiang2015doubly,thomas2016data,farajtabar2018more,kallus2019double}. The IS-based methods often suffer from high variance in long-horizon problems. To mitigate this problem, \citet{DBLP:conf/nips/LiuLTZ18} propose to directly estimate the stationary state visitation distribution instead of the cumulative importance ratio to break the curse of horizon, i.e., an excessively high variance in long horizon problems. Following this, many works reformulated the problem as a density ratio estimation between two stationary state visitation distributions with samples only from the behaviour distribution \citep{nachum2019dualdice,Zhang2020GenDICEGO,Zhang2020GradientDICERG,Nachum2019AlgaeDICEPG}.


Empirically, direct methods tend to perform better than IS variants and hybrid methods when encountering long horizons, policy mismatch, poor estimation of unknown behaviour policy and environment stochasticity. In particular, a recent comprehensive study \citep{voloshin2019empirical} shows that FQE tends to be more data-efficient than the other direct methods in applications with limited data. With powerful function approximation such as neural networks, FQE can be very competitive with hybrid methods 
\citep[see Figure~13]{voloshin2019empirical}. However, function approximation with insufficient representation power can hurt direct methods such as FQE. 


Theoretically, several sample efficiency results in offline RL in tabular and linear MDPs are obtained. In particular,
the sample efficiency guarantees are established in offline tabular RL \citep{xie2019optimal,DBLP:conf/aistats/YinW20,DBLP:conf/aistats/YinBW21,yin2021characterizing}. A Cramer-Rao lower bound for discrete-tree MDP is derived in \citep{DBLP:conf/icml/JiangL16}. While most existing theoretical results apply only to tabular MDP without function approximation, \citet{DBLP:journals/corr/abs-2002-09516} provide a minimax-optimal error bound for OPE with linear function approximation. Beyond linear function approximation, \citet{DBLP:conf/icml/0002VY19} provide an error bound of offline RL with general function approximation; however, they ignore the data-dependent structure in their analysis, dodging an important feature of offline RL in practice.  \citet{DBLP:journals/corr/abs-1901-00137} consider deep neural network approximation but they focus on analyzing deep Q-learning using a fresh batch of data for each iteration. Using such new batch of data is not efficient in offline RL as it scales the number of samples with the number of iterations which is arbitrarily large in practice. 

A recent line of work has focused on studying the pessimism principle in offline RL \citep{buckman2020importance}. The pessimism principle aims at penalizing erroneous out-of-distribution extrapolation when learning an optimal policy from offline data. Intuitively, if an estimator of the optimal value at some region is highly uncertain due to lack of (offline) data covering that region, the pessimism principle takes a conservative approach by tending to largely penalize the selection of actions from the region. Pessimism is provably efficient for offline RL in tabular settings \citep{rashidinejad2021bridging} and linear settings \citep{jin2020pessimism}. However, it is unclear whether sample efficiency in offline RL for more complex models is possible.

\section{Conclusion}
In this chapter, we have presented technical background and a brief literature for reinforcement learning. We close this chapter with several open challenges which we address in this thesis. 

\noindent \textbf{Bayesian Quadrature Optimization}. 
The BQO formulation in Subsection \ref{chap2_bqo} and its approach above rely on the assumption that the distribution $P(w)$ is known. In practice, however, this knowledge is not available. Instead we only have an access to a set of empirical samples of $P(w)$ obtained via previous interactions. The BQO approach in this practical setting can fail to obtain an robust solution to the original optimization problem with respect to the unknown of the distribution $P(w)$. In Chapter \ref{chap:three}, we address the above problem in a novel framework namely distributionally robust Bayesian quadrature optimization. 

\noindent \textbf{Distributional RL}. Though many distributional RL algorithms have been proposed, they \textit{all} suffer from the so-called curse of predefined statistics where they rely on a set of predefined statistics of the return distribution to approximate it. This is limited in both representation and learning as it requires a statistic constraint for such predefined statistics. Maintaining such statistic constraint is highly involved and difficult in practice. In Chapter \ref{chap:four}, we address this problem via a novel framework from statistical hypothesis testing and present a new understanding of distributional RL in such framework. We establish a novel scalable algorithm with a new state-of-the-art empirical performance in our framework. 

\noindent \textbf{Offline RL}. 
One of the main challenges of offline RL is to understand its feasibility in practical settings, in particular in high-dimensional complex models where it is necessary to use powerful function approximation such as deep neural networks for generalization. An analysis of offline RL in such situation that truly covers the offline learning setting remains to be studied. In Chapter \ref{chap:five}, we fill out this gap by analyzing sample complexity of offline RL with deep ReLU network function approximation under a \textit{new} dynamic condition and data-dependent structure. These conditions are more general than the conditions considered in prior analyses and allow improved sample complexity, rendering our work the first comprehensive analysis of offline RL with deep ReLU network function approximation.




\newpage{}

\chapter{Distributionally Robust Bayesian Quadrature Optimization \label{chap:three}}


In this chapter, we consider the first challenge of this thesis about learning and decision making in the face of an uncontrollable environmental variable. In particular, we study this problem in an instance of RL namely Bayesian quadrature optimization which aims at sequentially maximizing a quadrature objective. We propose a novel framework for this problem with provable robustness against the adversary of the uncontrollable environmental variable. Our proposed method is shown to be effective in both synthetic and real-world experiments. This chapter is based on our AISTAT'20 paper \citep{pmlr-v108-nguyen20a}. Our implementation for the proposed algorithm is available at \url{https://github.com/thanhnguyentang/drbqo}.

\section{Introduction}
\label{section:introduction}
Making robust decisions in the face of parameter uncertainty is critical to many real-world decision problems in machine learning, engineering and economics. {Besides the uncertainty that is inherent in data, a further difficulty arises due to the uncertainty in the context. A common example is hyperparameter selection of machine learning algorithms where cross-validation is performed using a small to medium sized validation set. Due to limited size of validation set, the variance across different folds might be high. 
Ignoring this uncertainty results in sub-optimal and non-robust decisions. {The problem of uncertain contexts can be further exacerbated as the outcome measurements may be noisy and the black-box function itself is expensive to evaluate.} 


{One way to capture the uncertainty in context is through a probability distribution. In this work, we consider the task of stochastic black-box optimization that is distributionally robust to the uncertainty in context. We formulate the problem as}
\begin{align}
    \label{eq:stochastic_opt}
    \max_{x \in \mathcal{X} \subset \mathbb{R}^d} g(x) := \max_{x \in \mathcal{X}} \mathbb{E}_{
    P_0(w)} [f(x, w)],
\end{align}
where $f$ is an expensive black-box function and $P_0$ is a distribution over context $w$. We assume distributional uncertainty in which the distribution $P_0$ is known only through a {limited} set of its i.i.d samples  $S_n = \{w_1, ..., w_n\}$. This is equivalent to the scenario in which we are able to evaluate $f$ only on $\mathcal{X} \times S_n$ during optimization.



In the case that $P_0$ is known (e.g., $P_0$ is either available in an analytical form or easy to evaluate), a standard solution to the problem in Equation (\ref{eq:stochastic_opt}) is based on Bayesian quadrature \citep{o1991bayes,nipsRasmussenG02,oates2016probabilistic,Oates19}. The main idea in this approach is that we {can} build a Gaussian Process (GP) model of $f$ and use the known relationship in the integral to imply a second GP model of $g$. This is possible because integration is a linear operator. 

Given the distributional uncertainty in which $P_0$ is only known through a limited set of its samples, a naive approach to the problem in Equation (\ref{eq:stochastic_opt}) is to maximize its Monte Carlo estimate:
\begin{align}
    \label{eq:mc}
    g_{mc}(x) :=  \mathbb{E}_{\hat{P}_n(w)}[f(x,w)],
\end{align}
where $\hat{P}_n(w) = \frac{1}{n} \sum_{i=1}^n \delta(w - w_i)$ and $\delta(.)$ is the Dirac distribution. When $n$ is sufficiently large, $g_{mc}(x)$ approximates $g(x)$ reasonably well as guaranteed by the weak law of large numbers; thus, the optimal solution of $g_{mc}(x)$ represents that of $g(x)$. In contrast, when $n$ is small, the optimal solution of $g_{mc}(x)$ might be sub-optimal to $g(x)$. {Since we are considering distributional perturbation, we cannot} guarantee the Monte Carlo estimate $g_{mc}(x)$ to be a good surrogate objective. 

A more {conservative} approach from statistical learning is to maximize the variance-regularized objective:
\begin{align}
    \label{eq:variance_reg}
    g_{bv}(x) := \mathbb{E}_{\hat{P}_n}[f(x, w)] - C_1 \sqrt{ \mathbb{V}_{\hat{P}_n}[ f(x, w) ] /n },
\end{align}
where $\mathbb{V}_{\hat{P}_n}$ denotes the empirical variance and $C_1$ is a constant determining the trade-off between bias and variance. {Thus}, given the context of {limited} samples, it is logical to use $g_{bv}(x)$ instead of $g_{mc}(x)$ as a surrogate objective for maximizing $g(x)$. 
However, unlike $g_{mc}(x)$, the variance term in $g_{bv}(x)$ breaks the linear relationship with respect to $f$. As a result, {though} $f$ is a Gaussian Process, $g_{bv}(x)$ {need} not be \citep{o1991bayes}. 
{Alternatively, we approach the distributional uncertainty problem above by formulating the distributionally robust Bayesian quadrature optimization}. In the face of the uncertainty about $P_0$, we seek to find a distributionally robust solution under the most adversarial distribution. Our approach is based on solving a surrogate distributionally robust optimization problem generated by posterior sampling at each time step. The surrogate optimization is solved efficiently via bisection search through any optimization. We demonstrate the efficiency of our algorithm in both synthetic and real-world problems. Our contributions in this chapter are:
\begin{itemize}
    
    \item Demonstrating the limitations of standard Bayesian quadrature optimization algorithms under distributional uncertainty (Section \ref{section:setup}), and introducing a new algorithm, namely DRBQO, that overcomes these limitations (Section \ref{section:alogirthm}); 
    \item  Introducing the concept of $\rho$-regret for measuring algorithmic performance in this formulation  (Section \ref{section:setup}), and characterizing the theoretical convergence of our proposed algorithm in sublinear Bayesian regret (Section \ref{section:theory}); 
    \item Demonstrating the efficiency of DRBQO in finding distributionally robust solutions in both synthetic and real-world problems (Section \ref{section:experiment}).
\end{itemize}

\section{Related Work}
\label{section:related}
\noindent \textbf{Bayesian Quadrature Optimization}. Our work is related to Bayesian quadrature optimization whose goal is to perform black-box global optimization of an expected objective {of the form} $\int f(x,w) P(w) dw$. This problem is known with various names such as optimization of integrated response functions \citep{Williams:2000:SDC:932015}, multi-task Bayesian optimization \citep{DBLP:conf/nips/SwerskySA13}, and optimization with expensive integrands \citep{Toscano_IntegralBO_18}. This direction approaches the problem by evaluating $f(x,w)$ at one or {several} values of $w$ given $x$. This ameliorates the need of evaluating $f(x,w)$ at all the values of $w$ and can outperform  methods that evaluate the full objective via numerical quadrature \citep{FrazierBOtut18,Toscano_IntegralBO_18}. {All the previous approaches assume the knowledge of the distribution in the expected function. The distinction of our formulation is that we are interested in the distributional uncertainty scenario in which the underlying distribution is unknown except its empirical estimate.} 


\noindent \textbf{Distributionally Robust Optimization}. 
Our work also {shares}  similarity with the distributionally robust optimization (DRO) literature \citep{DRO_review}. This problem setup considers the parameter uncertainty in real-world decision making problems. The uncertainty may be due to limited data and noisy measurements. DRO takes into account this uncertainty and approaches the problem by taking the worst-case of the underlying distribution within an uncertainty set of distributions. DRO variants distinguish each other in design choices of the {distributional uncertainty set} and in problem contexts. Regarding the design of uncertainty sets, common designs specify the set of distributions with respect to the nominal distribution via  distributional discrepancy such as $\chi^2$ divergence \citep{NamkoongD16}, Wasserstein distance \citep{WasserstainDRO19}, and Maximum Mean Discrepancy \citep{MMD_DRO19}.  Regarding studying DRO in different problem contexts, the following contexts have been investigated: robust optimization \citep{DBLP:journals/mansci/Ben-TalHWMR13}, robust risk minimization \citep{NamkoongD16}, sub-modular maximization \citep{DBLP:conf/aistats/StaibWJ19}, boosting algorithms \citep{DROBoosting}, graphical models \citep{DBLP:conf/nips/FathonyRBZZ18}, games \citep{sun2018distributional,DBLP:conf/aistats/Zhu0WGY19}, fairness in machine learning \citep{DBLP:conf/icml/HashimotoSNL18}, Markov Decision Process \citep{nipsXuM10}, reinforcement learning \citep{DRRL}, and model-agnostic meta-learning \citep{collins2020distributionagnostic}. The distinction of our work is in terms of the problem context where we study DRO in Bayesian quadrature optimization.

\section{Main Framework}

\subsection{Problem Setup} 

\label{section:setup}
\textbf{Model}. Let $f: \mathcal{X} \times \Omega \rightarrow \mathbb{R}$ be an element of a reproducing kernel Hilbert space (RKHS) $\mathcal{H}_k$ where $k$: $\mathcal{X} \times \Omega \times \mathcal{X} \times \Omega \rightarrow \mathbb{R}$ is a positive-definite kernel, and $\mathcal{X}$ and $\Omega$ are, unless explicitly mentioned otherwise, compact domains in $\mathbb{R}^d$ and $\mathbb{R}^m$ for some {dimensions} $d$ and $m$, respectively. We further assume that $k$ is continuous and bounded from above by $1$, and that $\| f \|_k = \sqrt{\langle f, f \rangle_k } \leq B$ for some $B > 0$.  Two commonly used kernels are Squared Exponential (SE) and Mat\'ern \citep{RasmussenW06} which are similarly defined on $\mathcal{X} \times \Omega$ as follows: 
\begin{align*}
    &k_{SE}(.,.; ., .) = \exp ( - d^2_{\theta, \psi} (.,.; ., .)), \\
    &k_{\nu}(.,.; ., .) = \frac{2^{1-\nu}}{\Gamma(\nu)} \sqrt{2\nu} d_{\theta, \psi}(.,.; ., .) J_{\nu}(\sqrt{2\nu} d_{\theta, \psi}(.,.; ., .)), 
\end{align*}
where $\theta$ and $\psi$ are the length scales, $\nu >0$ defines the smoothness in the Mat\'ern kernel, $J(\nu)$ and $\Gamma(\nu)$ define the Bessel function and the gamma function, respectively, and $d^2_{\theta, \psi}(x, w; x', w') = \sum_{i=1}^{d} (x_i - x'_i)^2 / \theta^2_i +  \sum_{j=1}^{m}  (w_j - w'_j)^2 / \psi^2_j$. 

Let $P_0$ be a distribution on $\Omega$, and $S_n = \{w_1, ..., w_n\}$ be a fixed set of samples drawn from $P_0$. Though $f$ is defined on $\mathcal{X} \times \Omega$, we are interested in the distributional uncertainty scenario in which we can query $f$ only on $\mathcal{X} \times S_n$ during optimization. At time $t$, we query $f$ at $(x_t, w_t) \in \mathcal{X} \times S_n$ and observe a noisy reward $y_t = f(x_t, w_t) + \epsilon_t$, where $\epsilon_t \sim \mathcal{N}(0,\sigma^2)$. Our goal is to find a robust solution point $x \in \mathcal{X}$ such that $ \mathbb{E}_{P(w)}[f(x,w)]$ remains high even under the most adversarial realization of the unknown distribution $P_0$.  

Given a sequence of noisy observations $(x_i, w_i, y_i)_{i=1}^t$, the posterior distribution under a GP(0, $k(.,.,;.,.)$) prior is also a GP with the following posterior mean and covariance: 
\begin{align*}
    \mu_t(x,w) &= k_t(x,w)^T(K_t + \sigma^2 I)^{-1} y_{1:t},  \\ 
    C_t(x,w; x', w') &= k(x,w; x', w') - k_t(x,w)^T (K_t + \sigma^2 I)^{-1} k_t(x', w'), 
\end{align*}
where  $y_{1:t}=(y_1, ..., y_t)$, $k_t(x,w) = [k(x_i, w_i; x, w)]_{i=1}^t$, and $K_t = [k(x_i, w_i; x_j, w_j)]_{1 \leq i,j \leq t}$ is the kernel matrix.

We define the quadrature functional as
\begin{align}
    \label{eq:g_functional}
    g(f, x, P) := \int P(w|x) f(x,w) dw,  
\end{align}
for any conditional distribution $P(.|x)$ on $\Omega$ for all $x \in \mathcal{X}$, i.e., $P \in \mathcal{P}_{n,\rho} \times \mathcal{X}$. 
As an extended result of Bayesian quadrature \citep{o1991bayes}, for any conditional distribution $P \in \mathcal{P}_{n,\rho} \times \mathcal{X}$, $g(f,x,P)$ also follows a GP with the following mean and variance: 
\begin{align}
    \label{eq:quadrature_mu}
    &\mu_t(x, P) := \mathbb{E}_t[g(f,x, P)] = \int P(w|x) \mu_t(x,w) dw \\
    &\sigma^2_t(x, P) := \mathbb{V}_t[g(f, x, P)] 
    \label{eq:quadrature_sigma}
= \int \int P(w|x) P(w'|x) C_t(x,w; x, w') dw dw'. 
\end{align}
\textbf{Optimization goal}. We seek to optimize a quadrature function under the most adversarial distribution over a distributional uncertainty set $\mathcal{P}_{n,\rho} := \{ P | D(P, \hat{P}_n) \leq \rho \}$ : 
\begin{align}
    \label{eq:drbqo}
    \max_{x \in \mathcal{X}} \min_{P \in \mathcal{P}_{n,\rho}} \mathbb{E}_{P(w)} [f(x,w)], 
\end{align}
where $\hat{P}_n(w) = \frac{1}{n} \sum_{i=1}^n \delta(w-w_i)$ is the empirical distribution, $\rho \geq 0$ is the confidence radius around the empirical distribution with respect to a distribution divergence $D(.,.)$ such as Wasserstein distance, maximum mean discrepancy, and $\phi$-divergence.  We can interpret $\mathcal{P}_{n,\rho}$ as the set of {perturbed distributions with respect to} the empirical distribution $\hat{P}_n$ within a confidence radius $\rho$. We then seek a robust solution in the face of {adversarial distributional perturbation} within $\mathcal{P}_{n,\rho}$. 

{For any} distribution divergence choice $D$, we define a $\rho$-robust point to be any $x^*_{\rho}$ such that 
\begin{align}
    x^*_{\rho} \in \operatorname*{arg\,max}_{x \in \mathcal{X}} \min_{P \in \mathcal{P}_{n,\rho}}  \mathbb{E}_{P(w)}[f(x,w)]. 
\end{align}
Our goal is to report after time $t$ a distributionally robust point $x_t$ in the sense that it has {small} $\rho$-\textbf{regret}, {which} is defined as
\begin{align}
    r_{\rho}(x)= g(f, x^*_{\rho}, P^*) - g(f, x, P^*),
    \label{eq:rho_regret}
\end{align}
where $P^*(.|x) = \argmin_{P \in \mathcal{P}_{n,\rho}} \sum_{w}P(w|x) f(x,w), \forall x$.

While our framework in this work can be adopted to various distribution divergences, we focus on the specific case when $D$ is $\chi^2$-divergence: 
$D(P, Q) = \frac{1}{2}\int_{\Omega} (\frac{dP}{dQ} -1)^2 dQ, \forall P,Q$. From here on, we {refer} $\mathcal{P}_{n,\rho}$ as the $\chi^2$ ball {with} $D$ being $\chi^2$-divergence. In particular, the distributionally robust optimization problem in Equation (\ref{eq:drbqo}) is equivalent to the variance-regularized optimization in Equation (\ref{eq:variance_reg}) when the variance is sufficiently high, as justified by the following theorem:
\begin{thm}[Modified from \citep{nipsNamkoongD17}]
Let $Z \in [M_0, M_1]$ be a random variable (e.g., $Z = f(x,w)$ for any fixed $x$), $\rho \geq 0$, $M = M_1 - M_0$, $s_n^2 = \mathbb{V}_{\hat{P}_n}[Z]$, $s^2 = \mathbb{V}[Z]$, and $OPT = \inf_{P} \left\{  \mathbb{E}_{P}[Z]:  P \in  \mathcal{P}_{n,\rho} \right\}$. Then $\max \left\{ \sqrt{2\rho s^2_n} - 2M\rho, 0 \right\}
    \leq \mathbb{E}_{\hat{P}_n}[Z] - OPT 
    \leq \sqrt{2\rho s^2_n}.$
Especially if $s^2 \geq \max\{ 24 \rho, \frac{16}{n},  \frac{1}{n s^2} \} \cdot M^2$, then $OPT = \mathbb{E}_{\hat{P}_n}[Z] -  \sqrt{2\rho s^2_n}$ 
with probability at least $1 - \exp(- \frac{n s^2}{36 M^2})$. 
\end{thm}
The intuition for this equivalence is that the $\chi^2$ ball and the variance penalty term in Equation (\ref{eq:variance_reg}) are both quadratic \citep{DBLP:conf/aistats/StaibWJ19}. Figure \ref{fig:2d_simplex} illustrates $\chi^2$ balls with various radii on the $3$-dimensional simplex. 

\begin{figure}[t]
    \centering
    \includegraphics[scale=0.45]{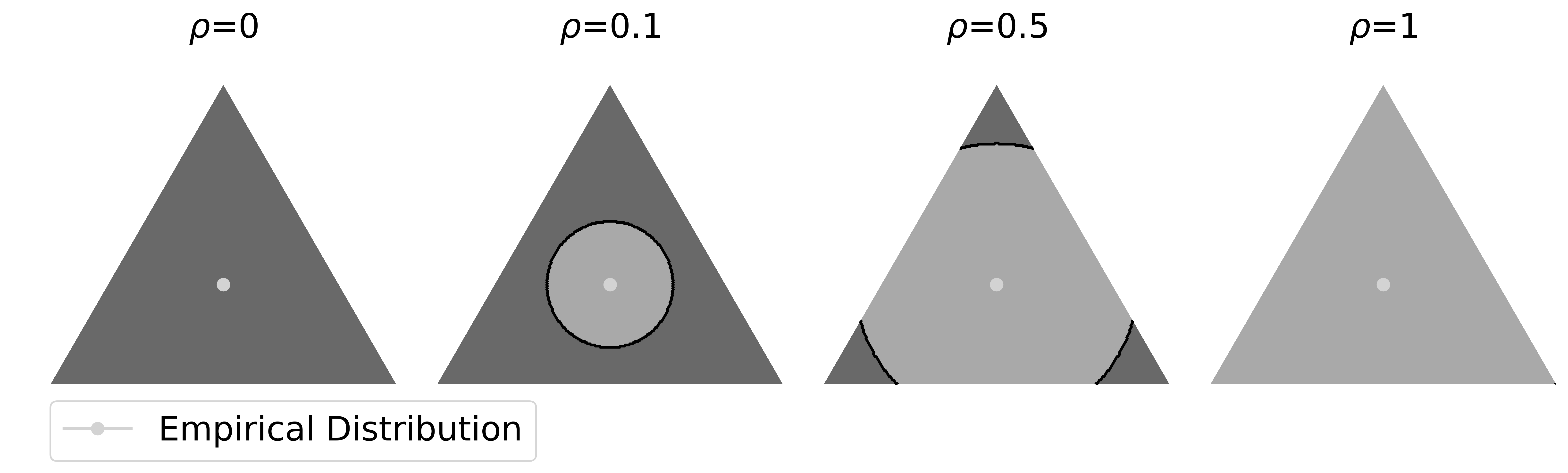}
    \caption{The $\chi^2$ balls with various radii $\rho$ on the $n$-dimensional simplex ($n=3$ in this example). The simplex, the $\chi^2$ balls and the empirical distribution are represented in dim gray, dark gray and light gray color, respectively. The $\chi^2$ ball with $\rho = 0$ reduces to a singleton containing only the empirical distribution while the ball becomes the entire simplex for $\rho \geq \frac{n-1}{2}$.}
    \label{fig:2d_simplex}
\end{figure}

\textbf{Failure of standard methods}.
Various methods have been developed for achieving small regret in maximizing $g(f,x, P_0) = \mathbb{E}_{P_0(w)} [f(x, w)]$ for some distribution $P_0(w|x) = P_0(w)$ \citep{Williams:2000:SDC:932015,DBLP:conf/nips/SwerskySA13,Toscano_IntegralBO_18}. These methods leverage the relationships in Equation (\ref{eq:quadrature_mu}) and (\ref{eq:quadrature_sigma}) to infer the posterior mean and variance of the expected function $g(f,x, P_0)$ from those of $f$. The inferred posterior mean and variance for $g(f,x, P_0)$ are then used in certain ways to acquire new points. While this is useful in the standard setting when we know $P_0$, {it is not useful when} we only have the empirical distribution $\hat{P}_n$. Specifically, an optimal solution found by these methods in the problem associated with the empirical distribution may be sub-optimal to that associated with the true distribution $P_0$. 

An illustrative example is depicted in Figure \ref{fig:figure1} where the averaged trajectories of our proposed DRBQO (detailed in Section \ref{section:alogirthm}) and a standard BQO baseline (detailed in Section \ref{section:experiment}) are also shown. Due to a limited number of samples of $P_0$, the Monte Carlo estimate $\mathbb{E}_{\hat{P}_n(w)}[f(x,w)]$ results in a spurious expected objective in this case. By resorting to the empirical distribution $\hat{P}_n$ constructed from the limited set of samples, the standard BQO baseline ignores the distributional uncertainty and converges to the optimum of the spurious expected objective. The same limitation applies to the standard BQO optimization methods, e.g., \citep{Williams:2000:SDC:932015,DBLP:conf/nips/SwerskySA13,Toscano_IntegralBO_18,DBLP:conf/wsc/PearceB17a} whose goal is to find a global non-robust maximum. 
\begin{figure}[t]
    \centering
    \includegraphics[scale=1]{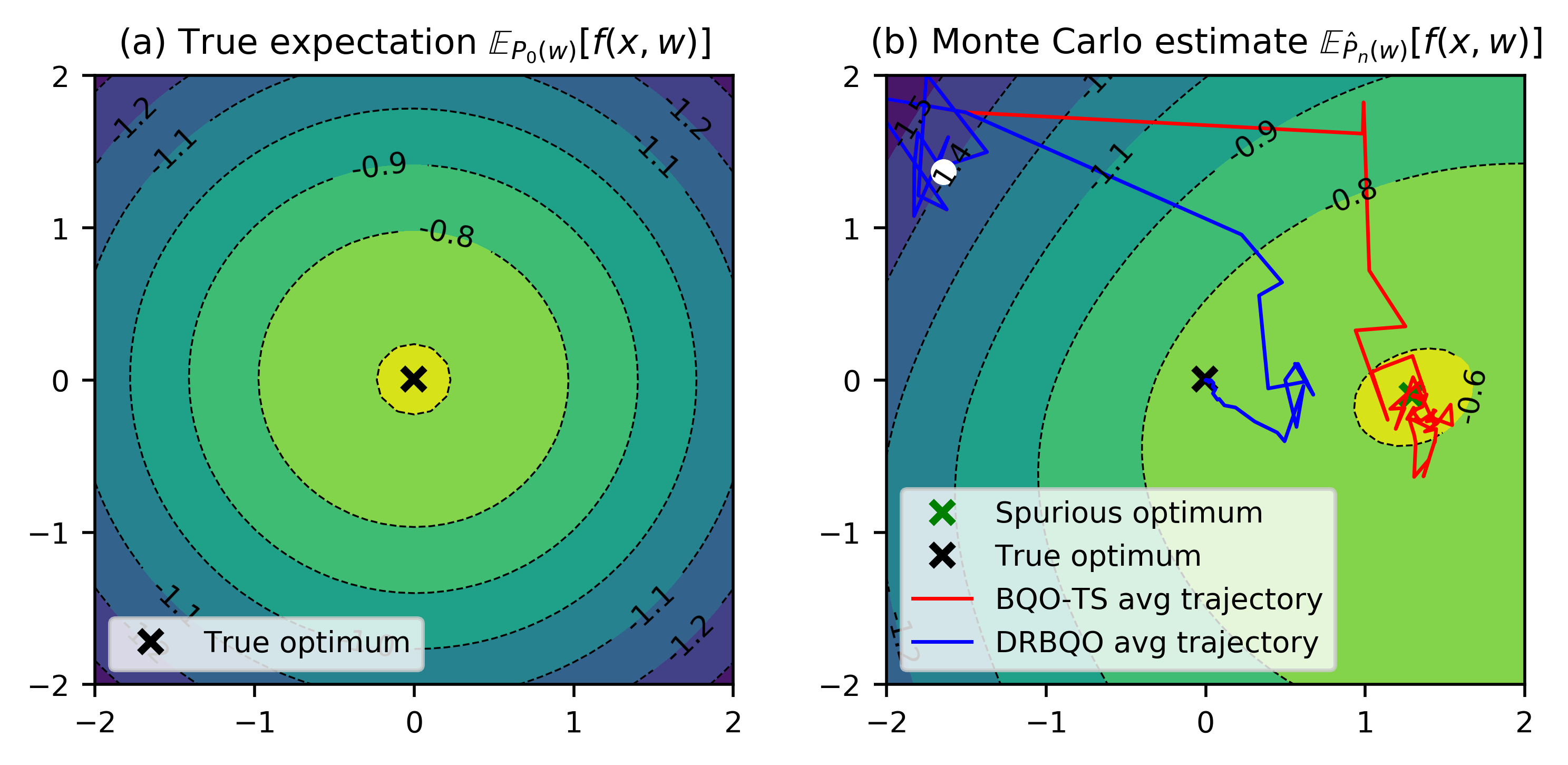}
    \caption{(a) The expected logistic function $g(x) = \mathbb{E}_{\mathcal{N}(w; 0, I)}[-\log(1 + e^{x^T w})]$ and (b) its Monte Carlo estimate using $10$ samples of $w$, and the averaged trajectories of our proposed algorithm DRBQO (detailed in Section \ref{section:alogirthm}) and a standard Bayesian Quadrature Optimization (BQO) baseline. Though being unbiased, Monte Carlo estimates can suffer from high variance given limited samples, resulting in spurious function estimates. Our proposed algorithm DRBQO approaches this mismatch problem by {finding} the distributionally robust solution under the most adversarial distribution over a $\chi^2$ distributional ball.}
    \label{fig:figure1}
\end{figure}
\subsection{Algorithmic Approach}
\label{section:alogirthm}
Our main proposed algorithm is presented in Algorithm \ref{pseudo_code:drbqo}.
In the standard Bayesian quadrature problem in Equation (\ref{eq:mc}), we can easily adopt standard Bayesian optimization algorithms such as expected improvement (EI) \citep{mockus1978application} and an upper confidence bound (UCB) (e.g., GP-UCB \citep{DBLP:conf/icml/SrinivasKKS10}) using quadrature relationships in Equation (\ref{eq:quadrature_mu}) and (\ref{eq:quadrature_sigma}) \citep{DBLP:conf/nips/SwerskySA13}. However, like $g_{bv}(x)$ in Equation (\ref{eq:variance_reg}), $\min_{P \in \mathcal{P}_{n,\rho}} \mathbb{E}_{P(w)} [f(x,w)]$ {does not follow} a GP if $f$ follows a GP. This difficulty hinders the adoption of EI-like and UCB-like algorithms to our setting. 
We overcome this problem using posterior sampling \citep{DBLP:journals/mor/RussoR14}. 

\begin{algorithm}[t]
  \caption{Distributionally robust Bayesian quadrature optimization}
\label{pseudo_code:drbqo}
\begin{algorithmic}[1]
  \STATE {\bfseries Input:}  Prior \texttt{GP}($\mu_0$, $k$), horizon $T$, sample set $S_n$,  confidence radius $\rho \geq 0$
  \STATE Set $C_0 \leftarrow k$ 
  \FOR{$t=1$ {\bfseries to} $T$}
  \STATE Sample $\tilde{f}_t \sim $ \texttt{GP}$\left(\mu_{t-1}, C_{t-1}\right)$  
  \label{alg:line:ps}
  \STATE Choose $\displaystyle x_t \in \operatorname*{arg\,max}_{x \in \mathcal{X}} \min_{P \in \mathcal{P}_{n,\rho}}  \mathbb{E}_{P} [\tilde{f}_t(x,w)]$ 
  \label{alg:line:dro}
  \STATE Choose $ \displaystyle w_t = \operatorname*{arg\,max}_{w \in S_n} C_{t-1}(x_t, w; x_t, w)$ 
  \label{alg:line:w}
  \STATE Observe reward $\hat{y}_t \leftarrow f(x_t, w_t) + \epsilon_t$
  \STATE Perform update \texttt{GP} to get $\mu_{t}$ and $C_{t}$
  \ENDFOR
\STATE {\bfseries Output:} $\displaystyle \operatorname*{arg\,max}_{ x \in \{x_1, ..., x_T\}} \min_{P \in \mathcal{P}_{n,\rho}}  \mathbb{E}_{P} [\mu_T(x,w)]$
\end{algorithmic}
\end{algorithm} 
The main idea of our algorithm is to sample and solve a surrogate distributionally robust optimization problem at each step guided by posterior sampling (lines \ref{alg:line:ps} and \ref{alg:line:dro} in Algorithm \ref{pseudo_code:drbqo}). In practice, we follow \cite{DBLP:conf/nips/Hernandez-LobatoHG14} to perform posterior sampling (line \ref{alg:line:ps} in Algorithm \ref{pseudo_code:drbqo}), with an explicit pseudo-code presented in Algorithm \ref{alg:ps_gp}. 
Similar to the way posterior sampling is applied to standard Bayesian optimization problem \citep{DBLP:conf/nips/Hernandez-LobatoHG14}, a new point is selected according to the probability it is optimal in the sense of distributional robustness. One of the advantages of posterior sampling is that it avoids the need for confidence bound such as UCB. This is useful for our setting because the non-Gaussian nature of the distributionally robust objective makes it difficult to construct a deterministic confidence upper bound. 

Due to the convexity of the expectation with respect to a distribution, we can efficiently compute the value (therefore the gradients) of the inner minimization in line \ref{alg:line:dro} of Algorithm \ref{pseudo_code:drbqo} in an analytical form via Lagrangian multipliers, {as presented in Proposition} \ref{prop:kkt}. 

\begin{prop}
\label{prop:kkt}
Let $l = (l_1, ..., l_n) \in \mathbb{R}^n$ (e.g., $l = \left( \tilde{f}_t(x,w_1), ..., \tilde{f}_t(x,w_n) \right)$ in line \ref{alg:line:dro} of Algorithm \ref{pseudo_code:drbqo}), $\hat{P}_n = (\frac{1}{n}, ..., \frac{1}{n})$ {being} the weights of the empirical distribution, $\Delta_n$ being the $n$-dimensional simplex,
    $\mathcal{P}_{n,\rho} = \bigg\{P \in \Delta_n  \bigg| \frac{1}{2}\int_{\Omega} (\frac{dP}{d\hat{P}_n} -1)^2 d\hat{P}_n \leq \rho \bigg\}$
being the $\chi^2$-ball around the empirical distribution with radius $\rho$. 
Then, the optimal weights $p = (p_1, ..., p_n) = \argmin_{q \in \mathcal{P}_{n,\rho}} q^T l$ satisfy the systems of {relations} with variables $(p, \lambda, \eta)$:
\begin{numcases}{}
\label{kkt:compute_pi}
\lambda p_i = \frac{1}{n} \max \{-l_i - \eta, 0\}, \forall 1 \leq i \leq n \nonumber \\
\label{kkt:compute_eta}
\eta |A| + n \lambda = -\sum_{i \in A} l_i \text{ where } A = \{i: l_i \leq -\eta \} \nonumber \\
\label{kkt:rho_equation}
\lambda \left(2\rho + 1 - n \| p\|_2^2 \right)  = 0 \\ 
n \| p\|_2^2 \leq 2\rho + 1, \text{ and } \gamma \geq 0. \nonumber 
\label{kkt:dual_1}
\end{numcases}

\end{prop}
\begin{proof}
The constrained minimization $\min_{p \in \mathcal{P}_{n,\rho}} p^T l$  is a convex optimization problem which forms the Lagrangian: $L(p, \lambda, \eta, \zeta) = p^Tl - \lambda \left(\rho - \frac{1}{2n} \sum_{i=0}^n (n p_i - 1)^2 \right)  -\eta (1 - \sum_{i=1}^n p_i) - \sum_{i=1}^n \zeta_i p_i$
where $p \in \mathbb{R}^n, \lambda \geq 0$, $\eta \in \mathbb{R}$, and $\zeta \in \mathbb{R}^n_{+}$. The system of linear equations in the proposition {emerges from}  Karush-Kuhn-Tucker (KKT) conditions and simple rearrangements. Note that since the primal problem is convex, the duality gap is zero and the KKT conditions are the sufficient and necessary conditions for the primal problem. 

In particular, consider the constrained optimization problem 
\begin{align}
    \label{eq:rho_quadrate_weights}
    \min_{p \in \mathcal{P}_{n,\rho}} \sum_{i=1}^n p_i l_i. 
\end{align}

This is a convex optimization problem which forms the Lagrangian: 
\begin{align}
    L(p, \lambda, \eta, \zeta) &= p^Tl - \lambda \left(\rho - \frac{1}{2n} \sum_{i=0}^n (n p_i - 1)^2 \right)  \nonumber -\eta (1 - \sum_{i=1}^n p_i) - \sum_{i=1}^n \zeta_i p_i, 
\end{align}
where $p \in \mathbb{R}^n, \lambda \geq 0$, $\eta \in \mathbb{R}$, and $\zeta \in \mathbb{R}^n_{+}$. The KKT conditions for the primal problem in Equation (\ref{eq:rho_quadrate_weights}) are: 
\begin{numcases}{}
\label{kkt:stationary}
l_i + \lambda (np_i - 1) + \eta - \zeta_i = 0, \forall 1 \leq i \leq n \\ 
\label{kkt:slack_1}
\lambda \left(2\rho + 1 - n \| p\|_2^2 \right)  = 0 \\ 
\label{kkt:primal_1}
n \| p\|_2^2 \leq 2\rho + 1  \\ 
\label{kkt:dual_1}
\lambda \geq 0 \\ 
\label{kkt:primal_2}
\sum_{i=1}^n p_i = 1 \\ 
\label{kkt:slack_2}
\zeta_i p_i = 0,  \forall 1 \leq i \leq n \\ 
\label{kkt:dual_2}
\zeta_i \geq 0, \forall 1 \leq i \leq n.
\end{numcases}  

We can see that the strong duality holds because the primal problem in Equation (\ref{eq:rho_quadrate_weights}) satisfies the Slater's condition; therefore the KKT conditions are the necessary and sufficient conditions for the primal optimal solution. It follows from Equations (\ref{kkt:stationary}), (\ref{kkt:slack_2}), and (\ref{kkt:dual_2}) that:
\begin{align}
    \label{eq:p}
    \lambda n p_i = (-l_i -\eta)_{+} := \max \{ -l_i -\eta,0 \},
\end{align}
which, combined with Equation (\ref{kkt:primal_2}), implies that: 
\begin{align}
    \label{eq:eta_to_lam}
    n \lambda = \sum_{i=1}^n (-l_i - \eta)_{+}.
\end{align}

From Equation (\ref{eq:eta_to_lam}), we have:
\begin{align}
    \label{eq:lam_to_eta}
    \eta &= \frac{-\sum_{i \in A} l_i - n \lambda}{ |A|},
\end{align}
where $A = \{ i: l_i + \eta \leq 0  \} $. Note that $|A| \geq 1$ because otherwise $p_i =0, \forall 1 \leq i \leq n$ which contradicts Equation (\ref{kkt:primal_2}). We then plug Equation (\ref{eq:lam_to_eta}) and (\ref{eq:p}) into Equation (\ref{kkt:primal_1}) to solve for $\lambda$. Note that $\|p\|_2^2$ is decreasing in $\lambda$, thus we can bisect to find the optimal $\lambda$ within its bound. We can easily obtain a bound on $\lambda$ from Equation (\ref{kkt:primal_1}): 
\begin{align*}
    0 \leq \lambda &\leq \max \left\{  \frac{-l_{min} + \sum_{i=1}^n l_i  }{\sqrt{1 + 2 \rho} - 1}, \frac{-l_{min} + l_{max}}{ \sqrt{1 + 2 \rho} }   \right \}, 
\end{align*}
where $l_{min} = \min_{1 \leq i \leq n} l_i$, and  $l_{max} = \max_{1 \leq i \leq n} l_i$.

The optimal distribution $\argmin_{p \in \mathcal{P}_{n,\rho}} \sum_{i=1}^n p_i l_i $ is not constant, but rather a function of $l$. Thus, its gradients with respect to some parameter $\psi$ must be computed from those of $l$. This becomes straightforward when we have solved $(p_i, \lambda, \eta)$ in terms of $l$ as in the results above:
 \begin{align*}
 \begin{cases}
    \label{eq:dr_gradient}
    \frac{\partial p_i}{\partial \psi}  = \frac{-1}{n \lambda^2} (-l_i - \eta) \frac{\partial \lambda}{\partial \psi}  + \frac{1}{n \lambda}(-\frac{\partial l_i}{\partial \psi} -\frac{\partial \eta}{\partial \psi}) \\
    |A| \frac{\partial \eta}{\partial \psi} =  -\sum_{i \in S} \frac{\partial l_i}{\partial \psi} - n \frac{\partial \lambda}{\partial \psi} \\ 
    \sum_{i \in S} p_i \frac{\partial p_i}{\partial \psi} = 0.      
 \end{cases}
 \end{align*}

\end{proof}

In practice, we can use bisection search \citep{NamkoongD16} to solve for $\lambda$ satisfying  Equation (\ref{kkt:rho_equation}), as presented in Algorithm \ref{alg:bisect}.

\begin{algorithm}[t]
\begin{algorithmic}[1]
\STATE \textbf{Input}: $ p = p(\lambda)$ computed Proposition \ref{prop:kkt}, $\epsilon \geq 0$
    
\STATE Set $\lambda_{min} \leftarrow 0$ 
    
\STATE Set $\lambda_{max} \leftarrow \max \left\{  \frac{-l_{min} + \sum_{i=1}^n l_i  }{\sqrt{1 + 2 \rho} - 1}, \frac{-l_{min} + l_{max}}{ \sqrt{1 + 2 \rho} }   \right \}$ 
    
\STATE Set $\lambda \leftarrow \lambda_{min}$ 

\WHILE{$\lambda_{max} - \lambda_{min} > \epsilon$}
\STATE $\lambda = \frac{1}{2}(\lambda_{max} + \lambda_{min})$ 
    
\IF{$n \| p(\lambda)\|_2^2 > 2\rho + 1$}
    \STATE $\lambda_{min} = \lambda$ 
    \ELSE
        \STATE $\lambda_{max} = \lambda$
\ENDIF
\ENDWHILE
\STATE \textbf{Output}: $\lambda, p(\lambda)$
\caption{Bisection search}
\label{alg:bisect}
\end{algorithmic}
\end{algorithm}

\subsection{Theoretical Analysis}
\label{section:theory}

For the sake of analysis, we adopt the definition of the $T$-period regret and Bayesian regret from \citep{DBLP:journals/mor/RussoR14} to our setting, as also discussed in Section \ref{chap2_section_bandit} of Chapter \ref{chap:Background}. In particular, we define a policy $\pi$ as a mapping from the history $H_t = (x_1, w_1, P_1, ..., x_{t-1}, w_{t-1}, P_{t-1})$ to $(x_t,w_t, P_t)$ where $P_i \in \mathcal{P}_{n,\rho} \times \mathcal{X}, \forall i$. 

\begin{defn}[$T$-period regret] The $T$-period regret of a policy $\pi$ is defined by
\begin{align*}
    Regret(T, \pi, f) = \sum_{t=1}^T \mathbb{E} \left[ g(f, x^*, P^*) - g(f, x_t, P_t) | f \right],
\end{align*}
where $T \in \mathbb{N}$, and 
\begin{align*}
    x^* &\in \operatorname*{arg\,max}_{x \in \mathcal{X}} \min_{P \in \mathcal{P}_{n,\rho}} \mathbb{E}_{P(w)}[f(x,w)], \\
    P^*(.|x) &= \argmin_{P \in \mathcal{P}_{n,\rho}}\mathbb{E}_{P(w)}[f(x,w)], \forall x \in \mathcal{X}. 
\end{align*}

\end{defn}
\begin{defn}[$T$-period Bayesian regret]
The $T$-period Bayesian regret of a policy $\pi$ is the expectation of the regret with respect to the prior over $f$, 
\begin{align}
    \text{BayesRegret}(T,\pi) &= \mathbb{E} [\text{Regret}(T, \pi, f)].
\end{align}
\end{defn}
For simplicity, we focus our analysis on the case where $\mathcal{X}$ is finite and $\mathcal{P}_{n,\rho}$ is a finite subset of the $\chi^2$ ball of radius $\rho$. Similar to \citep{DBLP:conf/icml/SrinivasKKS10}, the results can be extended to infinite sets $\mathcal{X}$ and the entire $\chi^2$ ball using discretization trick of \citep{DBLP:conf/icml/SrinivasKKS10} as long as a smoothness condition (i.e., the partial derivatives of $f$ are bounded with high probability) is satisfied \citep[Theorem~2]{DBLP:conf/icml/SrinivasKKS10}.  

\begin{thm} 
\label{thm:bayesregret}
Assume $\mathcal{X}$ is a finite subset of $\mathbb{R}^d$, and $\mathcal{P}_{n,\rho}$ is a finite subset of the $\chi^2$ ball of radius $\rho$. Let $\pi^{DRBQO}$ be the DRBQO policy presented in Algorithm \ref{pseudo_code:drbqo}, $\gamma_T$ be the maximum information gain defined in \cite{DBLP:conf/icml/SrinivasKKS10}, then for all $T \in \mathbb{N}$, 
\begin{align*}
    \text{BayesRegret}(T, \pi^{DRBQO}) &\leq 1  + \frac{(\sqrt{2\log \frac{(1+ T^2) |\mathcal{X}| |\mathcal{P}_{n,\rho}|}{ \sqrt{2\pi}} } +B) \sqrt{2\pi}}{|\mathcal{X}| |\mathcal{P}_{n,\rho}|} + \frac{2 \gamma_T \sqrt{(1 + 2\rho)n}}{1 + \sigma^{-2}} \\ 
    &+ 2\sqrt{ T \gamma_T (1 + \sigma^{-2})^{-1} \log \frac{(1+ T^2) |\mathcal{X}| |\mathcal{P}_{n,\rho}|}{ \sqrt{2\pi}}  }.
\end{align*}

\end{thm}
Note that $\gamma_T$ can be bounded for three common kernels: linear, SE and Mat\'ern kernels in \citep{DBLP:conf/icml/SrinivasKKS10}, which is summarized in Table \ref{table:information_gain}.
\begin{table}[h]
\begin{center}
\begin{tabular}{  l | l } 
\textbf{Kernel type} & \textbf{Information gain} $\gamma_T$ \\ 
\hline 
\hline 
Linear & $\mathcal{O}(d \log T)$  \\ 
\hline
Squared exponential & $\mathcal{O}(\log T)^{d+1})$  \\ 
\hline
Mat\'ern with $\nu > 1$ & $ \mathcal{O}( T^{d(d+1) / (2 \nu + d(d+1))} \log T ) $  \\ 
\end{tabular}
\end{center}
\caption{The upper bounds for the information gains $\gamma_T$ for various types of kernels. Here $d \in \mathbb{N}$ is the dimension of the search domain.}
\label{table:information_gain}
\end{table}
Using these bounds, Theorem \ref{thm:bayesregret} suggests that DRBQO has \textit{sublinear} Bayesian regret for common kernels such as linear, SE and Mat\'ern kernels.

Our Bayesian regret bound of DRBQO is of order $\sqrt{T \gamma_T \log( (1 + T^2) |\mathcal{X}| |\mathcal{P}_{n, \rho}| )}$ which matches the standard upper bounds (up to an extra log constant $\log |\mathcal{P}_{n, \rho}| )$ established in \citep{DBLP:journals/mor/RussoR14} and \citep{DBLP:conf/icml/SrinivasKKS10}. The extra log constant in our bound accounts for an additional decision space $\mathcal{P}_{n, \rho}$ for the parameter distribution in our problem. To our knowledge, the standard bound above is one of the best known upper bounds for GP optimization. \citet{scarlett2017lower} establish a lower bound for GP optimization suggesting that the standard bound above is near-optimal (w.r.t. the established lower bound) for the square exponential kernel.

\begin{proof}[Proof sketch]
We leverage two proof techniques from \citep{DBLP:journals/mor/RussoR14} to derive this bound including posterior sampling regret decomposition and the connection between posterior sampling and UCB. However, an extension from the Bayesian regret bound to our case is non-trivial. The main difficulty is that the $\rho$-robust quadrature distributions $\argmin_{P \in \mathcal{P}_{n,\rho}} \mathbb{E}_{P(w)}[f(x,w)]$ are random variables and the resulting quadrature $\min_{P \in \mathcal{P}_{n,\rho}} \mathbb{E}_{P(w)}[f(x,w)]$ does not follow a GP. We overcome this difficulty by decomposing the range $\mathbb{R}$ of $f(x,w)$ into a set of carefully designed disjoint subsets, using several concentration inequalities for Gaussian distributions, and leveraging the mild assumptions of $f$ from the problem setup. The detailed proof is presented in Section \ref{chap3_sec:proof}. 
\end{proof}

\begin{rem}[Extensions to other divergence measures beyond $\chi^2$ divergence]
We have focused on $\chi^2$ divergence mainly for simplicity. Our algorithmic and theoretical results can be potentially extended to $f$-divergence (including $\chi^2$, KL and Hellinger) that requires the involved distribution to have the same support as the nominal distribution $\hat{P}$. Regarding the algorithmic extension for $f$-divergence, since $f$ in $f$-divergence is convex, the surrogate DRO still reduces to convenient KKT conditions (as the strong duality still holds). Regarding the theoretical extension for $f$-divergence, the sublinear convergence rate in Theorem \ref{thm:bayesregret} remains valid because in our analysis the distribution-dependent term $\sum_i p_i^2$ is always bounded above by $1$ (though in the case of $\chi^2$ divergence, this bound can be tighter as shown in our proof of Theorem \ref{thm:bayesregret}). The current form of our framework cannot however be extended to divergences that are defined for distributions of continuous support such as Wasserstein because our analysis relies on the assumption of finite support for the distributional uncertainty set. This assumption is however very mild in practice because if one of the involved distributions is not discrete, computing the Wasserstein distance becomes intractable even with the simplest scenario where one distribution is uniform while the other is discrete with two atoms. In practice, we can usually avoid this intractability by discretizing the support via discrete distributions for the distributional uncertainty set, and thus can leverage our analytical insights.

\end{rem}
\section{Experiment}
\label{section:experiment}
In this section, we empirically validate the performance of DRBQO by comparing against several baselines in synthetic and $n$-fold cross-validation hyperparameter tuning experiments.  

We focus on the BQO baselines that directly substitute the inferred posterior mean $\mu_t(x, \hat{P}_n)$ (in Equation (\ref{eq:quadrature_mu})) and variance $\sigma_t^2(x, \hat{P}_n)$ (in Equation (\ref{eq:quadrature_sigma})) of  $g(f,x, \hat{P}_n)$ into any standard acquisition (e.g., EI and GP-UCB) to achieve small regret in maximizing $g(f,x, \hat{P}_n)$. More advanced BQO baseline methods, e.g.,  \citep{Toscano_IntegralBO_18} are expected to perform poorly in the distributional uncertainty setting because they are not set out to account for the robust solutions. There is a distinction between sampled points and report points by each baseline algorithm. A sampled point is a suggested point regarding where to sample next while a report point is chosen from all the sampled points (up to any iteration) based on the objective function that an algorithm aims at optimizing. In standard noiseless Bayesian optimization, sampled points and report points are identical. However, this is not necessarily the case in BQO where the objective function has expectation form and is not directly queried.  In particular, we consider the following baselines:   


\begin{itemize}
    \item MTBO: Multi-task Bayesian optimization  \citep{DBLP:conf/nips/SwerskySA13} is a typical BQO algorithm in which  the inferred posterior mean and variance are plugged into the standard EI acquisition to select $x_t$. In addition, each $w_t$ in this case represents a task and MTBO uses multi-task kernels to model the task covariance. Conditioned on $x_t$,  $w_t$ is selected such that the corresponding task yields the highest EI. We include MTBO only in the cross-validation hyperparameter tuning experiments.
    \item BQO-EI: This algorithm is similar to MTBO except for two distinctions. First,  $w_t$ is selected such that it yields the highest posterior variance on $f$, similar to our algorithm (see line \ref{alg:line:w} in Algorithm
    \ref{pseudo_code:drbqo}). Second, this uses kernels defined on the Cartesian product space $\mathcal{X} \times \Omega$ instead of the multi-task kernels as in MTBO. In addition, the report point at time $t$ is $\operatorname*{arg\,max}_{ x \in x_{1:t}}  \mathbb{E}_{\hat{P}_n(w)} [\mu_t(x,w)]$. 

    \item Maximin-BQO-EI: This method is the same as BQO-EI except that the report point is $\operatorname*{arg\,max}_{ x \in x_{1:t}} \min_{P \in \mathcal{P}_{n,\rho}}  \mathbb{E}_{P(w)} [\mu_t(x,w)]$. 
    
    \item BQO-TS: This method is a non-robust version of our proposed DRBQO. The only distinctions between BQO-TS and DRBQO are in the way $x_t$ is selected (line \ref{alg:line:dro} of Algorithm \ref{pseudo_code:drbqo}) and the way a report point is chosen. In BQO-TS, $x_t$ is selected with respect to the empirical distribution as follows: $x_t \in \operatorname*{arg\,max}_{x \in \mathcal{X}} \mathbb{E}_{\hat{P}_n(w)} [\tilde{f}_t(x,w)] $, and the report point at time $t$ is chosen as $\operatorname*{arg\,max}_{ x \in x_{1:t}}  \mathbb{E}_{\hat{P}_n(w)} [\mu_t(x,w)]$. 
    
    
    \item Maximin-BQO-TS: This is the same as BQO-TS except that the final report point is $\operatorname*{arg\,max}_{ x \in x_{1:t}} \min_{P \in \mathcal{P}_{n,\rho}}  \mathbb{E}_{P(w)} [\mu_t(x,w)]$. 
    
    \item Emp-DRBQO: This is the same as DRBQO except that the report point is chosen as $\operatorname*{arg\,max}_{ x \in x_{1:t}}  \mathbb{E}_{\hat{P}_n(w)} [\mu_t(x,w)]$.
\end{itemize}


\begin{figure}[t]
    \centering
    \includegraphics[scale=0.7]{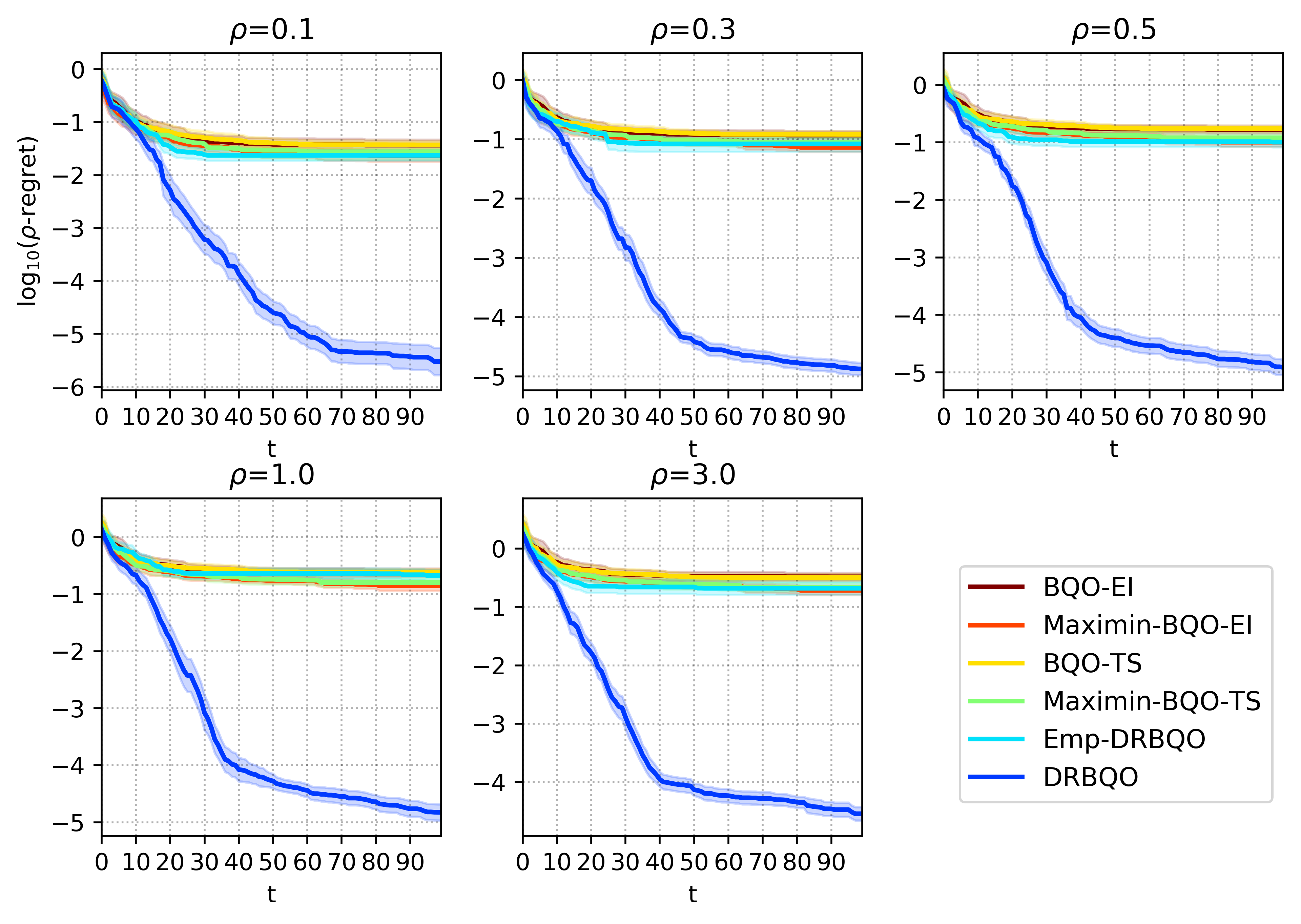}

    \caption{The best so-far $\rho$-regret values (plotted on the $\log_{10}$ scale) of the baseline BQO methods and our proposed method DRBQO for the synthetic function in Section \ref{section:experiment}. DRBQO significantly outperforms the baselines with respect to the $\rho$-regret in this experiment. {The larger the value of} $\rho$ (i.e., {the more conservative against the adversarial distributional perturbation), the higher is the} $\rho$-regret of the non-robust baselines.}
    \label{fig:plot_rho_regret_logistic}
\end{figure}

\subsection{Synthetic Functions}
The distributional uncertainty problem is more pronounced when $f(x,w)$ is more significantly distinct across different values of $w \in S_n$, i.e., $f(x,w)$ experiences high variance along the dimension of $w$. Inspired by the logistic regression and the experimental evaluation from the original variance-based regularization work \citep{NamkoongD16}, we use a logistic form for synthetic function: $f(x,w) = -\log(1 + \exp(x^T w))$, where $x,w \in \mathbb{R}^d$. The true distribution $P_0$ is the standard Gaussian $\mathcal{N}(0,I)$. In this example, we use $d=2$ for better visualization. We sample $n=10$ values of $w$ from $\mathcal{N}(w; 0,I)$ and fix this set for the empirical distribution $\hat{P}_n(w) = \frac{1}{n} \sum_{i=1}^n \delta(w-w_i)$. The true expected function $\mathbb{E}_{P_0(w)}[f(x,w)]$ and the empirical (Monte Carlo) estimate function $\mathbb{E}_{\hat{P}_n(w)}[f(x,w)]$ are illustrated in Figure \ref{fig:figure1} (a) and (b), respectively. In this illustration, the Monte Carlo estimate function catastrophically shifts the true optimum to a spurious point due to the limited data in estimating $P_0$.

\begin{figure}
    \centering
    \includegraphics[scale=0.9]{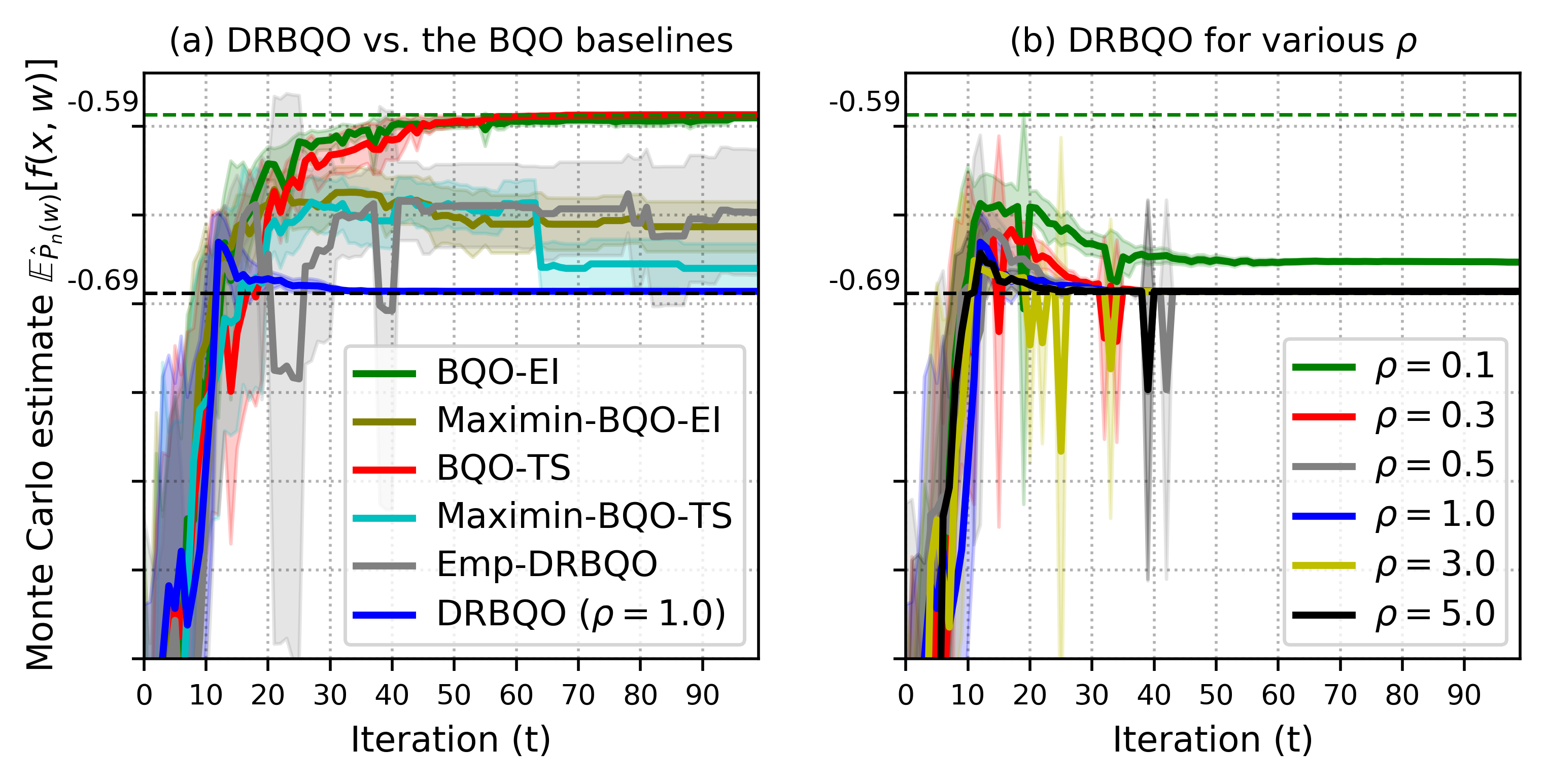}
    \caption{The empirical expected function $\mathbb{E}_{\hat{P}_n(w)}[f(x,w)]$ evaluated at each point $x$ reported at time $t$ by DRBQO and the standard BQO baselines (a), and by DRBQO for various values of $\rho$ (b). In this example, $\mathbb{E}_{\hat{P}_n(w)}[f(x,w)]$ has a maximum value of $-0.59$ while it has a value of $-0.69$ evaluated at the optimum of the true expected function $\mathbb{E}_{P_0(w)}[f(x,w)]$. 
    The BQO baselines achieve higher values of the empirical expected function than DRBQO but DRBQO converges to the distributionally robust solutions.}
    \label{fig:mc_values}
\end{figure}

We initialize the comparative algorithms by selecting $12$ uniformly random inputs $(x,w) \in \mathcal{X}\times S_n$, and {we} keep these initial points the same for all the algorithms. We use the squared exponential kernel $k_{SE}$ defined on the Cartesian product space of $x$ and $w$. We normalize the input and output values to the unit cube, and resort to marginal maximum likelihood to learn the GP hyperparameters \citep{RasmussenW06} every time we acquire a new observation. The time horizon for all the algorithms is $T=100$. We report the results using two evaluation metrics: the $\rho$-regret as defined in Equation (\ref{eq:rho_regret}) and the value of the empirical expected function $\mathbb{E}_{\hat{P}_n(w)}[f(x,w)]$ evaluated at point $x$ reported by an algorithm at time $t$. The former metric quantifies how close a certain point is to the distributionally robust solution while the latter measures the performance of each algorithm from a perspective of the empirical distribution. We repeat the experiment $30$ times and report the average mean and the $96\%$ confidence interval for each evaluation metric. 

The first results are presented in Figure \ref{fig:plot_rho_regret_logistic}. We report over a range of $\rho$ values $\{0.1, 0.3, 0.5, 1.0, 3.0\}$ capturing the degree of conservativeness against the distributional uncertainty. Note that if $\rho > \frac{n-1}{2} = 4.5$, it represents the most conservative case as the $\chi^2$ ball covers the entire $n$-dimensional simplex. We observe from Figure \ref{fig:plot_rho_regret_logistic} that DRBQO significantly outperforms the baselines in this experiment. Also notice that when we increase the conservativeness requirement (i.e., increasing the values of $\rho$), the standard BQO baselines have higher $\rho$-regret. This is because the standard BQO baselines are rigid and do not allow for any conservativeness in the optimization. Therefore, these algorithms converge to the optimum of the spurious Monte Carlo estimate function. 




We highlight the comparative algorithms in the second metric in Figure \ref{fig:mc_values} where we report the value of the empirical expected function $\mathbb{E}_{\hat{P}_n(w)}[f(x,w)]$ at each point $x$ reported by each algorithm at time $t$. Since the BQO baselines are set out to maximize the Monte Carlo estimate function, they achieve higher values in this metric than DRBQO. However, the non-robust solutions returned by the BQO baselines are sub-optimal with respect to the $\rho$-regret in this case, as seen from the corresponding results in Figure \ref{fig:plot_rho_regret_logistic}.

In addition, we evaluate the effectiveness of the selection of $w$ at line \ref{alg:line:w} in Algorithm \ref{pseudo_code:drbqo}. Currently, $w_t$ is selected such that it yields the highest posterior mean given $x_t$. This is to improve exploration in $f$. We compare this selection strategy with the random strategy in which $w_t$ is uniformly selected from $S_n$ regardless of $x_t$. The result is reported in Figure \ref{fig:w_selection}. In this figure, the post-fix RandW denotes the random selection of $w_t$. We observe that random selection of $w_t$ can hurt the convergence of both the standard BQO baselines and DRBQO. Furthermore, the selection of $w_t$ for the maxium posterior variance (line \ref{alg:line:w} of Algorithm \ref{pseudo_code:drbqo}) in DRBQO is also meaningful in proving Theorem \ref{thm:bayesregret}. 


\begin{figure}
    \centering
    \includegraphics[scale=0.9]{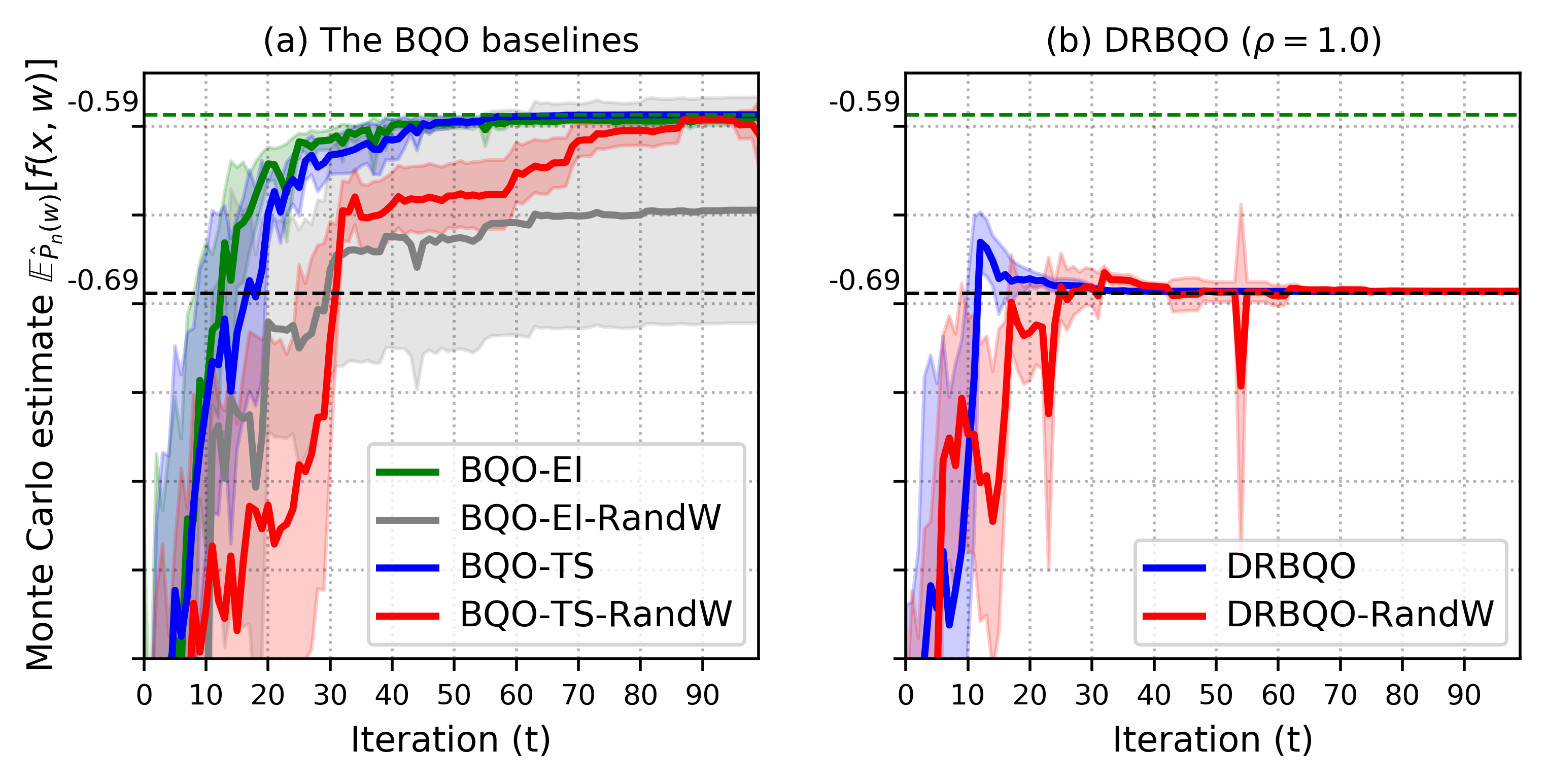}
    \caption{{The effect of different methods of selecting} $w_t$ on the performance of the BQO baselines (a) and DRBQO (b). We observe that random selection of $w_t$ can either slow down or prevent the convergence of both the standard BQO baselines and DRBQO in this experiment.}
    \label{fig:w_selection}
\end{figure}


We provide additional experimental evaluations in synthetic functions. The task in this experiment is to maximize $\mathbb{E}_{w \in \mathcal{N}(0,1)}[f(x,w)]$ where $f$ is a standard synthetic function such as Beale, Eggholder, Hartmann and Levy, $x$ is normalized to the unit cube and $f(x,w) := f(x+w)$. The performance metric used in this experiment is the $\rho$-robust values $\min_{P \in \mathcal{P}_{n,\rho}} \mathbb{E}_{P(w)}[f(x,w)]$. Here we use $n=10$ and $\rho=1.0$. We repeat the experiment $30$ times and report the average mean and the $96\%$ confidence interval for each evaluation
metric. The result is presented in Figure \ref{fig:synthetic_various}. The result shows that DRBQO achieves higher $\rho$-robust values than the baseline methods in all these functions except that in EggHolder function, DRBQO is compatible with BQO-EI but outperforms the other algorithms.

\begin{figure}[h]
    \centering
    \includegraphics[scale=1]{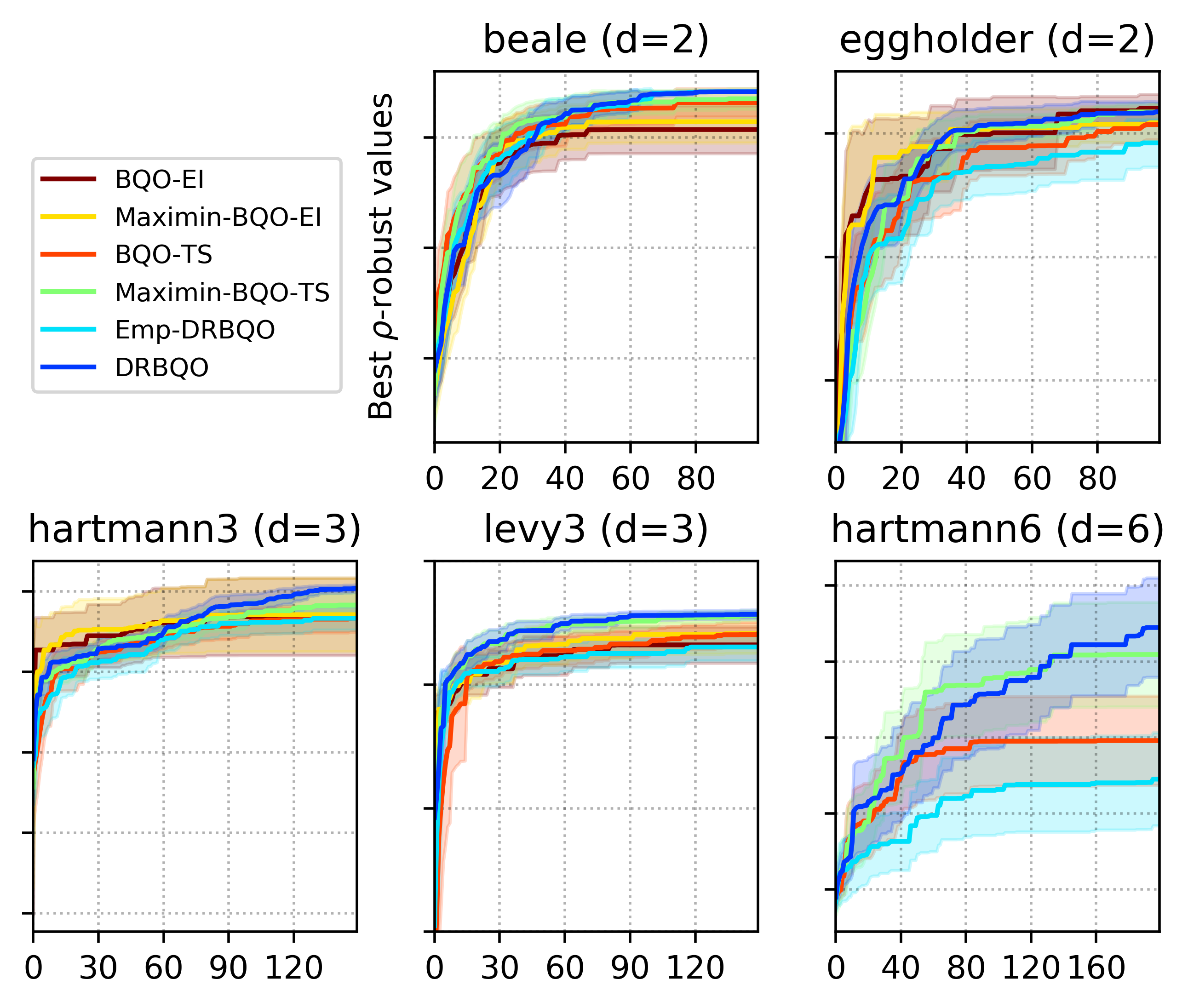}
    \caption{The performance of DRBQO and the baselines on the expected reformulation of various synthetic functions. Here we use $n=10$ and the best $\rho$ values are calculated with $\rho=1.0$. DRBQO achieves higher $\rho$-robust values than the BQO baselines in almost all the tested functions.}
    \label{fig:synthetic_various}
\end{figure}

\subsection{Cross-Validation Hyperparameter Tuning}
A typical real-world problem that possesses the quadrature structure of Equation ($\ref{eq:stochastic_opt}$) is $n$-fold cross-validation hyperparameter tuning. The $n$-fold cross-validation performance can be thought of as a Monte Carlo approximate of the true model performance. Given a fixed learning algorithm associated with a set of hyperparameter $x$, let $f(x,w)$ be an approximate model performance trained on $\mathcal{D} \backslash w$ and evaluated on the validation set $w$ where $\mathcal{D}$ denotes the training data set, $w$ denotes a subset of training points sampled from $\mathcal{D}$, and $\mathcal{D} \backslash w$ denotes everything in $\mathcal{D}$ but not in $w$. Increasing the number of folds reduces the variance in the model performance estimate, but it is expensive to evaluate the cross-validation performance for a large value of $n$. Therefore, a class of Bayesian quadrature optimization methods is beneficial in this case in which we actively select both the algorithm's hyperparameters $x_t$ and a fold $w_t$ to evaluate without the need of training the model in all $n$ folds \citep{DBLP:conf/nips/SwerskySA13}. 

However, the standard BQO methods assume the empirical distribution for each fold and are set out to maximize the average $n$-fold values.  In practice, the average $n$-fold value can be a spurious measure for model performance when there is sufficient discrepancy of the model performance across different folds. This scenario fits well into our distributional uncertainty problem in Equation (\ref{eq:stochastic_opt}) where the fold distribution $P_0(w)$ is unknown in practice. 
In addition, we use a one-hot $n$-dimensional vector to represent each of the $n$ folds. This offers two main advantages: (i) it allows us to leverage the standard kernel such as $k_{SE}$ on the product space $\mathcal{X} \times \Omega$; (ii) it is able to model different covariance between different pairs of folds. For example, the covariance between fold 1 and fold 3 is not necessary the same as that between fold 8 and fold 10 though the fold indicator difference are the same ($2 = 10 - 8 = 3 -1$ in this example). 

We evaluate this experiment on two common machine learning models using the MNIST dataset \citep{LeCun98}: ElasticNet and Convolutional Neural Network (CNN). For ElasticNet, we tune the $l_1$ and $l_2$ regularization hyperparamters, and use the SGDClassificer implementation from the scikit-learn package \citep{pedregosa2012scikitlearn}. For CNN, we use the standard architecture with 2 convolutional layers.
In CNN, we optimize over three following hyperparamters: the learning rate $l$ and the dropout rates in the first and second pooling layers. We used Adam optimizer \citep{DBLP:journals/corr/KingmaB14} in $20$ epochs with the batch size of $128$. 

\begin{table}
\begin{center}
\begin{tabular}{lccc}
\textbf{Methods}  &\textbf{ElasticNet} &\textbf{CNN} \\
\hline
\hline
MTBO       & $8.576 \pm 0.080$      & $1.712 \pm 0.263$ \\
BQO-EI      & $9.166 \pm 0.433$      & $1.634 \pm 0.157$\\
BQO-TS      & $8.625 \pm 0.116$      & $1.820 \pm 0.227$\\ 
\hline
DRBQO($\rho=0.1$)   & $\pmb{8.450} \pm 0.022$   &  $1.968 \pm 0.310$ \\
DRBQO($\rho=0.3$)   &  $\pmb{8.505} \pm 0.082$    & $\pmb{1.495} \pm 0.106$ \\ 
DRBQO($\rho=0.5$)   & $\pmb{8.515} \pm 0.075$   & $1.869 \pm 0.232$\\ 
DRBQO($\rho=1$)   &  $\pmb{8.526} \pm 0.065$  & $\pmb{1.444} \pm 0.071$ \\ 
DRBQO($\rho=3$)   & $\pmb{8.387} \pm 0.013$ &  $\pmb{1.374} \pm 0.066$ \\ 
DRBQO($\rho=5$)   &  $\pmb{8.380} \pm 0.022$ &  $\pmb{1.321} \pm 0.061$ \\ 
\end{tabular}
\end{center}
\caption{Classification error (\%) of ElasticNet and CNN on the MNIST test set tuned by different algorithms. Each bold number in the DRBQO group denotes the classification error that is smaller than any corresponding number in the baseline group. }
\label{table:cls_err}
\end{table}

In addition to the previous baselines in the synthetic experiment, we also consider the multi-task Bayesian optimization (MTBO) \citep{DBLP:conf/nips/SwerskySA13} baseline for this application. MTBO is a standard method for cross-validation hyperparameter tuning. 

In this experiment, we also use $k_{SE}$ kernel defined on the Cartesian product space $\mathcal{X} \times \Omega$ of $x$ and $w$ for all the methods except for MTBO which uses task kernel on the domain of $w$. We initialize $6$ (respectively $9$) initial points and keep these initial points the same for all the algorithms in ElasticNet (respectively CNN). Each of the algorithms are run for $T=60$ (respectively $T=90$) iterations in ElasticNet (respectively CNN). We repeat the experiment $20$ times and report the average and standard deviation values of an evaluation metric. We split the training data into $n=10$ folds and keep these folds the same for all algorithms. We compare DRBQO against the baselines via a practical metric: the classification error in the test set evaluated at the final set of hyperparameters reported by each algorithm at the final step $T$. This metric is a simple but practical measure of the robustness of the hyperparameters over the unknown data distribution $P_0$. The result is reported in Table \ref{table:cls_err}. We observe that DRBQO outperforms the baselines for most of the considered values of $\rho$, especially for large values of $\rho$ (i.e., $\rho \in \{1, 3, 5\}$ in this case).

We present more experimental results for the case of Support Vector Machine (SVM). We use glass and connectionist bench classification datasets from UCI machine learning repository. \footnote{http://archive.ics.uci.edu/ml} The glass dataset contains $214$ samples describing glass properties in $10$ features. The task associated with the glass dataset is to classify an example into one of $7$ classes. The connectionist bench dataset contains $208$ samples each of which has $60$ attributes. The task in the connectionist bench dataset is to classify whether sonar signals bounced off a metal cylinder or a roughly cylindrical rock. Each of the datasets is split into the training and test sets with the ratio of $80:20$. The training set is further split into $n=5$ folds for this experiment. 

Support vector machine (SVM) is a simple machine learning algorithm for classification problems. SVMs with RBF kernels have two hyperparameters: the misclassification trade-off $C$ and the RBF hyperparameter $\gamma$. We tuned these two hyperparameters in this example.  

The performance metric for this experiment is the classification accuracy of SVM in the test set.  We repeat the experiment $30$ times and report the average mean and the $96\%$ confidence interval for each evaluation
metric. The result is presented in Figure \ref{fig:svm}. In this example, DRBQO outperforms the baselines. 

\begin{figure}[h]
    \centering
    \includegraphics[scale=1]{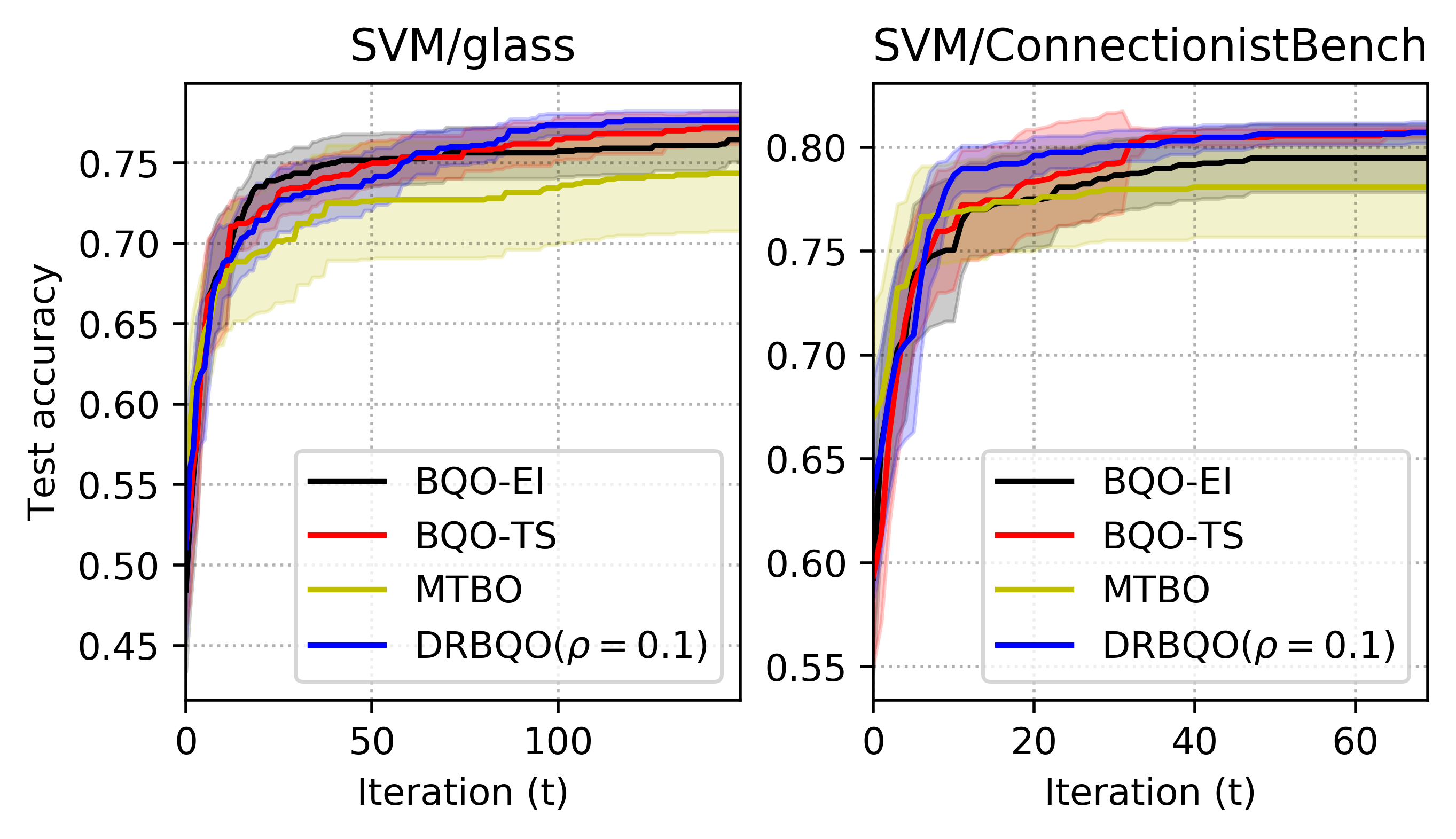}
    \caption{The test classification accuracy of SVM on glass and connectionist bench dataset tuned by DRBQO and the BQO baselines. In this example, we use $n=5$. }
    \label{fig:svm}
\end{figure}

\section{Conclusion}
\label{section:discussion}
In this chapter, we have proposed a posterior sampling based algorithm, namely DRBQO, that efficiently seeks for the robust solutions under the distributional uncertainty in Bayesian quadrature optimization. Compared to the standard BQO algorithms, DRBQO provides a flexibility to control the conservativeness against distributional perturbation. We have demonstrated the empirical effectiveness and characterized the theoretical convergence of DRBQO in sublinear Bayesian regret.

\section{Proofs}
\label{chap3_sec:proof}
In this section, we provide a detailed proof for Theorem \ref{thm:bayesregret} about a sublinear Bayesian regret of the DRBQO algorithm. For simplicity, we focus on the case where the decision space $\mathcal{X}$ and the distributional uncertainty set $\mathcal{P}_{n,\rho}$ are finite. The results can be extended to infinite sets using the discretization trick as in \citep{DBLP:conf/icml/SrinivasKKS10}. We present a series of lemmas that will culminate into the main theorem. 

\textbf{Notations and conventions}. Unless explicitly specified otherwise, we denote  a conditional distribution $P(. | x) \in \mathcal{P}_{n,\rho}, \forall x \in \mathcal{X}$ by $P$, i.e., $P \in  \mathcal{P}_{n,\rho} \times \mathcal{X}$. Recall the definition of the quadrature functional in the main text as
\begin{align*}
    g(f,x,P) = \int P(w|x) f(x,w) dw,
\end{align*}
for any $x \in \mathcal{X}$ and $P \in \mathcal{P}_{n,\rho} \times \mathcal{X}$.
Let $x^* \in \operatorname*{arg\,max}_{x \in \mathcal{X}} \min_{P \in \mathcal{P}_{n,\rho}} \mathbb{E}_{P(w)}[f(x,w)] $, and $P^*(.|x) = \argmin_{P \in \mathcal{P}_{n,\rho}}\mathbb{E}_{P(w)}[f(x,w)], \forall x \in \mathcal{X}$. Since $f$ is a stochastic process (a GP in our case), $x^*$ and $P^*$ are also random variables. The DRBQO algorithm $\pi^{DRBQO}$ maps at a time step $t$ the history $H_t = (x_1, w_1, P_1, ..., x_{t-1},w_{t-1}, P_{t-1})$ to a new decision $(x_t,w_t) \in \mathcal{X} \times S_n$ and conditional distribution $P_t \in \mathcal{P}_{n,\rho} \times \mathcal{X}$ as presented in line \ref{alg:line:dro}-\ref{alg:line:w} of Algorithm \ref{pseudo_code:drbqo}. The practical implementation of Algorithm \ref{pseudo_code:drbqo} samples $(x_t, P_t)$ as follows: 
$x_t \in \operatorname*{arg\,max}_{x \in \mathcal{X}} \min_{P \in \mathcal{P}_{n,\rho}} \mathbb{E}_{P(w)}[\tilde{f}_t(x,w)] $, and $P_t(.|x) = \argmin_{P \in \mathcal{P}_{n,\rho}}\mathbb{E}_{P(w)}[\tilde{f}_t(x,w)], \forall x \in \mathcal{X}$ where $\tilde{f}_t$ is a function sample of $f$ at time $t$ {from its posterior GP}.

\begin{lem}
\label{lem:ps_decompose}
For any sequence of deterministic functions $\{ U_t: \mathcal{X} \times \mathcal{P}_{n,\rho} \times \mathcal{X} \rightarrow \mathbb{R} | t \in \mathbb{N} \}$, 
\begin{align*}
    &\text{BayesRegret}(T, \pi^{DRBQO}) \\
    &= \mathbb{E} \sum_{t=1}^T \left[ U_t(x_t, P_t) - g(f, x_t, P_t)  \right] +\mathbb{E} \sum_{t=1}^T \left[ g(f, x^*, P^*) - U_t(x^*, P^*)    \right],
\end{align*}
for all $T \in \mathbb{N}$. 
\end{lem}

\begin{proof}[Proof of Lemma \ref{lem:ps_decompose}]
Given $H_t$, $\pi^{DRBQO}$ samples $(x_t, P_t)$ according to the probability they are optimal, i.e., $(x_t, P_t) \sim Pr( x^*, P^* | H_t)$. Thus, conditioned on $H_t$, $(x^*, P^*)$ and $(x_t, P_t)$ are identically distributed. As a result, given a deterministic function $U_t$, we have $\mathbb{E}[U_t(x^*, P^*)] = \mathbb{E}[ U_t(x_t, P_t)]$. Therefore, 
\begin{align*}
    &\mathbb{E} \left[ g(f, x^*, P^*) - g(f, x_t, P_t) \right] =  \mathbb{E} \left[ 
    \mathbb{E} \left[ g(f, x^*, P^*) - g(f, x_t, P_t) \right] | H_t
    \right] \\
    &= \mathbb{E} \left[ 
    \mathbb{E} \left[  U_t(x_t, P_t) - g(x_t, P_t) \right] | H_t
    \right]  + \mathbb{E} \left[ 
    \mathbb{E} \left[   g(f, x^*, P^*) - U_t(x^*, P^*)  \right] | H_t
    \right] \\
    &= \mathbb{E} \left[  U_t(x_t, P_t) - g(f, x_t, P_t) \right] 
    + \mathbb{E} \left[   g(f, x^*, P^*) - U_t(x^*, P^*)  \right].
\end{align*}
\end{proof}
\begin{lem}
\label{lem:lem2}
Let $X \sim \mathcal{N}(\mu, \sigma^2)$.  
\begin{enumerate}
    \item For all $\beta \geq 0$, we have 
    \begin{align*}
    Pr\{ |X - \mu| > \beta^{1/2} \sigma \} \leq e^{-\beta/2}. 
\end{align*}

\item If $\mu \leq 0$, then 
\begin{align*}
    \mathbb{E}[ \max\{X,0\} ] = \frac{\sigma}{\sqrt{2\pi}} e^{ \frac{-\mu^2}{2 \sigma^2} }.
\end{align*}

\item For all $a \leq b$, we have 
\begin{align*}
    \mathbb{E}[ X | a < X < b] = \mu - \sigma^2 \frac{p(a) - p(b)}{\phi(a) - \phi(b)}, 
\end{align*}
where $p(x)$ and $\phi(x)$ denote the density function and cumulative distribution function of $X$, respectively. 
\end{enumerate}
\end{lem}
\begin{proof}[Proof of Lemma \ref{lem:lem2}]
These are simple properties of normal distributions. 
\end{proof}

\begin{lem}
\label{lem:cov_inequality}
Given $H_{t}, \forall t \in \mathbb{N}$, let $\sigma_{t}^2(x,w) := C_{t}(x,w; x,w)$  be the variance of $f(x,w)$. Then, for all $P$, all $x$ and for $w^* = \operatorname*{arg\,max}_{w \in S_n} \sigma^2_{t}(x,w)$, we have
\begin{align*}
    \sigma^2_{t}(x, P) = \mathbb{V}[g(f, x, P) | H_t] \leq  \sigma_{t}^2(x, w^*). 
\end{align*}
\end{lem}
\begin{proof}[Proof of Lemma \ref{lem:cov_inequality}]
It follows from a simple property of posterior covariance that 
\begin{align*}
  \sigma^2_{t}(x, P) &= \sum_{w, w'} P(w|x) P(w'|x) C_{t}(x, w; x, w') \leq \sum_{w, w'} P(w|x) P(w'|x) C_{t}(x, w; x, w) \nonumber \\ 
    & \leq \sum_{w, w'} P(w|x) P_t(w'|x) \sigma_{t-1}^2(x, w^*) \nonumber = \sigma_{t}^2(x, w^*). 
\end{align*}
\end{proof}

\begin{lem}
\label{lemma:second_term}
If $U_t(x,P) = \mu_{t-1}(x,P) + \sqrt{\beta_t} \sigma_{t-1}(x,P) $ where 
\begin{align*}
    &\mu_{t-1}(x,P) := \int P(w|x) \mu_{t-1} (x,w) dw, \\ 
    &\sigma^2_{t-1}(x,P) := \int \int C_{t-1}(x,w; x,w') P(w|x) P(w'|x) dw dw', 
\end{align*}
and $\beta_t = 2 \log \frac{(t^2 + 1)|\mathcal{X}| |\mathcal{P}_{n,\rho}|}{\sqrt{2\pi}}$, then 
\begin{align*}
    \mathbb{E} \sum_{t=1}^T [g(f, x^*, P^*) - U_t(x^*, P^*)] \leq 1,
\end{align*}
for all $T \in \mathbb{N}$.
\end{lem}

\begin{proof}[Proof of Lemma \ref{lemma:second_term}]
The trick is to concentrate on the non-negative terms of the expectation. These non-negative terms can be bounded due to the specific choice of upper confidence bound $U_t$. 


Note that for any deterministic conditional distribution $P \in \mathcal{P}_{n,\rho} \times \mathcal{X}$, we have $g(f,x,P) \sim \mathcal{N}(\mu_{t-1}(x,P), \sigma_{t-1}^2(x,P))$, i.e., 
\begin{align*}
    g(f, x,P) - U_t(x,P) \sim \mathcal{N}(-\sqrt{\beta_t} \sigma_{t-1}(x,P),  \sigma_{t-1}^2(x,P) ). 
\end{align*}
It thus follows from Lemma \ref{lem:lem2} that: 
\begin{align*}
    \mathbb{E}[ \max\{  g(f,x,P) - U_t(x,P), 0 \} | H_t]  &= \frac{\sigma_{t-1}(x,P)}{\sqrt{2\pi}} \exp(\frac{-\beta_t}{2}) = \frac{\sigma_{t-1}(x,P)}{(t^2 +1) |\mathcal{X}| |\mathcal{P}_{n,\rho}|}\\
    &\leq \frac{1}{(t^2 +1) |\mathcal{X}| |\mathcal{P}_{n,\rho}|}.
\end{align*}

The final inequality above follows from Lemma \ref{lem:cov_inequality} and from the assumption that $\sigma_0(x,w) \leq 1, \forall x,w$, i.e.,  
\begin{align*}
    \sigma_{t-1}(x,P) \leq \sigma_{t-1}(x,w^*) \leq \sigma_0 (x,w^*) \leq 1, 
\end{align*}
where $w^* = \operatorname*{arg\,max}_{w} C_{t-1}(x,w; x, w)$. 

Therefore, we have 
\begin{align*}
    \mathbb{E} \sum_{t=1}^T [g(f, x^*, P^*) - U_t(x^*, P^*)] &\leq  \mathbb{E} \sum_{t=1}^T \mathbb{E} [ \max \{g(f, x^*, P^*) - U_t(x^*, P^*), 0 \} | H_t] \\ 
    &\leq \mathbb{E} \sum_{t=1}^T \sum_{x \in \mathcal{X}} \sum_{P \in \mathcal{P}_{n,\rho}}  \mathbb{E} [ \max \{g(f, x, P) - U_t(x, P), 0 \}] \\ 
    &\leq \sum_{t=1}^{\infty} \sum_{x \in \mathcal{X}} \sum_{P \in \mathcal{P}_{n,\rho}} \frac{1}{(t^2 +1) |\mathcal{X}| |\mathcal{P}_{n,\rho}|} = 1. 
\end{align*}
\end{proof}

\begin{lem} 
\label{lem:first_term}
{Given} the definition of the maximum information gain $\gamma_T$ as defined in \citep{DBLP:conf/icml/SrinivasKKS10}, we have
\begin{align*}
    \mathbb{E} \sum_{t=1}^T \left[ U_t(x_t, P_t) - g(f,x_t, P_t)  \right]  &\leq \frac{(\sqrt{\beta_T}+B) \sqrt{2\pi}}{|\mathcal{X}| |\mathcal{P}_{n,\rho}|} + 2 \gamma_T \sqrt{(1 + 2\rho)n} (1 + \sigma^{-2})^{-1} \\ 
    &+ 2\sqrt{ T \gamma_T (1 + \sigma^{-2})^{-1} \log \frac{(1+ T^2) |\mathcal{X}| |\mathcal{P}_{n,\rho}|}{ \sqrt{2\pi}}  } ,
\end{align*}
for all $T \in \mathbb{N}$. 
\end{lem}

\begin{proof}[Proof of Lemma \ref{lem:first_term}]
Now we bound the first term 
\begin{align*}
    L &:= \mathbb{E} \sum_{t=1}^T \left[ U_t(x_t, P_t) - g(f, x_t, P_t)  \right] = \mathbb{E} \sum_{t=1}^T \mathbb{E} [J(x_t, H_t) | x_t, H_t], 
\end{align*}
where
\begin{align*}
    J(x_t, H_t) = \mathbb{E}[  U_t(x_t, P_t) - g(f, x_t, P_t) | H_t, x_t ]. 
\end{align*}

While the second term of the Bayesian regret of DRBQO can be bounded as in Lemma \ref{lemma:second_term} by adopting the techniques from \citep{DBLP:journals/mor/RussoR14}, bounding $L$ in DRBQO is non-trivial. This is because $P_t(.|x)$ is also a random process on the simplex given $H_t$. Thus, $g(f, x_t, P_t) | H_t$ does not follow a GP as in the standard Quadrature formulae. In addition, we do not have a closed form of $\mathbb{E}[ g(f, x_t, P_t) | H_t ]$. We overcome this difficulty by decomposing $J$ into several terms that can be bounded more easily and leveraging the mild assumptions of $f$ in the problem setup. 

Given $(H_t, x_t)$, we are interested in bounding $J(x_t, H_t)$. The main idea for bounding this term is that we decompose the range $\mathbb{R}$ of the random variable $f(x_t,w), \forall w$ into three disjoint sets: 
\begin{align*}
    A_t(w) &= \bigg \{ f(x_t, w) \bigg| |f(x_t, w) - \mu_{t-1}(x_t, w)| \leq \sqrt{\beta_t} \sigma_{t-1}(x_t, w) \bigg\} \\ 
    B_t(w) &= \bigg \{ f(x_t, w) \bigg|  \mu_{t-1}(x_t, w) - f(x,w) >  \sqrt{\beta_t} \sigma_{t-1}(x_t, w) \bigg\} \\
    C_t(w) &= \bigg \{ f(x_t, w) \bigg|  \mu_{t-1}(x_t, w) - f(x,w) <  -\sqrt{\beta_t} \sigma_{t-1}(x_t, w) \bigg\}, 
\end{align*}
for all $w \in \Omega$. Note that $A_t(w) \cup B_t(w) \cup C_t(w) = \mathbb{R},  \forall w$.  We also denote $\bar{A}_t(w) = \mathbb{R} \backslash A_t(w) = B_t(w) \cup C_t(w), \forall w$. 

Since $f$ is bounded on $A_t$, there exists $P^*_t$ such that
\begin{align*}
    P^*_t(.|x) = \operatorname*{arg\,max}_{P \in \mathcal{P}_{n,\rho}} \{ U_t(x,P) - g(f, x,P)| f \in A_t \}, 
\end{align*}
for all $x \in \mathcal{X}$. 

Using the equation above, we decompose $J(x_t, H_t)$ as
\begin{align*}
    &J(x_t, H_t) = \mathbb{E}[  U_t(x_t, P_t) - g(f, x_t, P_t)| H_t, x_t] \\
    &=  \mathbb{E}_{f \in A_t}[  U_t(x_t, P_t) - g(f, x_t, P_t)| H_t, x_t] + \mathbb{E}_{f \in \bar{A}_t}[  U_t(x_t, P_t) - g(f, x_t, P_t)| H_t, x_t] \\
    &\leq \mathbb{E}_{f \in A_t}[  U_t(x_t, P^*_t) - g(f, x_t, P^*_t)| H_t, x_t] 
    + \mathbb{E}_{f \in \bar{A}_t}[  U_t(x_t, P_t) - g(f, x_t, P_t)| H_t, x_t] \\
    &= \mathbb{E}[  U_t(x_t, P^*_t) - g(f, x_t, P^*_t)| H_t, x_t] + \mathbb{E}_{f \in \bar{A}_t}[  U_t(x_t, P_t) - g(f, x_t, P_t)| H_t, x_t] \\
    &- \mathbb{E}_{f \in \bar{A}_t}[  U_t(x_t, P^*_t) - g(f, x_t, P^*_t)| H_t, x_t] \\ 
    &= J_1 + J_2 + J_3,
\end{align*}
where 
\begin{align*}
    J_1 &= \mathbb{E}[  U_t(x_t, P^*_t) - g(f, x_t, P^*_t)| H_t, x_t], \\
    J_2 &= \mathbb{E}_{f \in \bar{A}_t}[  U_t(x_t, P_t) - g(f, x_t, P_t)| H_t, x_t], \\ 
    J_3 &= \mathbb{E}_{f \in \bar{A}_t}[  g(f, x_t, P^*_t) - U_t(x_t, P^*_t)| H_t, x_t]. 
\end{align*}

It follows from Lemma \ref{lem:cov_inequality} and from the selection of $w_t$ for the highest posterior variance in the DRBQO algorithm (Algorithm \ref{pseudo_code:drbqo}) that for all $P$, we have
\begin{align*}
    \sigma^2_{t-1}(x_t, P) &= \sum_{w, w'} P(w|x) P(w'|x) C_{t-1}(x_t, w; x_t, w') 
    \leq \sigma_{t-1}^2(x_t, w_t). 
\end{align*}
Note that given $(H_t, x_t)$, $w_t$ is deterministic.  

For $J_1$, we have
\begin{align*}
    J_1 &= \mathbb{E}[  U_t(x_t, P^*_t) - g(f, x_t, P^*_t)| H_t, x_t] \\ 
    &= U_t(x_t, P^*_t) - \mathbb{E}[ g(f, x_t, P^*_t)| H_t, x_t] \\ 
    &= U_t(x_t, P^*_t) - \mu_{t-1}(x_t, P^*_t) \\ 
    &= \sqrt{\beta_t} \sigma_{t-1}(x_t, P^*_t) \\ 
    &\leq \sqrt{\beta_t} \sigma_{t-1}(x_t, w_t).
\end{align*}

For $J_2$, we have
\begin{align*}
    J_2 &= \mathbb{E}_{f \in \bar{A}_t}[  U_t(x_t, P_t) - g(f, x_t, P_t)| H_t, x_t] \\ 
    &=  \mathbb{E}_{f \in \bar{A}_t} [ \sqrt{\beta_t} \sigma_{t-1}(x_t, P_t)| H_t, x_t] +  \mathbb{E}_{f \in B_t} [ \sum_{w} P_t(w) (\mu_{t-1}(x_t, w) - f(x_t, w) | H_t, x_t] \\ 
    &+  \mathbb{E}_{f \in C_t} [ \sum_{w} P_t(w) (\mu_{t-1}(x_t, w) - f(x_t, w)   | H_t, x_t] \\ 
    &\leq \mathbb{E}_{f \in \bar{A}_t}[ \sqrt{\beta_t} \sigma_{t-1}(x_t, w_t) ] +  \mathbb{E}_{f \in B_t} [ \sum_{w} P_t(w) (\mu_{t-1}(x_t, w) - f(x_t, w)    | H_t, x_t] \\ 
    &\leq \sqrt{\beta_t} \sigma_{t-1}(x_t, w_t) e^{-\beta_t/2} +  \mathbb{E}_{f \in B_t} [ \sum_{w} P_t(w) (\mu_{t-1}(x_t, w) - f(x_t, w) | H_t, x_t] \\
    &\leq \sqrt{\beta_t} \sigma_{t-1}(x_t, w_t) e^{-\beta_t/2} \\
    &+ \mathbb{E}_{f \in B_t} \sqrt{\sum_{w}  (\mu_{t-1}(x_t, w) - f(x_t, w))^2 \sum_{w} P_t^2(w)} \\ 
    &\leq \sqrt{\beta_t} \sigma_{t-1}(x_t, w_t) e^{-\beta_t/2} + \mathbb{E}_{f \in B_t} \sqrt{\sum_{w}  (\mu_{t-1}(x_t, w) - f(x_t, w))^2 \frac{1+2\rho}{n}} \\ 
    &\leq \sqrt{\beta_t} \sigma_{t-1}(x_t, w_t) e^{-\beta_t/2} + \sqrt{\frac{1+2\rho}{n}} \mathbb{E}_{f \in B_t} \sum_{w}  (\mu_{t-1}(x_t, w) - f(x_t, w))  \\ 
    &= \sqrt{\beta_t} \sigma_{t-1}(x_t, w_t) e^{-\beta_t/2} + \sqrt{\frac{1+2\rho}{n}} \sum_{w} (\mu_{t-1}(x_t, w) - \mathbb{E}_{f \in B_t}[f(x_t, w)] )  \\
    &= \sqrt{\beta_t} \sigma_{t-1}(x_t, w_t) e^{-\beta_t/2} + \sqrt{\frac{1+2\rho}{n}} \sum_{w} \sigma^2_{t-1}(a_t,w)  \kappa(x_t, w)  \\
    &\leq \sqrt{\beta_t} \sigma_{t-1}(x_t, w_t) e^{-\beta_t/2} + \sqrt{\frac{1+2\rho}{n}} \sum_{w} \sigma^2_{t-1}(a_t,w) \\ 
    &\leq \sqrt{\beta_t} \sigma_{t-1}(x_t, w_t) e^{-\beta_t/2} + \sqrt{n(1+2\rho)} \sigma^2_{t-1}(a_t,w_t),  \\ 
\end{align*}
where 
\begin{align*}
    \kappa(x_t, w) := \frac{p(\mu_{t_1}(x_t,w) - \sqrt{\beta_t}  \sigma_{t-1}(x_t, w) )}{\phi(\mu_{t_1}(x_t,w) - \sqrt{\beta_t}  \sigma_{t-1}(x_t, w) )} \leq 1,
\end{align*}
and $p(.)$ and $\phi(.)$ denote the density function and the cumulative distribution function of the Gaussian distribution $\mathcal{N}(\mu_{t-1}(x_t, w), \sigma^2_{t-1}(x_t, w)), \forall w$. Here, the third inequality follows from the Cauchy-Schwartz inequality; the fourth inequality follows from the bound of the $\chi^2$ ball on the distributions in it; the fifth inequality follows from that fact that $\mu_{t-1}(x,w) - f(x,w) \geq \sqrt{\beta_t} \sigma_{t-1}(x,w) \geq 0$; and the final equation follows from Lemma \ref{lem:lem2}.

For $J_3$, we have 
\begin{align*}
    J_3 &= \mathbb{E}_{f \in \bar{A}_t}[  g(f, x_t, P^*_t) - U_t(x_t, P^*_t)| H_t, x_t] \\ 
    &= \mathbb{E}_{f \in \bar{A}_t}[  g(f, x_t, P^*_t) | H_t, x_t] + \mathbb{E}_{f \in \bar{A}_t} [ -U_t(x_t, P^*_t)] \\ 
    & = \mathbb{E}_{f \in \bar{A}_t}[ -U_t(x_t, P^*_t)] \\ 
    &=  \mathbb{E}_{f \in \bar{A}_t} [- \mu_{t-1}(x_t, P^*_t) - \sqrt{\beta_t} \sigma_{t-1}(x_t, P^*_t) ]\\ 
    &\leq \mathbb{E}_{f \in \bar{A}_t} [- \mu_{t-1}(x_t, P^*_t)]\\
    & \leq \mathbb{E}_{f \in \bar{A}_t} [B] \\
    & \leq B e^{-\beta_t/2}. 
\end{align*}

Here, the second equation follows from the property that $\mathbb{E}_{f\in\bar{A}_t}[f(x_t,w)] = 0$ since $f(x_t,w) \sim \mathcal{N}(\mu_{t-1}(x_t, w), \sigma^2_{t-1}(x_t,w)), \forall w$, and $\bar{A}_t(w)$ is a symmetric region in $\mathbb{R}$ with respect to (but not including) the line $x = \mu_{t-1}(x_t,w), \forall w$; the first inequality follows the non-negativity of the posterior variance $\sigma_{t-1}(x_t, P^*_t)$; the second inequality follows from that the posterior mean $\mu_{t-1}(x, w)$ of a GP is in the RKHS associated with kernel $k$ of the GP, thus is bounded above by $B$ by the mild assumption in the problem setup; and the final inequality follows from Lemma \ref{lem:lem2}.

Combining these results, we can finally bound the first term of the Bayesian regret of DRBQO,
\begin{align*}
    L &= \mathbb{E} \sum_{t=1}^T \mathbb{E} [J(x_t, H_t) | x_t, H_t] \\
    &\leq \mathbb{E} \sum_{t=1}^T \sqrt{\beta_t} \sigma_{t-1}(x_t, w_t) + \mathbb{E} \sum_{t=1}^T B e^{-\beta_t /2 } + \mathbb{E} \sum_{t=1}^T \sqrt{\beta_t} \sigma_{t-1}(x_t, w_t) e^{-\beta_t/2} \\
    &+ \mathbb{E} \sum_{t=1}^T \sqrt{n(1+2\rho)} \sigma^2_{t-1}(a_t,w_t) \\
    &\leq \mathbb{E} \sqrt{T \beta_T} \sqrt{\sum_{t=1}^T \sigma^2_{t-1}(x_t, w_t)} 
    + (B+\sqrt{\beta_T})\sum_{t=1}^{\infty} \frac{\sqrt{2\pi}}{(1+t^2) |\mathcal{X}| |\mathcal{P}_{n,\rho}|} \\
    &+  \sqrt{n(1+2\rho)} \mathbb{E} \sum_{t=1}^T \sigma^2_{t-1}(a_t,w_t) \\
    &\leq \sqrt{T \beta_T} \sqrt{2(1+\sigma^{-2})^{-1} \gamma_T } + \frac{(\sqrt{\beta_T}+B) \sqrt{2\pi}}{|\mathcal{X}| |\mathcal{P}_{n,\rho}|} + \sqrt{n(1 + 2\rho)} 2(1+\sigma^{-2})^{-1} \gamma_T, 
\end{align*}
where $\gamma_T$ is the maximum information gain defined in \citep{DBLP:conf/icml/SrinivasKKS10}, and we also use the following inequality of the maximum information gain 
\begin{align*}
    \sum_{t=1}^T \sigma_{t-1}^2(x_t, w_t) \leq 2(1 + \sigma^{-2})^{-1} \gamma_T.  
\end{align*}
\end{proof}

\begin{proof}[Proof of Theorem \ref{thm:bayesregret}]
Theorem \ref{thm:bayesregret} is a direct consequence of Lemma \ref{lem:ps_decompose}, Lemma \ref{lemma:second_term} and Lemma \ref{lem:first_term}. 
\end{proof}

\newpage 

\chapter{Distributional Reinforcement Learning via Moment Matching \label{chap:four}} 
In this chapter, we consider the second challenge of this thesis about scalable distributional learning in RL. For this, we design a novel theoretically grounded method for distributional RL that eschews the curse of predefined statistics, effectively learns the return distribution in deep RL setting, and is orthogonal to the recent modeling improvements in distributional RL. In particular, we consider the problem of learning a set of probability distributions from the empirical Bellman dynamics in distributional RL, a class of state-of-the-art methods that estimate the distribution, as opposed to only the expectation of the total return. We formulate a method that learns a finite set of statistics from each return distribution via neural networks, as in the distributional RL literature. Existing distributional RL methods however constrain the learned statistics to \textit{predefined} functional forms of the return distribution which is both restrictive in representation and difficult in maintaining the predefined statistics. Instead, we learn \textit{unrestricted} statistics, i.e., deterministic (pseudo-)samples, of the return distribution by leveraging a technique from hypothesis testing known as maximum mean discrepancy (MMD), which leads to a simpler objective amenable to backpropagation.  Our method can be interpreted as implicitly matching all orders of moments between a return distribution and its Bellman target. We establish sufficient conditions for the contraction of the distributional Bellman operator and provide finite-sample analysis for the deterministic samples in distribution approximation. Experiments on the suite of Atari games show that our method outperforms the distributional RL baselines and sets a new record in the Atari games for non-distributed agents. Our framework is also orthogonal to the recent modelling improvements in distributional RL and can be readily incorporated into these models. This chapter is largely based on our AAAI'21 work \citep{Nguyen-Tang_Gupta_Venkatesh_2021}. Our implementation for this work is available at \url{https://github.com/thanhnguyentang/mmdrl}.

\section{Introduction}

A fundamental aspect in reinforcement learning (RL) is the value of an action in a state which is formulated as the expected value of the \textit{return}, i.e., the expected value of the discounted sum of rewards when the agent follows a policy starting in that state and executes that action \cite{sutton1998introduction}.  Learning this expected action-value via Bellman's equation \citep{Bellman1957} is central to value-based RL such as temporal-difference (TD) learning \citep{DBLP:journals/ml/Sutton88}, SARSA \citep{Rummery94on-lineq-learning}, and Q-learning \citep{Watkins92q-learning}. Recently, however, approaches known as \textit{distributional} RL that aim at learning the distribution of the return have shown to be highly effective in practice \citep{DBLP:conf/uai/MorimuraSKHT10,DBLP:conf/icml/MorimuraSKHT10,DBLP:conf/icml/BellemareDM17,DBLP:conf/aaai/DabneyRBM18,DBLP:conf/icml/DabneyOSM18,yang2019fully}.

Despite many algorithmic variants with impressive practical performance \citep{DBLP:conf/icml/BellemareDM17,DBLP:conf/aaai/DabneyRBM18,DBLP:conf/icml/DabneyOSM18,yang2019fully}, they all share the same characteristic that they explicitly learn a set of statistics of \textit{predefined} functional forms to approximate a return distribution. Using predefined statistics can limit the learning due to the statistic constraints it imposes and the difficulty to maintain such predefined statistics. In this chapter, we propose to address these limitations by instead learning a set of \textit{unrestricted} statistics, i.e., deterministic (pseudo-)samples, of a return distribution that can be evolved into any functional form. We observe that the deterministic samples can be deterministically learned to simulate a return distribution by utilizing an idea from statistical hypothesis testing known as maximum mean discrepancy (MMD). This novel perspective requires a careful design of algorithm and a further understanding of distributional RL associated with MMD. 

Leveraging this perspective, we are able to provide a novel algorithm to eschew the predefined statistic limitations in distributional RL and give theoretical understanding of distributional RL within this perspective. Our approach is also conceptually amenable for natural extension along the lines of recent modelling improvements to distributional RL brought by Implicit Quantile Networks (IQN) \citep{DBLP:conf/icml/DabneyOSM18}, and Fully parameterized Quantile Function (FQF) \citep{yang2019fully}.
Our key contributions in this chapter are 

\begin{enumerate}
    \item We provide a novel approach to distributional RL using pseudo-samples via MMD that addresses the limitations in the existing distributional RL;

    \item We provide theoretical understanding of distributional RL within our framework, specifically the contraction property of the distributional Bellman operator and the non-asymptotic convergence of the approximate distribution from deterministic samples; 
    
    \item We demonstrate the scalability and the practical effectiveness of our framework in both tabular RL and large-scale experiments where 
    our method outperforms the standard distributional RL methods and even establishes a new record in the Atari games for non-distributed agents. 
\end{enumerate}



\section{Background and Related Work}

\subsection*{Expected RL}
In a standard RL setting, an agent interacts with an environment via a Markov Decision Process $(\mathcal{S}, \mathcal{A}, R, P, \gamma)$ \citep{puterman2014markov} where $\mathcal{S}$ and $\mathcal{A}$ denote state and action spaces, resp., $\mathcal{R}$ the reward measure, $P(\cdot| s,a)$ the transition kernel measure, and $\gamma \in [0,1)$ a discount factor. A policy $\pi(\cdot|s)$ maps a state to a distribution over the action space. 

Given a policy $\pi$, the discounted sum of future rewards following policy $\pi$ is the random variable 
\begin{align}
    Z^{\pi}(s,a) = \sum_{t=0}^{\infty} \gamma^{t} R(s_t, a_t),
    \label{eq:return_rv}
\end{align}
where $s_0 = s, a_0 = a, s_t \sim P(\cdot | s_{t-1}, a_{t-1})$, $a_t \sim \pi(\cdot | s_t)$, and $R(s_t,a_t) \sim \mathcal{R}(\cdot|s_t, a_t)$. The goal in expected RL is to find an optimal policy $\pi^*$ that maximizes the action-value function $Q^{\pi}(s,a) := \mathbb{E}[Z^{\pi}(s,a)]$. A common approach is to find the unique fixed point $Q^{*} = Q^{\pi^{*}}$ of the Bellman optimality operator \citep{Bellman1957} $T: \mathbb{R}^{\mathcal{S} \times \mathcal{A}} \rightarrow \mathbb{R}^{\mathcal{S} \times \mathcal{A}}$ defined by 
\begin{align*}
    T Q(s,a) := \mathbb{E}[R(s,a)] + \gamma \mathbb{E}_{P} [\max_{a'} Q(s', a')], \forall (s,a).
\end{align*}
A standard approach to this end is Q-learning \citep{Watkins92q-learning} which maintains an estimate $Q_{\theta}$ of the optimal action-value function $Q^{*}$ and iteratively improves the estimation via the Bellman backup 
\begin{align*}
    Q_{\theta}(s,a) \leftarrow \mathbb{E}[R(s,a)] + \gamma \mathbb{E}_{P} [\max_{a'} Q_{\theta}(s', a')].
\end{align*}
Deep Q-Network (DQN) \citep{mnih2015human} achieves human-level performance on the Atari benchmark by leveraging a convolutional neural network to represent $Q_{\theta}$ while using a replay buffer and a target network to update $Q_{\theta}$.

\subsection*{Additional Notations} 
Let $\mathcal{X} \subseteq \mathbb{R}^d$ be an open set. Let $\mathcal{P}(\mathcal{X})$ be the set of Borel probability measures on $\mathcal{X}$. Let $\mathcal{P}(\mathcal{X})^{\mathcal{S} \times \mathcal{A}}$ be the Cartesian product of $\mathcal{P}(\mathcal{X})$ indexed by $\mathcal{S} \times \mathcal{A}$. For any $\alpha \geq 0$, let $\mathcal{P}_{\alpha}(\mathcal{X}) := \{p \in \mathcal{P}(\mathcal{X}): \int_{\mathcal{X}} \|x\|^{\alpha} p(dx) < \infty  \}$. When $d=1$, let $m_n(p) := \int_{\mathcal{X}} x^n p(dx)$ be the $n$-th order moment of a distribution $p \in \mathcal{P}(\mathcal{X})$, and let  
\begin{align*}
    \mathcal{P}_{*}(\mathcal{X}) = \bigg\{p \in \mathcal{P}(\mathcal{X}):  \limsup_{n \rightarrow \infty} \frac{|m_n(p)|^{1/n}}{n} =0 \bigg\}. 
\end{align*}
Note that if $\mathcal{X}$ is a bounded domain in $\mathbb{R}$, then $\mathcal{P}_{*}(\mathcal{X}) = \mathcal{P}(\mathcal{X})$. Denote by $\delta_z$ the Dirac measure, i.e., the point mass, at $z$. Denote by $\Delta_n$ the $n$-dimensional simplex. 


\subsection*{Distributional RL}
Instead of estimating only the expectation $Q^{\pi}$ of $Z^{\pi}$, distributional RL methods \citep{DBLP:conf/icml/BellemareDM17,DBLP:conf/aaai/DabneyRBM18,DBLP:conf/icml/DabneyOSM18,DBLP:conf/aistats/RowlandBDMT18,yang2019fully} explicitly estimate the return distribution $\mu^{\pi}=\text{law}(Z^{\pi})$ as an auxiliary task. Empirically, this auxiliary task has been shown to significantly improve the performance in the Atari benchmark. Theoretically, in the policy evaluation setting, the distributional version of the Bellman operator is a contraction in the $p$-Wasserstein metric \citep{DBLP:conf/icml/BellemareDM17} and Cr\'amer distance \citep{DBLP:conf/aistats/RowlandBDMT18} (but not in total variation distance \citep{Chung1987DiscountedMD}, Kullback-Leibler divergence and Komogorov-Smirnov distance \citep{DBLP:conf/icml/BellemareDM17}). The contraction implies the uniqueness of the fixed point of the distributional Bellman operator. In control settings with tabular function approximations, distributional RL has a well-behaved asymptotic convergence in Cr\'amer distance when the return distributions are parameterized by categorical distributions \citep{DBLP:conf/aistats/RowlandBDMT18}. \citet{bellemare2019distributional} establish the asymptotic convergence of distributional RL in policy evaluation in linear function approximations. \citet{DBLP:conf/aaai/LyleBC19} examine behavioural differences between distributional RL and  expected RL, aligning the success of the former with non-linear function approximations. 

\subsubsection*{Categorical Distributional RL (CDRL)}
CDRL \citep{DBLP:conf/icml/BellemareDM17} approximates a distribution $\eta$ by a categorical distribution $\hat{\eta} = \sum_{i=1}^N \theta_i \delta_{z_i}$ where $z_1 \leq z_2 \leq ... \leq z_N $ is a set of fixed supports and $\{\theta_i\}_{i=1}^N$ are learnable probabilities. The learnable probabilities $\{\theta_i\}_{i=1}^N$ are found in such way that $\hat{\eta}$ is a projection of $\eta$ onto $\{ \sum_{i=1}^N p_i \delta_{z_i}: \{p_i\}_{i=1}^N \in \Delta_N \}$ w.r.t. the Cr\'amer distance \citep{DBLP:conf/aistats/RowlandBDMT18}. In practice, C51 \citep{DBLP:conf/icml/BellemareDM17}, an instance of CDRL with $N=51$, has shown to perform favorably in Atari games. 

\subsubsection*{Quantile Regression Distributional RL (QRDRL)}
QRDRL \citep{DBLP:conf/aaai/DabneyRBM18} approximates a distribution $\eta$ by a mixture of Diracs $\hat{\eta} = \frac{1}{N} \sum_{i=1}^N \delta_{\theta_i}$ where $\{\theta_i\}_{i=1}^N$ are learnable in such a way that $\hat{\eta}$ is a projection of $\eta$ on $\{  \frac{1}{N} \sum_{i=1}^N \delta_{z_i}: \{z_i\}_{i=1}^N \in \mathbb{R}^N \}$ w.r.t. to the 1-Wasserstein distance. Consequently, $\theta_i = F_{\eta}^{-1}( \frac{2i-1}{2N} )$ where $F_{\eta}^{-1}$ is the inverse cumulative distribution function of $\eta$. Since the quantile values $\{F_{\eta}^{-1}( \frac{2i-1}{2N})\}$ at the fixed quantiles $\{\frac{2i-1}{2N}\}$ is a minimizer of an asymmetric quantile loss from quantile regression literature (thus the name QRDRL) and the quantile loss is compatible with stochastic gradient descent (SGD), the quantile loss is used for QRDRL in practice. QR-DQN-1 \citep{DBLP:conf/aaai/DabneyRBM18}, an instance of QRDRL with Huber loss, performs favorably empirically in Atari games. 

\subsubsection*{Implicit Distributional RL} 
Some recent distributional RL methods have made modelling improvements to QRDRL. Two typical improvements are from Implicit Quantile Networks (IQN) \citep{DBLP:conf/icml/DabneyOSM18}, and Fully parameterized Quantile Function (FQF) \citep{yang2019fully}. IQN uses implicit models to represent the quantile values $\{\theta_i\}$ in QRDRL, i.e., instead of being represented by fixed network outputs, $\{\theta_i\}$ are the outputs of a differentiable function (e.g., neural networks) on the samples from a base sampling distribution (e.g., uniform). FQF further improves IQN by optimizing the locations of the base samples for IQN, instead of using random base samples as in IQN, i.e., both quantiles and quantile values are learnable in FQF. 



\subsection*{Predefined Statistic Principle}
Formally, a statistic is any functional $\zeta: \mathcal{P}(\mathcal{X}) \rightarrow \mathbb{R}$ that maps a distribution $p \in \mathcal{P}(\mathcal{X})$ to a scalar $\zeta(p)$, e.g., the expectation $\zeta(p) = \int_{\mathcal{X}} x p(dx)$ is a common statistic in RL. Here, we formally refer to a \textit{predefined statistic} as the one whose functional form is specified before the statistic is learned. In contrast, an \textit{unrestricted statistic} does not subscribe to any specific functional form (e.g., the median of a distribution $\eta$ is a predefined statistic as its functional form is predefined via $F^{-1}_{\eta}(\frac{1}{2})$ while any empirical sample $z \sim \eta$ can be considered an unrestricted statistic of $\eta$). 

Though CDRL and QRDRL are two different variants of distributional RL methodology, they share a unifying characteristic that they both explicitly learn a finite set of predefined statistics, i.e., statistics of \textit{predefined} functional forms \citep{DBLP:conf/icml/RowlandDKMBD19}. We refer to this as predefined statistic principle. This is clear for QRDRL as the statistics to be learned about a distribution $\eta$ are  $\{\zeta_1, ..., \zeta_N\}$ where 
\begin{align*}
    \zeta_i(\eta) := F^{-1}_{\eta}(\frac{2i-1}{N}), \forall i \in \{1,...,N\}.
\end{align*}
It is a bit more subtle for CDRL. It can be shown in \citep{DBLP:conf/icml/RowlandDKMBD19} that CDRL is equivalent to learning the statistics $\{\zeta_1, ..., \zeta_{N-1}\}$ where 
\begin{align*}
    \zeta_i(\eta) := \mathbb{E}_{Z \sim \eta} \left[ 1_{\{Z < z_i\}} + 1_{\{z_i \leq Z < z_{i+1}\}} \frac{z_{i+1}-Z}{z_{i+1} -z_i} \right], \forall i. 
\end{align*}

Learning predefined statistics as in CDRL and QRDRL however can suffer two limitations in (i) statistic representation and (ii) difficulty in maintaining the predefined statistics. Regarding (i), given the same fixed budget of $N$ statistics to approximate a return distribution $\eta$, CDRL restrictively associates the statistic budget to $N$ fixed supports $\{z_i\}_{i=1}^N$ while QRDRL constrains the budget to $N$ quantile values at specific quantiles. Instead, the statistic budget should be freely learned into any form as long as it could simulate the target distribution $\eta$ sensibly.  Regarding (ii), the fixed supports in CDRL require a highly involved projection step to be able to use KL divergence as the Bellman backup changes the distribution supports; QRDRL requires that the statistics must satisfy the constraints for valid quantile values at specific quantiles, e.g., the statistics to be learned are order statistics. In fact, QR-DQN  \citep{DBLP:conf/aaai/DabneyRBM18}, a typical instance of QRDRL, implicitly maintains the order statistics via an asymmetric quantile loss but still does not guarantee the monotonicity of the obtained quantile estimates. A further notice regarding (ii) recognized in \citep{DBLP:conf/icml/RowlandDKMBD19} is that since in practice we do not observe the environment dynamic but only samples of it, a naive update to learn the predefined statistics using such samples can collapse the approximate distribution due to the different natures of samples and statistics (in fact, \citet{DBLP:conf/icml/RowlandDKMBD19} proposes imputation strategies to overcome this problem). Instead, the statistics to be learned should be free of all such difficulties to reduce the learning burden. 

One might say that IQN/FQF (discussed in the previous subsection) can help QRDRL overcome these limitations. While the modeling improvements in IQN/FQF are practically effective, IQN/FQF however still embrace the predefined statistic principle above as they built upon QRDRL with an improved modelling capacity. In this work we propose an alternative approach to distributional RL that directly eschews the predefined statistic principle used in the prior distributional RL methods, i.e., the finite set of statistics in our approach can be evolved into any functional form and thus also reduces the need to maintain any statistic constraints. If we informally view improvements to CDRL/QRDRL into two dimensions: either modelling dimension or statistic dimension, IQN/FQF lie in the modelling dimension while our work belongs to the statistic dimension. We notice that this does not necessarily mean one approach is better than the other, but rather two orthogonal approaches where the modelling improvements in IQN/FQF can naturally apply to our work to further improve the modeling capacity. We leave these modeling extensions to the future work and focus the present work only on unrestricted statistics with the simplest modelling choice as possible. 

\begin{rem} We further clarify that a specified functional form in the predefined statistic principle section means that the parametric form $\zeta$, which possibly depends on some real parameters, is fully specified, e.g., quantile values $F^{-1}_{\eta}(\tau)$ (which depends on quantile level $\tau$) are predefined statistics. In this sense, FQF/IQN still have predefined statistics as they use quantile values to approximate a distribution. We emphasize that the flexibility of $\tau$ in FQF/IQN does not solve the problem of predefined statistics but rather only gives a finer-grained approximation by using more $\tau$ which in turn increases the budget of $N$ statistics. We remark that a flexible increase of $N$ (which can give a finer approximation in both predefined and unrestricted statistics) is not our focus in this chapter; thus, we find it useful to keep the same fixed budget of $N$ statistics when comparing predefined and unrestricted statistics. We also remark that a general empirical sample can be considered an unrestricted statistic as it does not subscribe to any predefined parametric functional form. As a concrete example, let us consider the task of approximating a distribution with only one (learnable) statistic: One approach targets the median, and the other uses an unrestricted statistic. Eventually, the former approach should converge its statistic to the median as it subscribes to this predefined functional form while the latter approach can evolve its statistic into any sample equally (not necessarily the median) as long as the sample can simulate the distribution in a certain sense (in our case, to match the moments of the empirical distribution with those of the target distribution). 
\end{rem}






\section{Main Framework}
\subsection{Maximum Mean Discrepancy} 
Let $\mathcal{F}$ be a reproducing kernel Hilbert space (RKHS) associated with a continuous kernel $k(\cdot, \cdot)$ on $\mathcal{X}$. Consider $p, q \in \mathcal{P}(\mathcal{X})$, and let $Z$ and $W$ be two random variables with distributions $p$ and $q$, respectively. The maximum mean discrepancy (MMD) \citep{DBLP:journals/jmlr/GrettonBRSS12} between $p$ and $q$ is defined as 
\begin{align*}
    \text{MMD}(p, q; \mathcal{F}) &:= \sup_{f \in \mathcal{F}: \|f\|_{\mathcal{F}} \leq 1} \left(\mathbb{E} [f(Z)] - \mathbb{E} [f(W)] \right) =  \| \psi_{p} - \psi_q \|_{\mathcal{F}} \\ 
    &= \bigg(  \mathbb{E} [k(Z,Z')] + \mathbb{E}[k(W,W')] - 2 \mathbb{E} [k(Z,W)]  \bigg)^{1/2}
\end{align*}
where $\psi_p := \int_{\mathcal{X}} k(x, \cdot) p(dx)$ is the Bochner integral, i.e., the mean embedding of $p$ into $\mathcal{F}$ \citep{DBLP:conf/alt/SmolaGSS07}, and $Z'$ (resp. $W'$) is a random variable with distribution $p$ (resp. $q$) and is independent of $Z$ (resp. $W$). In sequel, we interchangeably refer to MMD by $\text{MMD}(p,q; \mathcal{F}), \text{MMD}(p,q;k)$, or $\text{MMD}(p,q)$ if the context is clear. 




\subsubsection*{Empirical Approximation}
Given empirical samples $\{z_i\}_{i=1}^N \sim p$ and $\{w_i\}_{i=1}^M \sim q$, MMD admits a simple empirical estimate as  
\begin{align*}
    &\text{MMD}_{b}^2(\{z_i\}, \{w_i\}; k) = \frac{1}{N^2} \sum_{i,j} k(z_i, z_j) + \frac{1}{M^2} \sum_{i,j} k(w_i, w_j) - \frac{2}{N M} \sum_{i,j} k(z_i, w_j). 
\end{align*}
Though there is also a simple unbiased estimate of $\text{MMD}$, the \emph{biased estimate} $\text{MMD}_{b}$ has smaller variance in practice and thus is adopted in our work. 


\subsection{Problem Setting}
Consider $d=1$. For any policy $\pi$, let $\mu^{\pi}=\text{law}(Z^{\pi})$ be the law (distribution) of the return r.v. $Z^{\pi}$ as defined in Equation (\ref{eq:return_rv}). The distributional Bellman operator $\mathcal{T}^{\pi}$ \citep{DBLP:conf/icml/BellemareDM17} specifies the relation of different return distributions across state-action pairs along the Bellman dynamic; that is, for any $\mu \in \mathcal{P}(\mathcal{X})^{\mathcal{S} \times \mathcal{A}}$, and any $(s,a) \in \mathcal{S} \times \mathcal{A}$,
\begin{align*}
    \mathcal{T}^{\pi} \mu(s,a) := \nonumber \int_{\mathcal{S}} \int_{\mathcal{A}} \int_{\mathcal{X}} (f_{\gamma,r})_{\#} \mu(s',a') \mathcal{R}(dr|s,a) \pi(d a'|s') P(ds'|s,a),
\end{align*}
where $f_{\gamma,r}(z) := r + \gamma z, \forall z$ and $(f_{\gamma,r})_{\#} \mu(s',a')$ is the pushforward measure of $\mu(s',a')$ by $f_{\gamma,r}$. Note that $\mu^{\pi}$ is the fixed point of $\mathcal{T}^{\pi}$, i.e., $\mathcal{T}^{\pi} \mu^{\pi} = \mu^{\pi}$. We are interested in the  problem of learning $\mu^{\pi} \in \mathcal{P}(\mathcal{X})^{\mathcal{S} \times \mathcal{A}}$ via the distributional Bellman operator $\mathcal{T}^{\pi}$. 

\subsection{Algorithmic Approach}
In practical settings, we must approximate the return distribution via a finite set of statistics as the space of Borel probability measures is infinite-dimensional. Let $Z_{\theta}(s,a) := \{Z_{\theta}(s,a)_i\}_{i=1}^N$ be a set of parameterized statistics of $\mu^{\pi}(s,a)$ where $\theta$ represents the parameters of the model, e.g., neural networks. Instead of restricting $Z_{\theta}(s,a)$ to predefined statistic functionals, we model unrestricted statistics, i.e., deterministic samples where each $Z_{\theta}(s,a)_i$ can be evolved into any form of statistics and we use the Dirac mixture $\hat{\mu}_{\theta}{(s,a)} = \frac{1}{N} \sum_{i=1}^N \delta_{Z_{\theta}(s,a)_i}$ to approximate $\mu^{\pi}(s,a)$. We refer to the deterministic samples $Z_{\theta}(s,a)$ as particles, and our goal is reduced into learning the particles $Z_{\theta}(s,a)$ to approximate $\mu^{\pi}(s,a)$. To this end, the particles $Z_{\theta}(s,a)$ is deterministically evolved to minimize the MMD distance between the approximate distribution and its distributional Bellman target.
Algorithm \ref{alg:mmd-drl} below presents the generic update in our approach, namely MMDRL. 

\begin{algorithm}[t]
\caption{Generic moment matching distributional RL}
\label{alg:mmd-drl}
\begin{algorithmic}[1]
\STATE {\bfseries Input:} Number of particles $N$, kernel $k$, discount factor $\gamma \in [0,1]$, sample transition $(s, a, r, s')$.

\IF{policy evaluation}
\STATE $a^* \sim \pi(\cdot|s')$

\ELSIF{control setting}
\STATE $a^* \leftarrow \operatorname*{arg\,max}_{a' \in \mathcal{A}}  \frac{1}{N}\sum_{i=1}^N Z_{\theta}(s', a')_i$ 

\ENDIF

\STATE $\hat{T} Z_i \leftarrow r + \gamma Z_{\theta^{-}}(s', a^*)_i, \forall 1 \leq i \leq N$

\STATE {\bfseries Output:} $\text{MMD}_{b}^2\left(\{Z_{\theta}(s,a)_i\}_{i=1}^N, \{\hat{T} Z_i\}_{i=1}^N; k \right)$.
\end{algorithmic}
\end{algorithm}

\subsubsection*{Intuition} The MMDRL reduces into the standard TD or Q-learning when $N=1$. For $N > 1$, the objective $\text{MMD}_{b}$ when used with SGD and Gaussian kernels $k(x,y)=\exp(-|x-y|^2/h)$ contributes in two ways: (i) The term  $\frac{1}{N^2} \sum_{i,j} k(Z_{\theta}(s,a)_i, Z_{\theta}(s,a)_j)$ serves as a repulsive force that pushes the particles $\{Z_{\theta}(s,a)_i\}$ away from each other, preventing them from collapsing into a single mode, with force proportional to $\frac{2}{h}e^{-(Z_{\theta}(s,a)_i - Z_{\theta}(s,a)_j)^2/h}|Z_{\theta}(s,a)_i - Z_{\theta}(s,a)_j|$; (ii) the term $-\frac{2}{N^2} \sum_{i,j} k(Z_{\theta}(s,a)_i,\hat{T}Z_j)$ acts as an attractive force which pulls the particles $\{Z_{\theta}(s,a)_i\}$ closer to their target particles $\{\hat{T} Z_i\}$. This can also be intuitively viewed as a two-sample counterpart to Stein point variational inference \citep{DBLP:conf/nips/LiuW16,DBLP:conf/icml/ChenMGBO18}. 

\subsubsection*{Particle Representation}
We can easily extend MMDRL to DQN-like architecture to create a novel deep RL, namely MMDQN. In this work, we explicitly represent the particles $\{Z_{\theta}(s,a)_i\}_{i}^N$ in MMDQN via fixed $N$ network outputs as in QR-DQN \citep{DBLP:conf/aaai/DabneyRBM18} for simplicity. 

We emphasize that modeling improvements from IQN \citep{DBLP:conf/icml/DabneyOSM18} and FQF \citep{yang2019fully} can be naturally applied to MMDQN: we can implicitly generate $\{Z_{\theta}(s,a)_i\}_{i}^N$ via applying a neural network function to $N$ samples of a base sampling distribution (e.g., normal or uniform distribution) as in IQN, or we can use the proposal network in FQF to learn the weights of each Dirac components in MMDQN instead of using equal weights $1/N$.

\subsection{Theoretical Analysis}

Here we provide theoretical understanding of MMDRL. Before that, we define the notion of supremum MMD, a MMD counterpart to the supremum Wasserstein in \citep{DBLP:conf/icml/BellemareDM17}, to work on $\mathcal{P}(\mathcal{X})^{\mathcal{S} \times \mathcal{A}}$. 
\begin{defn}
Supremum MMD is a functional $\mathcal{P}(\mathcal{X})^{\mathcal{S} \times \mathcal{A}} \times \mathcal{P}(\mathcal{X})^{\mathcal{S} \times \mathcal{A}} \rightarrow \mathbb{R}$ defined by 
\begin{align*}
    \text{MMD}_{\infty}(\mu, \nu; k) := \sup_{(s,a) \in \mathcal{S} \times \mathcal{A}} \text{MMD}(\mu(s,a), \nu(s,a); k)
\end{align*}
for any $\mu, \nu \in \mathcal{P}(\mathcal{X})^{\mathcal{S} \times \mathcal{A}}$. 
\end{defn}

We are concerned with the following questions:
\begin{enumerate}
    \item \textbf{Metric property}: When does $\text{MMD}_{\infty}$ induce a metric on $\mathcal{P}(\mathcal{X})^{\mathcal{S} \times \mathcal{A}}$?
    \item \textbf{Contraction property}: When is $\mathcal{T}^{\pi}$ a contraction in $\text{MMD}_{\infty}$? 
    
    \item \textbf{Convergence property}: How fast do the particles returned by minimizing MMD approach the target distribution it approximates?
\end{enumerate}
The metric property in the first question ensures that $\text{MMD}_{\infty}$ is a meaningful test to distinguish two return distributions on $\mathcal{P}(\mathcal{X})^{\mathcal{S} \times \mathcal{A}}$.
The contraction property in the second question guarantees that following from Banach's fixed point theorem \citep{BanachSURLO}, $\mathcal{T}^{\pi}$ has a unique fixed point which is $\mu^{\pi}$. In addition, starting with an arbitrary point $\mu_0 \in \mathcal{P}(\mathcal{X})^{\mathcal{X} \times \mathcal{A}}$, $\mathcal{T}^{\pi} \circ \mathcal{T}^{\pi} \circ ... \circ \mathcal{T}^{\pi} \mu_{0}$ converges at an exponential rate to $\mu^{\pi}$ in $\text{MMD}_{\infty}$. We provide sufficient conditions to answer the first two questions and derive the convergence rate of the optimal particles in approximating a target distribution for the third question. In short, the first two properties highly depend on the underlying kernel $k$, and the particles returned by minimizing MMD enjoy a rate $O(1/\sqrt{n})$ regardless of the dimension $d$ of the underlying space $\mathcal{X}$. 

\subsubsection*{Metric Property}
\begin{prop}
Let $\tilde{\mathcal{P}}(\mathcal{S}) \subseteq \mathcal{P}(\mathcal{S})$ be some (Borel) subset of the space of the Borel probability measures. If $\text{MMD}$ is a metric on $\tilde{\mathcal{P}}(\mathcal{X})$, then $\text{MMD}_{\infty}$ is also a metric on $\tilde{\mathcal{P}}(\mathcal{X})^{\mathcal{S} \times \mathcal{A}}$. 
\label{prop:metric_over_state_action_space}
\end{prop}


Theorem \ref{theorem:metric_property} below provides \emph{sufficient conditions} for $\text{MMD}$ to induce a metric on $\tilde{\mathcal{P}}(\mathcal{X})$.



\begin{thm}
We have 
\begin{enumerate}
    \item 
    If the underlying kernel $k$ is \textbf{characteristic} (i.e., the induced Bochner integral $\psi_p$ is injective), e.g., Gaussian kernels, then MMD is a metric on $\mathcal{P}(\mathcal{X})$ \citep{DBLP:conf/nips/FukumizuGSS07,DBLP:journals/jmlr/GrettonBRSS12}. 

    \item 
    Define \textbf{unrectified} kernels $k_{\alpha}(x,y) := -\|x - y \|^{\alpha}, \forall \alpha \in \mathbb{R}, \forall x,y \in \mathcal{X}$. Then $\text{MMD}(\cdot, \cdot; k_{\alpha})$ is a metric on $\mathcal{P}_{\alpha}(\mathcal{X})$ for all $\alpha \in (0,2)$ but not a metric for $\alpha=2$ \citep{szekely2003statistics}. 
    
     \item MMD associated with the so-called \textbf{exp-prod} kernel $k(x,y) = \exp( \frac{xy}{\sigma^2})$ for any $\sigma > 0$ is a metric on $\mathcal{P}_*(\mathcal{X})$.  
    
\end{enumerate}
\label{theorem:metric_property}
\end{thm}


\subsubsection*{Contraction Property}
We analyze the contraction of $\mathcal{T}^{\pi}$ for several important classes of kernels. One such class is shift invariant and scale sensitive kernels. A kernel $k(\cdot, \cdot)$ is said to be \textit{shift invariant} if $k(x+c, y + c) = k(x,y), \forall x,y,c \in \mathcal{X}$; it is said to be 
 \textit{scale sensitive} with order 
$\alpha >0$ if $k(c x, c y) = |c|^{\alpha} k(x,y), \forall x,y \in \mathcal{X}$ and $c \in \mathbb{R}$. For example, the unrectified kernel $k_{\alpha}$ considered in Theorem \ref{theorem:metric_property} is both shift invariant and scale sensitive with order $\alpha$ while Gaussian kernels are only shift invariant. 

\begin{thm} We have 
\begin{enumerate}
    \item If the underlying kernel is $k = \sum_{i \in I} c_i k_i$ where each component kernel $k_i$ is both shift invariant and scale sensitive with order $\alpha_i > 0$, $c_i \geq 0$, and $I$ is a (possibly infinite) index set, then $\mathcal{T}^{\pi}$ is a ${\gamma}^{ \alpha_* /2 }$-contraction in $\text{MMD}_{\infty}$ where $\alpha_* := \min_{i \in I} \alpha_i$.  
    
    \item $\mathcal{T}^{\pi}$ is \textbf{not} a contraction in $\text{MMD}_{\infty}$ associated with either Gaussian kernels or exp-prod kernels $k(x,y) = \exp( \frac{xy}{\sigma^2})$. 
    
    
\end{enumerate}
\label{theorem:contraction_property}
\end{thm}

We present the detailed proofs for Theorem \ref{theorem:metric_property} and Theorem \ref{theorem:contraction_property} in Section \ref{chap4_section_proof}. 


\subsubsection*{Practical Consideration}
Theorem \ref{theorem:contraction_property} provides a negative result for the commonly used Gaussian kernel. In practice, however, we found that Gaussian kernels can promote to match the moments between two distributions and  have better empirical performance as compared to the other kernels analyzed in this section. In fact, MMD associated with Gaussian kernels $k(x,y) = \exp( -(x-y)^2 /(2 \sigma^2))$ can be decomposed into 
\begin{align*}
    \text{MMD}^2(\mu, \nu; k) = \sum_{n=0}^{\infty} \frac{1}{\sigma^{2n} n!} \left( \tilde{m}_n(\mu) - \tilde{m}_n(\nu) \right)^2
\end{align*}
where $\tilde{m}_n(\mu) =  \mathbb{E}_{x \sim \mu} \left[ e^{-x^2 /(2 \sigma^2 )} x^n \right]$, and similarly for $\tilde{m}_n(\nu)$. This indicates that MMD associated with Gaussian kernels approximately performs moment matching (scaled with a factor $e^{-x^2 /(2 \sigma^2 )}$ for each moment term).

\subsubsection*{Convergence Rate of Distribution Approximation} 
We justify the goodness of the particles obtained via minimizing MMD in terms of approximating a target distribution. 

\begin{thm}
Let $\mathcal{X} \subseteq \mathbb{R}^d$ and $P \in \mathcal{P}(\mathcal{X})$. For any $n \in \mathbb{N}$, let $\{x_i\}_{i=1}^n \subset \mathcal{X}$ be a set of $n$ deterministic points such that $\{x_i\}_{i=1}^n \in \arginf_{\{\tilde{x}_i\}_{i=1}^n} \text{MMD}(\frac{1}{n} \sum_{i=1}^n \delta_{\tilde{x}_i}, P; \mathcal{F})$. Then, $P_n := \frac{1}{n} \sum_{i=1}^n \delta_{x_i}$ converges to $P$ at a rate of $O(1/\sqrt{n})$ in the sense that for any function $h$ in the unit ball of $\mathcal{F}$, we have 
\begin{align*}
    \bigg|\int_{\mathcal{X}} h(x) dP_n(x) - \int_{\mathcal{X}} h(x) dP(x)\bigg| = O(1/\sqrt{n}).
\end{align*}
\label{theorem:convergence_rate}
\end{thm}

\begin{rem}
MMD enjoys a convergence rate of $O(n^{-1/2})$ regardless of the underlying dimension $d$ while $1$-Wasserstein distance has a convergence rate of $O(n^{-1/d})$ (if $d > 2$) \citep{fournier2015rate}, which is slower for large $d$.  
\end{rem}

\begin{rem}
Theorem \ref{theorem:convergence_rate} is concerned with the convergence of the optimal deterministic particles uniquely arisen in our problem setting where we deterministically evolve a set of particles to approximate a distribution in MMD. In other words, Theorem \ref{theorem:convergence_rate} does not fully analyze Algorithm \ref{alg:mmd-drl} but addresses one relevant yet important aspect: if a set of deterministic particles are evolved to simulate a distribution in MMD, how good is the approximation. This is different from the conventional setting in MMD which are often concerned with the convergence of empirical MMD derived from i.i.d. samples of each component distribution \citep{DBLP:journals/jmlr/GrettonBRSS12} (though we leverage similar proof techniques). We also remark that in Theorem \ref{theorem:convergence_rate}, we assume the attainability of the infimum but do not specify a practical algorithm to solve this infimum. While in practice, we use neural networks to represent the deterministic particles and use SGD to solve this optimization problem (as in MMDQN), the related literature of kernel herding can in fact provide a different approach with an improved analysis. Herding \citep{DBLP:conf/icml/Welling09} is a method that generates pseudo-samples (i.e., deterministic samples) from a distribution such that nonlinear moments of the sample set closely match those of the target distribution. A greedy selection of pseudo-samples can achieve a convergence rate of $O(1/n)$ \citep{DBLP:conf/uai/ChenWS10}. Different from greedy herding, MMDQN collectively find the set of pseudo-samples using SGD at each learning step. This collective herding by SGD is more effective in the distributional RL context than greedy herding as in distributional RL we need to perform herding for multiple distributions which themselves also evolve over learning steps. 
\end{rem}



\begin{proof}

We first present two relevant results below (whose detailed proofs are deferred to Section \ref{chap4_section_proof}) from which the theorem can follow. 

\begin{prop}
Let $(X_i)_{i=1}^n$ be $n$ i.i.d. samples of some distribution $P$. We have 
\begin{align*}
    \text{MMD}\left(\frac{1}{n} \sum_{i=1}^n \delta_{X_i}, P; \mathcal{F} \right)
    = O_p(1/\sqrt{n}),
\end{align*}
where $O_p$ denotes big-O in probability. 
\label{mmd_convergence_rate}
\end{prop}

\begin{lem}
Let $(a_n)_{n \in \mathbb{N}} \subset \mathbb{R}$ and $(X_n)_{n \in \mathbb{N}} \subset \mathbb{R}$ be sequences of deterministic variables and of random variables, respectively, such that for all $n$, $|a_n| \leq |X_n|$ almost surely (a.s.). Then, if $X_n = O_p(f(n))$ for some function $f(n) >0$, we have $a_n = O(f(n))$. 
\label{probability_bound_equals_standard_bound}
\end{lem}

It follows from the Cauchy-Schwartz inequality in $\mathcal{F}$ that for any function $h$ in the unit ball of $\mathcal{F}$, we have
\begin{align}
    \bigg|\int_{\mathcal{X}} h(x) dP_n(x) - \int_{\mathcal{X}} h(x) dP(x)\bigg| \nonumber &\leq \| h\|_{\mathcal{F}} \cdot \bigg \|\int_{\mathcal{X}} k(x, \cdot) dP_n(x) - \int_{\mathcal{X}} k(x, \cdot) dP(x)\bigg \|_{\mathcal{F}} \nonumber \\
    &\leq \text{MMD}\left(\frac{1}{n} \sum_{i=1}^n \delta_{x_i}, P; \mathcal{F} \right). 
    \label{eq:moment_bounded_by_mmd}
\end{align}

Now by letting $X_n = \text{MMD}(\frac{1}{n} \sum_{i=1}^n \delta_{\tilde{x}_i}, P; \mathcal{F})$, $a_n = \text{MMD}(\frac{1}{n} \sum_{i=1}^n \delta_{x_i}, P; \mathcal{F})$, and $f(n) = 1/\sqrt{n}$, and noting that $a_n \leq X_n, \forall n$, Proposition \ref{mmd_convergence_rate}, Lemma \ref{probability_bound_equals_standard_bound} and Equation (\ref{eq:moment_bounded_by_mmd}) immediately imply Theorem \ref{theorem:convergence_rate}. 
\end{proof}

\begin{figure*}
\centering
\includegraphics[scale=0.52]{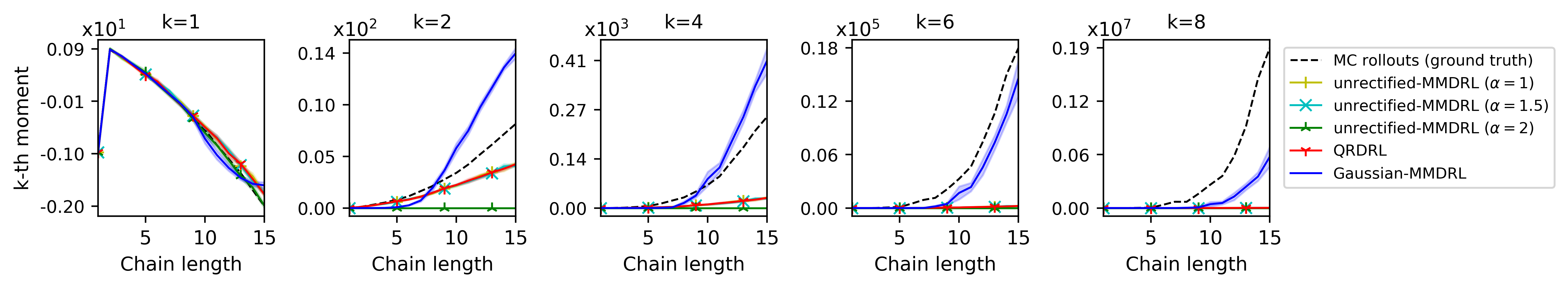}
\caption{Performance of different methods in approximating the optimal policy's return distribution at the initial state in the chain environment of various chain lengths $K = \{1,2,...,15\}$. The distribution approximation is evaluated in terms of how well a method can approximate the $k$-th central moment (except $k=1$ means the expectation) of the target distribution. $95\%$ C.I. with $30$ seeds. A variant (Gaussian-MMDRL) of our proposed MMDRL matches the MC rollouts (representing ground truth) much better than QRDRL.}
\label{fig:chain_experiment_result}
\end{figure*}


\section{Experiment}

We first present results with a tabular version of MMDRL to illustrate its behaviour in the distribution approximation task. We then combine the MMDRL update to the DQN-style architecture to create a novel deep RL algorithm namely MMDQN, and evaluate it on the Atari-57 games. 

\subsection{Tabular Policy Evaluation}

We empirically evaluate that MMDRL with Gaussian kernels ({Gaussian-MMDRL}) can approximately learn the moments of a policy's return distribution as compared to the MMDRL with unrectified kernels ({unrectified-MMDRL}) and the baseline QRDRL. 

We use a variant of the classic chain environment \citep{DBLP:conf/icml/RowlandDKMBD19} . The chain environment of length $K$ is a chain of $K$ states $s_0, ..., s_{K-1}$ where $s_0$ is the initial state and $s_{K-1}$ is the terminal state (see Figure \ref{fig:chain_mdp}). In each state, there are only two possible actions: (i) forward, which moves the agent one step to the right with probability $0.9$ and to $s_0$ with probability $0.1$, or (ii) backward, which transitions the agent to $s_0$ with probability $0.9$ and one step to the right with probability $0.1$. The agent receives reward $-1$ when transitioning to the initial state $s_0$, reward $1$ when reaching the terminal state $s_{K-1}$, and $0$ otherwise. The discount factor is $\gamma = 0.9$. We estimate $\mu^*_0$ the return distribution at the initial state of the optimal policy $\pi^*$ which selects forward action in every state. The longer the chain length $K$, the more stochastic the optimal policy's return distribution at $s_0$. We use $10,000$ Monte Carlo rollouts under policy $\pi^*$ to compute the central moments of  $\mu^*_0$ as ground truth values. Each method uses only $N=30$ samples to approximate the target distribution $\mu^*_0$. The algorithm details are presented in Algorithm \ref{alg:tabular_policy_evaluation}.

\begin{algorithm}[t]
\begin{algorithmic}[1]
\STATE \textbf{Input}: Number of particles $N$, kernel $k$, discount factor $\gamma \in [0,1]$, evaluation policy $\pi$, learning rate $\alpha_t$, tabular particles $\{\theta_i(s,a)\}_{i=1}^N$

\STATE \textbf{Initialization}: initial particles $\theta$, initial copy particles $\theta^{-} \leftarrow \theta$, and initial state $s_0$

\FOR{t=1,2,...}
\STATE Take action $a_t = \pi(s_t)$ and observe $s_{t+1} \sim P(\cdot|s_t, a_t)$ and $r_t \sim \mathcal{R}(s_t, a_t)$. 

\STATE Compute Bellman target particles 
\begin{align*}
    \hat{T} \theta_i^{-} \leftarrow r_t + \gamma \theta^{-}_i(s_{t+1}, \pi(s_{t+1})), \forall i \in \{1,...,N\}.  
\end{align*}

\STATE Compute TD gradient: 

For MMDRL: 
\begin{align*}
    g_i \leftarrow \frac{\partial}{ \partial \theta_i(s_t,a_t)} \text{MMD}_b^2( \{\theta_i(s_t,a_t)\}_{i=1}^N,  \{\hat{T} \theta_i^{-}\}_{i=1}^N; k), \forall i \in \{1,...,N\} 
\end{align*}
For QRDRL: 
\begin{align*}
   g_i \leftarrow \frac{\partial}{ \partial \theta_i(s_t,a_t)} \frac{1}{N} \sum_{j=1}^N \left( \hat{T} \theta_j^{-} - \theta_i(s_t,a_t) \right) \left( \frac{2i -1}{2N} - 1_{\{ \hat{T}\theta_j^{-} <  \theta_i(s_t,a_t)\}}  \right), \forall i
\end{align*}

\STATE Update 
\begin{align*}
    \theta_i(s_t, a_t) &\leftarrow \theta_i(s_t, a_t) - \alpha_t g_i,\forall i \in \{1,...,N\}\\ 
    \theta^{-} &\leftarrow \theta 
\end{align*}

\ENDFOR

\STATE \textbf{Output}: Approximate distribution $\mu(s,a) = \frac{1}{N} \sum_{i=1}^N \delta_{\theta_i(s,a)}$

\caption{Tabular policy evaluation}
\label{alg:tabular_policy_evaluation}
\end{algorithmic}
\end{algorithm}

In the tabular MDP, the particles $Z_{\theta}(s,a)$ reduces to tabular values $\sloppy \{ \theta_i(s,a)\}_{i=1}^N$, so a return distribution is represented as a mixture of Diracs $\mu(s,a) = \frac{1}{N} \sum_{i=1}^N \delta_{\theta_i(s,a)}$. 
The algorithm details used for tabular policy evaluation are presented in Algorithm \ref{alg:tabular_policy_evaluation}. In both tabular MMDRL and QRDRL cases, the TD gradients have a closed-form expression which we explicitly used in our tabular experiment. The detailed values of each (hyper-)parameters of the algorithms used in our experiment are reported in Table \ref{tab:tabular_policy_evaluation_params}. 
\begin{table}
    \centering
    \begin{tabular}{l|l} 
        \textbf{(Hyper-)Parameters} & \textbf{Values} \\
        \hline
        \hline
        Learning rate schedule & $\alpha_t = \frac{1}{t^{0.2}}$  \\
        Particle initialization & $\mathcal{N}(-1, 0.08)$ \\ 
        Number of episodes per iteration & $100$ \\ 
        Number of iterations & $15$ \\ 
        Number of particles $N$ & $30$ \\ 
        Number of MC rollouts & $10,000$ \\ 
        Kernel bandwidth $h$ (MMDRL only) & $\{8,10,12\}$ \\ 
        Quantiles (QRDRL only) & $\{ \frac{2i-1}{2N}: 1 \leq i \leq N \}$
    \end{tabular}
    \caption{The (hyper-)parameters of the algorithms used in our tabular policy evaluation experiment with the Chain MDP.}
    \label{tab:tabular_policy_evaluation_params}
\end{table}

The result is presented in Figure \ref{fig:chain_experiment_result}. While all the methods approximate the expectation of the target distribution well, their approximation qualities differentiate greatly when it comes to higher order moments. Gaussian-MMDRL, though with only $N=30$ particles, can approximate higher order moments more reasonably in this example whereas the rest highly suffer from underestimation.  We also experimented with the kernel considered in Theorem \ref{theorem:metric_property} in this tabular experiment and the Atari game experiment (next part) but found that it is highly inferior to the other kernel choices (even though it has an exact moment matching form as compared to Gaussian kernels) thus we did not include it (we speculate that the shift invariance of Gaussian kernels seems effective when interacting with transition samples from the Bellman dynamics). 

To demonstrate the effectiveness of MMDRL at scale, we combine the MMDRL in Algorithm \ref{alg:mmd-drl} with DQN-like architecture to obtain a deep RL agent namely MMDQN which is presented in details in Algorithm \ref{alg:mmd-dqn}. Specifically in this work, we used the same architecture as QR-DQN \citep{DBLP:conf/aaai/DabneyRBM18} for simplicity but more advanced modeling improvements from IQN \citep{DBLP:conf/icml/DabneyOSM18} and FQF \citep{yang2019fully} can naturally be used in combination to our framework. We use the same architecture of DQN except that we change the last layer to the size of $N \times |\mathcal{A}|$, instead of the size $|\mathcal{A}|$. In addition, we replace the squared loss in DQN by the empirical MMD loss. 

\begin{algorithm}[t]
\begin{algorithmic}[1]
\STATE \textbf{Input}: Number of particles $N$, kernel $k$ (e.g., Gaussian kernel), discount factor $\gamma \in [0,1]$, learning rate $\alpha$, replay buffer $\mathcal{M}$, main network $Z_{\theta}$, target network $Z_{\theta^{-}}$, and a policy $\pi$ (e.g., $\epsilon$-greedy policy w.r.t. $Q_{\theta}(s,a) =  \frac{1}{N}\sum_{i=1}^N Z_{\theta}(s, a)_i, \forall s,a$). 

\STATE Initialize $\theta$ and $\theta^{-} \leftarrow \theta$  \; 

\FOR{t = 1,2,...}

\STATE Take action $a_t \sim \pi(\cdot | s_t; \theta)$, receive reward $r_t \sim \mathcal{R}(\cdot| s_t, a_t)$, and observe $s_{t+1} \sim P(\cdot| s_t, a_t)$  

\STATE Store $(s_t, a_t, r_t, s_{t+1})$ to the replay buffer $\mathcal{M}$ 

\STATE Randomly draw a batch of transition samples $(s,a,r,s')$ from $\mathcal{M}$ 

\STATE Compute a greedy action 
\begin{align*}
    a^* \leftarrow \operatorname*{arg\,max}_{a' \in \mathcal{A}}  \frac{1}{N}\sum_{i=1}^N Z_{\theta^{-}}(s', a')_i
\end{align*} 

\STATE Compute the empirical Bellman target measure 
\begin{align*}
    \hat{T} Z_i^{-} \leftarrow r + \gamma Z_{\theta^{-}}(s', a^*)_i, \forall i \in \{1,...,N\}
\end{align*}

\STATE Update the main network 
\begin{align*}
    \theta \leftarrow \theta - \alpha \Delta_{\theta} \text{MMD}_b \left(\{Z_{\theta}(s,a)_i\}_{i=1}^N, \{\hat{T} Z_i^{-}\}_{i=1}^N; k \right)
\end{align*}
\; 

\STATE Periodically update the target network $\theta^{-} \leftarrow \theta$ 

\ENDFOR

\caption{Moment matching deep Q-networks}
\label{alg:mmd-dqn}
\end{algorithmic}
\end{algorithm}

We expect that our framework would also benefit from recent orthogonal improvements to DQN such as double-DQN \citep{DBLP:conf/aaai/HasseltGS16}, the dueling architecture \citep{DBLP:conf/icml/WangSHHLF16} and prioritized replay \citep{DBLP:journals/corr/SchaulQAS15} but did not include these for simplicity. In Table \ref{tab:hyperparameters}, we provide the hyperparameter details of QR-DQN and MMDQN used in the Atari games. The hyperparameters in MMDQN that share with QR-DQN are intentionally set the same to allow for fair comparison.

\begin{table}[h]
    \centering
    \begin{tabular}{l|l|l}
    \textbf{Hyperparameters}     &  \textbf{QR-DQN} & \textbf{MMDQN} \\
    \hline 
    \hline
    Learning rate       & $0.00005$ & $0.00005$ \\ 
    Optimizer & Adam & Adam \\ 
    $\epsilon_{ADAM}$ & $0.0003125$ & $0.0003125$ \\ 
    $N$ & $200$ & $200$ \\ 
    Quantiles & $\{ \frac{2i-1}{2N}: 1 \leq i \leq N \}$ & N/A \\
    Kernel bandwidth & N/A & $\{1,2,...,9,10\}$ \\
    \end{tabular}
    \caption{The MMDQN hyperparameters as compared to those of QR-DQN.}
    \label{tab:hyperparameters}
\end{table}

 We evaluated our algorithm on 55 \footnote{We failed to include Defender and Surround games using OpenAI and Dopamine framework.} Atari 2600 games \citep{Bellemare_2013} following the standard training and evaluation procedures \citep{mnih2015human,DBLP:conf/aaai/HasseltGS16}. For every 1M training steps in the environment, we computed the average scores of the agent by freezing the learning and evaluating the latest agent for 500K frames. We truncated episodes at 108K frames (equivalent to 30 minutes of game playing). We used the 30 no-op evaluation settings where we play a random number (up to 30) of no-op actions at the beginning of each episode during evaluation. We report the best score for a game by an algorithm which is the algorithm’s highest evaluation score in that game across all evaluation iterations during the training course (given the same hyperparameters are shared for all games). 

The human normalized scores of an agent per game is the agent's normalized scores such that 0\% corresponds to a random agent and 100\% corresponds to the average score of a human expert. The human-normalized scores used in the chapter are explicitly defined by 
\begin{align*}
    score = \frac{agent - random}{human - random},
\end{align*}
where $agent, human, random$ denotes the raw scores (undiscounted returns) for the given agent, the reference human player and the random player \citep{mnih2015human}, resp.,  in each game. 

From the human normalized scores for an agent across all games, we extracted three statistics for the agent's performance: the median, the mean and the number of games where the agent's performance is above the human expert's performance. 

The per-game percentage improvement (PI) of MMDQN over QR-DQN is computed as follow 
\begin{align*}
    PI = \frac{ score_{MMDQN} - score_{QR-DQN}  }{score_{QR-DQN}} \times 100 \%,
\end{align*}
where $ score_{MMDQN}$ and $score_{QR-DQN}$ are the best raw evaluation score of MMDQN and QR-DQN in the considered game. The log-scaled percentage improvement is computed as 
\begin{align*}
    log\_percentage\_improvement = 1_{\{ PI \geq 0 \}} \log(|PI| + 1).
\end{align*}

\subsubsection*{Baselines}
We categorize the baselines into two groups. The first group contains \textit{comparable} methods: DQN, PRIOR., C51, and QR-DQN-1, where DQN \citep{mnih2015human} and PRIOR. (prioritized experience replay \citep{DBLP:journals/corr/SchaulQAS15}) are classic baselines. The second group includes \textit{reference} methods: RAINBOW \citep{DBLP:conf/aaai/HesselMHSODHPAS18}, IQN, and FQF, which contain algorithmic/modeling improvements orthogonal to MMDQN: RAINBOW combines C51 with prioritized replay and $n$-step update while IQN and FQF contain modeling improvements as described in the related work section. Since in this work we used the same architecture as QR-DQN and C51 for MMDQN, we directly compare MMDQN with the first group while including the second group for reference. 




\subsubsection*{Hyperparameter Setting}
For fair comparison with QR-DQN, we used the same hyperparameters: $N=200$, Adam optimizer \citep{DBLP:journals/corr/KingmaB14} with learning rate $lr=0.00005$ and tolerance parameter $\epsilon_{ADAM} = 0.01/32$. We used $\epsilon$-greedy policy with $\epsilon$ being decayed at the same rate as in DQN but to a lower value $\epsilon=0.01$ as commonly used by the distributional RL methods. We used a target network to compute the distributional Bellman target as with DQN. Our implementation is based on OpenAI Gym \citep{brockman2016openai} and the Dopamine framework \citep{DBLP:journals/corr/abs-1812-06110}.

\subsubsection*{Kernel Selection}
We used Gaussian kernels $k_h(x,y) = \exp \left( - (x-y)^2 /h \right)$ where $h>0$. The kernel bandwidth $h$ is crucial to the statistical quality of MMD: overestimated bandwidth results in a flat kernel while underestimated one makes the decision boundary highly irregular. 
We utilize the kernel mixture trick in \citep{DBLP:conf/icml/LiSZ15} which is a mixture of $K$ kernels covering a range of bandwidths $k(x,y) = \sum_{i=1}^K k_{h_i}(x,y)$. The Gaussian kernel with a bandwidth mixture yields much a better performance than that with individual bandwidth and unrectified kernels in 6 tuning games: Breakout, Assault, Asterix, MsPacman, Qbert, and BeamRider (see Figure \ref{fig:hyperparameter_tunning} (a)). Figure \ref{fig:hyperparameter_tunning} (b) shows the sensitivity of MMDQN in terms of the number of particles $N$ in the 6 tuning games where too small $N$ adversely affects the performance.  

\begin{figure}
    \centering
    \includegraphics[scale=0.8]{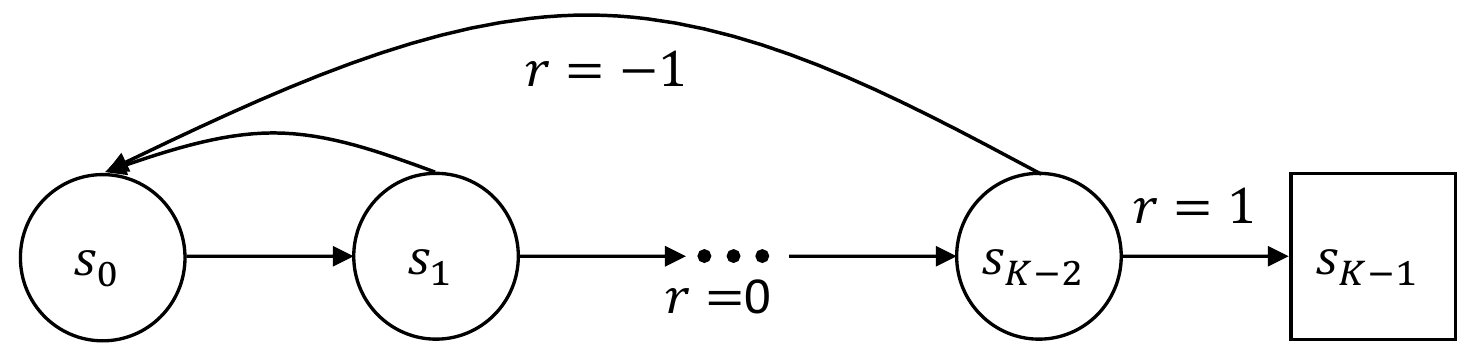}
    \caption{An illustration of a variant of the classic chain MDP with the chain length $K$.}
    \label{fig:chain_mdp}
\end{figure}

\subsection{Atari Games}
\begin{figure*}[h]
    \centering
    \begin{minipage}[t]{0.53\textwidth}
        \centering
        \includegraphics[scale=0.5]{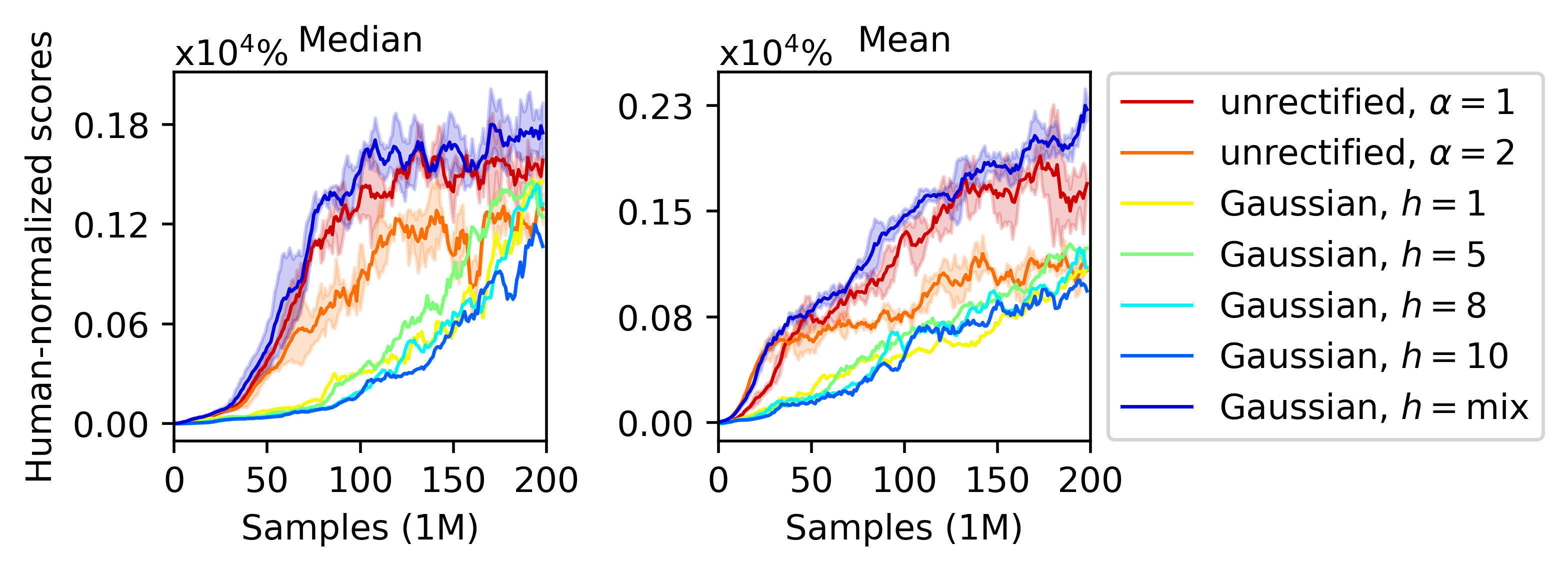} 
        \subcaption{Different kernel bandwidths (at $N=200$).}
    \end{minipage}\hfill
    \begin{minipage}[t]{0.46\textwidth}
        \centering
        \includegraphics[scale=0.5]{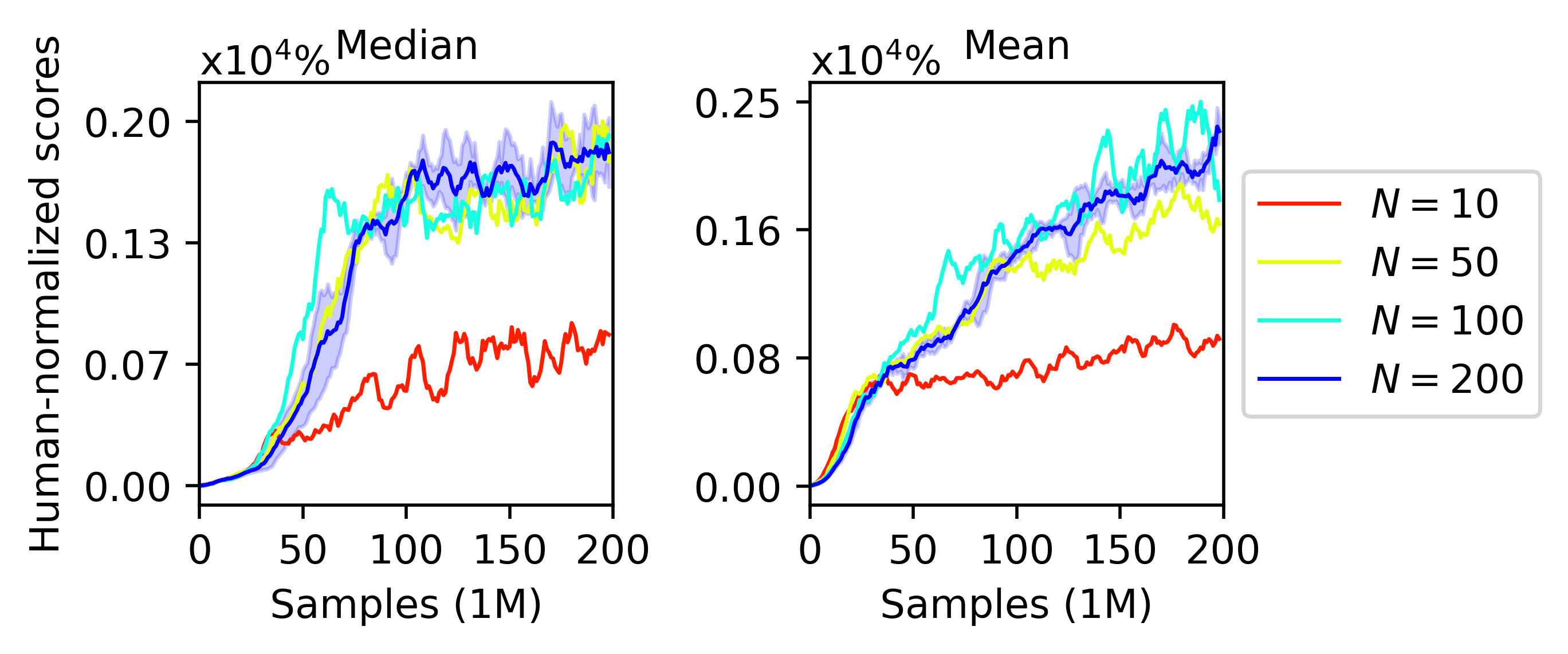} 
        \subcaption{Different values of $N$ (at $h=mix$).}
    \end{minipage}
    \caption{The sensitivity of (the human-normalized scores of) MMDQN in the 6 tuning games with respect to: (a) the kernel choices (Gaussian kernels with different bandwidths $h$ and unrectified kernels), and (b) the number of particles $N$. Here $h=mix$ indicates the mixture of bandwidth values in $\{1,2,...,10\}$. All curves are smoothed over $5$ consecutive iterations. $95\%$ C.I. for the $h=mix$, $N=200$, and unrectified kernel curves ($3$ seeds) and $1$ seed for the other curves.}
    
    \label{fig:hyperparameter_tunning}
\end{figure*}


\begin{table}
    \centering
    \setlength{\extrarowheight}{1.5pt}
    \begin{tabular}{l|l|l|l|l}
         & \textbf{Mean} & \textbf{Median} & $>$\textbf{Human} & $>$\textbf{DQN}  \\
    \hline 
    \hline
    DQN & 221\% & 79\% & 24 & 0 \\ 
    PRIOR. & 580\% & 124\% & 39 & 48\\ 
    C51 & 701\% & 178\% & 40 & 50 \\ 
    QR-DQN-1 & 902\% & 193\% & 41 & 54 \\ 
    \hline 
    RAINBOW & 1213\% & 227\% & 42 & 52 \\ 
    IQN & 1112\% & 218\% & 39 & 54 \\
    FQF & 1426\% & 272\% & 44 & 54  \\
    \hline 
    \textbf{MMDQN} & \textbf{1969}\% & 213\% & 41 & \textbf{55}
    \end{tabular}
    \caption{Mean and median of \textit{best} human-normalized scores across 55 Atari 2600 games. The results for MMDQN are averaged over 3 seeds and the reference results are from \citep{yang2019fully}.}
    \label{tab:main_table}
\end{table} 
The main empirical result is provided in Table \ref{tab:main_table} where we compute the mean and median of \textit{best} human normalized scores across 55 Atari games in the 30 no-op evaluation setting. The table shows that MMDQN significantly outperforms the comparable methods in the first group (DQN, PRIOR., C51 and QR-DQN-1) in all metrics though it shares the same network architecture with C51 and QR-DQN-1. Although we did not include any orthogonal algorithmic/modelling improvements from the reference methods to MMDQN, MMDQN still performs comparably with these methods and even achieve a state-of-the-art mean human-normalized score.  In Figure \ref{fig:percentage_improvement} we also provide the percentage improvement per-game of MMDQN over QR-DQN-1 where MMDQN offers significant gains over QR-DQN-1 in a large array of games. 

\begin{figure}
    \centering
    \includegraphics[scale=1]{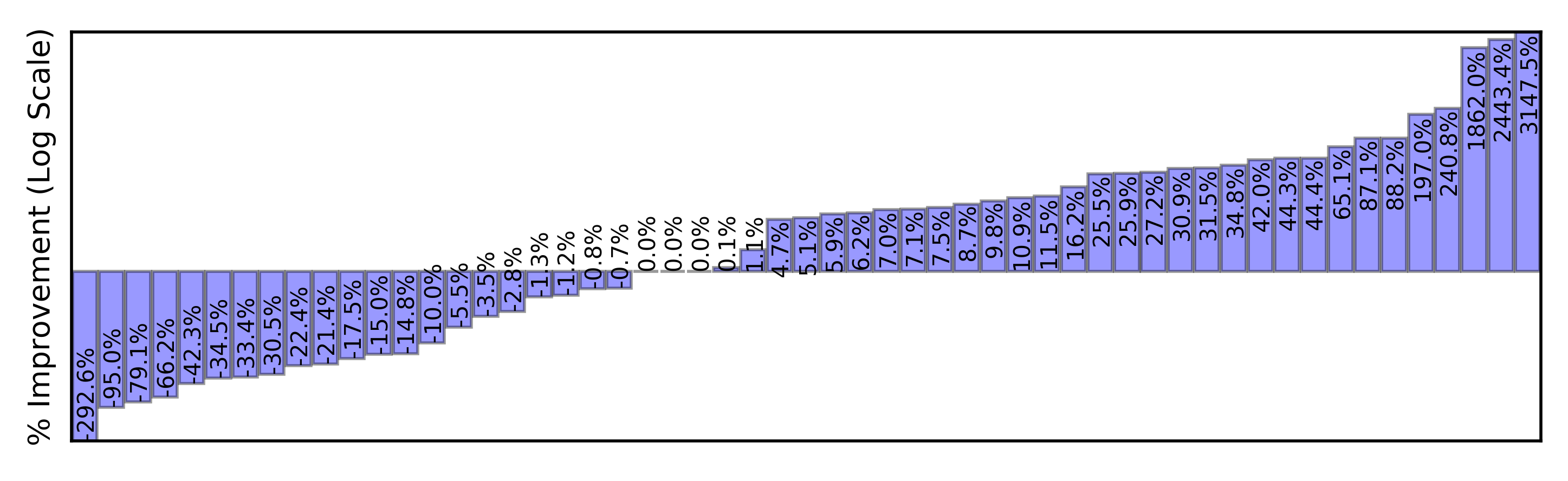}
    \caption{Percentage improvement per-game of MMDQN over QR-DQN-1.}
    \label{fig:percentage_improvement}
\end{figure}


In Figure \ref{fig:vis_mmd_dqn_breakout} we visualize the behaviour of our MMDQN in the Breakout game. Three rows correspond to  3 consecutive frames of the Breakout game accompanied by the approximate return distributions learnt by MMDQN. Since the particles learnt by MMDQN represent empirical samples of the return distributions, we can visualize the return distributions via the learnt particles by plotting the histogram (with $17$ bins in this example) of these particles. The learnt particles in MMDQN can maintain diversity in approximating the return distributions even though there is no order statistics in MMDQN as in the existing distributional RL methods such as QR-DQN. The 3 consecutive frames illustrate that the ball is moving away from the left to the right. In response, MMDQN also moves the paddle away from the left by gradually placing the probability mass of the return for the LEFT action towards smaller values. In particular, in the first frame where the ball is still far away from the ground, the MMDQN agent does not make a significant difference between actions. As the ball is moving closer the ground from the left (the second and third frame), the agent becomes clearer that the LEFT action is not beneficial, thus placing the action's probability mass to smaller values.  

\begin{figure}[!htb]
    \centering
    \includegraphics[scale=0.85]{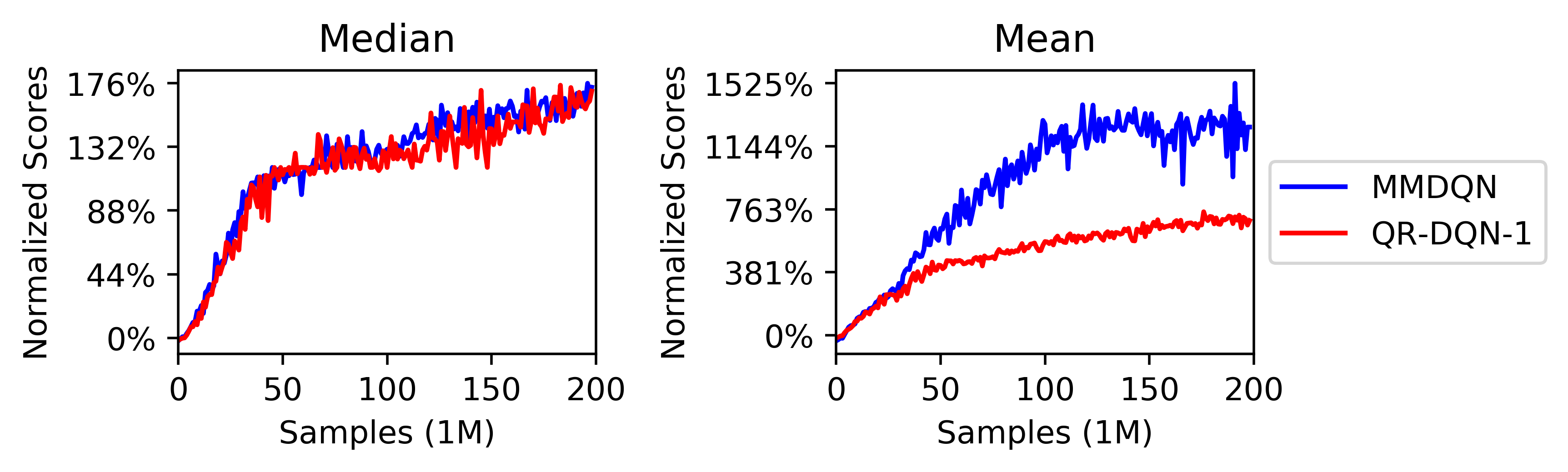}
    \caption{Median and mean of the test human-normalized scores across 55 Atari games for MMDQN (averaged over 3 seeds) and QR-DQN-1 (averaged over 2 seeds). }
    \label{fig:med_mean_55_games}
\end{figure}

\begin{figure}[t]
\centering
\subfloat{%
  \includegraphics[scale=0.9]{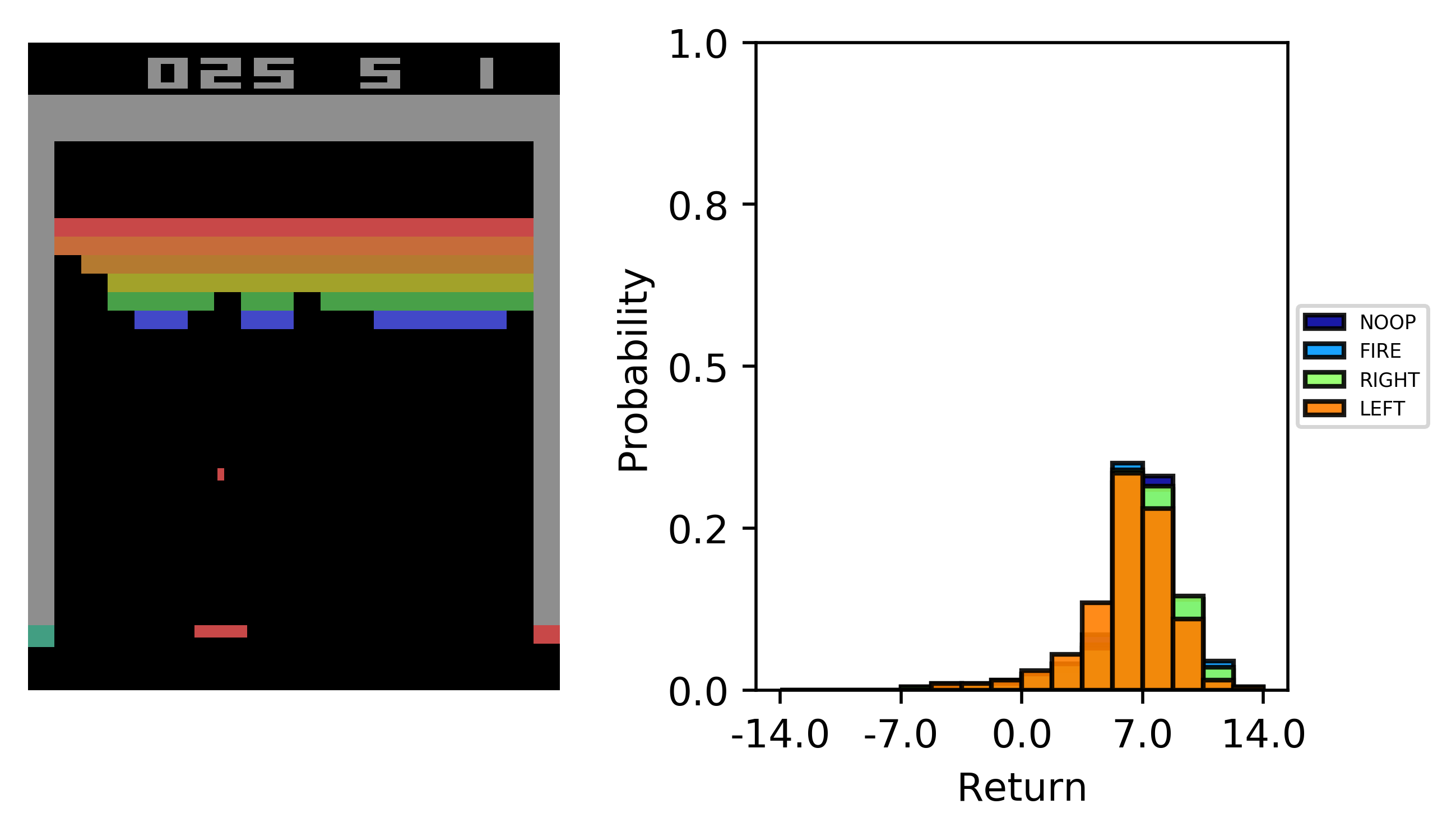}%
}

\subfloat{%
  \includegraphics[scale=0.9]{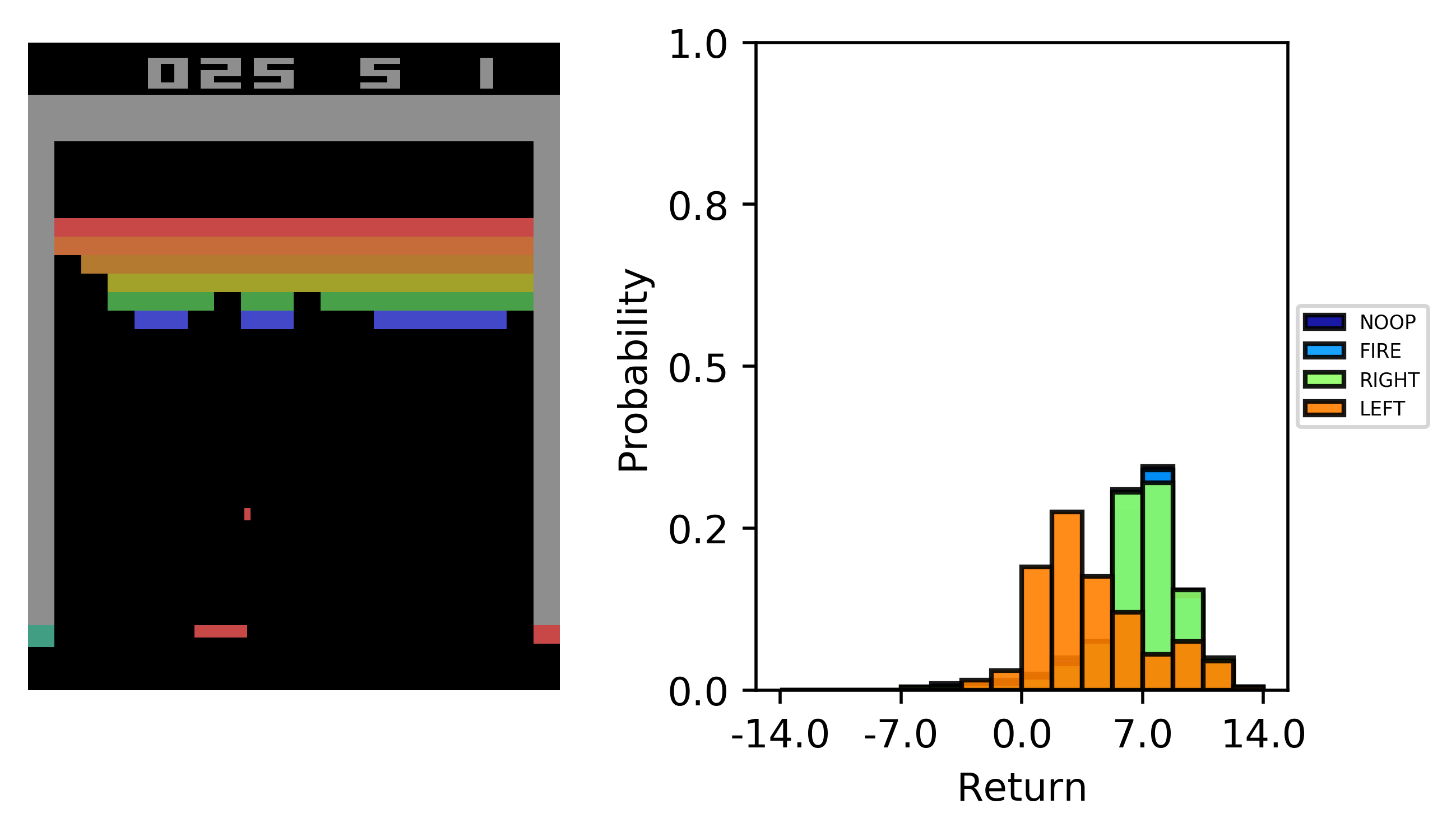}%
}

\subfloat{%
  \includegraphics[scale=0.9]{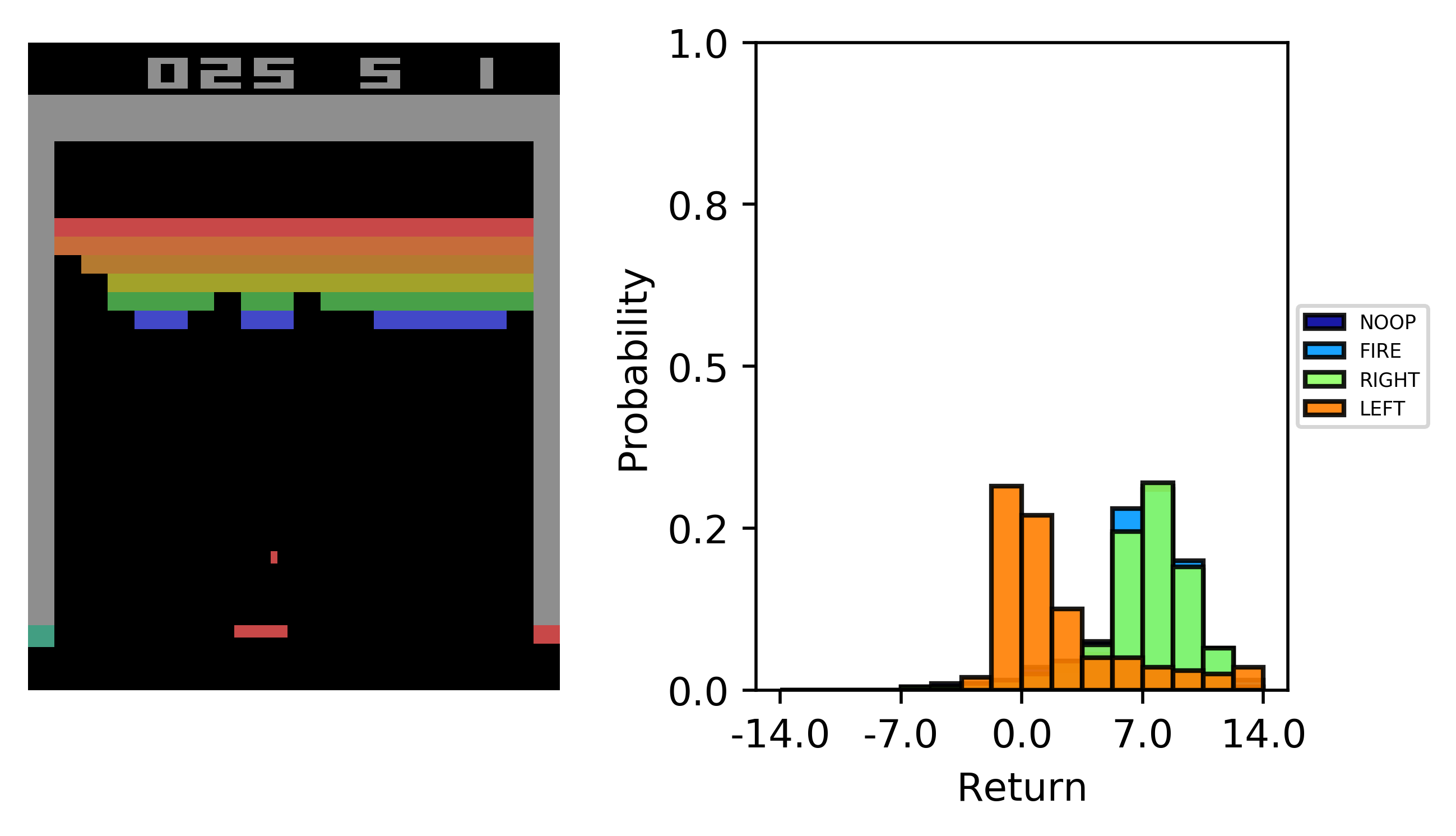}%
}
\caption{This example shows 3 consecutive frames and the approximate return distributions for all the actions in the Breakout game played by our MMDQN. The approximate return distributions plotted here are the histograms with $17$ bins constructed from the learnt particles by MMDQN.}
\label{fig:vis_mmd_dqn_breakout}
\end{figure}
We also include the recorded videos of the moves and approximate return distributions learnt by MMDQN for other Atari games (the exact video addresses are shown in Table \ref{tab:vid_url}).

\begin{table}[t]
    \centering
    \setlength{\extrarowheight}{4pt}
    \begin{tabular}{l|l}
         \textbf{Games} & \textbf{Video address}   \\
         \hline 
         \hline
         Breakout & \url{https://youtu.be/7P4oeJWJ6oE } \\ 
         BeamRider & \url{https://youtu.be/e6VQTynnbR8} \\ 
         BattleZone & \url{https://youtu.be/eXLs2pZJPCk} \\ 
         Qbert &  \url{https://youtu.be/64uHpoAPIvM} \\ 
         Pong & \url{https://youtu.be/NX5kXT59oJ4} \\ 
    \end{tabular}
    \caption{The recorded videos of the moves and approximate return distributions learnt by MMDQN.}
    \label{tab:vid_url}
\end{table}


We show the median and mean of the test human-normalized scores across all $55$ Atari games in Figure \ref{fig:med_mean_55_games}, the online learning curves in all the 55 Atari games in Figure \ref{fig:learing_curve_55_games} and provide the full raw scores of MMDQN in Table \ref{tab:raw_score_full}.

\begin{figure*}[t]
    \centering
    \includegraphics[scale=0.33]{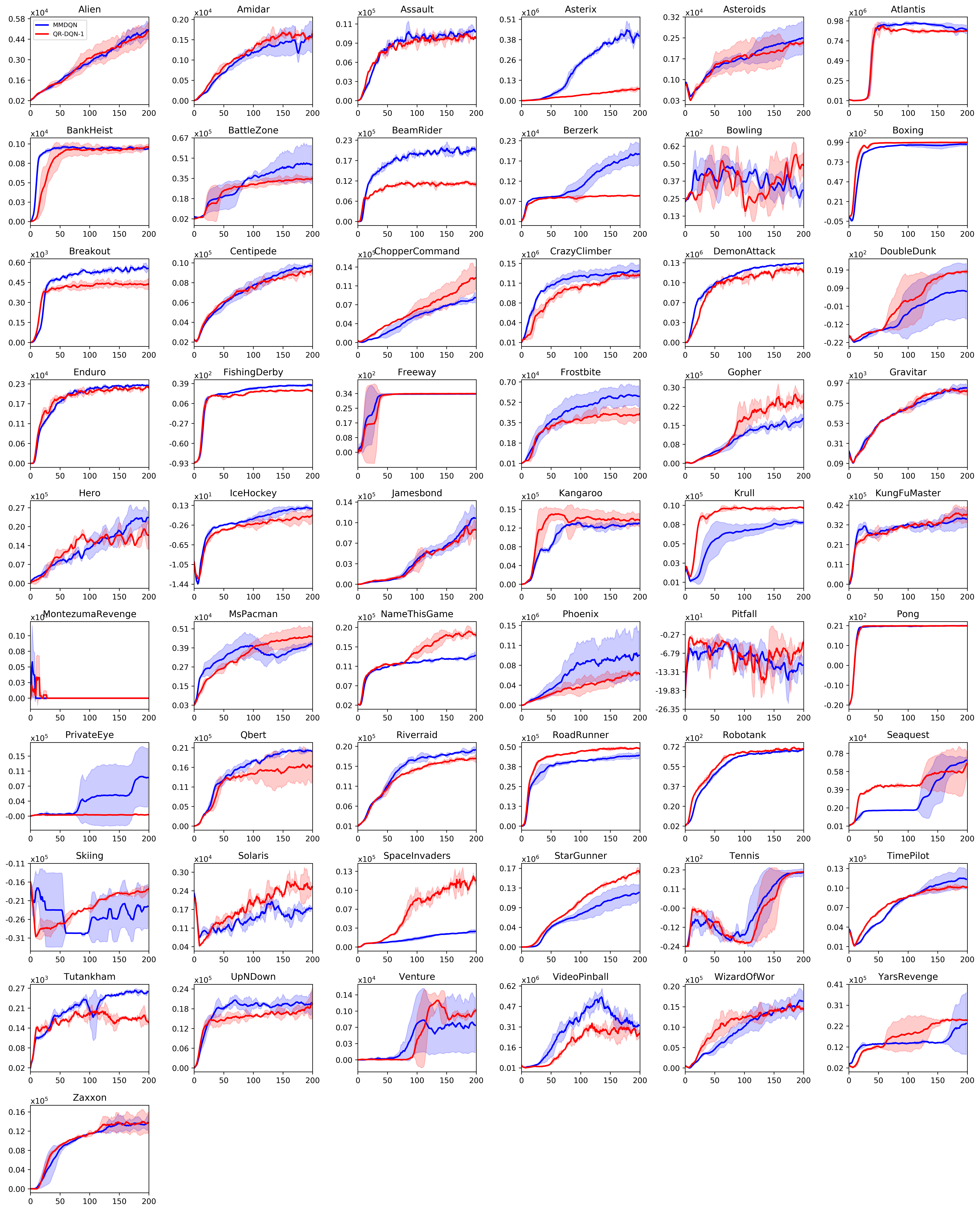}
    \caption{Online training curves for MMDQN (3 seeds) and QR-DQN-1 (2 seeds) on all 55 Atari 2600 games. Curves are averaged over the seeds and smoothed over a sliding window of 5 iterations. 95\% C.I. Reference values are from \citep{DBLP:conf/aaai/DabneyRBM18}.}
    \label{fig:learing_curve_55_games}
\end{figure*}

\begin{table*}[t]
    \centering
    \setlength{\extrarowheight}{3pt}
    \begin{adjustbox}{width=0.95\textwidth}
    \begin{tabular}{l|r|r|r|r|r|r|r}
    \textbf{GAMES} & 
    \textbf{RANDOM} & 
    \textbf{HUMAN} & 
    \textbf{DQN} & 
    \textbf{PRIOR. DUEL.} & 
    \textbf{C51} & 
    \textbf{QR-DQN-1} & 
    \textbf{MMDQN} \\
    \hline
        Alien & 227.8 & 7,127.7 & 1,620.0 & 3,941.0 & 3,166 & 4,871 & \textbf{6,918.8}\\
        Amidar & 5.8 & 1,719.5 & 978.0 & 2,296.8 & 1,735 & 1,641 & \textbf{2,370.1}\\
        Assault & 222.4 & 742.0 & 4,280.4 & 11,477.0 & 7,203 & \textbf{22,012} & 19,804.7\\
        Asterix & 210.0 & 8,503.3 & 4,359.0 & 375,080.0 & 406,211 & 261,025 & \textbf{775,250.9}\\
        Asteroids & 719.1 & 47,388.7 & 1.364.5 & 1,192.7 & 1,516 & \textbf{4,226} & 3,321.3 \\ 
        Atlantis & 12,850.0 & 29,028.1 & 279,987.0 & 841,075 & 395,762.0 & 971,850 & \textbf{1,017,813.3}\\ 
        BlankHeist & 14.2 & 753.1 & 455.0 & \textbf{1,503.1} & 976 & 1,249 & 1,326.6\\ 
        BattleZone & 2,360.0 & 37,187.5 & 29,900.0 & 35,520.0 & 28,742 & 39,268 & \textbf{64,839.8}\\ 
        BeamRider & 363.9 & 16,926.5 & 8,627.5 & 30,276.5 & 14,074 & \textbf{34,821} & 34,396.2 \\
        Berzerk &  123.7 & 2,630.4 & 585.6 & \textbf{3,409.0} & 1,645 & 3,117 & 2,946.1\\ 
        Bowling & 23.1 & 160.7 & 50.4 & 46.7 & \textbf{81.8} & 77.2 & 65.8\\ 
        Boxing & 0.1 & 12.1 & 88.0 & 98.9 & 97.8 & \textbf{99.9} & 99.2 \\
        Breakout & 1.7 & 30.5 & 385.5 & 366.0 & 748 & 742 & \textbf{823.1} \\
        Centipede & 2,090.9 & 12,017.0 & 4,657.7 & 7,687.5 & 9,646 & 12,447 & \textbf{13,180.9}\\
        ChopperCommand & 811.0 & 7,387.8 & 6,126.0 & 13.185.0 & {15,600} & 14,667 & \textbf{15,687.9} \\ 
        CrazyClimber & 10,780.5 & 35,829.4 & 110,763.0 & 162,224.0 & \textbf{179,877} & 161,196 & 169,462.0 \\ 
        DemonAttack & 152.1 & 1,971.0 & 12,149.4 & 72,878.6 & 130,955 & 121,551 & \textbf{135,588.7} \\
        DoubleDunk & -18.6 & -16.4 & -6.6 & -12.5 & 2.5 & \textbf{21.9} & {12.6} \\
        Enduro & 0.0 & 860.5 & 729.0 & 2,306.4 & \textbf{3,454} & 2,355 & {2,358.5} \\
        FishingDerby & -91.7 & -38.7 & -4.9 & 41.3 & 8.9 & 39.7 & \textbf{49.6} \\ 
        Freeway & 0.0 & 29.6 & 30.8 & 33.0 & 33.9 & \textbf{34} & 33.7 \\
        Frostbite & 65.2 & 4,334.7 & 797.4 & {7,413.0} & 3,965 & 4,384 & \textbf{8,251.4} \\ 
        Gopher & 257.6 & 2,412.5 & 8,777.4 & 104,368.2 & 33,641 & \textbf{113,585} & 38,448.1\\
        Gravitar & 173.0 & 3,351.4 & 473.0 & 238.0 & 440 & 995 & \textbf{1,092.5}\\
        Hero & 1,027.0 & 30,826.4 & 20,437.8 & 21,036.5 & \textbf{38,874} & 21,395 & 28,830.7\\
        IceHockey & -11.2 & 0.9 & -1.9 & -0.4 & -3.5 & -1.7 & \textbf{3.3} \\
        JamesBond & 29.0 & 302.8 & 768.5 & 812.0 & 1,909 & 4,703 & \textbf{16,028.9} \\
        Kangaroo &  52.0 & 3,035.0 & 7,259.0 & 1,792.0 & 12,853 & \textbf{15,356} & 15,154.2 \\
        Krull & 1,598.0 & 2,665.5 & 8,422.3 & 10,374.4 & 9,735 & \textbf{11,447} & 9,447.0\\
        KungFuMaster & 258.5 & 22,736.3 & 26,059.0 & 48,375.0 & 48,192 & \textbf{76,642} & 51,011.3\\
        MontezumaRevenge & 0.0 & 4,753.3 & 0.0 & 0.0 & 0.0 & 0.0 & 0.0 \\
        MsPacman &  307.3 & 6,951.6 & 3,085.6 & 3,327.3 & 3,415 & 5,821 & \textbf{6,762.8}\\
        NameThisGame & 2,292.3 & 8,049.0 & 8,207.8 & 15,572.5 & 12,542 & \textbf{21,890} & 15,221.2  \\
        Phoenix & 761.4 & 7,242.6 & 8,485.2 & 70,324.3 & 17,490 & 16,585 & \textbf{325,395.5} \\
        Pitfall & -229.4 & 6,463.7 & -286.1 & 0.0 & 0.0 & 0.0 & 0.0 \\
        Pong & -20.7 & 14.6 & 19.5 & 20.9 & 20.9 & 21.0 & 21.0 \\
        PrivateEye & 24.9 & 69,571.3 & 146.7 & 206.0 & \textbf{15,095} & 350 & {11,366.4}\\
        QBert & 163.9 & 13,455.0 & 13,117.3 & 18,760.3 & 23,784 & \textbf{572,510} & 28,448.0 \\
        Riverraid & 1,338.5 & 17,118.0 & 7,377.6 & 20,607.6 & 17,322 & 17,571 & \textbf{23000.0}\\
        RoadRunner & 11.5 & 7,845.0 & 39,544.0 & 62,151.0 & 55,839 &\textbf{ 64,262} & 54,606.8\\
        Robotank & 2.2 & 11.9 & 63.9 & 27.5 & 52.3 & 59.4 & \textbf{74.8} \\
        Seaquest & 68.4 & 42,054.7 & 5,860.6 & 931.6 & 266,434 & \textbf{8,268} & {7,979.3} \\
        Skiing & -17,098.1 & -4,336.9 & -13,062.3 & -19,949.9 & -13,901 & \textbf{-9,324} & {-9,425.3} \\
        Solaris & 1,236.3 & 12,326.7 & 3,482.8 & 133.4 & \textbf{8,342 }& 6,740 & 4,416.5 \\
        SpaceInvaders & 148.0 & 1,668.7 & 1,692.3 & 15,311.5 & 5,747 & \textbf{20,972} & 4,387.6 \\
        StarGunner & 664.0 & 10,250.0 & 54,282.0 & 125,117.0 & 49,095 & 77,495 & \textbf{144,983.7} \\
        Tennis & -23.8 & -8.3 & 12.2 & 0.0 & 23.1 & \textbf{23.6} & 23.0 \\
        TimePilot & 3,568.0 & 5,229.2 & 4,870.0 & 7,553.0 & 8,329 & 10,345 & \textbf{14,925.3} \\
        Tutankham & 11.4 & 167.6 & 68.1 & 245.9 & 280 & 297 & \textbf{319.4} \\
        UpNDown & 533.4 & 11,693.2 & 9,989.9 & 33,879.1 & 15,612 & \textbf{71,260} & 55,309.9 \\
        Venture & 0.0 & 1,187.5 & 163.0 & 48.0 & \textbf{1,520} & 43.9 & {1,116.6} \\
        VideoPinball & 16,256.9 & 17,667.9 & 196,760.4 & 479,197.0 & \textbf{949,604} & 705,662 & 756,101.8 \\
        WizardOfWor & 563.5 & 4,756.5 & 2,704.0 & 12,352.0 & 9,300 & 25,061 & \textbf{31,446.9}\\
        YarsRevenge & 3,092.9 & 54,576.9 & 18,098.9 & \textbf{69,618.1} & 35,050 & 26,447 & 28,745.7\\
        Zaxxon & 32.5 & 9,173.3 & 5,363.0 & 13,886.0 & 10,513 & 13,112 & \textbf{17,237.9} \\
    \end{tabular}
    \end{adjustbox}
    \caption{Raw scores of MMDQN  (averaged over 3 seeds) across all 55 Atari games starting with 30 no-op actions. Reference values are from  \citep{DBLP:conf/aaai/DabneyRBM18}. }
    \label{tab:raw_score_full}
\end{table*}






\section{Conclusion}
We have introduced a novel approach for distributional RL that eschews the predefined statistic principle used in the prior distributional RL. Our method deterministically evolves the (pseudo-)samples of a return distribution to approximately match moments of the resulting approximate distribution with those of the return distribution. We have also provided theoretical understanding of distributional RL within this framework. Our experimental results show that MMDQN, a combination of our approach with DQN-like architecture, achieves significant improvement in the Atari benchmark. 


\section{Proofs}
\label{chap4_section_proof}

\subsection*{Proof of Proposition \ref{prop:metric_over_state_action_space}}
\begin{proof}
If MMD is a metric in $\tilde{\mathcal{P}}(\mathcal{X})$. Then, it is obvious to see that $\text{MMD}_{\infty}(\mu,\nu) \geq 0, \forall \mu, \nu$ and that $\text{MMD}_{\infty}(\mu,\nu) = 0$ implies $\mu = \nu$. We now prove that $\texttt{MMD}_{\infty}$ satisfies the triangle inequality. Indeed, for any $\mu,\nu,\eta \in \tilde{\mathcal{P}}(\mathcal{X})^{\mathcal{S} \times \mathcal{A}}$, we have 
\begin{align*}
    \text{MMD}_{\infty}(\mu, \nu) &=  \sup_{(s,a) \in \mathcal{S} \times \mathcal{A}} \text{MMD}(\mu(s,a), \nu(s,a)) \\
    &\overset{(a)}{\leq} \sup_{(s,a) \in \mathcal{S} \times \mathcal{A}} \bigg\{ \text{MMD}(\mu(s,a), \eta(s,a)) + \text{MMD}(\eta(s,a), \nu(s,a)) \bigg\} \\ 
    &\overset{(b)}{\leq}  \sup_{(s,a) \in \mathcal{S} \times \mathcal{A}} \text{MMD}(\mu(s,a), \eta(s,a)) + \sup_{(s,a) \in \mathcal{S} \times \mathcal{A}} \text{MMD}(\eta(s,a), \nu(s,a)) \\
    &=  \text{MMD}_{\infty}(\mu, \eta) +  \text{MMD}_{\infty}(\eta, \nu),
\end{align*}
where $(a)$ follows from the triangle inequality for MMD and $(b)$ follows from that $ \sup (A + B) \leq \sup A + \sup B$ for any two sets $A$ and $B$ where $A+B := \{a + b: a \in A, b \in B\}$. 

\end{proof}

\subsection*{Proof of Theorem \ref{theorem:metric_property}}
We first present a relevant result for the proof. 
\begin{lem}
Let $\phi$ be the feature vector of $k$, i.e., $k(x,y) = \phi(x)^T \phi(y)$. Then, for any $\mu, \nu \in \mathcal{P}(\mathcal{X})$, we have 
\begin{align*}
    \text{MMD}(\mu; \nu; k) = \| u-v \|_2
\end{align*}
where $u = \mathbb{E}_{x \sim \mu} \phi(x)$ and $v = \mathbb{E}_{y \sim \nu} \phi(y)$.
\label{lemma:mmd_as_feature_mean_difference}
\end{lem}
\begin{proof}
Let $X, X' \overset{\text{i.i.d.}}{\sim} \mu$, $Y, Y' \overset{\text{i.i.d.}}{\sim} \nu$, and $X,X',Y,Y'$ are mutually independent. We have 
\begin{align*}
    \text{MMD}^2(\mu; \nu; k) &= \mathbb{E} \left[ k(X,X') \right] + \mathbb{E} \left[ k(Y,Y') \right] - 2 \left[ \mathbb{E} k(X,Y) \right] \\ 
    &= \mathbb{E} \left[ \phi(X)^T \phi(X') \right] + \mathbb{E} \left[ \phi(Y)^T \phi(Y') \right] - 2 \mathbb{E} \left[ \phi(X)^T \phi(Y) \right] \\ 
    &= u^T u + v^T v - 2 u^T v = \|u - v \|^2. 
\end{align*}
\end{proof}

\begin{proof}[Proof of Theorem \ref{theorem:metric_property}]
We prove only the third part of the theorem, as a proof for the first two parts can be found in  \citep{DBLP:conf/nips/FukumizuGSS07,DBLP:journals/jmlr/GrettonBRSS12} and \citep[c.f. Proposition 2]{szekely2003statistics}, respectively. For some $\sigma >0$, let $k(x,y) = \exp(xy/\sigma^2)$. Let $\phi(x) = [a_0(x), a_1(x), ..., a_n(x), ...]$ where $a_n(x) = \frac{1}{\sqrt{n!}} \frac{x^n}{\sigma^n}$. The Taylor expansion of $k$ yields $k(x,y) = \phi(x)^T \phi(y)$. It follows from Lemma \ref{lemma:mmd_as_feature_mean_difference} that 
\begin{align*}
    \text{MMD}^2(\mu, \nu; k) &= \| \mathbb{E} \phi(X) - \mathbb{E} \phi(Y) \|^2 = \sum_{n=0}^{\infty} \frac{1}{\sigma^{2n} n!} \left( \mathbb{E}[X^n] - \mathbb{E}[Y^n] \right)^2, 
\end{align*}
for any $\mu, \nu \in \mathcal{P}(\mathcal{X})$ where $X \sim \mu$ and $Y \sim \nu$. It is easy to see that $\text{MMD}(\mu, \nu; k) \geq 0$ and it satisfies the triangle inequality. We only need to prove that for any $\mu, \nu \in \mathcal{P}_*(\mathcal{X})$, if $\text{MMD}(\mu, \nu; k) = 0$, then $\mu = \nu$. Indeed, assume $\text{MMD}(\mu, \nu; k) = 0$, then $\mu$ and $\nu$ have equal moments of all orders. Note that a distribution $\mu$ is uniquely determined by its characteristic function $g_{\mu}(t) = \mathbb{E} \left[ \exp(itX) \right], \forall t$. Let $m_n(\mu) = \mathbb{E}[X^n]$ be the $n$-th moment of $\mu$. Taylor expansion of $g_{\mu}$ yields 
\begin{align*}
    g_{\mu}(t) = \sum_{n=0}^{\infty} \frac{ i^n t^n m_n(\mu) }{n!},
\end{align*}
a power series which is valid only within its radius of convergence. The radius of convergence of this power series is 
\begin{align*}
    r = \frac{1}{\limsup_{n \rightarrow \infty} \bigg|\frac{ m_n(\mu) }{ n!} \bigg|^{1/n}}. 
\end{align*}
Since $\mu \in \mathcal{P}_*(\mathcal{X})$, we have 
\begin{align*}
    \limsup_{n \rightarrow \infty} \frac{ |m_n(\mu)|^{1/n} }{ n} = 0. 
\end{align*}
Using Stirling's formular, this indicates that $\limsup_{n \rightarrow \infty} \bigg|\frac{ m_n(\mu) }{ n!} \bigg|^{1/n} = 0$, or $r = \infty$. Hence, the set of all moments of a distribution on $\mathcal{P}_*(\mathcal{X})$ uniquely determines the distribution. This concludes our proof. 
\end{proof}

\subsection*{Proof of Theorem \ref{theorem:contraction_property}}

\subsubsection*{Proof of the first part of Theorem \ref{theorem:contraction_property}}
\begin{lem}
Let $(\mu_i)_{i \in I}$ and $(\nu_i)_{i \in I}$ be two sets of Borel probability measures in $\mathcal{X}$ over some indices $I$. Let $p$ be any distribution induced over $I$, then we have
\begin{align*}
    \text{MMD}^2 \left( \sum_i p_i \mu_i, \sum_i p_i \nu_i \right) \leq \sum_{i} p_i \text{MMD}^2(\mu_i, \nu_i)
\end{align*}
\label{lemma:mixture_of_measures_scale_down_mmd}
\end{lem}
\begin{proof}
Denoting  $g_i = \psi_{\mu_i} - \psi_{\nu_i}, \forall i$, we have
\begin{align*}
    \text{MMD}^2 \left( \sum_i p_i \mu_i, \sum_i p_i \nu_i \right) &= \left\| \psi_{\sum_i p_i \mu_i} -  \psi_{\sum_i p_i \nu_i} \right \|^2_{\mathcal{F}} \\
    &\overset{(a)}{=} \left \|  \sum_i p_i (\psi_{\mu_i} - \psi_{\nu_i}) \right \|^2_{\mathcal{F}} \\
    &= \sum_i \langle p_i g_i, p_i g_i \rangle_{\mathcal{F}} + 2 \sum_{i \neq j} \langle p_i g_i, p_j g_j \rangle_{\mathcal{F}} \\
     &= \sum_i p_i^2 \langle g_i, g_i \rangle_{\mathcal{F}} + 2 \sum_{i \neq j} p_i p_j \langle  g_i, g_j \rangle_{\mathcal{F}} \\
     &\overset{(b)}{\leq} \sum_i p_i^2 \|g_i\|^2_{\mathcal{F}} + 2 \sum_{i \neq j} p_i p_j \|g_i \|_{\mathcal{F}} \|g_j\|_{\mathcal{F}} \\ 
     &=  \left( \sum_i \sqrt{p_i} \sqrt{p_i} \|g_i\|_{\mathcal{F}} \right)^2 \\
     &\overset{(c)}{\leq} (\sum_i p_i) \sum_{i} p_i  \left \|g_i \right \|^2_{\mathcal{F}} \\
     &= \sum_{i} p_i  \left \| \psi_{\mu_i} - \psi_{\nu_i} \right \|^2_{\mathcal{F}} = \sum_{i} p_i \text{MMD}^2(\mu_i, \nu_i),
\end{align*}
where (a) follows from that the Bochner integral is linear (w.r.t. the probability measure argument), i.e., $\psi_{\sum_i p_i \mu_i} = \sum_i p_i \psi_{\mu_i}$, and both inequalities (b) and (c) follow from Cauchy-Schwartz inequality. 
\end{proof}

\begin{lem}
If $k$ is shift invariant and  scale sensitive with order $\alpha >0$. Then for any $\mu, \nu \in \mathcal{P}(\mathcal{X})$ and any $r, \gamma \in \mathbb{R}$, we have 
\begin{align*}
    \text{MMD}^2 \bigg((f_{r,\gamma})_{\#} \mu, (f_{r,\gamma})_{\#} \nu; k \bigg) = |\gamma|^{\alpha} \text{MMD}^2(\mu,\nu; k),
\end{align*}
where $f_{r,\gamma}(z) := r + \gamma z$ and $\#$ denotes pushforward operator. 
\label{lemma:take_out_shift_invariant_and_scale_sensitive}
\end{lem}

\begin{proof}
It follows from the closed-form expression of MMD distance that we have 
\begin{align*}
    \text{MMD}^2 \bigg((f_{r,\gamma})_{\#} \mu, (f_{r,\gamma})_{\#} \nu; k\bigg) 
    &= \int \int k(z,z') (f_{r,\gamma})_{\#} \mu(dz) (f_{r,\gamma})_{\#} \mu(dz') \\
    &+ \int \int k(w,w') (f_{r,\gamma})_{\#} \nu(dw) (f_{r,\gamma})_{\#} \nu(dw') \\
    &-2 \int \int k(z,w) (f_{r,\gamma})_{\#} \mu(dz) (f_{r,\gamma})_{\#} \nu(dw) \\
    &= \int \int k(r+ \gamma z,r + \gamma z')  \mu(dz) \mu(dz') \\
    &+ \int \int k( w,r + \gamma w') \nu(dw) \nu(dw') \\
    &-2 \int \int k(r + \gamma z,r + \gamma w)  \mu(dz) \nu(dw) \\
    &= |\gamma|^{\alpha}  \int \int k(z,z')  \mu(dz) \mu(dz') \\
    &+ |\gamma|^{\alpha} \int \int k(w,w') \nu(dw) \nu(dw') \\
    &-2 |\gamma|^{\alpha} \int \int k(z,w)  \mu(dz) \nu(dw) \\ 
    &= |\gamma|^{\alpha} \text{MMD}^2(\mu,\nu; k). 
\end{align*}
\end{proof}

\begin{lem}
For any $\mu, \nu \in \mathcal{P}(\mathcal{X})$, $(\beta_i)_{i \in I} \subset \mathbb{R}$ for some indices $I$, we have 
\begin{align*}
    \text{MMD}^2(\mu, \nu; \sum_{i \in I} \beta_i k_i) = \sum_{i \in I} \beta_i \text{MMD}^2(\mu, \nu; k_i)
\end{align*}
\label{lemma:take_out_mixture_of_kernels}
\end{lem}

\begin{proof}
We have 
\begin{align*}
    \text{MMD}^2(\mu, \nu; \sum_{i \in I} \beta_i k_i) 
    &= \int \int \sum_{i \in I} \beta_i k_i(z,z') \mu(dz) \mu(dz') + \int \int \sum_{i \in I} \beta_i k_i(w,w') \nu(dw) \nu(dw') \\
    &-2 \int \int \sum_{i \in I} \beta_i k_i(z,w) \mu(dz) \nu(dw)\\
    &= \sum_{i \in I} \beta_i \bigg ( 
    \int \int   k_i(z,z') \mu(dz) \mu(dz') + \int \int  k_i(w,w') \nu(dw) \nu(dw') \\
    &-2 \int \int  k_i(z,w) \mu(dz) \nu(dw) \bigg ) \\
    &=  \sum_{i \in I} \beta_i \text{MMD}^2(\mu, \nu; k_i).
\end{align*}
\end{proof}

\begin{lem}
If $k$ is shift invariant and  scale sensitive with order $\alpha >0$, then 
\begin{align*}
    \text{MMD}_{\infty}(\mathcal{T}^{\pi} \mu, \mathcal{T}^{\pi} \nu; k) \leq \gamma^{\alpha/2} \text{MMD}_{\infty}(\mu, \nu; k), 
\end{align*}
for any $\mu,\nu \in \mathcal{P}(\mathcal{X})$ and any (stationary) policy $\pi$.
\label{lemma:shift_invariant_and_scale_sensitive_kernels_are_contractive}
\end{lem}
\begin{proof}
We have 
\begin{align*}
    &\text{MMD}^2 \bigg(\mathcal{T}^{\pi} \mu(s,a), \mathcal{T}^{\pi} \nu(s,a); k \bigg)\\ 
    &=\text{MMD}^2 \bigg(
    \int (f_{r,\gamma})_{\#} \mu(s',a') \pi(da'|s') P(ds'|s,a) \mathcal{R}(dr|s,a), \\
    &\int (f_{r,\gamma})_{\#} \nu(s',a') \pi(da'|s') P(ds'|s,a) \mathcal{R}(dr|s,a); k
    \bigg )\\
    &\overset{(a)}{\leq} \int \text{MMD}^2 \bigg( (f_{r,\gamma})_{\#} \mu(s',a'), (f_{r,\gamma})_{\#} \nu(s',a'); k \bigg) \pi(da'|s') P(ds'|s,a) \mathcal{R}(dr|s,a) \\ 
    &\overset{(b)}{=} \gamma^{\alpha} \int \text{MMD}^2 (  \mu(s',a'),  \nu(s',a') ) \pi(da'|s') P(ds'|s,a) \mathcal{R}(dr|s,a) \\ 
    &\leq \gamma^{\alpha} \text{MMD}^2_{\infty}(\mu, \nu). 
\end{align*}
Here $(a)$ follows from Lemma \ref{lemma:mixture_of_measures_scale_down_mmd}, and $(b)$ follows from Lemma \ref{lemma:take_out_shift_invariant_and_scale_sensitive}. Finally, the last inequality concludes the lemma by the definition of supremum MMD. 
\end{proof}

\begin{proof}[Proof of Theorem 2.1]
We have 
\begin{align*}
     \text{MMD}^2 \bigg(\mathcal{T}^{\pi} \mu(s,a), \mathcal{T}^{\pi} \nu(s,a); \sum_{i \in I} c_i k_i \bigg) 
     &\overset{(a)}{=} \sum_{i \in I} c_i \text{MMD}^2 \bigg(\mathcal{T}^{\pi} \mu(s,a), \mathcal{T}^{\pi} \nu(s,a); k_i \bigg) \\
     &\overset{(b)}{\leq} \sum_{i \in I} c_i \gamma^{\alpha_i} \text{MMD}^2 \bigg(\mu(s,a), \nu(s,a); k_i \bigg) \\ 
     &\overset{(c)}{\leq} \gamma^{\alpha_*} \sum_{i \in I} c_i  \text{MMD}^2 \bigg(\mu(s,a), \nu(s,a); k_i \bigg) \\  
     &\overset{(d)}{=} \gamma^{\alpha_*} \text{MMD}^2 \bigg(\mu(s,a), \nu(s,a); \sum_{i \in I} c_i k_i \bigg) \\
     &\leq \gamma^{\alpha_*} \text{MMD}^2_{\infty} \bigg(\mu, \nu; \sum_{i \in I} c_i k_i \bigg).
\end{align*}
Here, $(a)$ and $(d)$ follow from Lemma \ref{lemma:take_out_mixture_of_kernels}, $(b)$ follows from Lemma \ref{lemma:shift_invariant_and_scale_sensitive_kernels_are_contractive}, and $(c)$ follows from that $\gamma \leq 1$ and $c_i \geq 0$. Finally, the last inequality above concludes the proof by the definition of supremum MMD.
\end{proof}

\subsubsection*{Proof of the second part of Theorem \ref{theorem:contraction_property}}
\begin{proof}[Proof of the second part of Theorem \ref{theorem:contraction_property}]
We give a proof for two different cases. 

\textbf{Case 1}: Gaussian kernels $k(x,y) = \exp( -(x-y)^2 /(2 \sigma^2))$.  
\newline 

We will prove that $\text{MMD}_{\infty}$ associated with Gaussian kernels $k(x,y) = \exp( -(x-y)^2 /(2 \sigma^2))$ for some  $\sigma >0$ is not a contraction by contradiction and counterexamples. Assume by contradiction that there exists some $\alpha > 0$ such that 
\begin{align}
    \text{MMD}_{\infty}(\mathcal{T}^{\pi} \mu, \mathcal{T}^{\pi} \nu; k) \leq \gamma^{\alpha} \text{MMD}_{\infty}(\mu, \nu;k), 
    \label{eq:contraction_contradiction_assumption}
\end{align}
for all $\mu, \nu \in \mathcal{P}(\mathcal{X})$. If $\alpha \geq 1$, $\gamma^{\alpha} \leq \gamma^{\alpha'}, \forall \alpha' \in (0,1]$, thus without loss of generality, assume that $\alpha \in (0,1]$. 

\begin{figure}[t]
    \centering
    \includegraphics[scale=0.6]{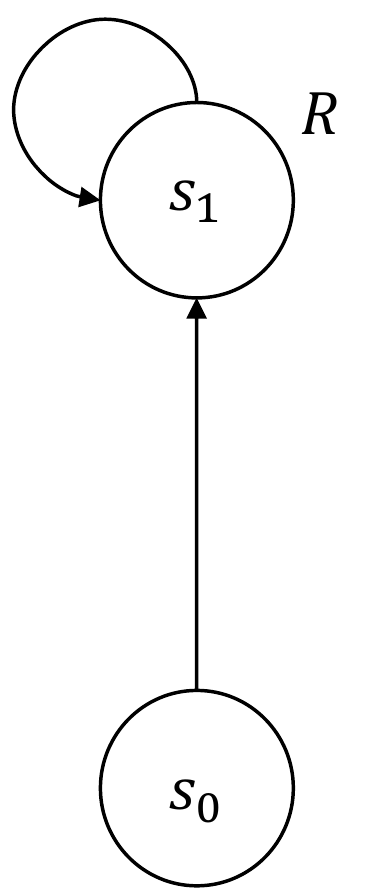}
    \caption{A simple MDP with 2 states: initial state $s_0$ and absorbing state $s_1$. Any agent receives a reward $r \sim R$ whenever it reaches state $s_1$.}
    \label{fig:simple_MDP}
\end{figure}

We provide a counterexample that contradicts Equation (\ref{eq:contraction_contradiction_assumption}). Consider a simple MDP with only 2 states: initial state $s_0$ and absorbing state $s_1$ where an agent receives a reward $r \sim R$ whenever it reaches $s_1$ (see Figure \ref{fig:simple_MDP}). Assume that the reward distribution $R$ has $dom(R) = \{r_1, ..., r_n\}$ with respective probabilities $\{p_{r_1}=\epsilon, p_{r_2}=\epsilon, ..., p_{r_{n-1}}=\epsilon, p_{r_n} = 1 - (n-1) \epsilon\}$ for some $\epsilon \in (0, \frac{1}{n-1})$ to be chosen later. Let $\mu, \nu \in \mathcal{P}(dom(R))^{\mathcal{S}}$ where $\mathcal{S} = \{s_0, s_1\}$ such that $\mu(s_0) = \mu(s_1) = p := \sum_{i=1}^n p_i \delta_{r_i}$ (for $\{p_i\}$ to be chosen later) and $\nu(s_0) = \nu(s_1) = q: = \sum_{i=1}^n q_i \delta_{r_i}$ (for $\{q_i\}$ to be chosen later). It is easy to verify that for any policy $\pi$, we have 
\begin{align*}
    \mathcal{T}^{\pi} \mu(s_0) &= \mathcal{T}^{\pi} \mu(s_1) = \mathbb{E}_{r \sim R} \left[ (f_{r,\gamma})_{\#} p \right] \\
    \mathcal{T}^{\pi} \nu(s_0) &= \mathcal{T}^{\pi} \nu(s_1) = \mathbb{E}_{r \sim R} \left[ (f_{r,\gamma})_{\#} q  \right],
\end{align*}
where $f_{r,\gamma}(z) := r + \gamma z, \forall z$ and $\#$ denotes pushforward operation. Note that $(f_{r,\gamma})_{\#} \mu(s_1)$ assigns probabilities $\{p_1, ..., p_n\}$ respectively to $r + \gamma dom(R) := \{r + \gamma r_1, ..., r + \gamma r_n\}$, and similarly for  $(f_{r,\gamma})_{\#} \nu(s_1)$. We have 
\begin{align*}
    \sum_{i=1}^n p_{r_i}^2 = (n-1) \epsilon^2 + (1 - (n-1) \epsilon)^2 \in [\frac{1}{n}, 1). 
\end{align*}
Note that since $\alpha \in (0,1]$, we have $\gamma^{2 \alpha} \in [\gamma^2, 1)$. Now, choose $\gamma \in (0,1)$ such that $\gamma^2 \geq \frac{1}{n}$, then there exists $\epsilon \in (0, \frac{1}{n-1})$ such that $ \sum_{i=1}^n p_{r_i}^2 = \gamma^{2 \alpha}$. Define \begin{align*}
    \eta_r &:= (f_{r,\gamma})_{\#} p = \sum_{i=1}^n p_i \delta_{r + \gamma r_i}, \\ 
    \kappa_r &:= (f_{r,\gamma})_{\#} q = \sum_{i=1}^n q_i \delta_{r + \gamma r_i}.
\end{align*}
It follows from the closed form of MMD that 
\begin{align}
    \text{MMD}^2(\mathcal{T}^{\pi} \mu(s_0), \mathcal{T}^{\pi} \nu(s_0); k) 
    &= \text{MMD}^2(\mathbb{E}_r[\eta_r] , \mathbb{E}_r[\kappa_r] ; k) \nonumber \\
    &= \sum_{r,r' \in dom(R)} p_r p_{r'} \text{MMD}^2(\eta_r, \kappa_{r'}; k)  \nonumber \\ 
    \label{eq:true_up_to_here_for_contradiction_examples}
    &> \sum_{i=1}^n p_{r_i}^2  \text{MMD}^2(\eta_{r_i}, \kappa_{r_i}; k) \\
    &\overset{(a)}{=} (\sum_{i=1}^n p_{r_i}^2) \text{MMD}^2(\eta_0, \kappa_0; k) \nonumber \\
    &\overset{(b)}{=} \gamma^{2 \alpha} \text{MMD}^2(\eta_0, \kappa_0; k) \nonumber.
\end{align}
Here $(a)$ follows from that Gaussian kernels are shift invariant and $(b)$ follows from the particular choice of $\epsilon$ such that $ \sum_{i=1}^n p_{r_i}^2 = \gamma^{2 \alpha}$. Similarly, we have 
\begin{align*}
    \text{MMD}^2(\mathcal{T}^{\pi} \mu(s_1), \mathcal{T}^{\pi} \nu(s_1); k) > \gamma^{2 \alpha} \text{MMD}^2(\eta_0, \kappa_0; k).
\end{align*}
It remains to choose particular values of $\sigma, n, p, q, \{r_i\}, \gamma$ such that $\gamma^2 \geq \frac{1}{n}$ and $\text{MMD}^2(\eta_0, \kappa_0; k) \geq \text{MMD}^2(p, q; k)$. It is indeed possible by choosing the values as in Table \ref{tab:realization_of_parameters_for_counterexample}. 

\begin{table}[]
    \centering
    \begin{tabular}{l|l}
        \textbf{Parameters} & \textbf{Values} \\
        \hline 
        $n$ & $5$ \\
        $\gamma$ & $0.8$ (or $0.9$) \\ 
        $\sigma$ & $0.1$ \\ 
        $\{p_i\}$ & $[0.4, 0.3, 0.2, 0.1, 0]$ \\ 
        $\{q_i \}$ & $[0, 0.1, 0.2, 0.3, 0.4]$ \\ 
        $\{r_i\}$ & $[0, 0.25, 0.5, 0.75, 1.]$
    \end{tabular}
    \caption{A realization of the parameters in the counterexample for proving Theorem 2.2}
    \label{tab:realization_of_parameters_for_counterexample}
\end{table}

So far, we have constructed a particular instance such that 
\begin{align*}
    \text{MMD}(\mathcal{T}^{\pi} \mu(s_0), \mathcal{T}^{\pi} \nu(s_0); k) &> \gamma^{\alpha} \text{MMD}(\mu(s_0), \nu(s_0); k) \\ 
    \text{MMD}(\mathcal{T}^{\pi} \mu(s_1), \mathcal{T}^{\pi} \nu(s_1); k) &> \gamma^{\alpha} \text{MMD}(\mu(s_1), \nu(s_1); k). 
\end{align*}
Thus, we have 
\begin{align*}
    \text{MMD}_{\infty}(\mathcal{T}^{\pi} \mu, \mathcal{T}^{\pi} \nu; k) > \gamma^{\alpha} \text{MMD}_{\infty}(\mu, \nu;k), 
\end{align*}
which contradicts Equation (\ref{eq:contraction_contradiction_assumption}). 

\noindent \textbf{Case 2}: For exp-prod kernels $k(x,y) = \exp( xy / \sigma^2)$. 

We follows the same procedure as in Case 1 but only up to Equation (\ref{eq:true_up_to_here_for_contradiction_examples}) as the exp-prodkernel is not shift invariant. Instead, define 
\begin{align*}
    r^* = \argmin_{r \in \{r_1, ..., r_n\}} \text{MMD}(\eta_{r}, \kappa_{r}; k).
\end{align*}
Then, we have 
\begin{align*}
    \text{MMD}^2(\mathcal{T}^{\pi} \mu(s_0), \mathcal{T}^{\pi} \nu(s_0); k)
    &> \sum_{i=1}^n p_{r_i}^2  \text{MMD}^2(\eta_{r_i}, \kappa_{r_i}; k) \\
    &\geq (\sum_{i=1}^n p_{r_i}^2) \text{MMD}^2(\eta_{r^*}, \kappa_{r^*}; k) \\
    &= \gamma^{2 \alpha} \text{MMD}^2(\eta_{r^*}, \kappa_{r^*}; k).
\end{align*}

Now it remains to choose particular values of $\sigma, n, p, q, \{r_i\}, \gamma$ such that $\gamma^2 \geq \frac{1}{n}$ and $\text{MMD}^2(\eta_{r^*}, \kappa_{r^*}; k) \geq \text{MMD}^2(p, q; k)$. In fact, the values chosen for Case 1 as in Table \ref{tab:realization_of_parameters_for_counterexample} already yields the previous inequalities for Case 2.

\end{proof}

\subsection*{Proofs of Lemma \ref{probability_bound_equals_standard_bound} and Proposition \ref{mmd_convergence_rate}}

\subsubsection*{Proof of Lemma \ref{probability_bound_equals_standard_bound}}
\begin{proof}
For all $n$, we have 
\begin{align}
    \frac{|a_n|}{f(n)} \leq \frac{|X_n|}{f(n)} \text{ a.s. }
    \label{eq:an_leq_Xn}
\end{align}
Denote by $P$ the probability measure of the underlying measurable space defining the random variables $X_n$. Since $X_n = O_p(f(n))$, for any $\epsilon > 0$, there exists $M_{\epsilon} > 0, N_{\epsilon} > 0$ such that 
\begin{align*}
    P\left( \frac{|X_n|}{f(n)} > M_{\epsilon} \right) \leq \epsilon, \forall n > N_{\epsilon}.
\end{align*}
It follows from Equation (\ref{eq:an_leq_Xn}) that $\forall n, \left\{ \frac{|a_n|}{f(n)} \leq M_{\epsilon} \right\} \supseteq \left\{ \frac{|X_n|}{f(n)} \leq M_{\epsilon} \right\}$ which implies that for all $n > N_{\epsilon}$, we have
\begin{align*}
    1_{\left\{ \frac{|a_n|}{f(n)} \leq M_{\epsilon} \right\}} &= P\left( \frac{|a_n|}{f(n)} \leq M_{\epsilon} \right) \geq P\left( \frac{|X_n|}{f(n)} \leq M_{\epsilon} \right) \geq 1 - \epsilon, 
\end{align*}
where the equality is due to that $a_n$ are deterministic. Picking any $\epsilon \in (0,1)$, we have 
\begin{align*}
    \frac{|a_n|}{f(n)} \leq M_{\epsilon}, \forall n > M_{\epsilon},
\end{align*}
which implies $a_n = O(f(n))$. 
\end{proof}

\subsubsection*{Proof of Proposition \ref{mmd_convergence_rate}}
First, we state the following proposition. 
\begin{prop}
Assume that $\sup_{x,y} k(x,y) \leq B$. Let $X_1, ..., X_n$ be $n$ i.i.d. samples of $P$ and denote $\xi(X_1, ..., X_n) := \text{MMD}\left(\frac{1}{n} \sum_{i=1}^n \delta_{X_i}, P; \mathcal{H} \right)$. For any $\epsilon > 0$, we have 
\begin{align*}
    P\left(\xi(X_1, ..., X_n) > \epsilon + 2 \sqrt{ \frac{B}{n} } \right) \leq \exp \left(  -\frac{\epsilon^2 n}{2B} \right).
\end{align*}
\label{prop:mmd_final_bound}
\end{prop}
Note that by setting $t = \exp \left(  -\frac{\epsilon^2 n}{2B} \right)$ in Proposition \ref{prop:mmd_final_bound}, we have 
\begin{align*}
     P\left(\xi(X_1, ..., X_n) /(1/\sqrt{n}) > M_t \right) \leq t,
\end{align*}
where $M_t := \sqrt{2B \log(1/t)} + 2 \sqrt{B}$. Thus, $\xi(X_1,...,X_n) = O_p(1/\sqrt{n})$, which concludes Proposition \ref{mmd_convergence_rate}.

Now, we only need to prove Proposition \ref{prop:mmd_final_bound}. The proof follows a standard procedure to bound an empirical process $\xi(X_1, ..., X_n) $ where we first bound it in probability w.r.t. its expectation using concentration inequalities and then we bound its expectation via a complexity notation of the witness function class $\mathcal{H}$.  

\noindent \textbf{Preliminaries}. Before proving Proposition \ref{prop:mmd_final_bound}, we present some relevant notations and preliminary results from which we combine to derive a proof for Proposition \ref{prop:mmd_final_bound}. For any function $g: \mathcal{X}^n \rightarrow \mathbb{R}$, denote 
\begin{align*}
    \delta_i g(x_1, ...,x_n) &:= \sup_{z} g(x_1, ...,x_{i-1}, z, x_{i+1}, ...,x_n) - \inf_{z} g(x_1, ...,x_{i-1}, z, x_{i+1}, ...,x_n), \\
    \|\delta_i g\|_{\infty} &:= \sup_{x_1,...,x_n} |\delta_i g(x_1, ...,x_n)|. 
\end{align*}
Denote by $\mathcal{H} := \{f \in \mathcal{F}: \|f\|_{\mathcal{F}} \leq 1\}$ the unit ball of the RKHS $\mathcal{F}$. For any $\{x_i\}_{i=1}^n \in \mathcal{X}^n$, we denote $\mathcal{H} \circ \{x_1, ..., x_n\} := \{ (f(x_1), ..., f(x_n)) \in \mathbb{R}^n: f \in \mathcal{H} \}$. 


\begin{defn}
The Rademacher complexity of a set $\mathcal{T} \subseteq \mathbb{R}^n$ is defined as 
\begin{align*}
    Rad(\mathcal{T}): = \mathbb{E} \sup_{t \in \mathcal{T}} \frac{1}{n} \sum_{i=1}^n \Omega_i t_i,
\end{align*}
where $\Omega_1, ..., \Omega_n \in \{-1,1\}$ are independent Rademacher random variables, i.e., $P(\Omega_i = 1) = P(\Omega_i = -1) = 1/2, \forall i$.
\end{defn}

\begin{lem}
We have 
\begin{align*}
    \mathbb{E} \left[\xi(X_1, ..., X_n) \right] \leq 2 \mathbb{E} Rad( \mathcal{H} \circ \{X_1, ..., X_n\} ).
\end{align*}
In addition, we have 
\begin{align*}
    \mathbb{E} Rad( \mathcal{H} \circ \{X_1, ..., X_n\} ) \leq \sqrt{ \frac{B}{n} }.
\end{align*}
\label{lemma:bound_in_expectation}
\end{lem}
\begin{proof}
Let $\tilde{X}_1, ..., \tilde{X}_n$ be new independent samples from $P$ and independent of $S := \{X_1, ..., X_n\}$.
We have 
\begin{align*}
    \mathbb{E} \left[\xi(X_1, ..., X_n) \right] &= \mathbb{E} \left[ \text{MMD}\left(\frac{1}{n} \sum_{i=1}^n \delta_{X_i}, P; \mathcal{F} \right) \right] \\
    &= \mathbb{E} \sup_{f \in \mathcal{H}} \left( \mathbb{E} f(X) - \frac{1}{n} \sum_{i=1}^n f(X_i) \right) \\
    &= \mathbb{E} \sup_{f \in \mathcal{H}} \frac{1}{n} \sum_{i=1}^n \mathbb{E} \left[ f(\tilde{X}_i) - f(X_i) \bigg | S \right] \\ 
    &\overset{(a)}{\leq} \mathbb{E} \mathbb{E}  \sup_{f \in \mathcal{H}} \frac{1}{n} \sum_{i=1}^n  \left( f(\tilde{X}_i) - f(X_i) \right) \\ 
    &= \mathbb{E}  \sup_{f \in \mathcal{H}} \frac{1}{n} \sum_{i=1}^n \left( f(\tilde{X}_i) - f(X_i) \right) \\ 
    &\overset{(b)}{=} \mathbb{E}  \sup_{f \in \mathcal{H}} \frac{1}{n} \sum_{i=1}^n \Omega_i \left( f(\tilde{X}_i) - f(X_i) \right) \\ 
    &\leq \mathbb{E}  \left[ \sup_{f \in \mathcal{H}} \frac{1}{n} \sum_{i=1}^n \Omega_i f(\tilde{X}_i) + \sup_{f \in \mathcal{H}} \frac{1}{n} \sum_{i=1}^n (-\Omega_i) f(X_i) \right] \\
    &= 2 \mathbb{E}  \sup_{f \in \mathcal{H}} \frac{1}{n} \sum_{i=1}^n \Omega_i f(X_i) \\
    &= 2 \mathbb{E} Rad(\mathcal{H} \circ \{X_1, ..., X_n\}). 
\end{align*}
Here $(a)$ follows from Jensen's inequality for convex function $sup$ and $(b)$ follows from $\left(f(\tilde{X}_i) - f(X_i) \right)_{1 \leq i \leq n}$ has the same distribution as $\left(\Omega_i (f(\tilde{X}_i) - f(X_i)) \right)_{1 \leq i \leq n}$. 

In addition, we have 
\begin{align*}
    \mathbb{E} Rad(\mathcal{H} \circ \{X_1, ..., X_n\}) &= \mathbb{E}  \sup_{f \in \mathcal{H}} \frac{1}{n} \sum_{i=1}^n \Omega_i f(X_i) \\ &= \mathbb{E} \sup_{f \in \mathcal{H}} \bigg \langle f, \frac{1}{n} \sum_{i=1}^n \Omega_i k(X_i, \cdot) \bigg \rangle_{\mathcal{F}} \\ 
    &\leq \mathbb{E} \sup_{f \in \mathcal{H}} \|f\|_{\mathcal{F}} \times \bigg\|  \frac{1}{n} \sum_{i=1}^n \Omega_i k(X_i, \cdot) \bigg \|_{\mathcal{F}} \text{(Cauchy-Schwartz ineq.)}\\
    &\leq \mathbb{E} \bigg\|  \frac{1}{n} \sum_{i=1}^n \Omega_i k(X_i, \cdot) \bigg \|_{\mathcal{F}} \\ 
    &= \frac{1}{n} \mathbb{E} \sqrt{  \sum_{i,j} \Omega_i \Omega_j k(X_i, X_j)  } \\
    &\leq \frac{1}{n} \sqrt{ \mathbb{E} \sum_{i,j} \Omega_i \Omega_j k(X_i, X_j)  } \text{ (Jensen's inequality)}\\ 
    &= \frac{1}{n} \sqrt{ \mathbb{E} \sum_{i=1}^n k(X_i,X_i) } \leq \frac{1}{n} \sqrt{ nB} = \sqrt{\frac{B}{n}}.
\end{align*}
The last equality is due to that $\Omega_i^2 = 1, \forall i$ and $\mathbb{E} \left[ \Omega_i \Omega_j \right] = 0, \forall i \neq j$.
\end{proof}

\begin{proof}[Proof of Proposition \ref{prop:mmd_final_bound}]
Now, we are ready to prove Proposition \ref{prop:mmd_final_bound}. 
For any $\{x_i\}_{i=1}^n \in \mathcal{X}^n$, we have 
\begin{align*}
    \delta_i \xi(x_1, ..., x_n) 
    &\leq \sup_{z} \bigg \| \frac{1}{n} \sum_{j=1, j \neq i}^n k(x_j, \cdot) + \frac{k(z, \cdot)}{n} - \int k(x, \cdot) dP(x) \bigg \|_{\mathcal{F}} \\ 
    &- \inf_{z} \bigg \| \frac{1}{n} \sum_{j=1, j \neq i}^n k(x_j, \cdot) +  \frac{k(z, \cdot)}{n} - \int k(x, \cdot) dP(x) \bigg \|_{\mathcal{F}} \\ 
    &\leq \sup_{z,z'} \bigg \| \frac{k(z, \cdot)}{n} - \frac{k(z', \cdot)}{n} \bigg\|_{\mathcal{F}} \leq \frac{2\sqrt{B}}{n}. 
\end{align*}
Thus, it follows from McDiarmid's inequality that 
\begin{align}
    P( \xi(X_1, ...,X_n) - \mathbb{E} \xi(X_1,...,X_n) \geq \epsilon) \leq \exp \left( - \frac{2 n \epsilon^2}{B} \right). 
    \label{eq:bound_in_probability}
\end{align}
Equation (\ref{eq:bound_in_probability}) and Lemma \ref{lemma:bound_in_expectation} immediately yield Proposition \ref{prop:mmd_final_bound}.
\end{proof}




\newpage 

\chapter{Offline Reinforcement Learning with Deep ReLU Networks}
\label{chap:five}
In this chapter, we consider the third challenge of this thesis about statistical efficiency of offline RL with infinitely large state spaces. In particular, we study the statistical theory of offline reinforcement learning (RL) with deep ReLU network function approximation. We analyze a variant of fitted-Q iteration (FQI) algorithm under a new dynamic condition that we call Besov dynamic closure, which encompasses the conditions from prior analyses for deep neural network function approximation. Under Besov dynamic closure, we prove that the FQI-type algorithm enjoys the sample complexity of $\tilde{\mathcal{O}}\left(  \kappa^{1 + d/\alpha} \cdot \epsilon^{-2 - 2d/\alpha} \right)$ where $\kappa$ is a distribution shift measure, $d$ is the dimensionality of the state-action space, $\alpha$ is the (possibly fractional) smoothness parameter of the underlying MDP, and $\epsilon$ is a user-specified precision. This is an improvement over the  sample complexity of $\tilde{\mathcal{O}}\left( K \cdot \kappa^{2 + d/\alpha} \cdot \epsilon^{-2 - d/\alpha}  \right)$ in the prior result \citep{DBLP:journals/corr/abs-1901-00137} where $K$ is an algorithmic iteration number which is arbitrarily large in practice. Importantly, our sample complexity is obtained under the new general dynamic condition and a data-dependent structure where the latter is either ignored in prior algorithms or improperly handled by prior analyses. This is the first comprehensive analysis for offline RL with deep ReLU network function approximation under a general setting. This chapter is based on our paper \citep{nguyentang2021sample}. 

\section{Introduction}
Offline reinforcement learning \citep{levine2020offline} is a practical paradigm of reinforcement learning (RL) where logged experiences are abundant but a new interaction with the environment is limited or even prohibited. The fundamental offline RL problems are how well previous experiences could be used to evaluate a new target policy, known as off-policy evaluation (OPE) problem, or to learn the optimal policy, known as off-policy learning (OPL) problem. We study these offline RL problems with infinitely large state spaces, where the agent must use function approximation such as deep neural networks to generalize across states from an offline dataset without any further exploration. Such problems form the core of modern RL in practical settings, but relatively few work provide a comprehensive and adequate analysis of the statistical efficiency for the problems. 

On the theoretical side, predominant sample efficiency results in offline RL focus on tabular environments with small finite state spaces \citep{DBLP:conf/aistats/YinW20,DBLP:conf/aistats/YinBW21,yin2021characterizing}, but as these methods scale with the number of states, they are infeasible for infinitely large state space settings. While this tabular setting has been extended to large state spaces via {linear} environments \citep{DBLP:journals/corr/abs-2002-09516,tran2021combining}, the linearity assumption often does not hold for many RL problems in practice. More relevant theoretical progress has been achieved for more complex environments with general and deep neural network function approximations, but these results are either inadequate or relatively disconnected from practical settings \citep{DBLP:journals/jmlr/MunosS08,DBLP:journals/corr/abs-1901-00137,DBLP:conf/icml/0002VY19}. In particular, their finite-sample results either (i) depend on a so-called inherent Bellman error \citep{DBLP:journals/jmlr/MunosS08,DBLP:conf/icml/0002VY19}, which could be arbitrarily large or uncontrollable in practice, (ii) 
 avoid the data-dependent structure in their algorithms at the cost of losing sample efficiency \citep{DBLP:journals/corr/abs-1901-00137} or improperly ignore it in their analysis \citep{DBLP:conf/icml/0002VY19}, or (iii) rely on relatively strong dynamics assumption  \citep{DBLP:journals/corr/abs-1901-00137}.





In this chapter, we study a variation of fitted-Q iteration (FQI) \citep{bertsekas1995dynamic,introRL} for the offline RL problems where we approximate the target $Q$-function from an offline data using a deep ReLU network. The algorithm is appealingly simple: it iteratively estimates the target $Q$-function via regression on the offline data and the previous estimate. This procedure, which intuitively does the best it could with the available offline data, forms the core of many current offline RL methods. With linear function approximation, \cite{DBLP:journals/corr/abs-2002-09516} show that this procedure yields a minimax-optimal sample efficient algorithm, provided the environment dynamics satisfy certain linear properties. While their assumptions generalize the tabular settings, they are restrictive for more complex environment dynamics where non-linear function approximation is required. Moreover, as they highly exploit the linearity structure, it is unclear how their analysis can accommodate non-linear function approximation such as deep ReLU networks. 

In this chapter, we provide the statistical theory of a FQI-type algorithm for both OPE and OPL problems with deep ReLU networks. In particular, we provide the first comprehensive analysis for offline RL under deep ReLU network function approximation. We achieve this generality in our result via two novel considerations. First, we introduce \textit{Besov dynamic closure} which is, to our knowledge, the most general assumption that encompasses the previous dynamic assumptions in offline RL.In particular, our Besov dynamic closure reduces into H\"older smoothness and Sobolev smoothness conditions as special cases. Moreover, the MDP under the Besov dynamic closure needs not be continuous, differentiable or spatially homogeneous in smoothness. 
Second, as each value estimate in a regression-based offline RL algorithm depends on the previous estimates and the entire offline dataset, a complicated data-dependent structure is induced. This data-dependent structure plays a central role in the statistical efficiency of the algorithm. While prior results ignore the data-dependent structure, either in their algorithm or their analysis, resulting a loss of sample efficiency or improper analysis, respectively, we consider it in a FQI-type algorithm and effectively handle it in our analysis (this is discussed further in Section \ref{section:background_and_related_literature} and \ref{section:main_results}). Under these considerations, we establish the sample complexity of offline RL with deep ReLU network function approximation that is both more general and more sample-efficient than the prior results, as summarized in Table \ref{tab:compare_literature} which will be discussed in details in Subsection \ref{subsect:main_result}. Moreover, as our technical proof combining a uniform-convergence analysis and local Rademacher complexities with a localization argument is sufficiently general and effective in handling complex function approximations, our proof could be of independent interest for other offline RL methods with non-linear function approximation. Our contributions for this chapter can be summarized as 
\begin{itemize}
    \item Introduce a new dynamic condition, namely Besov dynamic closure, that encompasses the dynamic conditions in the prior works; 
    \item Provide the first comprehensive analysis of sample complexity of offline RL with deep ReLU network function approximation under a data-dependent structure and the Besov dynamic closure.
\end{itemize}

\section{Related Work}
\label{section:background_and_related_literature}
The majority of the theoretical results for offline RL focus on tabular settings and mostly on OPE task where the state space is finite and an importance sampling -related approach is possible \citep{first_is,dudik2011doubly,jiang2015doubly,thomas2016data,farajtabar2018more,kallus2019double}. The main drawback of the importance sampling -based approach is that it suffers high variance in long horizon problems. The high variance problem is later mitigated by the idea of formulating the OPE problem as a density ratio estimation problem  \citep{DBLP:conf/nips/LiuLTZ18,nachum2019dualdice,Zhang2020GenDICEGO,Zhang2020GradientDICERG,Nachum2019AlgaeDICEPG} but these results do not provide sample complexity guarantees. The sample efficiency guarantees for offline RL are obtained in tabular settings in \citep{xie2019optimal,DBLP:conf/aistats/YinW20,DBLP:conf/aistats/YinBW21,yin2021characterizing}. A lower bound for tabular offline RL is obtained in \citep{DBLP:conf/icml/JiangL16} which in particular show a Cramer-Rao lower bound for discrete-tree MDPs. 

For the function approximation setting, as the state space of MDPs is often infinite or continuous, some form of function approximation is deployed in approximate dynamic programming such as fitted Q-iteration, least squared policy iteration \citep{bertsekas1995neuro,DBLP:conf/atal/JongS07,DBLP:journals/jmlr/LagoudakisP03,DBLP:conf/icml/GrunewalderLBPG12,DBLP:conf/icml/Munos03,DBLP:journals/jmlr/MunosS08,10.1007/s10994-007-5038-2,DBLP:conf/icml/TosattoPDR17}, and fitted Q-evaluation (FQE) \citep{DBLP:conf/icml/0002VY19}. A recent line of work studies offline RL in non-linear function approximation (e.g, general function approximation and deep neural network function approximation) \citep{DBLP:conf/icml/0002VY19,DBLP:journals/corr/abs-1901-00137}. In particular, \citet{DBLP:conf/icml/0002VY19} provide an error bound of OPE and OPL with general function approximation but they ignore the data-dependent structure in the FQI-type algorithm, resulting in an improper analysis. 
Moreover, their error bounds depend on the inherent Bellman error that can be large and controllable in practical settings. More closely related to our work is \citep{DBLP:journals/corr/abs-1901-00137} which also considers deep neural network approximation. In particular, \citet{DBLP:journals/corr/abs-1901-00137} focus on analyzing deep Q-learning using a fresh batch of data for each iteration. Such approach is considerably sample-inefficient in offline RL as it undesirably does not leverage the past data. As a result, their sample complexity scales with the number of iterations $K$ which is very large in practice. In addition, they rely on a relatively restricted smoothness assumption of the underlying MDPs that hinders their results from being widely applicable in more practical settings. We summarize the key differences between our work and the prior results in Table \ref{tab:compare_literature} which will be elaborated further in Subsection \ref{subsect:main_result}. 


Since the initial version of this chapter appeared, a concurrent work studies offline RL with general function approximation via local Rademacher complexities \citep{duan2021risk}. While both works independently have the same idea of using local Rademacher complexities as a tool to study sample complexities in offline RL, our work differs from \citep{duan2021risk} in three main aspects. First, we focus on infinite-horizon MDPs while \citep{duan2021risk} work in finite-horizon MDPs. Second, 
we focus on a practical setting of deep neural network function approximation with an explicit sample complexity while the sample complexity in \citep{duan2021risk} depends on the critical radius of local Rademacher complexity. Bounding the critical radius for a complex model under the data-dependent structure is highly non-trivial. \citet{duan2021risk} provide the specialized sample complexity for finite classes, linear classes, kernel spaces and sparse linear spaces but it is unclear how their result applies to more complex models such as a deep ReLU network. Importantly, we propose a new Besov dynamic closure and a uniform-convergence argument which appear absent in \citet{duan2021risk}.

\section{Preliminaries}
\label{section:preliminaries}
We consider reinforcement learning in an infinite-horizon discounted Markov decision process (MDP) with possibly infinitely large state space $\mathcal{S}$, continuous action space $\mathcal{A}$, initial state distribution $\rho \in \mathcal{P}(\mathcal{S})$, transition operator $P: \mathcal{S} \times \mathcal{A} \rightarrow \mathcal{P}(\mathcal{S})$, reward distribution $R: \mathcal{S} \times \mathcal{A} \rightarrow \mathcal{P}([0,1])$, and a discount factor $\gamma \in [0,1)$. Here we denote by $\mathcal{P}(\Omega)$ the set of probability measures supported in domain $\Omega$. For notational simplicity, we assume that $\mathcal{X} := \mathcal{S} \times \mathcal{A} \subseteq [0,1]^d$ but our main conclusions do not change when $\mathcal{A}$ is finite. 

A policy $\pi: \mathcal{S} \rightarrow \mathcal{P}(\mathcal{A})$ induces a distribution over the action space conditioned on states. The $Q$-value function for policy $\pi$ at state-action pair $(s,a)$, denoted by $Q^{\pi}(s,a) \in [0,1]$, is the expected discounted total reward the policy collects if it initially starts in the state-action pair, 
\begin{align*}
    Q^{\pi}(s,a) &:= \mathbb{E}_{\pi} \left[ \sum_{t=0}^{\infty} \gamma^t r_t | s_0 = s, a_0 = a \right], 
\end{align*}
where $r_t \sim R(s_t, a_t), a_t \sim \pi(\cdot|s_t)$, and $s_t \sim P(\cdot|s_{t-1}, a_{t-1})$. The value for a policy $\pi$ is simply $V^{\pi} = \mathbb{E}_{s \sim \rho, a \sim \pi(\cdot|s)} \left[Q^{\pi}(s,a) \right]$, and the optimal value is $V^* = \max_{\pi} V^{\pi}$ where the maximization is taken over all stationary policies. Alternatively, the optimal value $V^*$ can be obtained via the optimal $Q$-function $Q^* = \max_{\pi} Q^{\pi}$ as $V^* = \mathbb{E}_{s \sim \rho} \left[ \max_{a} Q^*(s,a) \right]$. Denote by $T^{\pi}$ and $T^*$ the Bellman operator and the optimality Bellman operator, i.e., for any $f: \mathcal{S}\times \mathcal{A} \rightarrow \mathbb{R}$
\begin{align*}
    [T^{\pi} f](s,a) &= \mathbb{E}_{r \sim R(s,a)}[r] + \gamma \mathbb{E}_{s' \sim P(\cdot|s,a), a' \sim \pi(\cdot|s')} \left[ f(s',a') \right] \\
    [T^* f](s,a) &= \mathbb{E}_{r \sim R(s,a)}[r] + \gamma \mathbb{E}_{s' \sim P(\cdot|s,a)} \left[ \max_{a'} f(s',a') \right], 
\end{align*}
we have $T^{\pi} Q^{\pi} = Q^{\pi}$ and $T^* Q^* = Q^*$.

We consider the offline RL setting where the agent cannot explore further the environment but has access to a fixed logged data $\mathcal{D} = \{(s_i, a_i, s'_i, r_i)\}_{i=1}^n$ collected a priori by certain behaviour policy $\eta$ where $(s_i, a_i) \overset{i.i.d.}{\sim} \mu(\cdot, \cdot) := \frac{1}{1 - \gamma} \sum_{t=0}^{\infty} \gamma^t P(s_t = \cdot, a_t = \cdot| \rho, \eta), s'_i \sim P(\cdot| s_i, a_i)$ and $r_i \sim R(s_i, a_i)$. Here $\mu$ is the (sampling) state-action visitation distribution. The goal of OPE and OPL are to estimate $V^{\pi}$ and $V^*$, respectively from $\mathcal{D}$, and in this chapter we measure performance by sub-optimality gaps. 

\noindent \textbf{For OPE}. Given a fixed target policy $\pi$, for any value estimate $\hat{V}$ computed from the offline data $\mathcal{D}$, the sub-optimality of OPE is defined as  
\begin{align*}
    \text{SubOpt}(\hat{V}; \pi) = |V^{\pi} - \hat{V}|. 
\end{align*}

\noindent \textbf{For OPL}. For any estimate $\hat{\pi}$ of the optimal policy $\pi^*$ that is learned from the offline data $\mathcal{D}$, we define the sup-optimality of OPL as 
\begin{align*}
    \text{SubOpt}(\hat{\pi}) = \mathbb{E}_{\rho} \left[ V^*(s) - Q^*(s, \hat{\pi}(s)) \right], 
\end{align*}
where $\mathbb{E}_{\rho}$ is the expectation with respect to $s \sim \rho$. 

\subsection{Deep ReLU Networks as Function Approximation}
\label{subsection:deep_relu_net}
In practice, the state space is often very large and complex, and thus function approximation is required to ensure generalization across different states. Deep networks with the ReLU activation offer a rich class of parameterized functions with differentiable parameters. Deep ReLU networks are state-of-the-art in many applications, e.g., \citep{krizhevsky2012imagenet,mnih2015human}, including offline RL with deep ReLU networks that can yield superior empirical performance \citep{voloshin2019empirical}. In this section, we describe the architecture of deep ReLU networks and the associated function space which we directly work on. A $L$-height, $m$-width ReLU network on $\mathbb{R}^d$ takes the form of 
\begin{align*}
    f_{\theta}^{L,m}(x) = W^{(L)} \sigma \left(  W^{(L-1)} \sigma \left( \hdots \sigma \left(W^{(1)} \sigma(x) + b^{(1)} \right) \hdots \right) + b^{(L-1)} \right) + b^{(L)},
\end{align*}
where $W^{(L)} \in \mathbb{R}^{1 \times m}, b^{(L)} \in \mathbb{R}, W^{(1)} \in \mathbb{R}^{m \times d}, b^{(1)} \in \mathbb{R}^m$, $W^{(l)} \in \mathbb{R}^{m \times m}, b^{(l)} \in \mathbb{R}^m, \forall 1 < l < L$, $\theta = \{W^{(l)}, b^{(l)}\}_{1 \leq l \leq L}$, and $\sigma(x) = \max\{x, 0\}$ is the (element-wise) ReLU activation. We define $\Phi(L, m, S,B)$ as the space of $L$-height, $m$-width ReLU functions $f_{\theta}^{L,m}(x)$ with sparsity constraint $S$, and norm constraint $B$, i.e., $\sum_{l=1}^L (\|W^{(l)}\|_0 + \| b^{(l)} \|_0) \leq S, \max_{1 \leq l \leq L} \|W^{(l)} \|_{\infty} \lor \|b^{(l)} \|_{\infty} \leq B$ where $\|\cdot\|_0$ is the $0$-norm, i.e., the number of non-zero elements, and $a \lor b = \max\{a,b\}$. Finally, for some $L,m \in \mathbb{N}$ and $S,B \in (0,\infty)$, we define the unit ball of ReLU network function space $\mathcal{F}_{NN}$ as 
\begin{align*}
    \mathcal{F}_{NN} := \bigg\{ f \in \Phi(L,m,S,B):  \| f \|_{\infty} \leq 1 \bigg\}. 
\end{align*}
We further write $\mathcal{F}_{NN}(\mathcal{X})$ to emphasize the domain $\mathcal{X}$ of deep ReLU functions in $\mathcal{F}_{NN}$ but often use $\mathcal{F}_{NN}$ when the domain context is clear. 

The main benefit of deep ReLU networks is that in standard non-parametric regression, they outperform any non-adaptive linear estimator due to their higher adaptivity to spatial inhomogeneity \citep{suzuki2018adaptivity}. Later, we show that this adaptivity benefit of deep ReLU networks transfers to the value regression problem in offline RL even though the value regression is much more complex than the standard non-parametric regression.

\subsection{Regularity}
In this section, we define a function space for the target functions for which we study offline RL. Note that a regularity assumption on the target function is necessary to obtain a nontrivial rate of convergence \citep{DBLP:books/daglib/0035701}. A common way to measure regularity of a function is through the $L^p$-norm of its local oscillations (e.g., of its derivatives if they exist). This regularity notion encompasses the classical Lipschitz, H\"older and Sobolev spaces. In particular in this work, we consider Besov spaces. Besov spaces allow \textit{fractional} smoothness that describes the regularity of a function more precisely and generalizes the previous smoothness notions. Besov spaces allow \textit{fractional} smoothness that describes the regularity of a function more precisely and generalizes the previous smoothness notions.   
There are several ways to characterize the smoothness in Besov spaces. Here, we pursue a characterization via moduli of smoothness as it is more intuitive, following \citep{gine2016mathematical}. 

\begin{defn}[\textit{Moduli of smoothness}]
For a function $f \in L^p(\mathcal{X})$ for some $p \in [1, \infty]$, we define its $r$-th modulus of smoothness as 
\begin{align*}
    \omega_r^{t,p}(f) := \sup_{0 \leq h \leq t} \| \Delta_h^r(f) \|_p, t > 0, r \in \mathbb{N}, 
\end{align*}
where the $r$-th order translation-difference operator $\Delta_h^r = \Delta_h \circ \Delta_h^{r-1}$ is recursively defined as 
\begin{align*}
    \Delta_h^r(f)(\cdot) &:= (f(\cdot + h) - f(\cdot))^r = \sum_{k=0}^r {{r}\choose{k}} (-1)^{r-k} f(\cdot + k\cdot h). 
\end{align*}
\end{defn}

\begin{rem}
The quantity $\Delta_h^r(f)$ captures the local oscillation of function $f$ which is not necessarily differentiable. In the case the $r$-th order weak derivative $D^r f$ exists and is locally integrable, we have 
\begin{align*}
    \lim_{ h \rightarrow 0} \frac{\Delta^r_h(f)(x)}{h^r} = D^r f(x), \frac{\omega^{t,p}_r(f)}{t^r} \leq \| D^r f \|_p \text{ and } \frac{\omega^{t,p}_{r+r'}(f)}{t^r} \leq \omega^{t,p}_{r'}(D^r f). 
\end{align*}
\end{rem}

\begin{defn}[\textit{Besov spac}e $B^{\alpha}_{p,q}(\mathcal{X})$]
For $1 \leq p,q \leq \infty$ and $\alpha > 0$, we define the norm $\|\cdot\|_{B^{\alpha}_{p,q}}$  of the Besov space $B^{\alpha}_{p,q}(\mathcal{X})$ as $\|f\|_{B^{\alpha}_{p,q}} := \|f\|_p + |f |_{B^{\alpha}_{p,q}}$ where 
\begin{align*}
    |f |_{B^{\alpha}_{p,q}} := 
    \begin{cases}
    \left( \int_{0}^{\infty} (\frac{\omega_{\floor{\alpha} + 1}^{t,p}(f)}{t^{\alpha}})^q \frac{dt}{t} \right)^{1/q}, & 1 \leq q < \infty, \\ 
    \sup_{t > 0} \frac{\omega_{\floor{\alpha} + 1}^{t,p}(f)}{t^{\alpha}}, & q = \infty, 
    \end{cases}
\end{align*}
is the Besov seminorm. Then, $B^{\alpha}_{p,q} := \{ f \in L^p(\mathcal{X}) : \|f \|_{B^{\alpha}_{p,q}} < \infty \}$. 
\end{defn}

Intuitively, the Besov seminorm $|f |_{B^{\alpha}_{p,q}}$ roughly describes the $L^q$-norm of the $l^p$-norm of the $\alpha$-order smoothness of $f$. Having defined Besov spaces, a natural question is what properties Besov spaces have and how these spaces are related to other function spaces considered in the current literature of offline RL? It turns out that Besov spaces are considerably general that encompass H\"older spaces and Sobolev spaces as well as functions with spatially inhomogeneous smoothness \citep{Triebel1983TheoryOF,besovbook,suzuki2018adaptivity,primer_besov,Nickl2007BracketingME}. 
We summarize the key intriguing characteristics of Besov spaces and their relation with other spaces:
\begin{itemize}
 
    \item (Monotonicity in $q$) For $1 \leq p \leq \infty, 1 \leq q_1 \leq q_2 \leq \infty$ and $\alpha \in \mathbb{R}$, $B^{\alpha}_{p,q_1}(\mathcal{X}) \hookrightarrow B^{\alpha}_{p,q_2}(\mathcal{X})$; 
    
     \item (With $L^p$ spaces) $L^2(\mathcal{X}) \hookrightarrow B^0_{2,2}(\mathcal{X})$, $B^0_{p,1}(\mathcal{X}) \hookrightarrow L^p(\mathcal{X}) \hookrightarrow B^0_{p,\infty}(\mathcal{X})$ for $1 \leq p \leq \infty$, and $B^{\alpha}_{p,q}(\mathcal{X}) \hookrightarrow L^r(\mathcal{X})$ for $\alpha > d(1/p - 1/r)_{+}$  where $r = \floor{\alpha}+1$;
     
     \item (With $C^0(\mathcal{X})$) $B^{\alpha}_{p,q}(\mathcal{X}) \hookrightarrow C^0(\mathcal{X})$ for $1 \leq p,q \leq \infty, \alpha > d/p$; 
    
    \item (With Sobolev spaces) $B^m_{2,2}(\mathcal{X}) = W^m_2(\mathcal{X})$ for $m \in \mathbb{N}$; 
    
    \item (With H\"older spaces) $B^{\alpha}_{\infty, \infty}(\mathcal{X}) = C^{\alpha}(\mathcal{X})$ for $\alpha = (0, \infty) \backslash \mathbb{N}$. 
\end{itemize}
In particular, the Besov space $B^{\alpha}_{p,q}$ reduces into the H\"older space $C^{\alpha}$ when $p=q=\infty$ and $\alpha$ is  positive and non-integer while it reduces into the Sobolev space $W^{\alpha}_2$ when $p=q=2$ and $\alpha$ is a positive integer. We further consider the unit ball of $B^{\alpha}_{p,q}(\mathcal{X})$: 
\begin{align*}
    \bar{B}^{\alpha}_{p,q}(\mathcal{X}) := \{g \in B^{\alpha}_{p,q}: \|g \|_{B^{\alpha}_{p,q}} \leq 1 \text{ and } \|g\|_{\infty} \leq 1\}. 
\end{align*}
To obtain a non-trivial guarantee, certain assumptions on the distribution shift and the MDP regularity are necessary. Here, we introduce such assumptions. The first assumption is a common restriction that handles distribution shift in offline RL.
\begin{assumption}[\textit{Concentration coefficient}]
\label{assumption:concentration_coefficient}
There exists $\kappa_{\mu} < \infty$ such that $ \| \frac{d\nu}{d\mu} \|_{\infty} \leq \kappa_{\mu}$ for any \textit{realizable} distribution $\nu$.  
\end{assumption}
Here, a distribution $\nu$ is said to be realizable if there exists $t \geq 0$ and policy $\pi_1$ such that $\nu(s,a) = \mathbb{P}(s_t = s, a_t=a|s_1 \sim \rho, \pi_1), \forall s,a$. Intuitively, the finite $\kappa_{\mu}$ in Assumption \ref{assumption:concentration_coefficient} asserts that the sampling distribution $\mu$ is not too far away from any realizable distribution uniformly over the state-action space. $\kappa_{\mu}$ is finite for a reasonably large class of MDPs, e.g., for any finite MDP, any MDP with bounded transition kernel density, and equivalently any MDP whose top-Lyapunov exponent is negative \citep{DBLP:journals/jmlr/MunosS08}. \citet{DBLP:conf/icml/ChenJ19} further provide natural problems with  rich observations generated from hidden states that has low concentration coefficients. These suggest that low concentration coefficients can be found in fairly many interesting problems in practice. 


%


\begin{assumption}[\textit{Besov dynamic closure}] $\forall f \in \mathcal{F}_{NN}(\mathcal{X}), \forall \pi, T^{\pi}f \in 
\bar{B}^{\alpha}_{p,q}(\mathcal{X})$ for some $ p,q \in [1, \infty]$ and  $\alpha > \frac{d}{p \land 2}$.
\label{assumption:completeness}
\end{assumption}

The assumption signifies that for any policy $\pi$, the Bellman operator $T^{\pi}$ applied on any ReLU network function in $\mathcal{F}_{NN}(\mathcal{X})$ results in a Besov function in $\bar{B}^{\alpha}_{p,q}(\mathcal{X})$. Moreover, as $T^{\pi_f} f = T^* f$ where $\pi_f$ is the greedy policy w.r.t. $f$, Assumption \ref{assumption:completeness} also implies that $T^*f \in \bar{B}^{\alpha}_{p,q}(\mathcal{X})$ if $f \in \mathcal{F}_{NN}(\mathcal{X})$. This kind of assumption is relatively standard and common in the offline RL literature \citep{DBLP:conf/icml/ChenJ19}. A natural example for this assumption to hold is when both the expected reward function $r(s,a)$ and the transition density $P(s'|s,a)$ for each fixed $s'$ are Besov functions. This specific example posits a general smoothness to the considered MDP which can be considered a way to impose restrictions in MDPs. We remark again that restrictions in MDPs are necessary to obtain non-trivial convergence rates. 

Importantly, as Besov spaces are more general than H\"older and Sobolev spaces, our Besov dynamic closure assumption is considerably general that encompasses the dynamic conditions considered in prior results \citep{DBLP:journals/corr/abs-1901-00137}. In particular, as remarked earlier, the Besov space $B^{\alpha}_{p,q}$ reduces into the H\"older space $C^{\alpha}$ and Sobolev space $W^{\alpha}_2$ at $p=q=\infty,\alpha \in (0, \infty) \backslash \mathbb{N}$, and at $p=q=2, \alpha \in \mathbb{N}$, respectively. Moreover, our dynamic assumption only requires the boundedness of a very general notion of local oscillations of the underlying MDP. In particular, the underlying MDP can be discontinuous or non-differentiable (e.g., when $\alpha \leq 1/2$ and $p=2$), or even have spatially inhomogeneous smoothness (e.g., when $p < 2$). These generality properties were not possible to be considered in the prior results. 

The condition $\alpha > \frac{d}{p \land 2}$ guarantees a finite bound for the compactness  and the (local) Rademacher complexity of the considered Besov space. When $p < 2$ (thus the condition above becomes $\alpha > d/p$), a function in the corresponding Besov space contains both spiky parts and smooth parts, i.e., the Besov space has \textit{inhomogeneous} smoothness \citep{suzuki2018adaptivity}. In particular, when $\alpha > d/p$, each \textit{equivalence} class $[f]_{\lambda}, f \in B^{\alpha}_{p,q}(\mathbb{R}^d)$, i.e., modulo equality $\lambda$-almost everywhere, contains a unique continuous representative. In addition, this representative has partial derivatives of order at least $\alpha - d/p$; thus $\alpha - d/p$ is called the \textit{differential dimension} of the Besov space. 
\section{Algorithm and Main Result}
\label{section:main_results}

    
    
    
    
    

\subsection{Algorithm}
Now we turn to the main algorithm and the main result. We study least-squares value iteration (LSVI) for both OPE and OPL with the pseudo-code presented in Algorithm \ref{alg:LSVI} where we denote $\rho^{\pi}(s,a) = \rho(s) \pi(a|s)$. The algorithm is nearly identical to \citep{DBLP:journals/corr/abs-2002-09516} but with deep neural network function approximation instead of linear models. As such, it can be considered as a generalization. 

\begin{algorithm}[h]
   \caption{Least-squares value iteration}
\label{alg:LSVI}
\begin{algorithmic}[1]
   \STATE Initialize $Q_0 \in \mathcal{F}_{NN}$.
   \FOR{$k=1$ {\bfseries to} $K$}
   \STATE If \textbf{OPE} (for a fixed policy $\pi$):  $y_i \leftarrow r_i + \gamma \int_{\mathcal{A}} Q_{k-1}(s'_i, a) \pi(da|s'_i), \forall i$
   \label{lsvi:ope}
 
   \STATE If \textbf{OPL}:  $y_i \leftarrow r_i + \gamma \max_{a' \in \mathcal{A}} Q_{k-1}(s'_i, a'), \forall i $
   \label{lsvi:learning}

   \STATE $Q_k \leftarrow \argmin_{f \in \mathcal{F}_{NN}} \frac{1}{n} \sum_{i=1}^n (f(s_i, a_i) - y_i)^2 $
   \label{lsvi:opt}
   \ENDFOR
   \STATE
   If \textbf{OPE}, return
   $$V_K = \|Q_K \|_{ \rho^{\pi}} = \sqrt{ \mathbb{E}_{\rho(s) \pi(a|s)} \left[ Q_K(s,a)^2 \right]}$$
   \STATE If \textbf{OPL}, return the greedy policy $\pi_K$ w.r.t. $Q_K$.
\end{algorithmic}
\end{algorithm}

The idea of LSVI is to do the best it could with all the offline data using least-squares regression over a function space. The algorithm arbitrarily initializes $Q_0 \in \mathcal{F}_{NN}$ and iteratively computes $Q_k$ as follows: at each iteration $k$, the algorithm constructs a new regression data $\{(x_i, y_i)\}_{i=1}^n$ where the covariates $x_i$ are $(s_i, a_i)$ and the Bellman targets $y_i$ are computed following dynamic programming style. In particular, depending on whether this is an OPE or OPL problem, $y_i$ are computed according to line \ref{lsvi:ope} and line \ref{lsvi:learning} of Algorithm \ref{alg:LSVI}, respectively. It then fits the function class $\mathcal{F}_{NN}$ to the constructed regression data by minimizing the mean squared error at line \ref{lsvi:opt}. This type of algorithm belongs to the fitted Q-iteration family \citep{DBLP:journals/jmlr/MunosS08,DBLP:conf/icml/0002VY19} that iteratively uses least-squares (value) regression to estimate the value functions. The main difference in the algorithm is here we use deep neural networks as function approximation for generalization to unseen states and actions in a complex MDP.  

On the computational side, solving the non-convex optimization at line \ref{lsvi:opt} of Algorithm \ref{alg:LSVI} can be highly involved and stochastic gradient descent is a dominant optimization method for such a task in deep learning. In particular, GD is guaranteed to converge to a global minimum under certain structural assumptions \citep{nguyen2021proof}. Here, as we focus on the statistical properties of LSVI, not on the optimization problem, we assume that the minimizer at line \ref{lsvi:opt} is attainable. Such a oracle assumption is common when analyzing the statistical properties of an RL algorithm with non-linear function approximation \citep{DBLP:journals/corr/abs-1901-00137,DBLP:conf/icml/ChenJ19,duan2021risk,DBLP:journals/corr/abs-1912-04136,DBLP:journals/corr/abs-2005-10804,DBLP:journals/corr/abs-2102-00815}. For the optimization problem of deep neural networks, we refer the readers to its  vast body of literature \citep[see e.g.][and references therein]{sun2019optimization}.


\subsection{Data-dependent Structure}
We remark the {data-dependent structure} in Algorithm \ref{alg:LSVI}. The target variable $y_i$ computed at line \ref{lsvi:ope} and line \ref{lsvi:learning} of the algorithm depends on the previous estimate $Q_{k-1}$ which in turn depends on the covariate $x_i := (s_i, a_i)$.
This induces a complex data-dependent structure across all iterations where the current estimate depends on all the previous estimates and the past data. In particular, one of the main difficulties caused by such data-dependent structure is that conditioned on each $x_i$, the target variable $y_i$ is no longer centered at $[T^*Q_{k-1}](x_i)$ for OPL (or at $[T^{\pi}Q_{k-1}](x_i)$ for OPE, respectively), i.e., $\mathbb{E} \left[ [T^*Q_{k-1}](x_i) - y_i | x_i \right] \neq 0$. This data-dependent structure hinders the use of any standard non-parametric regression analysis and concentration phenomenon typically used in supervised learning. Prior results either improperly ignore the data-dependent structure in their analysis \citep{DBLP:conf/icml/0002VY19} or directly avoid it by estimating each $Q_k$ on a separate subset of the original data \citep{DBLP:journals/corr/abs-1901-00137}. While the latter removes the data-dependent structure, it pays the undesirable cost of scaling the sample complexity with the number of iterations $K$ in the algorithm as it requires splitting the original data into $K$ disjoint subsets. In our work, we consider the data-dependent structure in LSVI and effectively handle it via a uniform-convergence argument and local Rademacher complexities. While our uniform-convergence argument overcomes the data-dependent structure by considering \textit{deterministic} coverings of the target function space $T^* \mathcal{F}_{NN}$ without the need for breaking the original data into $K$ disjoint subsets, local Rademacher complexities localize an original function space into local data-dependent balls which can then be gracefully integrated with the uniform-convergence argument and the complicated deep ReLU function approximation. The technical details for our handling method of the data-dependent structure are presented in Section \ref{chap4_section_proof}. 


\subsection{Main Result}
\label{subsect:main_result}
Our main result is a sup-optimality bound for LSVI in both OPE and OPL settings under Assumption \ref{assumption:concentration_coefficient} and Assumption \ref{assumption:completeness}. 

\begin{thm}
Under Assumption \ref{assumption:concentration_coefficient} and Assumption \ref{assumption:completeness}, for any $\epsilon > 0, \delta \in (0,1], K > 0$, and for $n \gtrsim \left(\frac{1}{\epsilon^2} \right)^{1 + \frac{d}{\alpha}} \log^6 n + \frac{1}{\epsilon^2}(\log(1/\delta) + \log \log n)$, with probability at least $1 - \delta$, the sup-optimality of Algorithm \ref{alg:LSVI} is 
\begin{align*}
    \begin{cases}
    \text{SubOpt}(V_K; \pi) \leq \frac{ \sqrt{ \kappa_{\mu}}}{1-\gamma} \epsilon + \frac{ \gamma^{K/2}}{(1-\gamma)^{1/2}} &\text{ for OPE}, \\ 
    \text{SubOpt}(\pi_K) \leq \frac{4 \gamma \sqrt{ \kappa_{\mu}}}{(1-\gamma)^2} \epsilon + \frac{4 \gamma^{1 + K/2}}{(1-\gamma)^{3/2}} &\text{ for OPL.}
    \end{cases}
\end{align*}
In addition, the optimal deep ReLU network $\Phi(L,m,S,B)$ that obtains such sample complexity (for both OPE and OPL) satisfies 
\begin{align*}
    L \asymp \log N, m \asymp N \log N, S \asymp N, \text{ and } B \asymp N^{1/d + (2 \iota)/ (\alpha - \iota) },
\end{align*}
where $\iota := d(p^{-1} - (1 + \floor{\alpha})^{-1})_{+}, N \asymp n^{\frac{(\beta + 1/2)d}{2 \alpha + d}}$, and $\beta = (2 + \frac{d^2}{\alpha(\alpha + d)})^{-1}$.
\label{thm:sample_complexity}
\end{thm}

The result states that LSVI incurs a sub-optimality which consists of the statistical error (the first term) and the algorithmic error (the second term). While the algorithmic error enjoys the fast linear convergence to $0$, the statistical error reflects the fundamental difficulty of the problems. The statistical errors for both OPE and OPL cases are bounded by the distributional shift $\kappa_{\mu}$, the effective horizon $1/(1-\gamma)$, and the user-specified precision $\epsilon$ for $n$ satisfying the inequality given in Theorem \ref{thm:sample_complexity}. In particular, the sample complexity does not depend on the number of states as in tabular MDPs \citep{DBLP:conf/aistats/YinW20,DBLP:conf/aistats/YinBW21,yin2021characterizing} or the inherent Bellman error as in the general function approximation \citep{DBLP:journals/jmlr/MunosS08,DBLP:conf/icml/0002VY19}. Instead, it explicitly scales with the (possible fractional) smoothness $\alpha$ of the underlying MDP and the dimension $d$ of the input space. Importantly, this guarantee is established under  the data-dependent structure of the algorithm and the Besov dynamic closure encompassing the dynamic conditions of the prior results. Thus, Theorem \ref{thm:sample_complexity} is the most comprehensive result we are aware of for offline RL with deep neural network function approximation.

Moreover, to develop further intuition on our sample complexity, we compare it with the prior results. Regarding the tightness of our result, our sample complexity $ \epsilon^{-2 - 2d/\alpha}$ (ignoring the log factor and the factor pertaining to $\kappa_{\mu}$ and effective horizon) nearly matches the nonparametric regression's minimax-optimal sample complexity $\epsilon^{-2 - d/\alpha}$ \citep{kerkyacharian1992density,gine2016mathematical} even though in our case we deal with a more complicated data-dependent structure in a value iteration problem instead of a standard non-parametric regression problem. This gap is necessary and expected due to the data-dependent structure in the algorithm. We remark that it is possible to retain the rate $\epsilon^{-2 - d/\alpha}$ if we split the offline data $\mathcal{D}$ into $K$ (given in Algorithm \ref{alg:LSVI}) disjoint subsets and estimate each $Q_k$ in Algorithm \ref{alg:LSVI} using a separate disjoint subsets. This however comes at the cost that the overall sample complexity scales with $K$ which could be arbitrarily large in practice.

\begin{table*}
    \centering
     \resizebox{\textwidth}{!}{  
    \begin{tabular}{l|l|l|l|l|l}
       \textbf{Work} &  \textbf{Functions} & \textbf{Regularity} & \textbf{Tasks} & \textbf{Sample complexity}  & \textbf{Remark}\\
    \hline
    \hline 
    
    \citet{DBLP:conf/aistats/YinW20}  & Tabular & Tabular & OPE & $\tilde{\mathcal{O}}\left( \kappa \cdot |\mathcal{S}|^2 \cdot |\mathcal{A}| ^2  \cdot \epsilon^{-2} \right)$  & minimax-optimal \\ 
    \hline 
    
    \citet{DBLP:journals/corr/abs-2002-09516} & Linear & Linear & OPE & $\tilde{\mathcal{O}}\left( \kappa \cdot d \cdot \epsilon^{-2} \right) $ & minimax-optimal\\ 
    \hline 
    \citet{DBLP:conf/icml/0002VY19} & General & General & OPE/OPL & N/A  & improper analysis\\ 
    \hline 
    
    \citet{DBLP:journals/corr/abs-1901-00137} & ReLU nets & H\"older & OPL & $ \tilde{\mathcal{O}}\left( K \cdot \kappa^{2 + d/\alpha} \cdot \epsilon^{-2 - d/\alpha}  \right)$  & no data reuse  \\
    \hline 
    
    \textbf{Ours} &  \textbf{ReLU nets} & \textbf{Besov}  & OPE/OPL & $\tilde{\mathcal{O}}\left(  \kappa^{1 + d/\alpha} \cdot \epsilon^{-2 - 2d/\alpha} \right)$ & \textbf{data reuse}
    
    \end{tabular}
    }
     \caption{Recent advances in the sample complexity of offline RL with various function approximations. Here, $|\mathcal{S}|$ and $|\mathcal{A}|$ are the cardinalities of the state and action space when they are finite, $\kappa$ is a measure of distribution shift, $\epsilon$ is the user-specified precision, $d$ is the dimension of the input space, $\alpha$ is the smoothness parameter of the underlying MDP, and $K$ is the algorithmic iteration number.}
    \label{tab:compare_literature}
\end{table*}

To show the significance of our sample complexity, we summarize our result and compare it with the prior results in Table \ref{tab:compare_literature}. From the leftmost column to the rightmost one, the table describes the related works, the function approximations being employed, the regularity conditions considered to establish theoretical guarantees, the offline RL tasks considered, the sample complexity obtained, and the important remarks or features of each work. Specifically, the ``data reuse'' in Table \ref{tab:compare_literature} means that an algorithm reuses the data across all iterations instead of splitting the original offline data into disjoint subsets for each iteration and the regularity column specifies the regularity assumption on the underlying MDP. Based on this comparison, we make the following observations. First, with simpler models such as tabular and linear MDPs, it requires less samples to achieve the same sub-optimality precision $\epsilon$ than more complex environments such as H\"older and Besov MDPs. This should not come as a surprise as the simpler regularities are much easier to learn but they are too strong as a condition to hold in practice. Second, as remarked earlier that Besov smoothness is more general than H\"older smoothness considered in \citep{DBLP:journals/corr/abs-1901-00137}, our setting is more practical and comprehensive as it covers more scenarios of the regularity of the underlying MDPs than the prior results. Third, our result obtains an improved sample complexity as compared to that in \citep{DBLP:journals/corr/abs-1901-00137} where we are able to get rid of the dependence on the algorithmic iteration number $K$ which can be arbitrarily large in practice. On the technical side, we provide a unifying analysis that allows us to account for the complicated data-dependent structure in the algorithm and handle the complex deep ReLU network function approximation. This can also be considered as a substantial technical improvement over \citep{DBLP:conf/icml/0002VY19} as \citet{DBLP:conf/icml/0002VY19} improperly ignores the data-dependent structure in their analysis. In addition, the result in \citep{DBLP:conf/icml/0002VY19} does not provide an explicit sample complexity as it depends on an unknown inherent Bellman error. Thus, our sample complexity is the most general result in a practical and comprehensive setting with an improved performance. 

Finally, we provide a detailed proof for Theorem \ref{thm:sample_complexity} in Section \ref{chap5_proofs}. The proof has four main components: a sub-optimality decomposition for error propagation across iterations, a Bellman error decomposition using a uniform convergence argument, a deviation analysis for least-squares value regression with deep ReLU networks using local Rademacher complexities via a localization argument, and an upper bound minimization step to obtain an optimal deep ReLU architecture.

\section{Conclusion}
\label{section:discussion}
This chapter presents the sample complexity of offline RL with deep ReLU network function approximation. We prove that the FQI-type algorithm with the data-dependent structure obtains an improved sample complexity of $\tilde{\mathcal{O}}\left(  \kappa^{1 + d/\alpha} \cdot \epsilon^{-2 - 2d/\alpha} \right)$ under a standard condition of distributional shift and a new dynamic condition namely Besov dynamic closure which encompasses the dynamic conditions considered in the prior results. Established under the data-dependent structure and the general Besov dynamic closure, our sample complexity is the most general result for offline RL with deep ReLU network function approximation. 

\section{Proofs}
\label{chap5_proofs}

We now provide a complete proof of Theorem \ref{thm:sample_complexity}. The proof has four main components: a sub-optimality decomposition for error propagation across iterations, a Bellman error decomposition using a uniform convergence argument, a deviation analysis for least squares with deep ReLU networks using local Rademacher complexities and a localization argument, and a upper bound minimization step to obtain an optimal deep ReLU architecture. 

\subsubsection*{Step 1: A sub-optimality decomposition}
The first step of the proof is a sub-optimality decomposition, stated in Lemma \ref{sub_opt_decompose}, that applies generally to any least-squares Q-iteration methods. 
\begin{lem}[\textit{A sub-optimality decomposition}]
Under Assumption \ref{assumption:concentration_coefficient},
the sub-optimality of $V_K$ returned by Algorithm \ref{alg:LSVI} is bounded as 
\begin{align*}
    \text{SubOpt}(V_K)  \leq
    \begin{cases}
    \frac{\sqrt{\kappa_{\mu}}}{1-\gamma} \displaystyle \max_{0 \leq k \leq K-1} \| Q_{k+1} - T^{\pi} Q_k \|_{\mu} + \frac{ \gamma^{K/2}}{(1-\gamma)^{1/2}} &\text{ for OPE}, \\ 
    \frac{4 \gamma \sqrt{\kappa_{\mu}}}{(1-\gamma)^2} \displaystyle \max_{0 \leq k \leq K-1} \| Q_{k+1} - T^* Q_k \|_{\mu} + \frac{4 \gamma^{1 + K/2}}{(1-\gamma)^{3/2}} &\text{ for OPL}.
    \end{cases}
\end{align*}
where we denote $\|f\|_{\mu} := \sqrt{ \int \mu(ds da) f(s,a)^2 }, \forall f: \mathcal{S} \times \mathcal{A} \rightarrow \mathbb{R}$. 
\label{sub_opt_decompose}
\end{lem}
The lemma states that the sub-optimality decomposes into a statistical error (the first term) and an algorithmic error (the second term). While the algorithmic error enjoys the fast linear convergence rate, the statistical error arises from the distributional shift in the offline data and the estimation error of the target $Q$-value functions due to finite data. Crucially, the contraction of the (optimality) Bellman operators $T^{\pi}$ and $T^*$ allows the sup-optimality error at the final iteration $K$ to propagate across all iterations $k \in [0,K-1]$. Note that this result is agnostic to any function approximation form and does not require Assumption \ref{assumption:completeness}. The result uses a relatively standard argument that appears in a number of works on offline RL \citep{DBLP:journals/jmlr/MunosS08,DBLP:conf/icml/0002VY19}. 
\begin{proof}[Proof of Lemma \ref{sub_opt_decompose}]
We will prove the sup-optimality decomposition for both settings: OPE and OPL.  

\noindent \textbf{(i) For OPE}.  
We denote the right-linear operator by $P^{\pi} \cdot : \{\mathcal{X} \rightarrow \mathbb{R} \} \rightarrow \{\mathcal{X} \rightarrow \mathbb{R} \}$ where
\begin{align*}
    (P^{\pi}f)(s,a) := \int_{\mathcal{X}} f(s',a')\pi(da'|s') P(ds'|s,a), 
\end{align*}
for any $f \in \{\mathcal{X} \rightarrow \mathbb{R} \}$. Denote Denote $\rho^{\pi}(ds da) = \rho(ds) \pi(da|s)$. Let $\epsilon_k := Q_{k+1} - T^{\pi} Q_k, \forall k \in [0,K-1]$ and $\epsilon_K = Q_0 - Q^{\pi}$. Since $Q^{\pi}$ is the (unique) fixed point of $T^{\pi}$, we have 
\begin{align*}
    Q_k - Q^{\pi} &= T^{\pi} Q_{k-1} - T^{\pi} Q^{\pi} + \epsilon_{k-1} = \gamma P^{\pi}(Q_{k-1} - Q^{\pi}) + \epsilon_{k-1}.
\end{align*}
By recursion, we have 
\begin{align*}
     Q_K - Q^{\pi} &= \sum_{k=0}^K (\gamma P^{\pi})^k \epsilon_k = \frac{1 - \gamma^{K+1}}{1 - \gamma} \sum_{k=0}^K \alpha_k A_k \epsilon_k  \\ 
\end{align*}
where $\alpha_k :=  \frac{(1-\gamma) \gamma^k}{1 - \gamma^{K+1}}, \forall k \in [K]$ and $A_k := (P^{\pi})^k,  \forall k \in [K]$. 
Note that $\sum_{k=0}^K \alpha_k = 1$ and $A_k$'s are probability kernels. Denoting by $|f|$ the point-wise absolute value $|f(s,a)|$, we have that the following inequality holds point-wise: 
\begin{align*}
    |Q_K - Q^{\pi}| \leq \frac{1-\gamma^{K+1}}{1-\gamma} \sum_{k=0}^K \alpha_k A_k |\epsilon_k|.
\end{align*}
We have 
\begin{align*}
    \|Q_K - Q^{\pi}\|^2_{\rho^{\pi}} &\leq \frac{(1-\gamma^{K+1})^2}{(1-\gamma)^2} \int \rho(ds) \pi(da|s) \left( \sum_{k=0}^K \alpha_k A_k |\epsilon_k|(s,a)\right)^2 \\ 
    &\overset{(a)}{\leq} \frac{(1-\gamma^{K+1})^2}{(1-\gamma)^2} \int \rho(ds) \pi(da|s) \sum_{k=0}^K \alpha_k  A_k^2 \epsilon_k^2(s,a)\\
    &\overset{(b)}{\leq} \frac{(1-\gamma^{K+1})^2}{(1-\gamma)^2} \int \rho(ds) \pi(da|s) \sum_{k=0}^K \alpha_k  A_k \epsilon_k^2(s,a) \\
    &\overset{(c)}{\leq} \frac{(1-\gamma^{K+1})^2}{(1-\gamma)^2} \left( \int \rho(ds) \pi(da|s) \sum_{k=0}^{K-1}\alpha_k  A_k \epsilon_k^2(s,a) + \alpha_K \right) \\
    &\overset{(d)}{\leq} \frac{(1-\gamma^{K+1})^2}{(1-\gamma)^2} \left( \int \mu(ds,da) \sum_{k=0}^{K-1} \alpha_k  \kappa_{\mu} \epsilon_k^2(s,a) + \alpha_K \right) \\ 
    &= \frac{(1-\gamma^{K+1})^2}{(1-\gamma)^2} \left(\sum_{k=0}^{K-1} \alpha_k  \kappa_{\mu}  \| \epsilon_k \|^2_{\mu} + \alpha_K \right) \\
    & \leq \frac{\kappa_{\mu}}{(1 - \gamma)^2} \max_{0 \leq k \leq K-1} \| \epsilon_k \|_{\mu}^2 + \frac{\gamma^K}{(1 - \gamma)}.
\end{align*}
The inequalities $(a)$ and $(b)$ follow from Jensen's inequality, $(c)$ follows from $\|Q_0\|_{\infty} \leq 1$,$ \|Q^{\pi}\|_{\infty} \leq 1$, and $(d)$ follows from Assumption \ref{assumption:concentration_coefficient} that $\rho^{\pi} A_k = \rho^{\pi} (P^{\pi})^k \leq \kappa_{\mu} \mu$. Thus we have 
\begin{align*}
    \text{SubOpt}(V_K; \pi) &= |V_K - V^{\pi}| \nonumber \\
    &= \bigg|\mathbb{E}_{\rho, \pi}[Q_K(s,a)] - \mathbb{E}_{\rho}[Q^{\pi}(s,a)] \bigg| \nonumber \\
    &\leq \mathbb{E}_{\rho, \pi} \left[ |Q_K(s,a) - Q^{\pi}(s,a)| \right] \nonumber \\
    &\leq \sqrt{   \mathbb{E}_{\rho, \pi} \left[ (Q_K(s,a) - Q^{\pi}(s,a))^2 \right] } \nonumber \\ 
    &= \|Q_K - Q^{\pi}\|_{\rho^{\pi}} \nonumber \\
    &\leq \frac{ \sqrt{\kappa_{\mu}} }{1-\gamma} \max_{0 \leq k \leq K-1} \| \epsilon_k \|_{\mu} + \frac{\gamma^{K/2}}{(1-\gamma)^{1/2}}. 
\end{align*}
\noindent \textbf{(ii) For OPL}. The sup-optimality for the OPL setting is  more complex than the OPE setting but the technical steps are relatively similar. In particular, let $\epsilon_{k-1} = T^* Q_{k-1} - Q_k, \forall k$ and $\pi^*(s) = \operatorname*{arg\,max}_{a} Q^*(s,a), \forall s$, we have 
\begin{align}
    Q^* - Q_K &= T^{\pi^*} Q^* - T^{\pi^*} Q_{K-1} + \underbrace{T^{\pi^*} Q_{K-1} - T^* Q_{K-1}}_{\leq 0} + \epsilon_{K-1} \nonumber \\ 
    &\leq \gamma P^{\pi^*} (Q^* - Q_{K-1}) + \epsilon_{K-1} \nonumber \\ 
    &\leq \sum_{k=0}^{K-1} \gamma^{K-k-1} (P^{\pi^*})^{K-k-1} \epsilon_k + \gamma^{K} (P^{\pi^*})^K (Q^* - Q_0) (\text{by recursion}).
    \label{eq:upper_bound_Qstar_QK}
\end{align}
Now, let $\pi_k$ be the greedy policy w.r.t. $Q_k$, we have 
\begin{align}
    Q^* - Q_K &= \underbrace{T^{\pi^*} Q^*}_{\geq T^{\pi_{K-1}} Q^*} - T^{\pi_{K-1}} Q_{K-1} + \underbrace{T^{\pi_{K-1}} Q_{K-1} - T^* Q_{K-1}}_{\geq 0} + \epsilon_{K-1} \nonumber \\ 
    &\geq \gamma P^{\pi_{K-1}}(Q^* - Q_{K-1}) + \epsilon_{K-1} \nonumber \\ 
    &\geq \sum_{k=0}^{K-1} \gamma^{K - k -1} (P^{\pi_{K-1}} \ldots P^{\pi_{k+1}}) \epsilon_k + \gamma^K (P^{\pi_{K-1}} \ldots P^{\pi_0}) (Q^* - Q_0).  
    \label{eq:lower_bound_Qstar_QK}
\end{align}
Now, we turn to decompose $Q^* - Q^{\pi_K}$ as 
\begin{align*}
    Q^* - Q^{\pi_K} &= (T^{\pi^*} Q^* - T^{\pi^*} Q_K) + \underbrace{(T^{\pi^*} Q_K - T^{\pi_K} Q_K)}_{\leq 0} + (T^{\pi_K} Q_K - T^{\pi_K} Q^{\pi_K}) \\ 
    &\leq \gamma P^{\pi^*} (Q^* - Q_K) + \gamma P^{\pi_K}(Q_K - Q^* + Q^* - Q^{\pi_K}). 
\end{align*}
Thus, we have 
\begin{align*}
    (I - \gamma P^{\pi_K}) (Q^* - Q^{\pi_K}) \leq \gamma (P^{\pi^*} - P^{\pi_K}) (Q^* - Q_K) .
\end{align*}
Note that the operator $(I - \gamma P^{\pi_K})^{-1} = \sum_{i=0}^{\infty} (\gamma P^{\pi_K})^i$ is monotone, thus 
\begin{align}
    Q^* - Q^{\pi_K} \leq \gamma (I - \gamma P^{\pi_K})^{-1} P^{\pi^*} (Q^* - Q_K) - \gamma  (I - \gamma P^{\pi_K})^{-1} P^{\pi_K} (Q^* - Q_K).
    \label{eq:upper_bound_Qstar_QpiK}
\end{align}
Combining (\ref{eq:upper_bound_Qstar_QpiK}) with (\ref{eq:upper_bound_Qstar_QK}) and (\ref{eq:lower_bound_Qstar_QK}), we have 
\begin{align*}
    Q^* - Q^{\pi_K} &\leq (I - \gamma P^{\pi_K})^{-1} \left( \sum_{k=0}^{K-1} \gamma^{K-k} (P^{\pi^*})^{K-k} \epsilon_k + \gamma^{K+1} (P^{\pi^*})^{K+1} (Q^* - Q_0) \right) - \\ 
    &  (I - \gamma P^{\pi_K})^{-1} \left( \sum_{k=0}^{K-1} \gamma^{K - k} (P^{\pi_{K}} \ldots P^{\pi_{k+1}}) \epsilon_k + \gamma^{K+1} (P^{\pi_{K}} \ldots P^{\pi_0}) (Q^* - Q_0) \right).
\end{align*}
Using the triangle inequality, the above inequality becomes 
\begin{align*}
    Q^* - Q^{\pi_K} \leq \frac{2 \gamma (1 - \gamma^{K+1})}{(1 - \gamma)^2}  \left( \sum_{k=0}^{K-1} \alpha_k A_k |\epsilon_k| + \alpha_K A_K |Q^* - Q_0| \right),
\end{align*}
where 
\begin{align*}
    A_k &= \frac{1-\gamma}{2} (I - \gamma P^{\pi_K})^{-1} \left( (P^{\pi^*})^{K-k} + P^{\pi_{K}} \ldots P^{\pi_{k+1}} \right), \forall k < K, \\ 
    A_K &= \frac{1-\gamma}{2} (I - \gamma P^{\pi_K})^{-1} \left( (P^{\pi^*})^{K + 1} + P^{\pi_{K}} \ldots P^{\pi_0} \right), \\
    \alpha_k &= \gamma^{K-k-1} (1- \gamma) / (1-\gamma^{K+1}), \forall k < K, \\ 
    \alpha_K &= \gamma^K (1 - \gamma) / ( 1- \gamma^{K+1}). 
\end{align*}
Note that $A_k$ is a probability kernel for all $k$ and $\sum_k \alpha_k = 1$. Thus, similar to the steps in the OPE setting, for any policy $\pi$, we have
\begin{align*}
    \| Q^* - Q^{\pi_K} \|_{\rho^{\pi}}^2 &\leq  \left[ \frac{2 \gamma (1 - \gamma^{K+1})}{(1 - \gamma)^2} \right]^2  \left( \int \rho(ds) \pi(da|s) \sum_{k=0}^{K-1}\alpha_k  A_k \epsilon_k^2(s,a) + \alpha_K \right) \\ 
    &\leq \left[ \frac{2 \gamma (1 - \gamma^{K+1})}{(1 - \gamma)^2} \right]^2  \left( \int \mu(ds,da) \sum_{k=0}^{K-1} \alpha_k  \kappa_{\mu} \epsilon_k^2(s,a) + \alpha_K \right) 
\end{align*}
\begin{align*} 
    &= \left[ \frac{2 \gamma (1 - \gamma^{K+1})}{(1 - \gamma)^2} \right]^2 \left(\sum_{k=0}^{K-1} \alpha_k  \kappa_{\mu}  \| \epsilon_k \|^2_{\mu} + \alpha_K \right) \\
    & \leq \frac{4 \gamma^2 \kappa_{\mu}}{(1 - \gamma)^4} \max_{0 \leq k \leq K-1} \| \epsilon_k \|_{\mu}^2 + \frac{ 4 \gamma^{K+2}}{(1 - \gamma)^3}.
\end{align*}
Thus, we have 
\begin{align*}
    \| Q^* - Q^{\pi_K} \|_{\rho^{\pi}} \leq \frac{2 \gamma \sqrt{\kappa_{\mu}}}{(1 - \gamma)^2} \max_{0 \leq k \leq K-1} \| \epsilon_k \|_{\mu} + \frac{ 2 \gamma^{K/2+1}}{(1 - \gamma)^{3/2}}.
\end{align*}
Finally, we have 
\begin{align*}
    \text{SubOpt}(\pi_K) &= \mathbb{E}_{\rho} \left[ Q^*(s, \pi^*(s)) - Q^*(s, \pi_K(s)) \right] \\ 
    &\leq \mathbb{E}_{\rho} \left[ Q^*(s, \pi^*(s))  - Q^{\pi_K}(s, \pi^*(s)) + Q^{\pi_K}(s, \pi_K(s)) - Q^*(s, \pi_K(s)) \right] \\ 
    &\leq \| Q^* - Q^{\pi_K} \|_{\rho^{\pi^*}} + \| Q^* - Q^{\pi_K} \|_{\rho^{\pi_K}} \\ 
    &\leq \frac{4 \gamma \sqrt{\kappa_{\mu}}}{(1 - \gamma)^2} \max_{0 \leq k \leq K-1} \| \epsilon_k \|_{\mu} + \frac{ 4 \gamma^{K/2+1}}{(1 - \gamma)^{3/2}}.
\end{align*}
\end{proof}
\subsubsection*{Step 2: A Bellman error decomposition}
The next step of the proof is to decompose the Bellman errors $\| Q_{k+1} - T^{\pi} Q_k \|_{\mu}$ for OPE and $\| Q_{k+1} - T^* Q_k \|_{\mu}$ for OPL. Since these errors can be decomposed and bounded similarly, we only focus on OPL here. 

The difficulty in controlling the estimation error $\| Q_{k+1} - T^* Q_k \|_{2,\mu}$ is that $Q_k$ itself is a random variable that depends on the offline data $\mathcal{D}$. In particular, at any fixed $k$ with Bellman targets $\{y_i\}_{i=1}^n$ where $y_i = r_i + \gamma \max_{a'} Q_k(s_i', a')$, it is not immediate that 
$\mathbb{E}\left[ [T^* Q_k](x_i) - y_i | x_i \right] = 0$ for each covariate $x_i := (s_i, a_i)$ as $Q_k$ itself depends on $x_i$ (thus the tower law cannot apply here). A naive and simple approach to break such data dependency of $Q_k$ is to split the original data $\mathcal{D}$ into $K$ disjoint subsets and estimate each $Q_k$ using a separate subset. This naive approach is equivalent to the setting in \citep{DBLP:journals/corr/abs-1901-00137} where a fresh batch of data is generated for different iterations. This approach is however not efficient as it uses only $n/K$ samples to estimate each $Q_k$. This is problematic in high-dimensional offline RL when the number of iterations $K$ can be very large as it is often the case in practical settings. We instead prefer to use all $n$ samples to estimate each $Q_k$. This requires a different approach to handle the complicated data dependency of each $Q_k$. To circumvent this issue, we leverage a uniform convergence argument by introducing a deterministic covering of $T^* \mathcal{F}_{NN}$. Each element of the deterministic covering induces a different regression target $\{r_i + \gamma \max_{a'} \tilde{Q}(s'_i, a')\}_{i=1}^n$ where $\tilde{Q}$ is a deterministic function from the covering which ensures that $\mathbb{E}\left[ r_i + \gamma \max_{a'} \tilde{Q}(s'_i, a') - [T^* \tilde{Q}](x_i) | x_i\right] = 0$. In particular, we denote 
\begin{align*}
    y_i^{Q_k} = r_i + \gamma  \max_{a'} Q_k(s'_i, a'), \forall i \text{ and } \hat{f}^{Q_k} := Q_{k+1} = \arginf_{f \in \mathcal{F}_{NN}} \sum_{i=1}^n l(f(x_i), y_i^{Q_k}), \\
    \text{ and } f_*^{Q_k} = T^* Q_k, 
\end{align*}
where $l(x,y) = (x-y)^2$ is the squared loss function. Note that for any deterministic $Q \in \mathcal{F}_{NN}$, we have $f_*^Q(x_1) = \mathbb{E}[y_1^Q|x_1], \forall x_1$, thus 
\begin{align}
    \mathbb{E}(l_f - l_{f_*^Q}) = \| f - f_*^Q \|_{\mu}^2, \forall f, 
    \label{eq:loss_to_norm}
\end{align}
where $l_f$ denotes the random variable $(f(x_1) - y_1^Q)^2$. Now letting $f_{\perp}^Q := \arginf_{f \in \mathcal{F}_{NN}} \|f - f_*^Q \|_{2, \mu}$ be the projection of $f_*^Q$ onto the function class $\mathcal{F}_{NN}$, we have 
\begin{align}
    \max_{k} \| Q_{k+1} - T^*Q_k \|_{\mu}^2 &= \max_{k} \| \hat{f}^{Q_{k}} - f_*^{Q_{k}} \|_{\mu}^2 \nonumber  \overset{(a)}{\leq} \sup_{Q \in \mathcal{F}_{NN}} \| \hat{f}^{Q} - f_*^{Q} \|_{\mu}^2 \nonumber \\
    &\overset{(b)}{=} \sup_{Q \in \mathcal{F}_{NN}} \mathbb{E}(l_{\hat{f}^Q} - l_{f_*^Q}) 
    \nonumber\\
    &\overset{(c)}{\leq} \sup_{Q \in \mathcal{F}_{NN}} \left\{ \mathbb{E}(l_{\hat{f}^Q} - l_{f_*^Q}) + \mathbb{E}_n(l_{f_{\perp}^Q} - l_{\hat{f}^Q}) \right\} \nonumber \\
    &= \sup_{Q \in \mathcal{F}_{NN}} \left\{ (\mathbb{E} - \mathbb{E}_n)(l_{\hat{f}^{Q}} - l_{f_*^{Q}}) + \mathbb{E}_n(l_{f_{\perp}^{Q}} - l_{f_*^{Q}}) \right \} \nonumber \\ 
    &\leq \underbrace{\sup_{Q \in \mathcal{F}_{NN}} (\mathbb{E} - \mathbb{E}_n)(l_{\hat{f}^Q} - l_{f_*^Q})}_{I_1, \text{empirical process term}} + \underbrace{\sup_{Q \in \mathcal{F}_{NN}} \mathbb{E}_n(l_{f_{\perp}^Q} - l_{f_*^Q})}_{I_2, \text{bias term}}, 
    \label{eq:decomposition}
\end{align}
where (a) follows from that $Q_{k} \in \mathcal{F}_{NN}$, (b) follows from Equation (\ref{eq:loss_to_norm}), and (c) follows from that $\mathbb{E}_n [l_{\hat{f}^Q}] \leq \mathbb{E}_n [l_{f^Q}], \forall f, Q \in \mathcal{F}_{NN}$. That is, the error is decomposed into two terms: the first term $I_1$ resembles the empirical process in statistical learning theory and the second term $I_2$ specifies the bias caused by the regression target $f_*^{Q}$ not being in the function space $\mathcal{F}_{NN}$.


\subsubsection*{Step 3: A deviation analysis}
The next step is to bound the empirical process term and the bias term via an intricate concentration, local Rademacher complexities and a localization argument. First, the bias term in Equation (\ref{eq:decomposition}) is taken uniformly over the function space, thus standard concentration arguments such as Bernstein's inequality and Pollard's inequality used in \citep{DBLP:journals/jmlr/MunosS08,DBLP:conf/icml/0002VY19} do not apply here. Second, local Rademacher complexities \citep{bartlett2005} are data-dependent complexity measures that exploit the fact that only a small subset of the function class will be used. Leveraging a localization argument for local Rademacher complexities \citep{farrell2018deep}, we localize an empirical Rademacher ball into smaller balls by which we can handle their complexities more effectively. Moreover, we explicitly use the sub-root function argument to derive our bound and extend the technique to the uniform convergence case. That is, reasoning over the sub-root function argument makes our proof more modular and easier to incorporate the uniform convergence argument.

Localization is particularly useful to handle the complicated approximation errors induced by deep ReLU network function approximation.

\subsubsection*{Step 3.a: Bounding the bias term via a uniform convergence concentration inequality}
We define the inherent Bellman error as  $d_{\mathcal{F}_{NN}} := \sup_{Q \in \mathcal{F}_{NN}} \inf_{f \in \mathcal{F}_{NN}}  \| f - T^* Q\|_{\mu}$. This implies that 
\begin{align}
    d_{\mathcal{F}_{NN}}^2 := \sup_{Q \in \mathcal{F}_{NN}} \inf_{f \in \mathcal{F}_{NN}}  \| f - T^* Q\|_{\mu}^2 = \sup_{Q \in \mathcal{F}_{NN}} \mathbb{E}(l_{f_{\perp}^Q} - l_{f_*^Q}). 
\end{align}
We have $ |l_f - l_g| \leq 4  |f-g| \text{ and } |l_f - l_g| \leq 8$, and
\begin{align*}
    &H(\epsilon, \{l_{f_{\perp}^Q} - l_{f_*^Q}: Q \in \mathcal{F}_{NN}\}| \{x_i, y_i\}_{i=1}^n, n^{-1} \| \cdot \|_1 ) \\
    &\leq H(\frac{\epsilon}{4}, \{f_{\perp}^Q - f_*^Q: Q \in \mathcal{F}_{NN}\}| \{x_i\}_{i=1}^n, n^{-1} \| \cdot \|_1 ) \\ 
    &\leq H(\frac{\epsilon}{4 }, (\mathcal{F} - T^* \mathcal{F}_{NN})| \{x_i \}_{i=1}^n, n^{-1} \| \cdot \|_1 ) \\ 
    &\leq H(\frac{\epsilon}{8 }, \mathcal{F}_{NN}| \{x_i\}_{i=1}^n, n^{-1} \| \cdot \|_1) + H(\frac{\epsilon}{8  }, T^* \mathcal{F}_{NN}| \{x_i\}_{i=1}^n, n^{-1} \| \cdot \|_1 ) \\ 
    &\leq H(\frac{\epsilon}{8 }, \mathcal{F}_{NN}| \{x_i\}_{i=1}^n,  \| \cdot \|_{\infty}) + H(\frac{\epsilon}{8 }, T^* \mathcal{F}_{NN},  \| \cdot \|_{\infty})
\end{align*}

For any $\epsilon' > 0$ and $\delta' \in (0,1)$, it follows from Lemma \ref{lemma:sup_concentration} with $\epsilon = 1/2$ and $\alpha = \epsilon'^2$, with probability at least $1 - \delta'$, for any $Q \in \mathcal{F}_{NN}$, we have 
\begin{align}
    \mathbb{E}_n(l_{f_{\perp}^Q} - l_{f_*^Q}) \leq 3 \mathbb{E}(l_{f_{\perp}^Q} - l_{f_*^Q}) + \epsilon'^2 \leq 3 d_{\mathcal{F}_{NN}}^{2} + \epsilon'^2, 
    \label{eq:bound_E_minus_En_lf_minus_lfstar}
\end{align}
given that 
\begin{align*}
    n \approx \frac{1}{\epsilon'^2}\left( \log(4/\delta') + \log \mathbb{E} N(\frac{\epsilon'^2}{40 }, (\mathcal{F}_{NN} - T^* \mathcal{F}_{NN})| \{x_i\}_{i=1}^n, n^{-1} \| \cdot \|_1) \right) .
\end{align*}

Note that if we use Pollard's inequality \citep{DBLP:journals/jmlr/MunosS08} in the place of Lemma \ref{lemma:sup_concentration}, the RHS of Equation (\ref{eq:bound_E_minus_En_lf_minus_lfstar}) is bounded by $\epsilon'$ instead of $\epsilon'^2$(i.e., $n$ scales with $O(1/\epsilon'^4)$ instead of $O(1/\epsilon'^2)$). In addition, unlike \citep{DBLP:conf/icml/0002VY19}, the uniform convergence argument hinders the application of Bernstein's inequality. We remark that \citealt{DBLP:conf/icml/0002VY19} make a mistake in their proof by ignoring the data-dependent structure in the algorithm (i.e., they wrongly assume that $Q^k$ in Algorithm \ref{alg:LSVI} is fixed and independent of $\{s_i, a_i\}_{i=1}^n$). Thus, the uniform convergence argument in our proof is necessary. 

\subsubsection*{Step 3.b: Bounding the empirical process term via local Rademacher complexities}
For any $Q \in \mathcal{F}_{NN}$, we have 
\begin{align*}
    |l_{f_{\perp}^Q} - l_{f_*^Q}| &\leq 2 |f_{\perp}^Q - f_*^Q| \leq 2, \\
    \mathbb{V}[l_{f_{\perp}^Q} - l_{f_*^Q}] &\leq \mathbb{E}[(l_{f_{\perp}^Q} - l_{f_*^Q})^2] \leq 4 \mathbb{E} (f_{\perp}^Q - f_*^Q)^2. 
 \end{align*}
Thus, it follows from Lemma \ref{lemma:local_rademacher_complexity_basics_first_part} (with $\alpha = 1/2$) that with any $r > 0, \delta \in (0,1)$, with probability at least $1 - \delta$, we have
\begin{align*}
    &\sup \{(\mathbb{E} - \mathbb{E}_n)(l_{\hat{f}^Q} - l_{f_*^Q}): Q \in \mathcal{F}_{NN}, \|\hat{f}^Q - f_*^Q\|^2_{\mu} \leq r \} \\
    &\leq \sup \{(\mathbb{E} - \mathbb{E}_n)(l_f - l_g): f \in \mathcal{F}_{NN}, g \in T^{*} \mathcal{F}, \|f - g\|^2_{\mu} \leq r \} \\
     &\leq 3 \mathbb{E} R_n \left\{l_f - l_g: f \in \mathcal{F}_{NN}, g \in T^{*} \mathcal{F}_{NN}, \|f - g \|_{\mu}^2 \leq r \right \} + 2 \sqrt{ \frac{2 r \log(1/\delta) }{n} } \\
     &+ \frac{28 \log(1/\delta)}{3n} \nonumber \\ 
     &\leq 6 \mathbb{E} R_n \left\{f - g: f \in \mathcal{F}_{NN}, g \in T^{*} \mathcal{F}_{NN}, \|f - g \|_{\mu}^2 \leq r \right \} + 2 \sqrt{ \frac{2 r \log(1/\delta) }{n} } \\
     &+ \frac{28 \log(1/\delta)}{3n}. 
\end{align*}

\subsubsection*{Step 3.c: Bounding $\|Q_{k+1} - T^* Q_k \|_{\mu}$ using localization argument via sub-root functions}
We bound $\|Q_{k+1} - T^* Q_k \|_{\mu}$ using the localization argument, breaking down the Rademacher complexities into local balls and then build up the original function space from the local balls. Let $\psi$ be a sub-root function  \citep[Definition~3.1]{bartlett2005} with the fixed point $r_*$ and assume that for any $ r \geq r_*$, we have 
\begin{align}
    \psi(r) \geq  3 \mathbb{E} R_n \left\{f - g: f \in \mathcal{F}_{NN}, g \in T^{*} \mathcal{F}_{NN}, \|f - g \|_{\mu}^2 \leq r \right \}.
    \label{eq:sub_root_function_bounding_empirical_Rademacher}
\end{align}

We recall that a function $\psi: [0, \infty) \rightarrow [0, \infty)$ is \textit{sub-root} if it is non-negative, non-decreasing and $r \mapsto \psi(r) / \sqrt{r}$ is non-increasing for $r > 0$. Consequently, a sub-root function $\psi$ has a unique fixed point $r_*$ where $r_* = \psi(r_*)$. In addition, $\psi(r) \leq \sqrt{r r_*}, \forall r \geq r_*$. In the next step, we will find a sub-root function $\psi$ that satisfies the inequality above, but for this step we just assume that we have such $\psi$ at hand. Combining Equations (\ref{eq:decomposition}), (\ref{eq:bound_E_minus_En_lf_minus_lfstar}), and (\ref{eq:sub_root_function_bounding_empirical_Rademacher}), we have: for any $r \geq r_*$ and any $\delta \in (0,1)$, if $\| \hat{f}^{Q_{k-1}} - f_*^{Q_{k-1}} \|_{2,\mu}^2 \leq r$, with probability at least $1 - \delta$, 
\begin{align*}
    \| \hat{f}^{Q_{k-1}} - f_*^{Q_{k-1}} \|_{2,\mu}^2 &\leq 2\psi(r) +  2 \sqrt{ \frac{2 r \log(2/\delta) }{n} } + \frac{28 \log(2/\delta)}{3n} + 3 d^2_{\mathcal{F}} + \epsilon'^2 \\ 
    &\leq \sqrt{r r_*} + 2 \sqrt{ \frac{2 r \log(2/\delta)}{n} } + \frac{28 \log(2/\delta)}{3n} + (\sqrt{3} d_{\mathcal{F}} + \epsilon' )^2, 
\end{align*} 
where 
\begin{align*}
    n \approx \frac{1}{4\epsilon'^2}\left( \log(8/\delta) + \log \mathbb{E} N(\frac{\epsilon'^2}{20}, (\mathcal{F}_{NN} - T^* \mathcal{F}_{NN})| \{x_i\}_{i=1}^n, n^{-1}\| \cdot \|_1) \right).
\end{align*}

Consider $r_0 \geq r_*$ (to be chosen later) and denote the events 
\begin{align*}
    B_k := \{ \| \hat{f}^{Q_{k-1}} - f_*^{Q_{k-1}} \|^2_{2,\mu} \leq 2^k r_0   \}, \forall k \in \{0,1,...,l\},
\end{align*}
where $l = \log_2(\frac{1}{r_0}) \leq \log_2( \frac{1}{r_*} )$. We have $B_0 \subseteq B_1 \subseteq ... \subseteq B_l$ and since $\|f - g\|_{\mu}^2 \leq 1, \forall |f|_{\infty}, |g|_{\infty} \leq 1$, we have $P(B_l) = 1$. 
If $\|\hat{f}^{Q_{k-1}} - f_*^{Q_{k-1}} \|^2_{\mu} \leq 2^i r_0$ for some $i \leq l$, then with probability at least $1 - \delta$, we have
\begin{align*}
    \| \hat{f}^{Q_{k-1}} - f_*^{Q_{k-1}} \|_{2,\mu}^2 
    &\leq \sqrt{2^i r_0 r_*} + 2 \sqrt{ \frac{2^{i+1} r_0 \log(2/\delta)}{n} } + \frac{28 \log(2/\delta)}{3n} + (\sqrt{3} d_{\mathcal{F}_{NN}} + \epsilon' )^2 \\
    &\leq 2^{i-1} r_0, 
\end{align*} 
if the following inequalities hold 
\begin{align*}
    \sqrt{ 2^i r_*} + 2\sqrt{ \frac{2^{i+1} \log(2/\delta)}{n} } &\leq \frac{1}{2} 2^{i-1} \sqrt{r_0},  \\ 
    \frac{28 \log(2/\delta)}{3n} + (\sqrt{3} d_{\mathcal{F}_{NN}} + \epsilon' )^2 &\leq \frac{1}{2} 2^{i-1} r_0. 
\end{align*}

We choose $r_0 \geq r_*$ such that the inequalities above hold for all $0 \leq i \leq l$. This can be done by simply setting
\begin{align*}
    \sqrt{r_0} &= \frac{2}{2^{i-1}}  \left( \sqrt{ 2^i r_*} + 2\sqrt{ \frac{2^{i+1} \log(2/\delta) }{n} } \right) \bigg |_{i=0} + \sqrt{  \frac{2}{2^{i-1}} \left(  \frac{28 \log(2/\delta)}{3n} + (\sqrt{3} d_{\mathcal{F}_{NN}} + \epsilon' )^2  \right)   } \bigg |_{i=0} \\
    &\lesssim  d_{\mathcal{F}_{NN}} + \epsilon' + \sqrt{\frac{\log(2/\delta)}{n}} + \sqrt{r_*}.
\end{align*}

Since $\{B_i\}$ is a sequence of increasing events, we have 
\begin{align*}
    P(B_0) &= P(B_1) - P(B_1 \cap B_0^c ) = P(B_2) - P(B_2 \cap B_1^c) - P(B_1 \cap B_0^c) \\ 
    &=P(B_l) - \sum_{i=0}^{l-1} P(B_{i+1} \cap B_i^c) \geq 1 - l \delta.
\end{align*}

Thus, with probability at least $1 - \delta$, we have
\begin{equation}
\| \hat{f}^{Q_{k-1}} - f_*^{Q_{k-1}} \|_{\mu} \lesssim  d_{\mathcal{F}_{NN}} + \epsilon' + \sqrt{\frac{\log(2l/\delta)}{n}} + \sqrt{r_*}
\label{eq:general_bounding}
\end{equation}
where 
\begin{align*}
    n \approx \frac{1}{4\epsilon'^2}\left( \log(8l/\delta) + \log \mathbb{E} N(\frac{\epsilon'^2}{20}, (\mathcal{F}_{NN} - T^* \mathcal{F}_{NN})| \{x_i\}_{i=1}^n, n^{-1}\| \cdot \|_1)) \right) .
\end{align*}

\subsubsection*{Step 3.d: Finding a sub-root function and its fixed point}
It remains to find a sub-root function $\psi(r)$ that satisfies Equation (\ref{eq:sub_root_function_bounding_empirical_Rademacher}) and thus its fixed point. The main idea is to bound the RHS, the local Rademacher complexity, of Equation (\ref{eq:sub_root_function_bounding_empirical_Rademacher}) by its empirical counterpart as the latter can then be further bounded by a sub-root function represented by a measure of compactness of the function spaces $\mathcal{F}_{NN}$ and $T^{*}\mathcal{F}_{NN}$. 



For any $\epsilon > 0$, we have the following inequalities for entropic numbers:
\begin{align}
H(\epsilon, \mathcal{F}_{NN} - T^* \mathcal{F}_{NN}, \| \cdot \|_{n}) &\leq H(\epsilon/2, \mathcal{F}_{NN}, \| \cdot \|_{n}) + H(\epsilon/2,  T^* \mathcal{F}_{NN}, \| \cdot \|_{n}), \nonumber \\
H(\epsilon, \mathcal{F}_{NN}, \|\cdot\|_{n}) &\leq 
H(\epsilon, \mathcal{F}_{NN}| \{x_i\}_{i=1}^n, \| \cdot \|_{\infty}) \overset{(a)}{\lesssim} N[ (\log N)^2 + \log(1/\epsilon)],  \\
H(\epsilon, T^* \mathcal{F}_{NN}, \|\cdot \|_{n}) &\leq H(\epsilon, T^* \mathcal{F}_{NN}, \| \cdot \|_{\infty}) \leq H_{[]}(2 \epsilon, T^*\mathcal{F}_{NN}, \| \cdot \|_{\infty}) \nonumber \\
&\overset{(b)}{\leq} H_{[]}(2\epsilon, \bar{B}^{\alpha}_{p,q}(\mathcal{X}), \| \cdot \|_{\infty}) \overset{(c)}{\lesssim} (2\epsilon)^{-d/\alpha},
\end{align}
where $N$ is a hyperparameter of the deep ReLU network described in Lemma \ref{lemma:approximation_power_for_Besov}, (a) follows from Lemma \ref{lemma:approximation_power_for_Besov}, and (b) follows from  Assumption \ref{assumption:completeness}, and (c) follows from Lemma \ref{lemma:entropic_number_of_Besov}. Let $\mathcal{H} := \mathcal{F}_{NN} - T^{*} \mathcal{F}_{NN}$, it follows from Lemma \ref{lemma:refined_entropy_integral} with $\{\xi_k := \epsilon / 2^k\}_{k \in \mathbb{N}}$ for any $\epsilon > 0$ that 
\begin{align*}
    &\mathbb{E}_{\sigma} R_n \{ h \in \mathcal{H} - \mathcal{H}: \|h\|_{n} \leq \epsilon \} \leq 4 \sum_{k=1}^{\infty} \frac{\epsilon}{2^{k-1}} \sqrt{ \frac{H(\epsilon/2^{k-1}, \mathcal{H}, \| \cdot \|_{n})}{n}  } \nonumber \\ 
    &\leq  4 \sum_{k=1}^{\infty} \frac{\epsilon}{2^{k-1}} \sqrt{ \frac{H(\epsilon/2^{k}, \mathcal{F}_{NN}, \| \cdot \|_{\infty})}{n}  } + 4 \sum_{k=1}^{\infty} \frac{\epsilon}{2^{k-1}} \sqrt{ \frac{H(\epsilon/2^{k}, T^{\pi} \mathcal{F}_{NN}, \| \cdot \|_{\infty})}{n}  }  \\ 
    &\leq \frac{4 \epsilon}{\sqrt{n}} \sum_{k=1}^{\infty} 2^{-(k-1)} \sqrt{N \left( (\log N)^2 + \log(2^k/\epsilon) \right)} + \frac{4 \epsilon}{\sqrt{n}} \sum_{k=1}^{\infty} 2^{-(k-1)} \sqrt{\left(\frac{\epsilon}{2^{k-1}} \right)^{-d/\alpha}} \\ 
    &\lesssim \frac{\epsilon}{\sqrt{n}} \sqrt{N ((\log N)^2 + \log(1/\epsilon))} + \frac{\epsilon^{1 - \frac{d}{2 \alpha}}}{\sqrt{n}},
\end{align*}
where we use $\sqrt{a + b} \leq \sqrt{a} + \sqrt{b}, \forall a,b \geq 0$, $\sum_{k=1}^{\infty} \frac{\sqrt{k}}{2^{k-1}} < \infty$, and $\sum_{k=1}^{\infty} \left( \frac{1}{2^{1 - \frac{d}{2 \alpha}}} \right)^{k-1} < \infty$.

It now follows from Lemma
\ref{lemma:local_empirical_rademacher_bounded_by_covering_number_with_empirical_norm} that 
\begin{align*}
&\mathbb{E}_{\sigma} R_n \{f \in \mathcal{F}, g \in T^{*} \mathcal{F}: \| f -g \|_{n}^2 \leq r \} \\
&\leq \inf_{\epsilon > 0} \bigg[ \mathbb{E}_{\sigma} R_n \{ h \in \mathcal{H} - \mathcal{H}: \| h \|_{\mu} \leq \epsilon \}  + \sqrt{ \frac{2r H(\epsilon/2, \mathcal{H}, \|\cdot \|_{n}) }{n}  } \bigg] \\
&\lesssim  \bigg[ \frac{\epsilon}{\sqrt{n}} \sqrt{N ((\log N)^2 + \log(1/\epsilon))} + \frac{\epsilon^{1 - \frac{d}{2 \alpha}}}{\sqrt{n}} + 
\sqrt{\frac{2r}{n}} \sqrt{N((\log N)^2 + \log(4 /\epsilon))} \\
&+ \sqrt{\frac{2r}{n}}(\epsilon/2)^{\frac{-d}{2 \alpha}}
\bigg] \bigg |_{\epsilon = n^{-\beta}}\\ 
&\asymp n^{-\beta -1/2} \sqrt{N (\log^2 N + \log n )} + n^{-\beta(1 - \frac{d}{2 \alpha}) - 1/2} + \sqrt{\frac{r}{n}}  \sqrt{N (\log^2 N + \log n )} \\
&+ \sqrt{r} n^{-\frac{1}{2}(1 - \frac{\beta d}{\alpha})} =: \psi_1(r),
\end{align*}
where $\beta \in (0, \frac{\alpha}{d})$ is an absolute constant to be chosen later.

Note that $\mathbb{V}[(f-g)^2] \leq \mathbb{E}[(f-g)^4] \leq \mathbb{E}[(f-g)^2]$ for any $f \in \mathcal{F}_{NN}, g \in T^* \mathcal{F}_{NN}$. Thus, for any $r \geq r_*$, it follows from Lemma \ref{lemma:local_rademacher_complexity_basics} that with probability at least $1 - \frac{1}{n}$, for any $f \in \mathcal{F}_{NN}, g \in T^* \mathcal{F}_{NN}$ such that $\|f - g\|^2_{\mu} \leq r$, we have 
\begin{align*}
&\|f - g\|_{n}^2 \\
&\leq \|f - g \|^2_{\mu} + 3 \mathbb{E} R_n \{(f-g)^2: f \in \mathcal{F}_{NN}, g \in T^* \mathcal{F}_{NN}, \| f - g \|^2_{\mu} \leq r\} +  \sqrt{ \frac{2r \log n}{n} } \\
&+ \frac{56}{3} \frac{\log n}{n} \\ 
&\leq  \|f - g \|^2_{\mu} + 3 \mathbb{E} R_n \{f-g: f \in \mathcal{F}_{NN}, g \in T^* \mathcal{F}_{NN}, \| f - g \|^2_{\mu} \leq r\} + \sqrt{ \frac{2r \log n }{n} } \\
&+ \frac{56}{3} \frac{\log n}{n} \\ 
&\leq r + \psi(r) + r + r \leq 4r,
\end{align*}
if $r \geq r_* \lor \frac{2 log n}{n} \lor \frac{56 log n}{3 n}$. For such $r$, denote $E_r = \{ \| f - g \|_{n}^2 \leq 4r \} \cap \{\|f - f_*\|_{\mu}^2 \leq r \}$, we have $P(E_r) \geq 1 - 1/n$ and
\begin{align*}
&3 \mathbb{E} R_n \{f - g: f \in \mathcal{F}_{NN}, g \in T^* \mathcal{F}_{NN}, \|f - g\|^2_{\mu} \leq r \} \\
&=
3 \mathbb{E} \mathbb{E}_{\sigma} R_n \{f - g: f \in \mathcal{F}_{NN}, g \in T^{*} \mathcal{F}_{NN}, \|f - g\|^2_{\mu} \leq r \} \\
&\leq 3 \mathbb{E} \bigg[ 1_{E_r} \mathbb{E}_{\sigma} R_n \{f - g: f \in \mathcal{F}_{NN}, g \in T^* \mathcal{F}_{NN}, \|f - g\|^2_{\mu} \leq r \} +  (1 - 1_{E_r})  \bigg] \\ 
&\leq 3 \mathbb{E} \bigg[ \mathbb{E}_{\sigma} R_n \{f - g: f \in \mathcal{F}_{NN}, g \in T^* \mathcal{F}_{NN}, \|f - g\|^2_{n} \leq 4r \} +  (1 - 1_{E_r})  \bigg] \\
&\leq 3( \psi_1(4r) + \frac{1}{n}) \\
&\lesssim n^{-\beta -1/2} \sqrt{N (\log^2 N + \log n )} + n^{-\beta(1 - \frac{d}{2 \alpha}) - 1/2} + \sqrt{\frac{r}{n}}  \sqrt{N (\log^2 N + \log n )} \\
&+ \sqrt{r} n^{-\frac{1}{2}(1 - \frac{\beta d}{\alpha})} + n^{-1} =: \psi(r)
\end{align*}

It is easy to verify that $\psi(r)$ defined above is a sub-root function. The fixed point $r_*$ of $\psi(r)$ can be solved analytically via the simple quadratic equation $r_* = \psi(r_*)$. In particular, we have 
\begin{align}
    \sqrt{r_*} &\lesssim n^{-1/2} \sqrt{N (\log^2 N + \log n )} + n^{-\frac{1}{2}(1 - \frac{\beta d}{\alpha})} + n^{-\frac{\beta}{2} - \frac{1}{4}} [N (\log^2 N + \log n )]^{1/4} \nonumber \\
    &+ n^{-\frac{\beta}{2}(1 - \frac{d}{2 \alpha}) - \frac{1}{2}} + n^{-1/2} \nonumber \\ 
    &\lesssim n^{-\frac{1}{4} ( (2\beta) \land 1) + 1)} \sqrt{N (\log^2 N + \log n )} + n^{-\frac{1}{2}(1 - \frac{\beta d}{\alpha})} + n^{-\frac{\beta}{2}(1 - \frac{d}{2 \alpha}) - \frac{1}{2}} + n^{-1/2}
    \label{eq:sandwich_fixed_point}
\end{align}

It follows from Equation (\ref{eq:general_bounding}) (where $l \lesssim \log(1/r_*)$), the definition of $d_{\mathcal{F}_{NN}}$, Lemma \ref{lemma:approximation_power_for_Besov}, and (\ref{eq:sandwich_fixed_point}) that for any $\epsilon' >0$ and $\delta \in (0,1)$, with probability at least $1 - \delta$, we have 
\begin{align}
    \max_{k} \|Q_{k+1} - T^* Q_k \|_{\mu} &\lesssim N^{-\alpha / d} + \epsilon' + n^{-\frac{1}{4} ( (2\beta) \land 1) + 1)} \sqrt{N (\log^2 N + \log n )} + n^{-\frac{1}{2}(1 - \frac{\beta d}{\alpha})} \nonumber \\
    &+ n^{-\frac{\beta}{2}(1 - \frac{d}{2 \alpha}) - \frac{1}{2}}
    + n^{-1/2}\sqrt{\log(1/\delta) + \log \log n  } 
    \label{eq:final_parametric_upper_bound}
\end{align}
where 
\begin{align}
    n &\gtrsim \frac{1}{4\epsilon'^2}\bigg( \log(1/\delta) + \log \log n + \log \mathbb{E} N(\frac{\epsilon'^2}{20}, (\mathcal{F}_{NN} - T^* \mathcal{F}_{NN})| \{x_i\}_{i=1}^n, n^{-1} \cdot \| \cdot \|_1)) \bigg) .
    \label{eq:sample_complexity_n}
\end{align}


\subsubsection*{Step 4: Minimizing the upper bound}
The final step for the proof is to minimize the upper error bound obtained in the previous steps w.r.t. two free parameters $\beta \in (0, \frac{\alpha}{d})$ and $N \in \mathbb{N}$. Note that $N$ parameterizes the deep ReLU architecture $\Phi(L,m,S,B)$ given Lemma \ref{lemma:approximation_power_for_Besov}. 
In particular, we optimize over $\beta \in (0, \frac{\alpha}{d})$ and $N \in \mathbb{N}$ to minimize the upper bound in the RHS of Equation (\ref{eq:final_parametric_upper_bound}). The RHS of Equation (\ref{eq:final_parametric_upper_bound}) is minimized (up to $\log n$-factor) by choosing
\begin{align}
    N \asymp n^{\frac{1}{2}((2 \beta \land 1) + 1) \frac{d}{2\alpha + d}} \text{ and }
    \beta = \left(2 + \frac{d^2}{\alpha (\alpha + d)} \right)^{-1}, 
    \label{eq:optimal_value_N_beta}
\end{align}
which results in $N \asymp n^{\frac{1}{2}(2 \beta  + 1) \frac{d}{2\alpha + d}}$. At these optimal values, Equation (\ref{eq:final_parametric_upper_bound}) becomes 
\begin{align}
    \max_{k} \|Q_{k+1} - T^* Q_k\|_{\mu} &\lesssim \epsilon' + n^{-\frac{1}{2}\left( \frac{2 \alpha}{2 \alpha + d} + \frac{d}{ \alpha} \right)^{-1}} \log n + n^{-1/2}\sqrt{\log(1/\delta) + \log \log n  },
    \label{eq:final_upper_bound}
\end{align}
where we use inequalities $n^{-\frac{\beta}{2}(1 - \frac{d}{2 \alpha}) - \frac{1}{2}} \leq n^{-\frac{1}{2}(1 - \frac{\beta d}{\alpha})} \asymp N^{-\alpha /d} = n^{-\frac{1}{2}\left( \frac{2 \alpha}{2 \alpha + d} + \frac{d}{ \alpha} \right)^{-1}} $.

Now, for any $\epsilon > 0$, we set $\epsilon' = \epsilon/3$ and let 
\begin{align*}
    n^{-\frac{1}{2}\left( \frac{2 \alpha}{2 \alpha + d} + \frac{d}{ \alpha} \right)^{-1}} \log n \lesssim \epsilon / 3 \text{ and } n^{-1/2}\sqrt{\log(1/\delta) + \log \log n  } \lesssim \epsilon / 3.
\end{align*}
It then follows from Equation (\ref{eq:final_upper_bound}) that with probability at least $1 - \delta$, we have $\max_{k} \| Q_{k+1} - T^* Q_k \|_{\mu} \leq \epsilon$ if $n$ simultaneously satisfies Equation (\ref{eq:sample_complexity_n}) with $\epsilon' = \epsilon/3$ and 
\begin{align}
    n \gtrsim \left(\frac{1}{\epsilon^2} \right)^{ \frac{2 \alpha}{ 2 \alpha + d} + \frac{d}{\alpha}} (\log^2 n)^{ \frac{2 \alpha}{ 2 \alpha + d}+ \frac{d}{\alpha}} \text{ and } n \gtrsim \frac{1}{\epsilon^2} \left( \log(1/\delta) + \log \log n  \right).
    \label{eq:explicit_n_1}
\end{align}

Next, we derive an explicit formula of the sample complexity satisfying Equation (\ref{eq:sample_complexity_n}). Using Equations (\ref{eq:final_parametric_upper_bound}), (\ref{eq:explicit_n_1}), and (\ref{eq:optimal_value_N_beta}), we have that $n$ satisfies Equation (\ref{eq:sample_complexity_n}) if 
\begin{align}
\begin{cases}
        n &\gtrsim \frac{1}{\epsilon^2} \left[ n^{\frac{2\beta + 1}{2} \frac{d}{2 \alpha + d}} (\log^2 n + \log(1/\epsilon))\right], \\ 
    n &\gtrsim \left( \frac{1}{\epsilon^2} \right)^{1 + \frac{d}{\alpha}}, \\
    n &\gtrsim \frac{1}{\epsilon^2} \left( \log(1/\delta) + \log \log n  \right).
\end{cases}
\label{eq:explicit_n_2}
\end{align}

Note that $\beta \leq 1/2$ and $\frac{d}{\alpha} \leq 2$; thus, we have 
\begin{align*}
    \left(1 - \frac{2\beta + 1}{2} \frac{d}{2 \alpha + d} \right)^{-1} \leq 1 + \frac{d}{\alpha} \leq 3.
\end{align*}
Hence, $n$ satisfies Equations (\ref{eq:explicit_n_1}) and (\ref{eq:explicit_n_2}) if 
\begin{align*}
    n \gtrsim \left(\frac{1}{\epsilon^2} \right)^{1 + \frac{d}{\alpha}} \log^6 n + \frac{1}{\epsilon^2}(\log(1/\delta) + \log \log n).
\end{align*}

\section*{Technical Lemmas}
\label{appendix:B}

\begin{lem}[\citep{bartlett2005}]
Let $r > 0$ and let 
\begin{align*}
    \mathcal{F} \subseteq \{f: \mathcal{X} \rightarrow [a,b] : \mathbb{V}[f(X_1)] \leq r\}. 
\end{align*}
\begin{enumerate}
    \item 
    For any $\lambda > 0$, we have with probability at least $1 - e^{-\lambda}$, 
    \begin{align*}
        \sup_{f \in \mathcal{F}} \left(\mathbb{E}f - \mathbb{E}_n f \right) 
        \leq \inf_{\alpha >0 } \left( 2(1+\alpha) \mathbb{E} \left[R_n \mathcal{F} \right] + \sqrt{ \frac{2 r \lambda}{ n}} + (b-a)\left(\frac{1}{3} + \frac{1}{\alpha} \right)\frac{\lambda}{n} \right).
    \end{align*}
    \label{lemma:local_rademacher_complexity_basics_first_part}

    \item 
    With probability at least $1 - 2 e^{-\lambda}$, 
    \begin{align*}
        &\sup_{f \in \mathcal{F}} \left(\mathbb{E}f - \mathbb{E}_n f \right) \leq\\
        &\inf_{\alpha \in (0,1) } \left( \frac{2(1+\alpha)}{(1 - \alpha)} \mathbb{E}_{\sigma} \left[R_n \mathcal{F} \right] + \sqrt{ \frac{2 r \lambda}{ n}} + (b-a)\left(\frac{1}{3} + \frac{1}{\alpha} + \frac{1 + \alpha}{2 \alpha (1 - \alpha)} \right)\frac{\lambda}{n} \right).
    \end{align*}
\end{enumerate}
Moreover, the same results hold for $\sup_{f \in \mathcal{F}} \left(  \mathbb{E}_n f - \mathbb{E}f\right) $. 

\label{lemma:local_rademacher_complexity_basics}
\end{lem}



\begin{lem}[{\citep[Theorem~11.6]{DBLP:books/daglib/0035701}}]
Let $B \geq 1$ and $\mathcal{F}$ be a set of functions $f: \mathbb{R}^d \rightarrow [0,B]$. Let $Z_1, ..., Z_n$ be i.i.d. $\mathbb{R}^d$-valued random variables. For any $\alpha > 0$, $0 < \epsilon < 1$, and $ n \geq 1$, we have 
\begin{align*}
    P \left\{ \sup_{f \in \mathcal{F}} \frac{\frac{1}{n}\sum_{i=1}^n f(Z_i) - \mathbb{E}[f(Z)]}{\alpha + \frac{1}{n} \sum_{i=1}^n f(Z_i) + \mathbb{E}[f(Z)]} > \epsilon \right\} \leq 4 \mathbb{E} N(\frac{\alpha \epsilon}{5}, \mathcal{F}|Z_1^n, n^{-1} \| \cdot \|_1) \exp \left( \frac{-3 \epsilon^2 \alpha n}{40 B} \right).
\end{align*}
\label{lemma:sup_concentration}
\end{lem}


\begin{lem}[\textit{Contraction property} \citep{afol_lecture2}]
Let $\phi: \mathbb{R} \rightarrow \mathbb{R}$ be a $L$-Lipschitz, then 
\begin{align*}
    \mathbb{E}_{\sigma} R_n \left( \phi \circ \mathcal{F} \right) \leq L \mathbb{E}_{\sigma} R_n \mathcal{F}.
\end{align*}
\end{lem}

\begin{lem}[{\citep[Lemma~1]{DBLP:journals/ijon/LeiDB16}}]
Let $\mathcal{F}$ be a function class and $P_n$ be the empirical measure supported on $X_1, ..., X_n \sim \mu$, then for any $r >0$ (which can be stochastic w.r.t $X_i$), we have 
\begin{align*}
    \mathbb{E}_{\sigma} R_n \{f \in \mathcal{F}: \| f \|_{n}^2 \leq r \} &\leq \inf_{\epsilon > 0} \bigg[ \mathbb{E}_{\sigma} R_n \{ f \in \mathcal{F} - \mathcal{F}: \|f \|_{\mu} \leq \epsilon \}  \\
    &+ \sqrt{ \frac{2r \log N(\epsilon/2, \mathcal{F},\|\cdot \|_{n}) }{n}  } \bigg].
\end{align*}
\label{lemma:local_empirical_rademacher_bounded_by_covering_number_with_empirical_norm}
\end{lem}

\begin{lem}[{\citep[modification]{DBLP:journals/ijon/LeiDB16}}]
Let $X_1, ..., X_n$ be a sequence of samples and $P_n$ be the associated empirical measure. For any function class $\mathcal{F}$ and any monotone sequence $\{\xi_k\}_{k=0}^{\infty}$ decreasing to $0$, we have the following inequality for any non-negative integer $N$ 
\begin{align*}
    \mathbb{E}_{\sigma} R_n \{f \in \mathcal{F}: \|f \|_n \leq \xi_0 \} \leq 4 \sum_{k=1}^N \xi_{k-1} \sqrt{ \frac{\log \mathcal{N}(\xi_k, \mathcal{F}, \| \cdot \|_{n})}{n}  } + \xi_N.
\end{align*}
\label{lemma:refined_entropy_integral}
\end{lem}

\begin{lem}[\textit{Pollard's inequality}] 
Let $\mathcal{F}$ be a set of measurable functions $f: \mathcal{X} \rightarrow [0,K]$ and let $\epsilon >0, N$ arbitrary. If $\{X_i\}_{i=1}^N$ is an i.i.d. sequence of random variables taking values in $\mathcal{X}$, then 
\begin{align*}
    P \left( \sup_{f \in \mathcal{F}} \bigg| \frac{1}{N} \sum_{i=1}^N f(X_i) - \mathbb{E}[f(X_1)] \bigg| > \epsilon \right) \leq 8 \mathbb{E} \left[ N(\epsilon/8, \mathcal{F}|_{X_{1:N}}) \right] e^{\frac{-N \epsilon^2}{ 128 K^2}}.
\end{align*}

\end{lem}


\begin{lem}[\textit{Properties of (bracketing) entropic numbers}]
Let $\epsilon \in (0, \infty)$. We have 
\begin{enumerate}
    \item $H(\epsilon, \mathcal{F}, \|\cdot \|) \leq H_{[]}(2 \epsilon, \mathcal{F}, \|\cdot \|)$;
    
    \item $H( \epsilon, \mathcal{F}|\{x_i\}_{i=1}^n, n^{-1/p} \cdot\| \cdot \|_p ) = H( \epsilon, \mathcal{F}, \| \cdot \|_{p,n} ) \leq H(\epsilon, \mathcal{F}| \{x_i\}_{i=1}^n, \| \cdot \|_{\infty}) \leq H(\epsilon, \mathcal{F}, \| \cdot \|_{\infty})$ for all $\{x_i\}_{i=1}^n \subset dom(\mathcal{F})$.
    
    \item $H(\epsilon, \mathcal{F} - \mathcal{F}, \| \cdot \|) \leq 2 H (\epsilon/2, \mathcal{F}, \| \cdot \|))$, 
    where $\mathcal{F} - \mathcal{F} := \{f - g: f, g \in \mathcal{F}\}$.
\end{enumerate}

\end{lem}

\begin{lem}[\textit{Entropic number of bounded Besov spaces} {\citep[Corollary~2.2]{Nickl2007BracketingME}}]
For $1 \leq p,q \leq \infty$ and $\alpha > d/p$, we have 
\begin{align*}
    H_{[]}(\epsilon, \bar{B}^{\alpha}_{p,q}(\mathcal{X}), \| \cdot \|_{\infty})  \lesssim \epsilon^{-d/\alpha}.
\end{align*}
\label{lemma:entropic_number_of_Besov}
\end{lem}

\begin{lem}[\textit{Approximation power of deep ReLU networks for Besov spaces} \citep{suzuki2018adaptivity}]
Let $1 \leq p,q \leq \infty$ and $\alpha \in (\frac{d}{p \land 2}, \infty)$. For sufficiently large $N \in \mathbb{N}$, there exists a neural network architecture $\Phi(L, m,S,B)$ with 
\begin{align*}
    L \asymp \log N, m \asymp N \log N, S \asymp N, \text{ and } B \asymp N^{d^{-1} + \nu^{-1}},
    \label{parameterize_net}
\end{align*}
where $\nu := \frac{\alpha - \delta}{2 \delta}$ and $\delta := d(p^{-1} - (1 + \floor{\alpha})^{-1})_{+}$
such that 
\begin{align*}
    \sup_{f_* \in \bar{B}^{\alpha}_{p,q}(\mathcal{X})} \inf_{f \in \Phi(L,W,S,B)} \|f - f_*\|_{\infty} \lesssim N^{-\alpha/d}.
\end{align*}
\label{lemma:approximation_power_for_Besov}
\end{lem} 
\newpage

\chapter{Conclusion\label{chap:Conclusion}}
In this thesis, we have introduced three novel frameworks and methods to address the challenges of reinforcement learning under practical considerations. Our methods obtain provable robustness, scalability and statistical efficiency by approaching these diverse challenges from a unifying perspective: a distributional perspective. We summarize our contributions below and close this thesis with our discussion on some potential improvements, open questions and future work in this direction. 

\section{Contributions}
In Chapter \ref{chap:three}, we have proposed a novel framework for addressing the problem of sequential decision making under the presence of an uncontrollable environmental variable, namely distributionally robust Bayesian optimization (DRBQO). In the presence of the uncontrollable environmental variable with unknown distribution, the prior Bayesian optimization and Bayesian quadrature optimization methods can converge to a spurious optimum falsely depicted via the empirical distribution of the uncontrollable environmental variable. In our framework, we instead seek for a solution that is guaranteed to perform well for all the possible environmental distributions close to the empirical distribution with respect to the $\chi^2$-divergence. Our approach leads to a practical algorithm that is proven to converge to an optimal robust solution in a sublinear time. Our method works effectively in both synthetic and real-world experiments. 

In Chapter \ref{chap:four}, we have proposed a novel method for distributional RL leveraging the idea of statistical hypothesis testing. All the predominant distributional RL methods suffer from the so-called curse of predefined statistics where they approximate the return distribution via a set of statistics with a predefined functional form. This imposes unnecessary restrictions on the statistic representation and makes the learning update more involved and difficult as it requires non-trivial projections to maintain the unnecessary restrictions. Our proposed framework, namely distributional RL via moment matching, eschews the curse of predefined statistics by considering the distributional learning as an evolution of the pseudo-samples of the return distribution. The pseudo-samples are not entitled to any predefined functional form, thus are free to be learned to simulate the return distribution. We provide insights of distributional RL within our framework via our theoretical analysis. In addition our framework obtains scalability as it is orthogonal to the modelling improvements in distributional RL and is easily extended to the deep RL setting. In fact, in the deep RL setting, our framework achieves a new state-of-the-art performance in the Atari game benchmark. 

In Chapter \ref{chap:five}, we have provided the first comprehensive analysis of offline RL with deep ReLU network function approximation. In particular, we introduce a new dynamic condition, namely Besov dynamic closure, that encompasses the dynamic conditions considered in the prior work. In addition, we analyze offline RL under the data-dependent structure induced by fitted-Q iteration update type. The data-dependent structure is ignored either in previous algorithms and prior analyses, leading to a sample-inefficient algorithm or an improper analysis, respectively. Our analysis shows an improved sample complexity of offline RL with deep ReLU network function approximation as compared to the literature. Technically, we establish this result via a combination of a uniform convergence argument, local Rademacher complexities and a localization argument which could be of independent interest. 

\section{Future Directions}
We discuss several possible improvements and open questions for each chapter. 

\subsection*{Distributionally Robust Bayesian Quadrature Optimization}
The confidence $\rho$ in our proposed DRBQO in Chapter \ref{chap:three} is a problem-dependent hyperparameter and depends on the variance of $f(x,w)$ along $w$. Intuitively the higher the variance, more conservative we would like to be by setting the larger $\rho$ value in the range of $[0, (n-1)/2]$. If there is no prior knowledge of the variance, we can heuristically perform grid search for $\rho$ in $[0, (n-1)/2]$. A future research direction is to investigate an automatic selection of $\rho$ in a data-driven manner.

\subsection*{Distributional RL via Moment Matching}
We discuss some potential improvements for our proposed framework MMDRL in Chapter \ref{chap:four} and some open questions. 

\noindent \textbf{Automatic Kernel Selection for MMDQN}. The kernel used in MMDQN plays a crucial role in achieving a good empirical performance and using the same kernel to perform well in all the games is a highly non-trivial task. Our current work uses a relatively simple but effective heuristics which uses a mixture of Gaussian kernels with different bandwidths. We speculate that a systematic way of selecting a kernel can even boost the empirical performance of MMDQN further. A promising direction is that instead of relying on a predefined kernel, we can train an adversarial kernel \citep{DBLP:conf/nips/SriperumbudurFGLS09,DBLP:conf/nips/LiCCYP17} to provide a stronger signal about a discrepancy between two underlying distributions; that is, $\min_{\theta \in \Theta} \max_{k \in \mathcal{K}} \text{MMD}( Z_{\theta}(x,a), [\mathcal{T}  Z_{\theta}](x,a); k), \forall (x,a)$ where $\mathcal{K}$ is a set of kernels. 

\noindent \textbf{Modeling Improvement for MMDQN}. 
We focus our current work only on the statistical aspect of distributional RL and deliberately keep all the other design choices similar to the basic QR-DQN (e.g, we did not employ any modeling improvements and uncertainty-based exploration).  As our framework does not require the likelihood but only (pseudo-)samples from the
return distribution, it is natural to build an implicit generative model (as in IQN) for the return
distribution in MMDQN where we transform via a deterministic parametric function the samples
from a base distribution, e.g., a simple Gaussian distribution, to the samples of the return distribution. The weights of the empirical distribution can be made learnable by a proposal network as in FQF.

\noindent \textbf{An Open Question about The Necessary Condition for Contraction}. In this work, we prove that the distributional Bellman operator is not a contraction in MMD with Gaussian kernels using the scale-insensitivity of Gaussian kernels. On the other hands, we show that the distributional Bellman operator is a contraction in MMD with shift-invariant and scale-sensitive kernels.  This suggests a question of whether the scale sensitivity is a necessary condition for the contraction under MMD. Another direction is an understanding of a precise notion and the role of approximate contraction in practical setting as here the distributional Bellman operator is not a contraction in MMD with Gaussian kernels but the Gaussian kernels still give a favorable empirical performance in the Atari games as compared to the other kernels. 

\noindent \textbf{Robust Off-policy Estimation in Distributional RL}. Another potential direction from the current work is to estimate the return distributions merely from offline data generated by some behaviour policies. Since the estimation is constructed from finite offline data, robustness is key to avoid a spurious estimation \citep{pmlr-v108-nguyen20a}. 
\subsection*{Offline RL with Function Approximations}
We conclude Chapter \ref{chap:five} with some open problems. First, although the finite concentration coefficient is a uniform data coverage assumption that is relatively standard in offline RL, can we develop a weaker, non-uniform assumption that can still accommodate offline RL with non-linear function approximation \citep{nguyen2021offline}? While such a weaker data coverage assumptions do exist for offline RL in tabular settings \citep{rashidinejad2021bridging}, it seems difficult to generalize this condition to function approximation. Another important direction is to investigate the sample complexity of \textit{pessimism} principle \citep{buckman2020importance} in offline RL with non-linear function approximation, which is currently studied only in tabular and linear settings \citep{rashidinejad2021bridging,jin2020pessimism}. 

\newpage{}

\appendix
\clearpage{}

\bibliographystyle{plainnat}
\bibliography{refs,Chap3/Chap3_refs,Chap2/Chap2_refs,Chap4/Chap4_refs,Chap5/chap5_refs}
\textbf{Every reasonable effort has been made to acknowledge the owners of copyright material. I would be pleased to hear from any copyright owner who has been omitted or incorrectly acknowledged.}






\end{document}